\documentclass[12pt]{article}
\usepackage[margin=1in]{geometry}
\usepackage{times}

\usepackage{url,booktabs}
\usepackage{nicefrac,microtype,xcolor}
\usepackage{graphicx,tikz}
\usetikzlibrary{calc,arrows.meta}
\usepackage{subfigure}
\usepackage{amsmath,amsthm,amssymb,amsfonts}
\usepackage{mathtools,mathrsfs,bm}
\usepackage{algorithm,algorithmic}

\mathtoolsset{showonlyrefs}

\usepackage{caption}
\usepackage{authblk}
\usepackage{etoc}
\usepackage{wrapfig}

\usepackage{newtxtext,newtxmath}

\usepackage{graphicx} 

\usepackage[
    colorlinks=true,
    linkcolor=blue,
    anchorcolor=blue,
    citecolor=blue
]{hyperref}

\usepackage[shortlabels]{enumitem}

\usepackage{xcolor}

\usepackage{
  mathtools,
  thmtools,
  thm-restate,
  bbm,
  bm,
  cancel,
  url,
} 

\newtheorem{theorem}{Theorem}[section]
\newtheorem{lemma}{Lemma}[section]
\newtheorem{proposition}[lemma]{Proposition}
\newtheorem{corollary}[lemma]{Corollary}

\newtheorem{assumption}[lemma]{Assumption}

\theoremstyle{definition} %

\newcommand\locallabel[1]{\label{\currentprefix:#1}}
\newcommand\localref[1]{\ref{\currentprefix:#1}}

\newcommand{\eps}{\varepsilon}
\newcommand{\rd}{\mathrm{d}}
\newcommand{\R}{\mathbb{R}}
\renewcommand{\S}{\mbb{S}}
\newcommand{\Id}{\mbf{I}}
\newcommand{\x}{{\mbf{x}}}
\newcommand{\y}{\mbf{y}}

\newcommand{\inprod}[2]{\left\langle #1, #2\right\rangle}
\newcommand{\norm}[1]{\left\|#1\right\|}
\newcommand{\abs}[1]{ {\left| #1 \right|} }
\newcommand{\braces}[1]{ \left\{ #1 \right\} }

\newcommand{\inv}{^{-1}}
\newcommand{\trans}{^\top}

\newcommand{\indi}{\mathbbm{1}}
\newcommand{\E}{\mathop{\mathbb{E\/}}}
\renewcommand{\P}{\mathop{\mathbb{P\/}}}
\newcommand{\Var}{\mathop{\mathbf{Var\/}}}

\DeclareMathOperator*{\argmax}{argmax}
\DeclareMathOperator{\poly}{poly}

\newcommand{\lhs}{\mrm{LHS}}
\newcommand{\rhs}{\mrm{RHS}}

\newcommand{\mbb}{\mathbb}
\newcommand{\mrm}{\mathrm}
\newcommand{\mbf}{\bm}
\newcommand{\mcal}{\mathcal}

\renewcommand{\v}{{\mbf{v}}}
\newcommand{\z}{{\mbf{z}}}

\renewcommand{\a}{\mbf{a}}

\newcommand{\Loss}{\mcal{L}}
\newcommand{\Gaussian}[2]{\mcal{N}\left(#1, #2\right)}

\newcommand{\Tmp}{\textnormal{\texttt{Tmp}}}

\newcommand{\Term}{\textnormal{\texttt{T}}}

\DeclareMathOperator{\sgn}{sgn}
\newcommand{\e}{\mbf{e}}

\newtheorem*{theorem*}{Theorem}
\newtheorem*{proposition*}{Proposition}

\newtheorem*{remark*}{Remark}

\newtheorem{inductionH}[theorem]{Induction Hypothesis}

\newcommand{\cF}{\mathcal{F}}

\renewcommand{\u}{\mbf{u}}

\newcommand{\IE}{\mrm{IE}}
\newcommand{\Unif}{\mrm{Unif}}

\renewcommand{\H}{\mbf{H}}

\newcommand{\cV}{\mcal{V}}

\hypersetup{
  colorlinks,
  citecolor={blue!60!black},
  citecolor={blue!60!black},
}

\newcommand*\samethanks[1][\value{footnote}]{\footnotemark[#1]}

\ifdefined\usebigfont

\usepackage{times}
\usepackage[fontsize=13pt]{scrextend}
\makeatletter
\@ifpackageloaded{geometry}{\AtBeginDocument{\newgeometry{letterpaper,left=1.56in,right=1.56in,top=1.71in,bottom=1.77in}}}{\usepackage[letterpaper,left=1.56in,right=1.56in,top=1.71in,bottom=1.77in]{geometry}}
\AtBeginDocument{\newgeometry{letterpaper,left=1.56in,right=1.56in,top=1.71in,bottom=1.77in}}
\usepackage{hyperref} %

\else
\fi

\title{\vspace{-3mm}Emergence and scaling laws in SGD learning of \\ shallow neural networks} 

\author{
{Yunwei Ren}\thanks{Equal contribution. \vspace{-2mm}}$^{~\,,1}$\!,\,~ 
{Eshaan Nichani}\samethanks[1]$^{~\,,1}$\!,\,~ 
{Denny Wu}$^{2,3}$\!,\,~ 
{Jason D.~Lee}$^{1}$
\\
\normalsize{$^{1}$Princeton University},\,\,\, 
\normalsize{$^{2}$New York University},\,\,\, 
\normalsize{$^{3}$Flatiron Institute}
\vspace{1.5mm}
\\
{\small
\texttt{\{yunwei.ren,eshnich,jasonlee\}@princeton.edu}, \,\, \texttt{dennywu@nyu.edu}
}
}

\begin{document}
\etocdepthtag.toc{mtchapter}
\etocsettagdepth{mtchapter}{subsection}
\etocsettagdepth{mtappendix}{none}

\maketitle 

\vspace{-3mm}

\begin{abstract}%
We study the complexity of online stochastic gradient descent (SGD) for learning a two-layer neural network with $P$ neurons on isotropic Gaussian data: $f_*(\boldsymbol{x}) = \sum_{p=1}^P a_p\cdot \sigma(\langle\boldsymbol{x},\boldsymbol{v}_p^*\rangle)$, $\boldsymbol{x} \sim \mathcal{N}(0,\boldsymbol{I}_d)$, where the activation $\sigma:\mathbb{R}\to\mathbb{R}$ is an even function with information exponent $k_*>2$ (defined as the lowest degree in the Hermite expansion), $\{\boldsymbol{v}^*_p\}_{p\in[P]}\subset \mathbb{R}^d$ are orthonormal signal directions, and the non-negative second-layer coefficients satisfy $\sum_{p} a_p^2=1$. We focus on the challenging ``extensive-width'' regime $P\gg 1$ and permit diverging condition number in the second-layer, covering as a special case the power-law scaling $a_p\asymp p^{-\beta}$ where $\beta\in\mathbb{R}_{\ge 0}$. We provide a precise analysis of SGD dynamics for the training of a student two-layer network to minimize the mean squared error (MSE) objective, and explicitly identify sharp transition times to recover each signal direction. In the power-law setting, we characterize scaling law exponents for the MSE loss with respect to the number of training samples and SGD steps, as well as the number of parameters in the student neural network. Our analysis entails that while the learning of individual teacher neurons exhibits abrupt transitions, the juxtaposition of $P\gg 1$ emergent learning curves at different timescales leads to a smooth scaling law in the cumulative objective.
\end{abstract}

\allowdisplaybreaks

\section{Introduction}

Recent works have studied the gradient-based training of shallow neural networks for learning low-dimensional target functions (i.e., functions in $\R^d$ that depend on $P\ll d$ directions), such as single-index models \cite{benarous2021online,ba2022high,bietti2022learning,damian2023smoothing,berthier2023learning,damian2024computational} and multi-index models \cite{damian2022neural,abbe2022merged,bietti2023learning,collins2023hitting,ben_arous_stochastic_2024,troiani2024fundamental}, to illustrate the adaptivity (and hence the improved statistical efficiency) of neural networks through feature learning. 
For such target functions on unstructured (isotropic) input data, it is known that optimization may exhibit an \textit{emergent} risk curve: learning undergoes an extensive ``search phase'' during which the loss plateaus (the length of which depends on properties of the nonlinearity), followed by a sharp ``descent phase'' where strong recovery is achieved rapidly. For instance, when the target is a single-index model $f_*(\x) = \sigma(\x\cdot\boldsymbol{\theta}), \boldsymbol{\theta}\in\R^d$, the initial search phase of online SGD scales as $t\asymp d^{\Theta(k_*)}$, where $k_*\in\mathbb{R}_+$ is the \textit{information exponent} of the link function $\sigma$ (defined as the index of its first nonzero Hermite coefficient \cite{dudeja2018learning,benarous2021online}), whereas the final descent phase occurs in $\eta t=\tilde{\Theta}(1)$ time.

The sharp phase transition observed in the gradient-based learning of low-dimensional target functions may seem at odds with the phenomenon of \textit{neural scaling laws} \cite{hestness2017deep,kaplan2020scaling,hoffmann2022training}, where increasing compute and data empirically leads to a predictable power-law decay in the loss. A plausible explanation lies in considering an \textit{additive model}, where the objective can be decomposed into a large number of distinct ``skills'', each of which occupies only a small fraction of the trainable parameters \cite{dai2021knowledge,elhage2022toy,panigrahi2023task}. While the acquisition of individual skills may exhibit abrupt transitions -- empirically observed in \cite{wei2022emergent,ganguli2022predictability} -- the juxtaposition of numerous emergent learning curves occurring at different timescales results in a smooth power-law rate for the cumulative objective \cite{michaud2024quantization,nam2024exactly}.

Motivated by the above, we consider an idealized setting where each learning task is represented by a Gaussian single-index model, so the additive model reduces to a two-layer neural network
\[\textstyle
f_*(\x) = \sum_{p=1}^P a_p\, \sigma(\v^*_p\cdot\x), \quad \x\sim\mathcal{N}(0,\Id_d), 
\]
where $\{\v^*_p\}_{p=1}^P$ are orthonormal index features, $a_1 \ge \cdots \ge a_P \ge 0$ are second-layer weights ordered in descending magnitude, and $\sigma:\R\to\R$ is an even activation function with information exponent $k_*>2$; this implies that (online) SGD learning of each task has an emergent learning curve with $\text{poly}(d)$ initial plateau.
This target function is a subclass of multi-index models (with ridge-separable nonlinearity), for which the complexity of gradient-based optimization has been recently studied \cite{oko2024learning,simsek2024learning,ren2024learning}. We highlight the following technical challenges to be addressed.

\begin{itemize}[leftmargin=*]
    \item \textbf{Extensive width} ($P\gg 1$). Most existing results on SGD learning have focused on the ``narrow-width'' regime such as $P=1$ for single-index models \cite{benarous2021online,damian2023smoothing,mousavi2022neural,dandi2024benefits,lee2024neural} and $P = O_d(1)$ for multi-index models \cite{damian2022neural,bietti2023learning,dandi2023learning,ben_arous_stochastic_2024,zhou2024does}. However, to obtain a smooth power-law scaling from a sum of ``discrete'' learning curves, the number of tasks should be large; this motivates us to study the extensive-width regime where we allow $P\to\infty$ as $d\to\infty$, which yields an \textit{infinite-dimensional} effective dynamics \cite{benarous2022high}. 
    
    \item \textbf{Large condition number} 
    ($\frac{a_{\max}}{a_{\min}}\gg 1$). Existing works in the extensive-width regime usually assumed identical second-layer coefficients ($a_1=...=a_P$) \cite{ren2024learning,simsek2024learning} or proved optimization complexity that scales exponentially with the condition number $\kappa=\frac{a_{\max}}{a_{\min}}$ \cite{li2020learning,oko2024learning} (to our knowledge the only exceptions are \cite{ge_understanding_2021,ben_arous_stochastic_2024} which considered algorithms that are unnatural for neural network training, e.g., Stiefel constraint or tensor deflation with re-initialization). Such exponential dependency implies that in the poly-time learnable regime $\kappa=O_d(1)$, the signal strength for individual tasks can only differ by constant factors, and consequently, there is insufficient timescale separation to produce a power-law risk curve. We thus focus on the challenging large condition number regime, allowing $\kappa\to\infty$ as $d\to\infty$.
    
    \item \textbf{Single-phase training.} Prior works on multi-index learning typically employed a layer-wise training procedure, where correlation loss SGD is first applied to the first-layer parameters to recover the index features, followed by convex optimization to solve for the optimal second layer \cite{damian2022neural,ba2022high,abbe2023sgd,oko2024learning}. Such stage-wise training creates complications in the scaling law description due to the changing computational procedure. Hence we aim to characterize a natural, single-phase algorithm where both layers are updated simultaneously.
\end{itemize}

\subsection{Our Contributions}

We study the learning of an additive model target function \eqref{eq:teacher} with orthogonal first-layer weights and even activation with information exponent $k_*>2$, using a student two-layer neural network with $m$ neurons trained via online SGD to minimize the mean squared error (MSE) loss. We consider the extensive-width regime $P\gg 1$, and allow the scale of second-layer parameters of the target (teacher model) to depend polynomially on the width $P$. 
Our main contribution is establishing a polynomial optimization and sample complexity for single-phase SGD training and providing a sharp characterization of the recovery time for each teacher neuron.

\begin{theorem*}[(Informal) sample complexity] {\it
Assume the teacher model has $P \lesssim d^{c}$ orthogonal neurons for some small but fixed $c>0$, and the activation $\sigma$ is an even function with information exponent $k_*>2$. To recover the top $P_*\le P$ teacher directions, we can train a student network \eqref{eq:student} with $m=\tilde{\Theta}(P_*)$ neurons via online SGD with sample and runtime complexity $n \asymp T \asymp a_{P_*}^2 \cdot d^{k_*-1} \mathrm{poly}(P)$.} 
\end{theorem*}
As a corollary, we know that a student width $m=\tilde{\Theta}(P)$ and sample size $n = \tilde{\Theta}(a_{\min}^2 d^{k_*-1} \mathrm{poly}(P))$ are sufficient to learn all teacher neurons and achieve small population error, where $a_{\min} := \min_{p \in [P]} a_p$. 
Prior to our work, \cite{oko2024learning} studied the learning of the same target function class using a layer-wise training procedure that deviates from common practice. Their analysis established optimization guarantees that require $m\gtrsim P^{\Omega(1/a_{\min})}$ student neurons, which is computationally prohibitive since $P,a_{\min}^{-1}$ can both scale with the ambient dimensionality $d$. Interestingly, we show that this limitation can be overcome by considering an arguably more natural single-phase training algorithm. 
At a technical level, our analysis leverages the following key ingredients.
\begin{itemize}[leftmargin=*]
    \item \textit{Single-stage training.} We consider a 2-homogeneous student model and simultaneously train both layers via online SGD under the MSE loss; this differs from prior layer-wise analyses where the first-layer parameters are optimized under correlation loss. In our large condition number setting, the correlation loss analysis yields super-polynomial computational complexity in order to compensate for the signal discrepancy across different tasks \cite{oko2024learning}; in contrast, our single-phase MSE dynamics circumvents this issue by automatically removing the learned tasks from the loss, analogous to a deflation process \cite{ge_understanding_2021}. 
    \item \textit{Decoupled dynamics.} In the extensive-width $P\gg 1$ regime, the effective dynamics of SGD cannot be captured by a finite set of summary statistics. To understand the convergence of this high-dimensional system, we show that the evolution of different signal directions can be approximately decoupled (see Section \ref{sec: idealized to gf}) using the ``automatic'' deflation mechanism and carefully controlling the influence of the irrelevant coordinates. 
\end{itemize}

Applying our general learnability result, we precisely characterize the scaling of the population loss along the online SGD trajectory in the following power-law setting.

\begin{proposition*}[(Informal) scaling law] {\it 
Under the same conditions and hyperparameters as the previous theorem, and assuming $a_p \asymp p^{-\beta}$ for $\beta>1/2$, then (ignoring logarithmic factors) we have 
\begin{enumerate}[(a), leftmargin=*,itemsep=0.mm]
    \item \textbf{Emergence.} The $p$-th teacher neuron (where $p\lesssim m$) is recovered at time $\eta t\sim p^\beta d^{k_*/2-1}$. 
    \item \textbf{Scaling law.} The population squared error follows a power-law decay up to approximation barrier $\Loss(t) \sim \left(t\eta d^{1-k_*/2}\right)^{\frac{1-2\beta}{\beta}} \vee m^{1-2\beta}$. 
\end{enumerate}    
}
\end{proposition*}

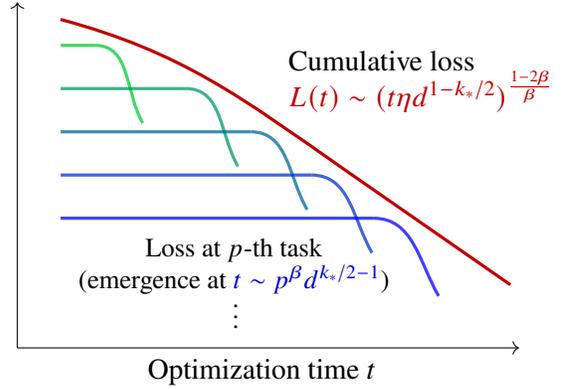
\begin{wrapfigure}{r}{0.46\textwidth} 
\centering
    \begin{tikzpicture}[scale=1.15]   

    \draw[->] (0,0) -- (5.8,0) node[pos=0.56, below] {\parbox{4cm}{\small Optimization time $t$}};
    \draw[->] (0,0) -- (0,4) node[below] {};

    \definecolor{task5}{rgb}{0.0, 0.0, 1.0}  %
    \definecolor{task4}{rgb}{0.0, 0.2, 0.8}  %
    \definecolor{task3}{rgb}{0.0, 0.4, 0.6}  %
    \definecolor{task2}{rgb}{0.0, 0.6, 0.4}  %
    \definecolor{task1}{rgb}{0.0, 0.8, 0.2}  %
    
    \draw[very thick, red!75!black] 
        (0.5,3.8) to[out=345, in=145] (3.8,2.05)  %
        -- (5.7,0.73);  

    \node[above] at (4.7,2.6) {\parbox{3.6cm}{\small Cumulative loss  \\ {\color{red!75!black} $L(t)\sim (t\eta d^{1-k_*/2})^{\frac{1-2\beta}{\beta}}$}}};

    \foreach \x/\y/\d/\color in {0.9/3.5/0.55/task1, 1.95/3.0/0.6/task2, 2.7/2.5/0.65/task3, 3.4/2.0/0.7/task4, 4.12/1.5/0.75/task5} {
        \draw[very thick, color=\color, opacity=0.75] (0.5,\y) -- (\x,\y);  %
        \draw[very thick, color=\color, opacity=0.75] (\x,\y) to[out=0, in=120] (\x+\d, \y-0.9); %
    }
    \node[below] at (2.5,1.4) {\parbox{4.2cm}{\centering\footnotesize Loss at $p$-th task \\ (emergence at {\color{blue!80!black}$t\sim p^\beta d^{k_*/2-1}$}\color{black}) \\ \vspace{-1.mm}\,\vdots }}; 

\end{tikzpicture}
\vspace{-2.mm}
\caption{\small Power-law scaling of MSE loss as a result of superposition of emergent learning curves. }
\label{fig:intro}
\vspace{-5mm}
\end{wrapfigure}

This proposition confirms the additive model intuition from \cite{michaud2024quantization,nam2024exactly} in a high-dimensional feature learning setting, where the length of the ``search phase'' (plateau) for each feature direction $\v_p^*$ is modulated by the magnitude of the second-layer coefficient $a_p$, and the simultaneous learning of all directions yields a power-law decay in the cumulative loss (see Figure~\ref{fig:intro}). However, unlike these prior works, our problem setting does not imply that the learning of different tasks can be decoupled \textit{a priori}, as student neurons may be attracted to multiple teacher directions and
also interact with each other through the squared loss.

\paragraph{Organization.}
The rest of this paper is organized as follows. In Section~\ref{sec: setup}, we describe our problem setting and present the main theorems. Section \ref{sec: proof sketch} provides proof sketches of our main results: in Section~\ref{sec: idealized dynamics and scaling laws} we discuss the idealized training dynamics and scaling laws; in Section~\ref{sec: idealized to gf} and \ref{sec: online sgd}, we show that gradient flow can approximate the idealized dynamics and that online SGD can track the gradient flow, respectively. Formal proofs and additional related works are deferred to the appendix.

\section{Problem Setting and Main Results}
\label{sec: setup}

In this section, we present our main results on SGD learning and scaling laws. 

\subsection{Setting and Algorithm}

\paragraph{Architecture: two-layer neural network.} Let $\sigma: \R \to \R$ denote the nonlinear link function. We assume the target function is given by the following additive model
\begin{equation}
  \textstyle
  f_*(\x) = \sum_{p=1}^P a_p \sigma( \v_p^* \cdot \x ), \quad \forall \x \in \R^d, 
  \label{eq:teacher}  
\end{equation}
where $\x \sim \gamma := \Gaussian{0}{\Id_d}$ is the input, $\{ \v_p^* \}_{p \in [P]} \subset \R^d$ are orthonormal with $P\gg 1$,
$\sigma \in L^2(\gamma)$ satisfies Assumption~\ref{assumption: link function}, 
and $a_1 \ge \cdots \ge a_P\ge 0$ are normalized so that $\sum_p a_p^2 = 1$. Since the input distribution and our learning algorithm are rotationally invariant, we may assume w.l.o.g.~that $\v_p^* = \e_p$, where $\e_p \in \R^d$ is the $p$-th standard basis vector.
While our scaling results will assume $a_p$ follows a power law decay, no such 
assumptions are required for our optimization results. 
\begin{assumption}[Link function]
  \label{assumption: link function}
  Let $\{ h_k \}_{k \in \mbb{N}_{\ge 0}}$ denote the normalized Hermite polynomials. 
  \begin{enumerate}[(a),leftmargin=*]
    \item $\sigma$ is even and has information exponent $\IE(\sigma) = 2 I$ for $I > 1$, that is, the Hermite expansion of $\sigma$ is given as $\sigma = \sum_{i=I}^\infty \hat\sigma_{2i} h_{2i}$, and we require $\hat\sigma_{2I} \ge c_\sigma$; we also assume $\|\sigma\|_{L^2(\gamma)}=1$, and $\|\sigma'\|_{L^2(\gamma)}, \|\sigma''\|_{L^2(\gamma)}\le C_\sigma$, where constants $c_\sigma, C_\sigma > 0$.
    \item $\sigma$ and $\sigma'$ have polynomial growth. That is, there exist universal constants $C, Q > 0$
        such that $|\sigma(x)| \vee |\sigma'(x)| \le C(1 + x^2)^{Q/2}$ for all $x \in \mbb{R}$.
  \end{enumerate}
\end{assumption}
\begin{remark*}
We focus on high information exponent $\IE(\sigma)>2$ link functions as in \cite{oko2024learning,simsek2024learning,glasgow2025propagation}. This setting entails that the learning of each single-index task is ``hard" in the sense that online SGD exhibits a long loss plateau, and we utilize this assumption to prove (approximate) decoupling of individual tasks. The condition on even $\sigma$ simplifies the analysis by removing the $1/2$ probability of neurons initialized in the wrong hemisphere (see e.g., \cite{benarous2021online}). 
\end{remark*}

Our learner network (student model) is a width-$m$ two-layer neural network:
\begin{equation}
  \label{eq:student}
  f(\x) 
  := f\left(\x; \{\v_k\}_{k=1}^m \right)
  = \sum_{k=1}^m \norm{\v_k}^2 \sigma( \bar{\v}_k \cdot \x ), 
\end{equation}
where $\{\v_k\}_{k=1}^m \subset \R^d$ are trainable parameters and $\bar \v_k := \v_k/\norm{\v_k}$. Note that this student network is parameterized to be $2$-homogeneous in each $\v_k$, i.e., the second-layer coefficients are coupled with the norm of the first-layer weights. We make the following remarks.

\begin{remark*}
The $2$-homogeneous parameterization has been used in prior works 
\cite{li2020learning,wang_beyond_2020,ge_understanding_2021}; this setting originated from the analysis of training both layers of ReLU networks under balanced initialization (see e.g., \cite{chizat2020implicit}). One of our technical contributions is that when both layers are trained simultaneously under this parameterization, the growth of the second-layer norm $\norm{\v_k}$ is coupled to the directional convergence of the first layer $\bar \v_k$, enabling an ``automatic deflation" process and making the single-phase training dynamics amenable to analysis. We believe that a similar proof strategy can be applied to simultaneous training of networks with decoupled second-layer weights.
\end{remark*}

\paragraph{Algorithm: online SGD.} 
The performance of the learner is measured using the mean squared error (MSE) loss. 
For each $\x \in \R^d$, the per-sample MSE loss is defined as 
\begin{equation}
  \label{eq: per-sample loss}
  l(\x) 
  = l\left(\x; \{\v_k\}_{k=1}^m \right)
  = \frac{1}{2} \left( f_*(\x) - f\left(\x; \{\v_k\}_{k=1}^m \right) \right)^2. 
\end{equation}
Using a Hermite expansion calculation (\cite{ge2018learning}), one can show that the population MSE loss can 
be expressed as a tensor decomposition loss as follows: 
\begin{equation}
  \label{eq: population loss}
  \Loss := \!\E_{\Gaussian{0}{\Id_d}}\![l(\x)]
  = \sum_{i=I}^\infty \hat\sigma_{2i}^2 \left(
      \frac{\norm{\a}^2 }{2}
      - \sum_{p=1}^P \sum_{k=1}^{m} a_p \norm{\v_k}^2 \inprod{\bar{\v}_k}{\v_p^*}^{2i}
      + \frac{1}{2} \sum_{k,l=1}^{m} \norm{\v_k}^2 \norm{\v_l}^2 \inprod{\bar{\v}_k}{\bar{\v}_l}^{2i}
    \right).
\end{equation}
In Lemma~\ref{lemma: population and per-sample gradients} we decompose the population gradient into the radial and tangent components, and derive concentration bounds for the empirical gradients.

We use online stochastic gradient descent (SGD) to train the learner model. 
Let $\{ (\x_t, f_*(\x_t)) \}_{t\in \mbb{N}}$ be our dataset with $\x_t \overset{\text{i.i.d.}}{\sim} \Gaussian{0}{\Id_d}$ 
being the fresh sample at step $t$. We initialize the student neurons $\v_k\sim\Unif(\S^{d-1}(\sigma_0))$, where $\sigma_0 = 1 / \poly(d)$ is a parameter we specify in the sequel. Let $\eta > 0$ be the step size. 
At each step, we update the neurons using vanilla gradient descent: 
$\v_k(t+1) = \v_k(t) - \eta \nabla_{\v_k} l(\x_t)$, for all $k \in [m]$, 
where $l$ is the per-sample loss defined in \eqref{eq: per-sample loss}. 

We also include in Appendix \ref{sec:gradient-flow} 
a full proof for population gradient flow (GF), which offers a cleaner analysis that captures the core aspects of the learning problem. 
The population gradient estimations derived in the GF analysis will also be reused in the SGD analysis.

\subsection{Complexity of SGD Learning}

Our main theorem provides a sharp characterization of the sample complexity of online SGD and the recovery time of individual single-index tasks. To characterize the learning order of the first $P_*\le P$ tasks, we introduce an ordering of student neurons $\v_1, \dots, \v_m$ and a mapping $\pi : [P_*] \rightarrow [P]$ that specifies which student neurons converge to a particular task (teacher neuron). This mapping function is explicitly defined via the greedy maximum selection procedure \eqref{eq: greedy maximum selection} which we explain in Section~\ref{sec: idealized dynamics and scaling laws} --- intuitively speaking, after the reordering, for $p \in [P_*]$, $\v_p$ is the neuron that eventually converges to direction $\v^*_{\pi(p)}$, and the directions are learned sequentially based on the signal strength $\{a_p\}_{p=1}^P$ and their overlap with the closest student neuron at initialization. 

Let $\bar v_{p, q}(t) := \langle \bar\v_p, \v^*_q\rangle$ denote the normalized overlap between the $p$-th student neuron (ordered) and the $q$-th teacher neuron at time $t\ge 0$. The following theorem describes the convergence of student neuron $\v_p$ to the corresponding teacher $\v_{\pi(p)}^*$ (defined by the mapping $\pi$) in terms of \textit{direction}: $\bar v_{p, \pi(p)}^2(t)\to 1$, as well as \textit{norm}: $\|{\v_p(t)}\|^2 \to a_{\pi(p)}$.

\begin{restatable}[Main theorem for online SGD]{theorem}{onlinesgd}\label{thm:online-sgd}
    Let $C,C'>0$ be large universal constants, depending only on $I$ and $\sigma$, and set the initialization scale as $\sigma_0 = d^{-C}$. Let $P_* \in [P]$, $a_{\min_*} = \min_{p\in[P_*]} a_p$, and $\delta^*_{\P}$ be the target failure probability. Define $\Delta \simeq \frac{\delta^*_{\P}}{mP\max(m, P)} = o_d(1)$.  Assume the dimension $d$, width $m$, learning rate $\eta$ and target accuracies $\eps_D, \eps_R = o_d(1)$ satisfy 
    \begin{gather*}
        d \gtrsim \norm{\a}_1^4\Delta^{-8}a_{\min_*}^{-4}\log^{8I}d,\quad
        m \gtrsim P_*\log(P_*/\delta^*_{\P}) \lor \log(P/\delta_{\P}^*), \quad \frac{m}{\log^3 m} \gtrsim \log^2(P_*/\delta^*_{\P}), \\
        \frac{\Delta^6}{d \log^{4(I-1)}d} \gtrsim \eps_D \gtrsim \frac{\norm{\a}_1}{a_{\min_*}d^{I - 1/4}},\quad
        P_*^{-1/2}\eps_D^{1/2} \gtrsim \eps_R \gtrsim \eps_D\log(1/\sigma_0^2), \\
          \\
        \eta \lesssim \frac{a_{\min_*}\norm{\a}_1^{-2}m^{-1}P^{-1}\delta^*_{\P}}{\log^{ C}\left(\textstyle\frac{md}{\delta^*_{\P}}\right)}\min(\Delta^2d^{-I}, \eps_D^2).
    \end{gather*}
    With probability $1 - \delta_{\P}^*$, there exists an ordering of the student neurons $\v_1, \dots, \v_m$ and a mapping $\pi : [P_*] \rightarrow [P]$ of student neurons to teacher neurons (see Equation \eqref{eq: greedy maximum selection}) such that, defining
    \begin{align*}
        T_p := \frac{1}{4I(I-1)\hat\sigma_{2I}^2a_{\pi(p)}\eta\bar v_{p, \pi(p)}^{2I-2}(0)}\quad\forall p \in [P_*], \quad \text{and} \quad T_{\max} := \left(1 + {\Delta}/{4}\right)\max_{p \in [P_*]}T_p
    \end{align*}
    we have:
    \begin{enumerate}[label=(\alph*),leftmargin=*]
        \item \textbf{(Unused neurons).} $\norm{\v_k(t)}^2 \le d^{-C'} =: \sigma_1^2$ for all $k > P_*$.
        \item \textbf{(Convergence).} $\bar v_{p, \pi(p)}^2(t) \ge 1 - \eps_D$ and $\norm{\v_p(t)}^2 = a_{\pi(p)} \pm \eps_R$ for all $p \in [P_*]$, $(1 + \Delta/4)T_p \le t \le T_{\max}$.
        \item \textbf{(Sharp Transition).} $\bar v_{p, \pi(p)}^2(t) \le d^{-1/2}$ and $\norm{\v_p(t)}^2 \le \sigma_1^2$ for all $p \in [P_*]$, $t \le (1 - \Delta/256)T_p$.
        \item \textbf{(Loss Value).}
        At time $t$, the population loss of the student network can be bounded by
        \begin{align*}
            \hspace{-6.4mm} 1 -\!\! \sum_{p \in [P_*]} a_{\pi(p)}^2 \indi\left\{ t \ge (1 \!-\! \Delta/4)T_p \right\} - O(\eps_D) \le \mathcal{L}(t) \le 1 -\!\! \sum_{p \in [P_*]} a_{\pi(p)}^2 \indi\left\{t \ge (1 \!+\! \Delta/4)T_p \right\} + O(\eps_D). 
        \end{align*} 
    \end{enumerate}
\end{restatable}

\noindent We observe the following conclusions about Theorem \ref{thm:online-sgd}. 
\begin{itemize}[leftmargin=*]
    \item Points (b) and (c) suggest a \emph{sharp transition} in the learning of the teacher neuron $\v^*_{\pi(p)}$ around time $T_p \simeq (\eta a_{\pi(p)} \langle\bar \v_{p}(0), \v^*_{\pi(p)}\rangle^{2(I-1)})^{-1}$. In particular, for time $t \le (1 - o(1))T_p$, minimal progress is made on the learning of $\v^*_{\pi(p)}$, as $\langle\bar \v_{p}, \v^*_{\pi(p)}\rangle^2, \|{\v_p}\|^2/a_{\pi(p)}\ll 1$. Then, at some point during the short time interval $(1 \pm o(1))T_p$, both directional and norm convergence occur rapidly as the quantities $\langle\bar \v_{p}, \v^*_{\pi(p)}\rangle^2$ and $\|{\v_p}\|^2$ approach $1$ and $a_{\pi(p)}$ respectively.
    
    \item The theorem implies that a student width of $m \gtrsim P_*\log(P_*)$ is sufficient to recover $P_*$ teacher neurons; this minimal (logarithmic) overparameterization allows us to establish near-optimal width dependence for the scaling laws in the ensuing section. 
    
    \item Selecting $\eta = \tilde{\Theta}(a_{\min}d^{-I}\poly(m, P))$, the runtime (and sample complexity) required to recover all directions $\{\v_k^*\}_{k \in [P]}$ up to $1/d$ error, and thus obtain a population loss of $O(1/d)$, is $T = \tilde{\Theta}(d^{2I-1}\poly(P) a_{\min}^{-2}) = d^{\text{IE}(\sigma)-1}P^{\Theta(1)}$, which is polynomial in all problem parameters --- this contrasts with the exponential dependence on the condition number in \cite{li2020learning,oko2024learning}. Note that the $d^{\text{IE}(\sigma)-1}$ factor matches the dimension scaling of online SGD for learning Gaussian single-index models \cite{benarous2021online}.
    Moreover, our Assumption~\ref{assumption: link function} permits high-degree link functions; hence when $\text{deg}(\sigma)\gg\text{IE}(\sigma)$, the sample complexity established in Theorem \ref{thm:online-sgd} is far superior to the $n\gtrsim d^{\text{deg}(\sigma)}$ rate for neural networks in the kernel/lazy regime \cite{jacot2018neural,chizat2018note,ghorbani2019linearized}.
\end{itemize} 

\subsection{Neural Scaling Laws}

As an application of Theorem \ref{thm:online-sgd}, we have the following proposition on the scaling 
law of the MSE loss when $a_p$ follows a power law decay. 

\begin{proposition}[Scaling laws]
  \label{thm: scaling law}
  Consider the same setting as Theorem~\ref{thm:online-sgd}, and suppose $a_p = p^{-\beta} / Z$ where $\beta > 1/2$ and $Z = \sum_{p=1}^P p^{-2\beta}$ is the normalizing
  constant. Then, with high probability,
  \begin{enumerate}[label=(\alph*),leftmargin=*]
  \item For $p \le P_* = \tilde\Theta(m)$, the $p$-th teacher neuron $\v^*_p$ is learned at time $t=\tilde\Theta(p^\beta d^{I-1}\eta^{-1})$.
  \item There exist constants $0 < c_\beta < C_\beta$ and $0 < c_\beta' < C_\beta'$ that can depend 
  only on $\beta$ such that
  \[
  c_\beta\left[\left(\frac{m}{\log m}\right)^{1 - 2\beta} + \left(\frac{K_0\eta t}{d^{I-1}}\right)^\frac{1-2\beta}{\beta}\right] - O(\eps_D)\le \Loss(t) \le C_\beta \left[\left(\frac{m}{\log m}\right)^{1 - 2\beta} +\left(\frac{K_0\eta t}{d^{I-1}}\right)^\frac{1-2\beta}{\beta}\right] + O(\eps_D)
  \]
  for all $t \in [T_{\min}, T_{\max}]$, where $K_0 := \log^{2I-2} m / Z$, $T_{\min} = C_\beta' d^{I-1}/(K_0\eta)$ and 
  $T_{\max} = c_\beta' P^\beta d^{I-1} /(K_0\eta)$. 
  \end{enumerate}
\end{proposition}
\begin{remark*}
We make the following remarks.
\begin{itemize}[leftmargin=*]
\item As in the literature on neural scaling laws~\cite{kaplan2020scaling,hoffmann2022training,paquette2024phases}, our scaling law in Proposition \ref{thm: scaling law} consists of the \textit{approximation bottleneck} $\tilde \Theta(m^{1 - 2\beta})$, governed by the width of the student network, and the \textit{optimization bottleneck} $\Theta\big((\eta t d^{1 - I})^{(1 - 2\beta)/\beta}\big)$, governed by the number of online SGD steps (or equivalently number of samples). 
\item  Note that the times the first and last directions get learned are approximately $d^{I-1}/(K_0\eta)$ and $P^\beta d^{I-1} / (K_0\eta)$. Hence $[T_{\min}, T_{\max}]$ covers the time interval where most directions are learned.
\item We state the risk scaling for square-summable second-layer coefficients $\beta>1/2$ similar to prior theoretical works on scaling laws \cite{bordelon2024dynamical,lin2024scaling}. In the ``heavy-tailed'' regime ($\beta<1/2$), we can also apply Theorem~\ref{thm:online-sgd} to obtain $\Loss(t) = \tilde\Theta\big((1-(P/m)^{1-2\beta})_+ \vee (1 - (t\eta d^{I-1})^{(1-2\beta)/\beta})_+\big)$. Note that in this setting, the required student width is roughly proportional to the teacher width $m = \tilde\Theta(P)$ in order to achieve small approximation error.
\end{itemize}
\end{remark*}

\paragraph{``Unstable'' discretization.} 
Given a fixed training budget $t$, it can be quite pessimistic to choose the learning rate $\eta \propto a_{\min_*} \asymp a_{P_*}$ for $P_* = \tilde \Theta(m)$, since at any $t \ll (\eta a_{\pi(P_*)}\bar v_{P_*, \pi(P_*)}^{2I-2}(0))^{-1}$, far fewer than $P_*$ directions are learned. As such, consider pre-specifying the runtime $t$ (or equivalently the number of samples $n$). If we only are interested in learning the top $p$ neurons, we can apply Theorem \ref{thm:online-sgd} with $P_* = p$, which gives a larger learning rate of $\eta = \tilde\Theta(\frac{a_{\pi(p)}d^{-I}}{\poly(P)})$. The $p$-th direction is now learned at $T_p = \tilde \Theta(a_{\pi(p)}^{-2}d^{2I-1}\poly(P)) = \tilde\Theta(p^{2\beta}d^{2I-1}\poly(P))$. 
This leads to the following ``unstable" scaling law. 
\begin{corollary}[Unstable scaling law]
    Let $m$ be the student network width and $n$ be the total number of training examples. Then, there exists a choice of learning rate $\eta$ (depending on $n$, $m$) such that with high probability the population loss after $t = n$ steps of online SGD is
    \begin{align*}
        \Loss(n) = \tilde \Theta\left( m^{1 - 2\beta} + \left(\frac{n}{d^{2I-1}\poly(P)}\right)^{\frac{1 - 2\beta}{2\beta}}\right) \pm O(\eps_D).
    \end{align*}
\label{cor:unstable-scaling-law}
\end{corollary}
We remark that the above sample size scaling matches the minimax optimal rate for Gaussian sequence models (see e.g., \cite{johnstone2017gaussian}), and the exponent is consistent with existing scaling law analyses of SGD on linear models \cite{bordelon2024dynamical,lin2024scaling,paquette2024phases}. 
Note that despite the matching exponents (in terms of the decay rate $\beta$), the underlying mechanism and our theoretical analysis differ from these prior results due to the presence of nonlinear feature learning, which is reflected, for example, by the learning rate selection in our unstable discretization --- see Section~\ref{sec: online sgd} for more discussions.

\subsection{Simulations: Compute-optimal Frontier}  

In Figure~\ref{fig:scaling-law}, we plot $(a)$ the idealized scaling curves assuming decoupled learning and an exact emergence time for each task (see Section~\ref{sec: idealized dynamics and scaling laws}), and $(b)$ the MSE loss curves for GD training (with fixed step size) on the population loss, where we set $d=2048, P=1024, \sigma=h_4$, and vary the student width. 
While the idealized scaling law does not exactly hold at finite $d$, the slope of MSE loss vs.~compute (on logarithmic scale) is independent of the problem dimension; we therefore compare the slope of the compute-optimal frontier in $(a)(b)$. Omitting the dimensionality $d$ (which does not vary across models) in Proposition~\ref{thm: scaling law}, we know that given a fixed computational budget $\mathcal{T}\asymp mt$, the compute-optimal model under constant learning rate exhibits the following scaling,
\[
\Loss\sim \mathcal{T}^{\frac{1-2\beta}{1+\beta}}, \quad m\sim \mathcal{T}^{\frac{1}{1+\beta}}. 
\]
We set the power-law exponent to be $\beta=0.8$ in Figure~\ref{fig:scaling-law}. Observe that:  

\begin{figure}[t]
\centering
\begin{minipage}[t]{0.49\linewidth}
\centering
{\includegraphics[height=0.72\textwidth]{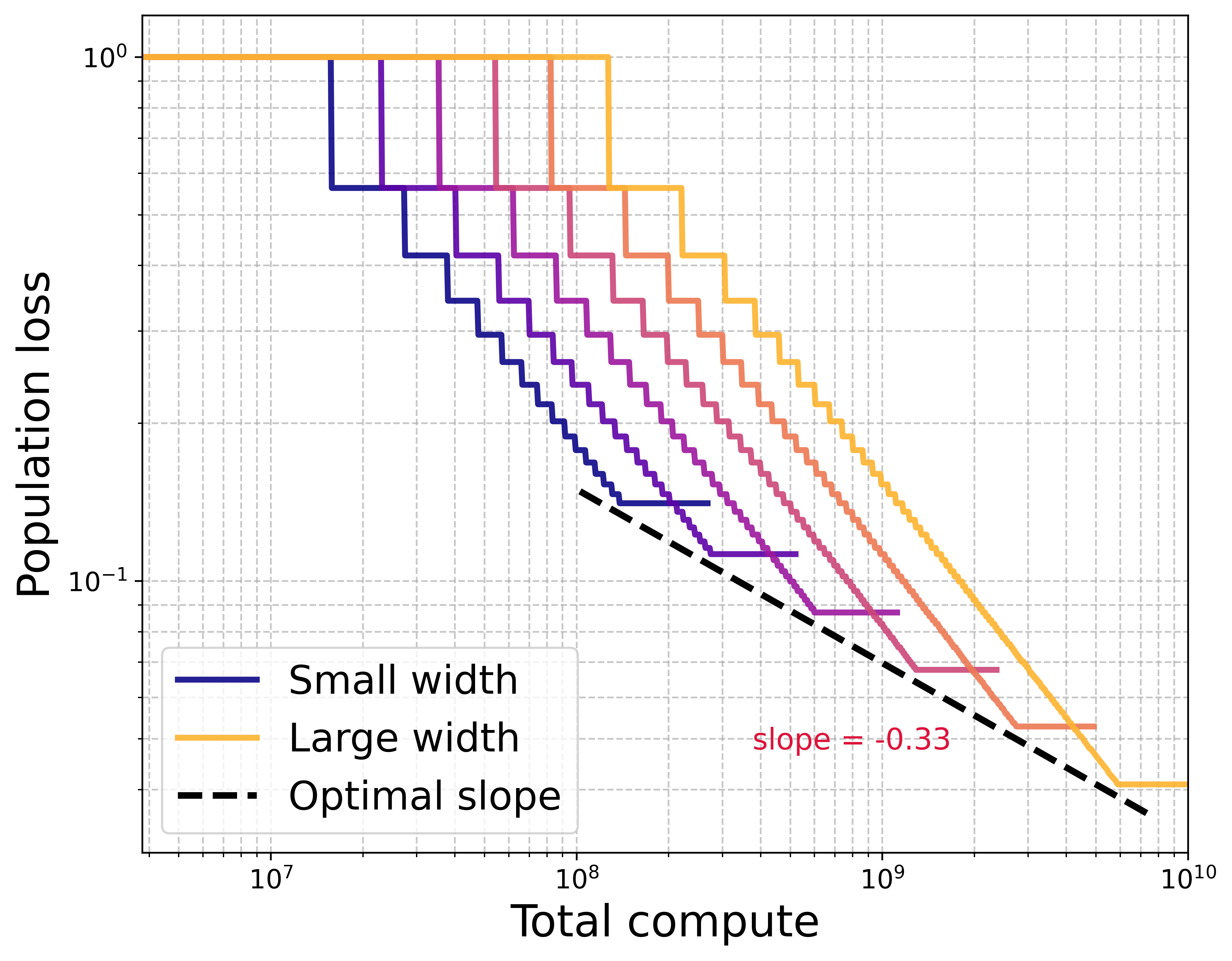}}  \\ \vspace{-1mm}
\small (a) Theoretical scaling law. 
\end{minipage}%
\begin{minipage}[t]{0.49\linewidth}
\centering 
{\includegraphics[height=0.72\textwidth]{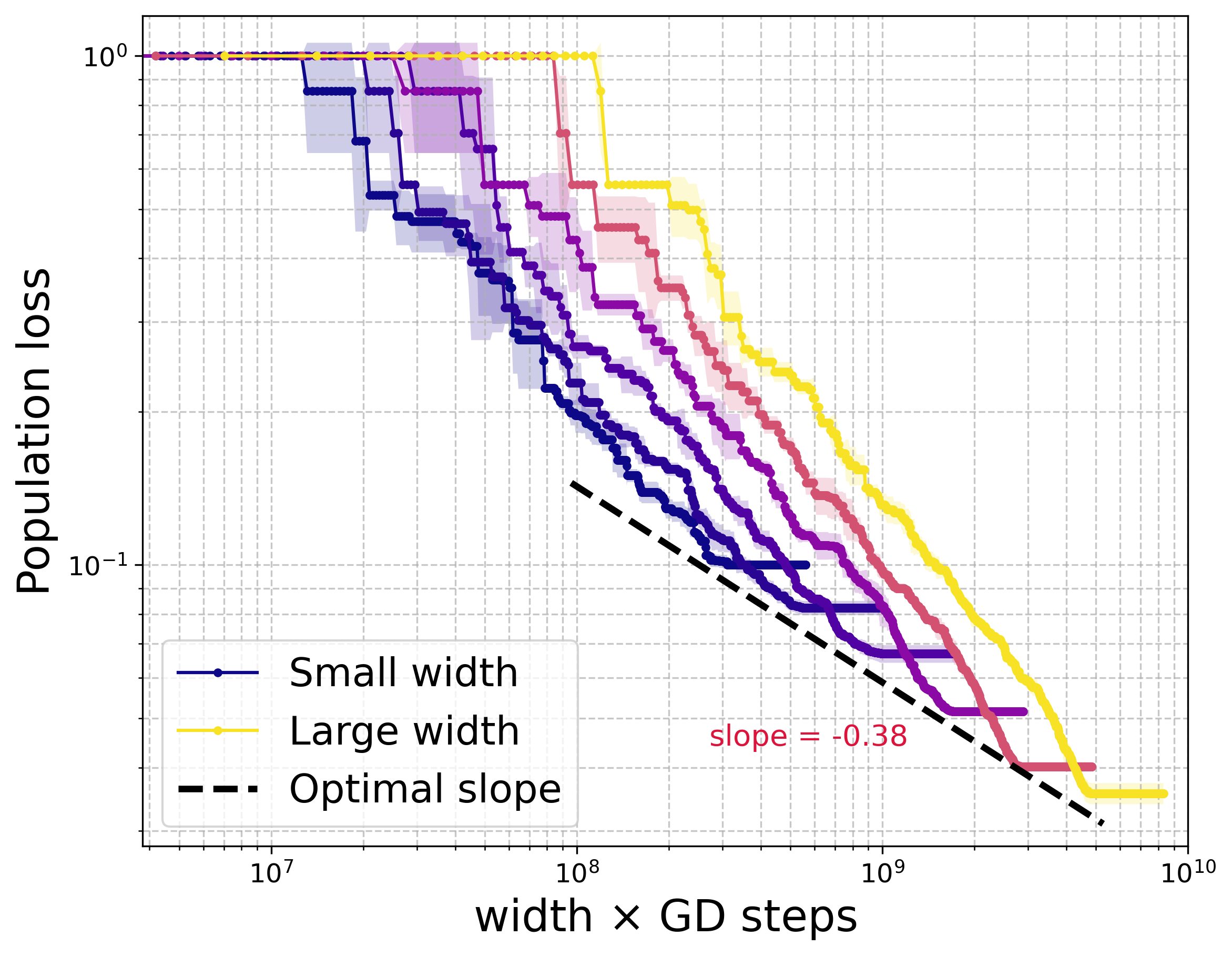}} \\ \vspace{-1mm}
\small (b) Empirical scaling law. 
\end{minipage}%
\caption{\small Theoretical and empirical risk curves with $\beta=0.8$. $(a)$ Idealized scaling curves described in Section~\ref{sec: idealized dynamics and scaling laws}. $(b)$ Empirical scaling curve of GD training on the population loss with $d=2048, P=1024$. } 
\label{fig:scaling-law} 
\end{figure}  

\begin{itemize}[leftmargin=*]
    \item The sum of staircase-like emergent learning curves yields a smooth power-law scaling in the cumulative MSE loss towards the tail, followed by a plateau due to the approximation error. 
    \item The compute-optimal slope (dashed black line) is roughly consistent between the theoretical and empirical risk curves. Specifically, for $\beta=0.8$ we theoretically predict a loss scaling of $\Loss\sim (mt)^{1/3}$ for the compute-optimal model; note that the empirical slope is slightly steeper due to the finite-width truncation error of the infinite power-law sum. 
\end{itemize}

\section{Overview of Proof Ideas}
\label{sec: proof sketch}

In Section~\ref{sec: idealized dynamics and scaling laws}, we describe the idealized dynamics, and show that they imply a loss scaling law when 
the signal strength $\{a_p\}_{p=1}^P$ follows a power law. In Section~\ref{sec: idealized to gf} we show that gradient flow 
approximates this idealized dynamics, and in Section~\ref{sec: online sgd} we discretize the gradient flow with online SGD. 
For ease of presentation, we will assume a Hermite-4 link function 
$\sigma = h_4$ in this section; the same argument follows for more general activations. 

\subsection{The Idealized Learning Dynamics}
\label{sec: idealized dynamics and scaling laws}

\paragraph{Learning a single task.}
First, consider the single-index setting and suppose the target function is $\x \mapsto a h_4(\e_1 \cdot \x)$. 
Let $\v \in \R^d$ denote the learner neuron. It is known that, under gradient flow, the correlation of $\v$ with the ground-truth
direction $\e_1$ approximately follows the quadratic ODE: $\frac{\rd}{\rd t} \bar{v}_1^2 \approx 8 a \bar{v}_1^4$ prior to weak recovery, i.e., when $\bar{v}_1^2=o(1)$ \cite{benarous2021online}. This ODE has a closed-form solution: 
$\bar{v}_1^2(t) = \left( 1/\bar{v}_1^2(0) - 8 a t \right)\inv$. We can make two immediate observations from 
this formula: 
\begin{enumerate}[(i),leftmargin=*] 
  \item $\bar{v}_1^2 = \langle\bar{\v},\e_1\rangle^2$ will grow from $\tilde{\Theta}(1/d)$ to a nontrivial value around time $( 8 a \bar{v}_1^2(0) )\inv$. 
  \item The growth of $\bar{v}_1^2$ exhibits a sharp transition. That is, $\bar{v}_1^2$ stays near its initial value for 
    most of the time and then suddenly increases around time $( 8 a \bar{v}_1^2(0) )\inv$.
\end{enumerate}
The above claims imply an emergent learning curve for the directional recovery of the single-index task. Due to the 2-homogeneous parameterization, we can show that the norm of $\v$ will not grow until strong recovery is achieved, and the norm growth occurs at a much shorter timescale than the dynamics of $\bar{\v}$. Consequently, the MSE loss remains nearly constant for an extensive period of time, followed by a sharp drop by $a^2/2$ at the aforementioned critical time.

\paragraph{Decoupled learning of multiple tasks.}
Next consider the multi-index setting where we have $P$ orthonormal ground-truth directions $\{\e_p\}_{p\in [P]}$
with signal strength $\{a_p\}_{p \in [P]}$. Assume these $P$ single-index models are fully decoupled, i.e., 
for each $p \in [P]$, there is exactly one learner neuron $\v_p$ associated with direction $\e_p$, and the learning of different directions do not interfere --- in other words, we are learning $P$ single-index 
models independently and simultaneously. 
Then from our previous discussion, we know that direction $\e_p$ will be learned around time $( 8 a_p \bar{v}_{p, p}^2(0) )\inv$ and the MSE loss will have a sudden 
drop of size $a_p^2$. Therefore, the idealized loss can be expressed as the sum of loss decrements at different times (we omit the constant factor $1/2$ for concise presentation)
\[ 
  \tilde{L}(t)
  = \sum_{p=1}^{P} a_p^2 \indi\braces{ t < \big( 8a_p \bar{v}_{p, p}^2(0) \big)\inv }. 
\]
See Figure~\ref{fig:scaling-law}(a) for illustration. Based on this heuristic, we can derive the iteration/sample scaling in Proposition~\ref{thm: scaling law}.
Suppose that the signal strength follows a power law $a_p = p^{-\beta}$ for some $\beta > 1/2$, and assume identical initial overlap for all neurons $\bar{v}_{p, p}^2(0) = v^2$ for all $p \in [P]$, so that direction $\e_p$ is learned at exactly 
$t=p^{\beta} v^{-2}/8$. Then, when $P$ is large, we have 
\[
  \tilde{L}\big( p^{\beta} v^{-2}/8 \big)
  \approx \sum_{q=p}^\infty q^{-2\beta}
  \approx \int_p^\infty s^{-2\beta} \,\rd s
  = \frac{p^{1-2b}}{2b-1}.
\]
Applying the change-of-variables $t = p^{\beta} v^{-2}/8$, $p = (8 v^2 t)^{1/\beta}$, we obtain 
the loss scaling
\[
  \tilde{L}(t)
  \approx (2b-1)^{-1}{(8 v^2)^{(1-2b)/b}} \cdot t^{-(2b-1)/b}.
\]
To make the above approximations rigorous, it suffices to control the difference between gradient flow and the idealized decoupled dynamics, and estimate the fluctuation caused by the randomness of $\bar{v}_{p, p}^2(0)$, which we handle in Appendix~\ref{sec: scaling law}.

\paragraph{Width scaling.} To obtain the student width dependence, we show that a width-$m$ student network can learn $\tilde{\Theta}(m)$ directions -- note that this is sharp up to logarithmic factors.   
Hence the approximation error can be computed as a truncation of the top $\tilde{\Theta}(m)$ tasks: $\sum_{q=\tilde{\Theta}(m)}^P q^{-2\beta} \approx \tilde{\Theta}(m^{1 - 2\beta})$.

\subsection{The Gradient Flow Dynamics}
\label{sec: idealized to gf}

In the previous section, we assumed complete decoupling of the learning of the single-index tasks. We now discuss how this 
condition holds approximately under gradient flow. Note that, \textit{a priori}, there is no reason to believe these single-index tasks can be decoupled, even when the norm of the leaner neurons, and therefore their interaction, is small, as those 
larger teacher directions will attract all the learner neurons and the model could potentially collapse to a few larger 
directions. We show that (i) if a learner neuron is sufficiently random (or, more precisely, incoherent), then the 
influence of different teacher directions can be decoupled, and (ii) thanks to the sharp transitions in the training dynamics, 
when a large teacher direction gets fitted by a learner neuron, there is still enough randomness in the remaining leanrer
neurons.

\paragraph{Re-indexing and greedy maximum selection.}
To simplify notation, we first re-index the neurons based on the initial correlation with the ground-truth directions. Let $\cV \subset \R^d$ be the collection of initialized neurons. Define 
\(
  (\pi(1), \v_1) := \argmax_{q \in [P], \v \in \cV} a_q \bar{v}_q^{2I-2}. 
\)
By our previous heuristic argument, we expect $\e_{\pi(1)}$ to be the first direction recovered, and 
$\v_1$ -- which achieves maximal overlap (weighted by $a_{\pi(1)}$) with $\e_{\pi(1)}$ at initialization -- to be the student neuron that converges to this direction first. 
After $\e_{\pi(1)}$ is fitted by $\v_1$, we remove this task from the cumulative objective; assuming the remaining student neurons have not moved too much during this process, we can determine the next task to be learned and the corresponding neuron via 
\begin{equation}
  \label{eq: greedy maximum selection}
  (\pi(p+1), \v_{p+1})
  = \textstyle\argmax_{ 
      \substack{ q \in [P]\setminus\{\pi(1), \dots, \pi(p)\} \\ \v \in \cV\setminus\{\v_1, \dots, \v_p\} } 
    } 
    \, a_q \bar{v}_q^{2I-2},
  \quad \forall p \in [ \min\{P, m\} - 1 ]. 
\end{equation} 
Finally, if $P < m$ we index the remaining unused neurons as $\{ \v_{P+1}, \dots, \v_m \}$, and if $m < P$ we assign $\{\pi(m+1), \dots, \pi(P)\}$ to the unlearned teacher neurons arbitrarily so that $\pi$ is a permutation of $[P]$. Following \cite{ben_arous_stochastic_2024}, we call \eqref{eq: greedy maximum selection}
the \textit{greedy maximum selection} scheme and the matrix $\{ a_{\pi(p)} \bar{v}_{k, \pi(p)}^{2I-2}(0)  \}_{k \in [m], p \in [P]}$
the greedy maximum selection matrix (cf.~Figure~\ref{fig: greedy max matrix}).
Note that by construction, $a_{\pi(p)} \bar{v}_{p, \pi(p)}^{2I-2}$ is larger than all entries below it or on its right-hand 
side. We have the following quantitative estimates on the gaps between the on-diagonal and remaining entries of the maximum selection matrix at initialization. See Appendix~\ref{sec: proof of lemma: initialization} for the proof.

\begin{lemma}[Initialization]
  \label{lemma: initialization}
  Let $\delta_{\P} \in (e^{-\log^2 d}, 1)$ be the target failure probability. Suppose that $\delta_r = \frac{\delta_{\P} \pi}{2 m P^2}, \delta_t = \delta_c = \frac{\delta_{\P} \pi}{12 m^2 P}, d \ge \frac{400 (I - 1)^2}{\delta_c^2} \log\left( \frac{2 \pi}{3 \delta_c} \right), m \ge 4  P_* \log(P_* / \delta_{\P}) \lor 100\log(P/\delta_{\P}), \frac{m }{\log^3 m} \ge 512 \log^2(P_*/\delta_{\P})$
  Then, the following holds with probability at least $1 - O(\delta_{\P})$. 
  \begin{enumerate}[(a),leftmargin=*]
    \item \textbf{(Row gap)}. For any $p \in [P_*]$ and $p < q \in [P]$, we have 
      $a_{\pi(p)} \bar{v}_{p, \pi(p)}^{2I-2} \ge (1 + \delta_r) a_{\pi(q)} \bar{v}_{p, \pi(q)}^{2I-2}$.
    \item \textbf{(Column gap)}. For any $p \in [P_*]$ and $p < k \in [m]$, we have 
      $\bar{v}_{p, \pi(p)}^{2I-2} \ge (1 + \delta_c) \bar{v}_{k, \pi(p)}^{2I-2}$.
    \item \textbf{(Threshold gap)}. For any $P_* < q \le P$, $P_* < k \le m$, we have 
      $a_{\pi(P_*)} \bar{v}_{P_*, \pi(P_*)}^{2I-2} \ge (1 + \delta_t) a_{\pi(q)} \bar{v}_{k, \pi(q)}^{2I-2}$.
    \item \textbf{(Regularity conditions)}. 
      $\max_{k \in [m]} \norm{\bar{\v}_k}_\infty^2 \le \log^2 d / d$, 
      $\min_{p \in [P_*]} \bar{v}_{p, \pi(p)}^2 \ge (\log P_*) / d$ and $\min_{q \in [P]}\max_{j > P_*}\bar v^2_{j, q} \ge 1/d$.
  \end{enumerate}
\end{lemma}

\begin{figure}[t]
\centering
\begin{minipage}[c]{0.4\textwidth} 
  
\scalebox{1}{%
\begin{tikzpicture}[x=0.72cm,y=-0.63cm,>=latex]
  \def\xsh{-12pt}   %
  \def\ysh{7pt}   %

  \def\p{2}        %
  \def\Pstar{4.5}    %
  \def\P{7}        %
  \def\m{9}        %
  \def\sq{0.3}     %
  \def\ext{0.5}    %
  \pgfmathsetmacro{\axisEnd}{(\P+\m)/2}
  \pgfmathsetmacro{\midDown}{(\P+\axisEnd)/2}

  \draw[thick,dashed,shorten >=12pt]  (-\ext,\Pstar) -- (\axisEnd,\Pstar);
  \draw[thick,dotted,shorten >=12pt]  (-\ext,\P)      -- (\axisEnd,\P);
  \draw[thick,dashed]  (\Pstar,-\ext)  -- (\Pstar,\axisEnd);
  \draw[thick,dotted]  (\P,-\ext)      -- (\P,\axisEnd);

  \draw[thick,shorten >=12pt] (0,0) -- (\axisEnd,0);    %
  \draw[thick] (0,0) -- (0,\axisEnd);    %

  \node[anchor=south,xshift=\xsh,yshift=-2pt,font=\footnotesize] at (1,0)           {$\pi(1)$};
  \node[anchor=south,xshift=\xsh,yshift=-2pt,font=\footnotesize] at (2,0)           {$\pi(2)$};
  \node[anchor=south,xshift=\xsh,yshift=-2pt,font=\footnotesize] at ({(2+\Pstar)/2},0)   {$\cdots$};
  \node[anchor=south,xshift=\xsh,yshift=-2pt,font=\footnotesize] at (\Pstar-0.1,0)      {$\pi(P_*)$};
  \node[anchor=south,xshift=\xsh,yshift=-2pt,font=\footnotesize] at ({(\Pstar+\P)/2+0.1},0) {$\cdots$};
  \node[anchor=south,xshift=\xsh,yshift=-2pt,font=\footnotesize] at (\P,0)          {$\pi(P)$};

  \node[anchor=east,xshift=1pt,yshift=\ysh   ,font=\footnotesize] at (0,0.7)                {1};
  \node[anchor=east,xshift=1pt,yshift=\ysh   ,font=\footnotesize] at (0,1.8)                {2};
  \node[anchor=east,xshift=1pt,yshift=\ysh   ,font=\footnotesize] at (0,3)                {$\vdots$};
  \node[anchor=east,xshift=1.25pt,yshift=\ysh   ,font=\footnotesize] at (0,\Pstar)           {$P_*$};
  \node[anchor=east,xshift=1pt,yshift=\ysh   ,font=\footnotesize] at (0,{(\Pstar+\P)/2})  {$\vdots$};
  \node[anchor=east,xshift=1pt,yshift=\ysh   ,font=\footnotesize] at (0,\P)               {$P$};
  \node[anchor=east,xshift=1pt,yshift=\ysh/2   ,font=\footnotesize] at (0,\midDown)         {$\vdots$};
  \node[anchor=east,xshift=1pt,yshift=1pt   ,font=\footnotesize] at (0,\axisEnd)         {$m$};

  \fill[magenta!30,opacity=0.5]  (\p+\sq,\p-\sq) rectangle (\P,\p+\sq);
  \draw[magenta!80!black,ultra thick,-latex] (\p+\sq,\p) -- (\P,\p);

  \fill[blue!30,opacity=0.5]  (\p-\sq,\p+\sq) rectangle (\p+\sq,\axisEnd);
  \draw[blue!80!black,ultra thick,-latex] (\p,\p+\sq) -- (\p,\axisEnd);

  \fill[green!30,opacity=0.5]  (\Pstar,\Pstar) rectangle (\P,\axisEnd);
  \draw[green!60!black,ultra thick,-latex] (\Pstar,\Pstar) -- (\P,\Pstar);
  \draw[green!60!black,ultra thick,-latex] (\Pstar,\Pstar) -- (\Pstar,\axisEnd);
  \draw[green!60!black,ultra thick,-latex] (\Pstar,\Pstar) -- (\P,\P);

  \draw[red,ultra thick]         (0,0) -- (\p-\sq,\p-\sq);
  \draw[red,ultra thick,-latex]  (\p+\sq,\p+\sq) -- (\Pstar-\sq,\Pstar-\sq);

  \filldraw[white,draw=red,ultra thick]
    (\p-\sq,\p-\sq) rectangle (\p+\sq,\p+\sq);
  \filldraw[white,draw=red,ultra thick]
    (\Pstar-\sq,\Pstar-\sq) rectangle (\Pstar+\sq,\Pstar+\sq);

  \node[red,anchor=south west,font=\small,xshift=7pt,yshift=0.5mm]
    at (\p-\sq,\p) {$a_{\pi(\!p\!)}\,\bar v^{\!2I-2}_{\!p,\pi(\!p\!)}$};
  \node[red,anchor=south west,font=\small,xshift=-7pt,yshift=0.5mm]
    at (\Pstar+\sq,\Pstar) {$a_{\pi(\!P_*\!)}\,\bar v^{\!2I-2}_{\!P_*,\pi(\!P_*\!)}$};

\end{tikzpicture}
}
\end{minipage}
\hfill
\begin{minipage}[c]{0.58\textwidth}
    \caption{\small 
       The greedy maximum selection matrix. The red diagonal entries represent the relevant neurons that eventually achieve overlap close to $1$. The remaining irrelevant entries can be 
       partitioned into three groups: the upper triangular entries $\bar{v}_{p, \pi(q)}$ with $p \in [P_*]$ and 
       $p < q \in [P]$, the lower triangular entries, $\bar{v}_{k, \pi(p)}$ with $p \in [P_*]$ and $p < k \in [m]$,
       and the lower right block $\bar{v}_{k, \pi(q)}$ with $k > P_*, q > P_*$. We will control these blocks using the 
       row gap (purple arrow), column gap (blue arrow), and the threshold gap (green arrows), respectively.
    } \label{fig: greedy max matrix}
  \end{minipage}
\end{figure}
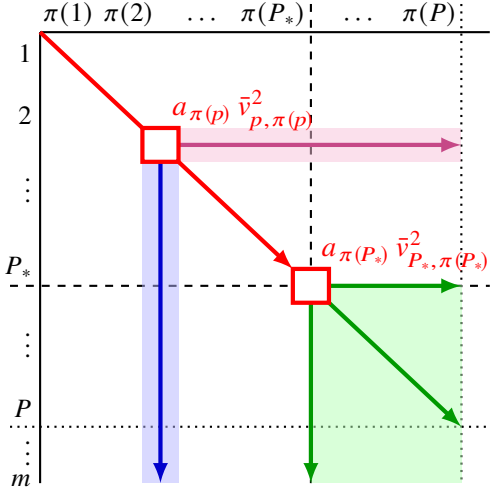

\paragraph{Approximately decoupled dynamics.}
We claim that when all irrelevant coordinates are small, the learning of different teacher directions can still be approximately decoupled. 
By Lemma~\ref{lemma: population and per-sample gradients}, the dynamics of the overlap $\bar{v}_{p, \pi(p)}^2$ can be decomposed into a primary signal term and the sum of contributions from the remaining coordinates: 
\[
  \textstyle
  \frac{\rd}{\rd t} \bar{v}_{p, \pi(p)}^2 
  \approx 8 \big( 
      a_{\pi(p)} \big( 1 - \bar{v}_{p, \pi(p)}^2 \big) \bar{v}_{p, \pi(p)}^2 
      - \sum_{q: q \ne p} a_{\pi(q)} \bar{v}_{p, \pi(q)}^4 \big)
    \bar{v}_{p, \pi(p)}^2.  
\]
When the overlap $\bar{v}_{p, \pi(p)}^2$ is small, the signal term is of order 
$a_{\pi(p)} \bar{v}_{p, \pi(p)}^2 = \Omega(a_{\pi(p)} / d )$. Also, if we assume all irrelevant coordinates (i.e., $\bar v_{p, \pi(q)}^2$ for $q \neq p$) are small, say bounded by $d^{-0.9}$, then
\[ 
  \sum_{q: q \ne p} a_{\pi(q)} \bar{v}_{p, \pi(q)}^4  
  \le d^{-1.8} \sum_{q: q \ne p} a_{\pi(q)} 
  \le P^{1/2} d^{-1.8} 
  \ll a_{\pi(p)} / d, 
\]
as long as $a_{\min} P^{1/2} \gg d^{-0.8}$. As a result, when $\bar{v}_{p, \pi(p)}^2$ is still small, we have 
\[\textstyle
  \frac{\rd}{\rd t} \bar{v}_{p, \pi(p)}^2 
  \approx \left( 1 \pm a_{\min}\inv d^{-0.8} \right) 
      \times 8 a_{\pi(p)}  \bar{v}_{p, \pi(p)}^4.
\]
Now suppose $a_{\min} \gg d^{-0.3}$. Then, the above implies that $\bar{v}_{p, \pi(p)}^2$
has a sharp transition around time $\left(1 \pm o(1) \right) ( 8 a_{\pi(p)} \bar{v}_{p, \pi(p)}^2 )\inv
= \tilde{\Theta}(d/a_{\pi(p)})$, 
and the $o(1)$ error term can be made much smaller than 
$1 / \poly(P)$ when $d$ is large --- this will be useful in bounding the growth of irrelevant coordinates. 

Similar to the analysis in \cite{ge_understanding_2021}, we know that once $\bar{\v}_p$ converges to $\e_{\pi(p)}$, the convergence of norm $a_{\pi(p)}$ occurs within $O(\log d)$ time, and its dynamics become local in the sense that the influence of other teacher neurons becomes negligible. 
In addition, after $\e_{\pi(p)}$ is learned, the remaining learner neurons will no longer be affected by this target direction.

\paragraph{Bounding the irrelevant coordinates.} We show that the irrelevant coordinates, i.e., ones that are not in $\{ \bar{v}_{p, \pi(p)} \}_{p \in [P_*]}$ (cf.~Figure~\ref{fig: greedy max matrix}), stay small throughout training using the fact that the dynamics have sharp transitions. Here, we only consider the lower triangular entries of the greedy maximum selection matrix, i.e., $\bar{v}_{k, \pi(p)}$ with $p \in [P_*]$ and $p < k \in [m]$, which we control using the column gap. The other entries can be controlled using similar strategies -- see Appendix~\ref{sec: gf: induction hypotheses} for details.
Recall that $\frac{\rd}{\rd t} \bar{v}_{k, \pi(p)}^2 \approx 8 a_{\pi(p)} \bar{v}_{k, \pi(p)}^4$, which has a sharp transition around time $( 8 a_{\pi(p)} \bar{v}_{k, \pi(p)}^2(0) )\inv$. From the column gap in Lemma~\ref{lemma: initialization}, this implies that $\bar{v}_{k, \pi(p)}^2$ stays small before $\v_p$ fits $a_{\pi(p)} \e_{\pi(p)}$. After that, the signal from $a_{\pi(p)} \e_{\pi(p)}$ will be close to $0$, and consequently $\bar{v}_{k, \pi(p)}^2$ will cease to grow. 

\subsection{Online Stochastic Gradient Descent}
\label{sec: online sgd}

In this section, we outline the proof of Theorem \ref{thm:online-sgd}, and demonstrate how to convert our analysis of the gradient flow dynamics into an analysis for the online SGD trajectory. At a high level, our proof relies on the martingale-plus-drift argument used in prior works (\cite{benarous2021online,abbe2023sgd,damian2023smoothing,oko2024learning,ren2024learning}). In order to rigorously handle the interdependence of the different martingale arguments, we rely on the stochastic induction arguments of \cite{ren2024learning}. 
The complete proof of Theorem \ref{thm:online-sgd} is presented in Appendix \ref{sec:online-sgd-proofs}.

\paragraph{Controlling the irrelevant coordinates.}
As in the gradient flow setting, we begin by bounding the growth of the irrelevant coordinates $\bar v^2_{k, \pi(q)}$ for $(k, \pi(q)) \not\in \{(p, \pi(p))\}_{p \in [P_*]}$. By Lemma \ref{lemma: population and per-sample gradients} and a similar simplifying argument as in Section~\ref{sec: idealized to gf}, one can show that the update on $\bar v^2_{k, \pi(q)}$ is given by
\begin{align*}
    \bar v_{k, \pi(q)}^2(t+1) \le \bar v_{k, \pi(q)}^2(t) + 8\eta a_{\pi(q)}\bar v^4_{k, \pi(q)}(t) + \xi_{t+1} + Z_{t+1},
\end{align*}
where $\xi_{t+1} \ll 1$ is an error term we will ignore for ease of exposition, and $Z_{t+1}$ is a martingale term defined by
\begin{align*}
    Z_{t+1} = \frac{2\eta \bar v_{k,\pi(q)}(t)}{\norm{\v_k(t)}} \big\langle (\Id - \bar\v_k(t)\bar\v_k(t)^\top)\left(\nabla_{\v_k(t)}l(\x_t) - \nabla_{\v_k(t)}\mathcal{L}\right), \e_{\pi(q)}\big\rangle.
\end{align*}
By Lemma \ref{lemma: population and per-sample gradients}, the conditional variance can be bounded as $\E[Z_{t+1}^2 \mid \mathcal{F}_t] \lesssim \eta^2 \bar v^2_{k,p}(t)$. One can then bound the total contribution of the martingale terms via Doob's inequality:
\begin{align*}
    \P\left[\sup_{r \le t}\abs{\sum_{s=1}^r Z_s} \ge M \right] \le M^{-2}\sum_{s=1}^t\E[Z_{s}^2] \lesssim  M^{-2}T\eta^2 d^{-1},
\end{align*}
where we heuristically use the fact that the ``typical" size of $\bar v_{k,\pi(q)}^2$ is $d^{-1}$. \cite{benarous2021online} selects $M = \frac12\bar v_{k, \pi(q)}^2(0) = \tilde \Theta(d^{-1})$, which requires a learning rate of $\eta \lesssim d^{-2}a_{\pi(q)}$, so that $\bar v_{k, \pi(q)}^2(t)$ can be coupled to the deterministic process $\hat x_{t+1} = \hat x_t + 8\eta a_{\pi(q)}\hat x_t^2$ with $\hat x_0 = 1.5\bar v_{k, \pi(q)}^2(0)$. 

Unfortunately, this only guarantees that the escape time of online SGD matches the corresponding gradient flow escape time of $(8\eta a_{\pi(q)}\bar v_{k, \pi(q)}^2(0))^{-1}$ up to constant factor. This is problematic, as we wish to argue $\bar v_{k, \pi(q)}^2$ stays small for the entirety of the time it takes for either $\bar v_{q, \pi(q)}^2$ (when $q < k$) or $\bar v_{k, \pi(k)}^2$ (when $k < q$) to grow close to 1. Let us begin by assuming that $(k, \pi(q))$ is a lower triangular entry, i.e $q \in [P_*]$ and $q < k$. The gradient flow escape time of $\bar v^2_{q, \pi(q)}$ is $(8\eta a_{\pi(q)}\bar v_{q, \pi(q)}^2(0))^{-1}$. By Lemma \ref{lemma: initialization}, this is only smaller than the escape time of $\bar v^2_{k, \pi(q)}$ by a multiplicative factor of $1 + \delta_c$, where $\delta_c = o(1)$ is the column gap. As such, only proving that the online SGD escape time is within a constant factor of the corresponding gradient flow escape time is insufficient.

Instead, by choosing the smaller learning rate $\eta \lesssim d^{-2}\delta_c^{2}a_{\pi(q)}$ we can now bound the total martingale term by $M \lesssim \delta_c \bar v_{k,\pi(q)}^2(0)$. The online SGD escape times for both $\bar v_{k, \pi(q)}^2$ and $\bar v_{q, \pi(q)}^2$ are now within a $(1 + \delta_c)$ multiplicative factor of their corresponding gradient flow escape times. Therefore $\bar v_{k, \pi(q)}^2$ is guaranteed to stay small in the time it takes for $\bar v^2_{q, \pi(q)}$ to grow to $\approx 1$. Afterwards, as in the gradient flow setting, the signal from $\e_{\pi(q)}$ will be close to 0, and $\bar v^2_{k, \pi(q)}$ will stop growing. The upper triangular entries ($k \in [P_*], k < q$) can be handled similarly, by scaling the learning rate $\eta$ with the row gap $\delta_r^2$.

\paragraph{On the unstable discretization.} There is a subtle challenge with handling the entries $\bar v_{k, \pi(q)}^2$ where $q > P_*$. As discussed above, the ``standard" online SGD analysis bounds the martingale term by $\Theta(\eta\sqrt{T/d})$. Since the convergence time is $\tilde \Theta(d\eta^{-1}a_{\pi(q)}^{-1})$, this corresponds to a learning rate of $\eta \propto d^{-2}a_{\pi(q)}$. However, it is quite pessimistic to scale the learning rate with the signal strength of a neuron which is not learned, as this can be arbitrarily small. Instead, we observe that if we are only interested in recovering the top $P_*$ directions, then it suffices to couple to the corresponding deterministic process up to time $T_{P_*} = \tilde \Theta(d\eta^{-1}a_{\pi(P_*)}^{-1})$. We therefore only need to scale $\eta$ with $a_{\pi(P_*)} \gg a_{\pi(q)}$. This can be interpreted as an ``unstable discretization," as the choice of learning rate $\eta$ is too large for any of the directions $\pi(q)$ with $q > P_*$ to be learned. While these $\bar v_{q, \pi(q)}^2$ will never converge to 1, we are still able to control their growth and show that they are small until the time that the $\pi(P_*)$th teacher neuron is learned. Altogether, it suffices to choose $\eta \propto a_{\pi(P_*)}\Delta^2d^{-2}$, where $\Delta := \min(\delta_r, \delta_c, \delta_t)$. The final result bounding the growth of the irrelevant coordinates is given in Lemma \ref{lem:total-growth-failed-coords}.

\paragraph{Controlling the relevant coordinates.} We next consider the growth of the relevant coordinates $\bar v_{p, \pi(p)}^2$ for $p \in [P_*]$. Following the argument in Section~\ref{sec: idealized to gf}, the update on $\bar v_{p, \pi(p)}^2$ is approximately
\begin{align*}
    \bar v_{p, \pi(p)}^2(t+1) \approx \bar v_{p, \pi(p)}^2(t) + 8\eta a_{\pi(p)}\left(1 - \bar v^2_{p, \pi(p)}(t)\right)\bar v^4_{p, \pi(p)}(t) + Z_{t+1},
\end{align*}
where the martingale term $Z_{t+1}$ satisfies $\E[Z_{t+1}^2 \mid \mathcal{F}_t] \lesssim \eta^2 \bar v^2_{p,\pi(p)}(t)$. Similarly to the irrelevant coordinates, by choosing the learning rate $\eta \lesssim d^{-2}\Delta^2 a_{\pi(p)}$, we can bound $\bar v_{p, \pi(p)}^2(t)$ between two deterministic processes $(x^+_t)_t, (x^-_t)_t$ which satisfy $x^{\pm}_0 = (1 \pm O(\Delta))\bar v_{p, \pi(p)}^2(0)$ and follow the updates $x^{\pm}_{t+1} = x^{\pm}_t + 8\eta a_{\pi(p)}(x^{\pm}_t)^2$. This guarantees that $\bar v^2_{p,\pi(p)} \ll 1$ up to a time of $(1 - O(\Delta))(8\eta a_{\pi(p)}\bar v_{p, \pi(p)}^2(0))^{-1}$.

However, lower bounding the process $v_{p, \pi(p)}^2(t)$ is less straightforward. The main challenge is that the variance of the martingale term also scales with $\bar v_{p, \pi(p)}^2$. When $x^+_t$ is small, it suffices to upper bound the variance by $O(\eta^2x_t^{\pm})$. However, when $t \ge (8\eta a_{\pi(p)}\bar v_{p, \pi(p)}^2(0))^{-1}$, then the process $x^+_t$ will have already diverged to $\infty$, while $x_t^-$ is still only $\Theta(\Delta^{-1}\bar v_{p, \pi(p)}^2(0))$. To handle this, we will split the interval $[\delta/d, 1/3]$ into the smaller subintervals $[\delta/d, \delta^2/d], [\delta^2/d, \delta^4/d], [\delta^4/d, \delta^8/d],$ etc, where $\delta = \tilde \Theta(\Delta^{-1})$. We then run separate martingale-plus-drift arguments on each subinterval, starting from some $t$ where $\bar v_{p, \pi(p)}^2(t) \ge \delta^{2^k}/d$ and using $\eta^2 \delta^{2^{k+1}}/d$ as an upper bound on the variance in this interval. Altogether, we can show that the total time required for $\bar v_{p, \pi(p)}^2$ to grow from $\delta/d$ to $\frac13$ can be upper bounded by $O(\Delta(8\eta a_{\pi(p)}\bar v_{p, \pi(p)}^2(0))^{-1})$. We conclude by showing that once $\bar v^2_{p, \pi(p)}(t)$ crosses $1/3$, it rapidly converges to $1-\eps$. Altogether, in Lemma \ref{thm:directional-convergence-proof}, we show that $\bar\v_p$ indeed converges to $\e_{\pi(p)}$ in time $(1 \pm O(\Delta))(8\eta a_{\pi(p)}\bar v_{p, \pi(p)}^2(0))^{-1}$.

\paragraph{Norm convergence.} To conclude, we must analyze the dynamics of the norm $\norm{\v_k(t)}^2$. Similarly to the gradient flow setting, we show that the norm of $p$-th neuron $\norm{\v_p}^2$ only begins to grow once strong recovery ($\bar v_{p, \pi(p)}^2 \ge 1 - \eps$) is achieved, and moreover that $\norm{\v_p}^2\to a_{\pi(p)}$ rapidly in this stage.

\section{Conclusion}

In this work, we study the (online) SGD training dynamics and sample complexity of learning a two-layer neural network with orthogonal ground truth weights and signal strengths $\{a_p\}_{p \in [P]} \subset \R_{\ge 0}$, where the width $P$ and the condition number $a_{\max}/a_{\min}$ can potentially be large. We establish a sample and runtime complexity that is polynomial in the problem dimensionality, teacher width, and condition number; as an application of our sharp analysis, when the second-layer coefficients of the teacher model follow a power law $a_p \asymp p^{-\beta}$ for $\beta>1/2$, we derive scaling laws for the population MSE as a function of the student network width and the number of SGD steps.

Our current results assume input data with identity covariance; one interesting extension is to consider anisotropic data $\x \sim \Gaussian{0}{\mathbf{\Sigma}}$ analogous to \cite{mousavi2023gradient,braun2025learning}, and derive a two-parameter scaling law when the eigenvalues of $\mathbf{\Sigma}$ also follow a power law. Another future direction is to consider a decaying learning rate schedule that achieves the unstable scaling law (Corollary~\ref{cor:unstable-scaling-law}) at any time $t$. Finally, our analysis relies on high information exponent link functions to decouple the learning of different directions, which does not cover the case of $\text{IE}(\sigma)=2$ studied in \cite{martin2023impact,ren2024learning} --- for this setting, the scaling behavior for SGD training is studied in a companion work \cite{benarous2025learning} for the special case of quadratic activation function.

\bigskip

\subsection*{Acknowledgments}

The authors would like to thank Alberto Bietti, Theodor Misiakiewicz, Elliot Paquette and Nuri Mert Vural for discussion and feedback. 
JDL acknowledges support of the NSF CCF 2002272, NSF IIS 2107304, and NSF CAREER Award 2144994. This work was done in part while DW and JDL
were visiting the Simons Institute for the Theory of Computing. 

\bigskip

{

\fontsize{11}{11}\selectfont     

\bibliography{reference}  
\bibliographystyle{alpha} 

\newpage

}

\appendix

{
  \hypersetup{linkcolor=black}
  \tableofcontents
}

\newpage 

\allowdisplaybreaks

\section{Additional Related Works}

\paragraph{Theory of scaling laws.} Neural scaling laws describe how the performance of deep learning models improves predictably as a power-law function of increased computational resources, data, and model size \cite{hestness2017deep,kaplan2020scaling,hoffmann2022training,bahri2024explaining}. When the optimization algorithm is not taken into account, such scaling relations have been established for the approximation and estimation errors of deep neural networks \cite{pinkus1997approximating,suzuki2018adaptivity,schmidt2020nonparametric}, as well as for the (precise) generalization error of simple closed-form estimators such as ridge regression \cite{cui2021generalization,maloney2022solvable,defilippis2024dimension,atanasov2024scaling}.
Recent works have also studied the loss scaling in distillation and synthetic data \cite{ildiz2024high,jain2024scaling}, associative memory 
\cite{cabannes2023scaling,nichani2024understanding} and hierarchical models \cite{cagnetta2024towards,cagnetta2024deep,arnal2024scaling,pan2025understanding}, among other theoretical settings.

The scaling laws of SGD in sketched linear regression have been characterized in \cite{bordelon2024dynamical,paquette2024phases,lin2024scaling} --- this problem setting corresponds to a two-layer linear network with random, untrained first-layer weights, and is parallel to earlier works \cite{rudi2017generalization,nitanda2020optimal} on learning random features model under source and capacity conditions (see e.g., \cite{caponnetto2007optimal,velikanov2024tight}). 
However, this linear setup fails to capture the feature learning efficiency of neural networks. On the other hand, existing scaling analyses for the additive setting \cite{hutter2021learning,michaud2024quantization,nam2024exactly} explicitly decompose the loss into an independent sum, simplifying the analysis due to task decoupling. We aim to understand a more natural -- yet arguably more challenging -- nonlinear feature learning scenario where the individual tasks are not decoupled.
 
\paragraph{Learning shallow neural networks.} The learning of two-layer neural networks with near-orthogonal neurons has been extensively studied in the deep learning theory literature. Existing works have studied the optimization dynamics for variants of ReLU \cite{li2020learning,zhou2021local,chizat2022sparse}, quadratic \cite{ghorbani2019limitations,sarao2020optimization,martin2023impact}, and general Hermite activation functions \cite{oko2024learning,ren2024learning,simsek2024learning}. In the absence of the (near-)orthogonality assumption, this function class can be computationally hard to learn, as suggested by statistical query lower bounds \cite{diakonikolas2020algorithms,goel2020superpolynomial}. Our target function is a subclass of additive models \cite{stone1985additive,hastie1987generalized}, where the individual components take the form of single-index models --- see \cite{bach2017breaking,oko2024learning} for further discussion.

\bigskip
\section{Structure of Gradient and Initialization}

\subsection{Population and Per-sample Gradients}
\label{sec: proof of lemma: population and per-sample gradient}

In this subsection, we compute the population gradient and derive variance and tail bounds for the per-sample gradient. Namely, we prove the following lemma. 
\begin{lemma}
  \label{lemma: population and per-sample gradients}
  Consider the setting described in Section~\ref{sec: setup}. 
  Assume w.l.o.g.~that $\v_p^* = \e_p$ for $p \in [P]$. 
  The radial and tangent components of the population gradient are given by 
  \begin{align*}
    - \inprod{\nabla_{\v_k} \Loss}{\v_k}
    &= 2 \norm{\v_k}^2 \sum_{i=I}^\infty \hat\sigma_{2i}^2 \sum_{p=1}^P a_p \bar{v}_{k, p}^{2i}
      - 2 \norm{\v_k}^2 \sum_{i=I}^\infty \hat\sigma_{2i}^2 \sum_{l=1}^m \norm{\v_l}^2 \inprod{\bar{\v}_k}{\bar{\v}_l}^{2i}, \\
    - \frac{\left[ (\Id - \bar{\v}_k\bar{\v}_k\trans)\nabla_{\v_k} \Loss \right]_p}{\norm{\v_k}}
    &= \sum_{i=I}^\infty 2i \hat\sigma_{2i}^2 \left( a_p \bar{v}_{k, p}^{2i-2} - \sum_{q=1}^P a_q \bar{v}_{k, q}^{2i} \right) \bar{v}_{k, p}
      \\
      &\qquad
      - \sum_{i=I}^\infty 2 i \hat\sigma_{2i}^2 \sum_{l : l \ne k} \norm{\v_l}^2 
        \inprod{\bar{\v}_k}{\bar{\v}_l}^{2i-1} \inprod{(\Id - \bar{\v}_k\bar{\v}_k\trans) \bar{\v}_l}{\e_p}. 
  \end{align*}
  Suppose that $\sum_{k=1}^{m} \norm{\v_k}^2 = O(\norm{\a}_1)$. Let $\u \in \S^{d-1}$ be a fixed direction. 
  Put $\tilde{Q} = 4 (1 + Q)$. Then, there exists a universal constant $C \ge 1$ such that, for any $s \ge C$, 
  \[
    \Var \frac{\inprod{\nabla_{\v_k} l(\x)}{\u}}{\norm{\v_k}} \le C \norm{\a}_1^2, \quad 
    \P\left(
      \left| \frac{\inprod{\nabla_{\v_k} l(\x)}{\u}}{\norm{\v_k}} \right| \ge s
    \right)
    \le C m \exp\left( - C\inv \left(s/\norm{\a}_1\right)^{2/\tilde{Q}} \right).
  \]
\end{lemma}
\begin{proof}
  The proof of the variance and tail bounds is essentially the same as the proof of Lemma~A.5 of 
  \cite{ren2024learning}.\footnote{
    Note that though Lemma~A.3 of \cite{ren2024learning} is stated for i.i.d.~random variables, the original theorem in \cite{kuchibhotla_moving_2022} requires only independence and therefore applies to our setting.
  } 
  Now, we compute the population gradient. 
  First, recall from \eqref{eq: population loss} that the population loss is given as
  \[
    \Loss 
    = \sum_{i=I}^\infty \hat\sigma_i^2 \left(
        \frac{\norm{\a}^2 }{2}
        - \sum_{p=1}^P \sum_{k=1}^{m} a_p \norm{\v_k}^2 \inprod{\bar{\v}_k}{\v_p^*}^i
        + \frac{1}{2} \sum_{k,l=1}^{m} \norm{\v_k}^2 \norm{\v_l}^2 \inprod{\bar{\v}_k}{\bar{\v}_l}^i 
      \right) 
    =: \sum_{i = I}^\infty \Loss_i.
  \]
  For its gradient, first note that for each $i \ge I$, 
  \begin{align*}
    \nabla_{\v} \left( \norm{\v}^2 \inprod{\bar\v}{\u}^i \right)
    = \nabla_{\v} \left( \frac{\inprod{\v}{\u}^i}{\norm{\v}^{i-2}} \right)
    &= \frac{\nabla_{\v} \inprod{\v}{\u}^i}{\norm{\v}^{i-2}}
      - \frac{\inprod{\v}{\u}^i}{\norm{\v}^{i-2}} \frac{\nabla_{\v} \norm{\v}^{i-2}}{\norm{\v}^{i-2}} \\
    &= \frac{i \inprod{\v}{\u}^{i-1} \u}{\norm{\v}^{i-2}}
      - \frac{\inprod{\v}{\u}^i}{\norm{\v}^{i-2}} 
        \frac{(i - 2) \norm{\v}^{i-3} \bar{\v}}{\norm{\v}^{i-2}} \\
    &= i \inprod{\bar{\v}}{\u}^{i-1} \norm{\v} \u
      - (i - 2) \inprod{\bar{\v}}{\u}^i \v. 
  \end{align*}
  Then, for each $k \in [m]$, we compute 
  \begin{align*}
    \nabla_{\v_k} \Loss_i 
    &= - \hat\sigma_i^2 \norm{\v_k} \sum_{p=1}^P a_p \left(
        i \bar{v}_{k, p}^{i-1} \e_p 
        - (i - 2) \bar{v}_{k, p}^i \bar{\v}_k
      \right) \\
      &\qquad
      + 2 \hat\sigma_i^2 \norm{\v_k}^2  \v_k 
      + \hat\sigma_i^2 \norm{\v_k} \sum_{l : l \ne k} \norm{\v_l}^2 
      \left(
        i \inprod{\bar{\v}_k}{\bar{\v}_l}^{i-1} \bar{\v}_l
        - (i - 2) \inprod{\bar{\v}_k}{\bar{\v}_l}^i \bar{\v}_k
      \right). 
  \end{align*}
  Hence, for the radial component, we have 
  \[
    \inprod{\nabla_{\v_k} \Loss_i}{\v_k}
    = - 2 \hat\sigma_i^2 \norm{\v_k}^2 \sum_{p=1}^P a_p \bar{v}_{k, p}^i 
      + 2 \hat\sigma_i^2 \norm{\v_k}^2 \sum_{l=1}^m \norm{\v_l}^2 \inprod{\bar{\v}_k}{\bar{\v}_l}^i .
  \]
  Meanwhile, for the tangent component, we have 
  \begin{align*}
    (\Id - \bar{\v}_k\bar{\v}_k\trans)\nabla_{\v_k} \Loss_i 
    &= - \hat\sigma_i^2 \norm{\v_k} \sum_{p=1}^P a_p i \bar{v}_{k, p}^{i-1} (\Id - \bar{\v}_k\bar{\v}_k\trans) \e_p 
      \\
      &\qquad
      + \hat\sigma_i^2 \norm{\v_k} \sum_{l : l \ne k} \norm{\v_l}^2 
        i \inprod{\bar{\v}_k}{\bar{\v}_l}^{i-1} 
        (\Id - \bar{\v}_k\bar{\v}_k\trans) \bar{\v}_l \\
    &= - \hat\sigma_i^2 \norm{\v_k} \sum_{p=1}^P a_p i \bar{v}_{k, p}^{i-1} \left( \e_p - \bar{v}_{k, p} \bar{\v}_k \right) 
        \\
        &\qquad
        + \hat\sigma_i^2 \norm{\v_k} \sum_{l : l \ne k} \norm{\v_l}^2 
          i \inprod{\bar{\v}_k}{\bar{\v}_l}^{i-1} \left(
            \bar{\v}_l - \inprod{\bar{\v}_k}{\bar{\v}_l} \bar{\v}_k 
          \right). 
  \end{align*}
  In particular, for each $p \in [P]$, we have 
  \begin{align*}
    \frac{\left[ (\Id - \bar{\v}_k\bar{\v}_k\trans)\nabla_{\v_k} \Loss_i  \right]_p}{\norm{\v_k}}
    &= - i \hat\sigma_i^2 \left( a_p \bar{v}_{k, p}^{i-2} - \sum_{q=1}^P a_q \bar{v}_q^i \right) \bar{v}_{k, p}
      \\
      &\qquad
      + i \hat\sigma_i^2 \sum_{l : i \ne k} \norm{\v_l}^2 
        \inprod{\bar{\v}_k}{\bar{\v}_l}^{i-1} \left(
          \bar{v}_{l, p} - \inprod{\bar{\v}_k}{\bar{\v}_l} \bar{v}_{k, p}
        \right). 
  \end{align*}
  Sum over $i \ge I$, and we obtain 
  \begin{align*}
    \inprod{\nabla_{\v_k} \Loss}{\v_k}
    &= - 2 \norm{\v_k}^2 \sum_{i=I}^\infty \hat\sigma_i^2 \sum_{p=1}^P a_p \bar{v}_{k, p}^i 
      + 2 \norm{\v_k}^2 \sum_{i=I}^\infty \hat\sigma_i^2 \sum_{l=1}^m \norm{\v_l}^2 \inprod{\bar{\v}_k}{\bar{\v}_l}^i, \\
    \frac{\left[ (\Id - \bar{\v}_k\bar{\v}_k\trans)\nabla_{\v_k} \Loss \right]_p}{\norm{\v_k}}
    &= - \sum_{i=I}^\infty i \hat\sigma_i^2 \left( a_p \bar{v}_{k, p}^{i-2} - \sum_{q=1}^P a_q \bar{v}_q^i \right) \bar{v}_{k, p}
      \\
      &\qquad
      + \sum_{i=I}^\infty i \hat\sigma_i^2 \sum_{l : l \ne k} \norm{\v_l}^2 
        \inprod{\bar{\v}_k}{\bar{\v}_l}^{i-1} \left(
          \bar{v}_{l, p} - \inprod{\bar{\v}_k}{\bar{\v}_l} \bar{v}_{k, p}
        \right). 
  \end{align*}
\end{proof}

\subsection{Initialization}
\label{sec: proof of lemma: initialization}

In this subsection, we prove Lemma~\ref{lemma: initialization}. 

\begin{proof}[Proof of Lemma~\ref{lemma: initialization} (row gap)]
  Consider an arbitrary neuron $\v$ and let $\z \sim \Gaussian{0}{\Id_d}$. Note that 
  $\bar{\v} \overset{d}{=} \z / \norm{\z}$ and therefore, for any $i \ne j$, 
  $\bar{v}_i / \bar{v}_j \overset{d}{=} z_i / z_j$, which follows the standard Cauchy distribution. 
  We know that $\P[ z_i/z_j \le z ] = \pi\inv \arctan(z) + 1/2$.  Fix $i \ne j$, we compute  
  \begin{align*}
    \P\left[  
      a_i \bar{v}_i^{2I-2} \in (1 \pm \delta_r) a_j \bar{v}_j^{2I-2}
    \right]
    &= 2 \P\left[  
        \left( (1 - \delta_r) \frac{a_j}{a_i} \right)^{\frac{1}{2I-2}}
        \le \frac{\bar{v}_i }{\bar{v}_j } 
        \le \left( (1 + \delta_r) \frac{a_j}{a_i} \right)^{\frac{1}{2I-2}}
      \right] \\
    &= \frac{2}{\pi} \left(
        \arctan \left( (1 + \delta_r) \frac{a_j}{a_i} \right)^{\frac{1}{2I-2}} 
        - \arctan\left( (1 - \delta_r) \frac{a_j}{a_i} \right)^{\frac{1}{2I-2}}
      \right) \\
    &= \frac{2}{\pi} \arctan\left(
        \frac{
          \left( (1 + \delta_r) \frac{a_j}{a_i} \right)^{\frac{1}{2I-2}} 
          - \left( (1 - \delta_r) \frac{a_j}{a_i} \right)^{\frac{1}{2I-2}}
        }{
          1 
          + \left( 
              (1 + \delta_r) (1 - \delta_r) \frac{a_j^2}{a_i^2} 
          \right)^{\frac{1}{2I-2}} 
        }
      \right),
  \end{align*}
  where the last line comes from $\arctan a - \arctan b = \arctan\frac{a - b}{1 + ab}$. 
  Note that for any $p \in (0, 1)$,
  by the concavity of $z \mapsto z^p$, we have $a^p - b^p \le p b^p (a - b)$. Therefore,

  \[
    \left( (1 + \delta_r) \frac{a_j}{a_i} \right)^{\frac{1}{2I-2}} 
    - \left( (1 - \delta_r) \frac{a_j}{a_i} \right)^{\frac{1}{2I-2}}
    \le \frac{1}{2 I - 2} \left( (1 - \delta_r) \frac{a_j}{a_i} \right)^{\frac{1}{2I-2} - 1} \delta_r \frac{a_j}{a_i}
    \le \frac{1}{2 I - 2} \left( \frac{a_j}{a_i} \right)^{\frac{1}{2I-2}} \delta_r. 
  \]
  Recall that $\arctan z \le z$. Thus, 
  \begin{align*}
    \P\left[  
      a_i \bar{v}_i^{2I-2} \in (1 \pm \delta_r) a_j \bar{v}_j^{2I-2}
    \right]
    \le \frac{2}{\pi} 
      \frac{
        \frac{1}{2 I - 2} \left( \frac{a_j}{a_i} \right)^{\frac{1}{2I-2}} \delta_r
      }{
        1 + \left( (1 - \delta_r^2) \frac{a_j^2}{a_i^2} \right)^{\frac{1}{2I-2}} 
      }  
    &= \frac{\delta_r}{(I - 1) \pi} 
      \frac{
        \left( a_i a_j \right)^{\frac{1}{2I-2}} 
      }{
        (a_i^2)^{\frac{1}{2I-2}} 
        + (1 - \delta_r^2)^{\frac{1}{2I-2}} 
        (  a_j^2 )^{\frac{1}{2I-2}} 
      }  \\ 
    &\le 
      \frac{\delta_r}{(I - 1) \pi} 
      \frac{
        \left( a_i^2 \vee a_j^2 \right)^{\frac{1}{2I-2}} 
      }{
        (a_i^2)^{\frac{1}{2I-2}} 
        + (1 - \delta_r^2)^{\frac{1}{2I-2}} 
        (  a_j^2 )^{\frac{1}{2I-2}} 
      }  \\ 
    &\le \frac{\delta_r}{(I - 1) \pi (1 - \delta_r^2)^{\frac{1}{2I-2}} } .
  \end{align*}
  The last term is upper bounded by $2 \delta_r / \pi$ as long as $\delta_r \le 1/2$. Apply union bound over 
  all $m$ neurons and all $P^2$ $(i, j)$-pairs, and we get 
  \[
    \P\left[ 
      \exists k \in [m], i \ne j \in [P],  
      a_i \bar{v}_{k, i}^{2I-2} \in (1 \pm \delta_r) a_j \bar{v}_{k, j}^{2I-2}
    \right]
    \le \frac{2 m P^2 }{\pi} \delta_r.
  \]
  Choose $\delta_r = \frac{\delta_{\P} \pi}{2 m P^2}$, so that the above implies 
  $a_i \bar{v}_{k, i}^{2I-2} \notin (1 \pm \delta_r) a_j \bar{v}_{k, j}^{2I-2}$ for all $k \in [m]$ and 
  $i \ne j \in [P]$ with probability at least $1 - \delta_{\P}$. To complete the proof, recall that 
  by the definition of the greedy maximum selection process, we have 
  $a_{\pi(p)} \bar{v}_{p, \pi(p)}^2 \ge a_{\pi(q)} \bar{v}_{p, \pi(q)}^2$.
\end{proof}

\begin{proof}[Proof of Lemma~\ref{lemma: initialization} (column gap)]
  Let $\z_1, \dots, \z_m$ be independent $\Gaussian{0}{\Id_d}$ variables. 
  Fix $k \ne l \in [m]$ and $p \in [P]$. Note that $(\bar{v}_{k, \pi(p)}, \bar{v}_{l, \pi(p)})
  \overset{d}{=} ( z_{k, p} / \norm{\z_k}, z_{l, p} / \norm{\z_l} )$. Hence, we can write 
  \begin{align*}
    \P\left[ \bar{v}_{k, \pi(p)}^{2I-2} \in (1 \pm \delta_c) \bar{v}_{l, \pi(p)}^{2I-2} \right]
    &= \P\left[ 
        \left( \frac{z_{k, p}}{ z_{k, l} } \right)^{2I-2} 
        \in (1 \pm \delta_c) \left( \frac{ \norm{\z_k} }{\norm{\z_l}} \right)^{2I-2}
      \right] \\
    &\le \P\left[ 
        \left( \frac{z_{k, p}}{ z_{k, l} } \right)^{2I-2} 
        \in 1 \pm 3 \delta_c 
      \right] 
      + \P\left[
        \left( \frac{ \norm{\z_k} }{\norm{\z_l}} \right)^{2I-2}
        \notin 1 \pm \delta_c
      \right].
  \end{align*}
  By our previous calculation, we know the first term is bounded by $6 \delta_c / \pi$. Meanwhile, 
  by the standard concentration results for $\Gaussian{0}{\Id_d}$, we have 
  \[
    \P\left[ \left| \frac{\norm{\z_k}}{\E \norm{\z_k}} - 1 \right| \ge t \right]
    \le 2 \exp\left( - (\E \norm{\z_k})^2 t^2/2 \right)
    \le 2 \exp\left( - d t^2/4 \right), \quad 
    \forall t \ge 0.
  \]
  In other words, with probability at least $1 - 4 \exp\left( - d t^2/4 \right)$, we have 
  \[
    \norm{\z_k}^{2I-2} = (1 \pm t)^{2I-2} = 1 \pm 4(I-1) t, \quad  
    \norm{\z_l}^{2I-2} = 1 \pm 4(I-1) t, 
  \]
  and therefore $\left( \norm{\z_k} / \norm{\z_l}  \right)^{2I-2} = 1 \pm 10 (I - 1) t$. Choose 
  $t = \delta_c / (10 (I - 1))$, and we obtain 
  \[
    \P\left[
      \left( \frac{ \norm{\z_k} }{\norm{\z_l}} \right)^{2I-2}
      \notin 1 \pm \delta_c
    \right]
    \le 4 \exp\left( - \frac{d}{4} \frac{\delta_c^2}{100 (I - 1)^2} \right).
  \]
  As a result, we have 
  \[
    \P\left[ \bar{v}_{k, \pi(p)}^{2I-2} \in (1 \pm \delta_c) \bar{v}_{l, \pi(p)}^{2I-2} \right]
    \le \frac{6 \delta_c}{\pi} + 4 \exp\left( - \frac{d}{4} \frac{\delta_c^2}{100 (I - 1)^2} \right).
  \]
  Take union bound over $k \ne l \in [m]$ and $p \in [P]$, and we get 
  \[
    \P\left[ \exists k \ne l \in [m], p \in [P], \bar{v}_{k, \pi(p)}^{2I-2} \in (1 \pm \delta_c) \bar{v}_{l, \pi(p)}^{2I-2} \right]
    \le m^2 P
      \left( \frac{6 \delta_c}{\pi} + 4 \exp\left( - \frac{d}{4} \frac{\delta_c^2}{100 (I - 1)^2} \right) \right).
  \]
  For the RHS to be bounded by $\delta_{\P}$, it suffices to require
  \begin{align*}
    m^2 P \frac{12 \delta_c}{\pi}
    \le \delta_{\P} 
    &\quad\Leftarrow\quad 
    \delta_c 
    \le \frac{\delta_{\P} \pi}{12 m^2 P}, \\
    4 \exp\left( - \frac{d}{4} \frac{\delta_c^2}{100 (I - 1)^2} \right) 
    \le \frac{6 \delta_c}{\pi}
    &\quad\Leftarrow\quad 
    d \ge \frac{400 (I - 1)^2}{\delta_c^2} \log\left( \frac{2 \pi}{3 \delta_c} \right).
  \end{align*}
  To complete the proof, recall that by the definition of the greedy maximum selection, we have $\bar{v}_{p, \pi(p)}^2 
  \ge \bar{v}_{k, \pi(p)}^2$ when $k > p$.
\end{proof}

\begin{proof}[Proof of Lemma~\ref{lemma: initialization} (threshold gap)]
  Consider arbitrary $k \ne l \in [m]$ and $p \ne q \in [P]$. We estimate the gap between 
  $a_{\pi(p)} \bar{v}_{k, \pi(p)}^{2I-2}$ and $a_{\pi(q)} \bar{v}_{l, \pi(q)}^{2I-2}$. 
  Let $\z_k, \z_l$ be independent $\Gaussian{0}{\Id_d}$ variables; we have 
  $(\bar{v}_{k, \pi(p)}, \bar{v}_{l, \pi(q)}) \overset{d}{=} ( z_{k, p} / \norm{\z_k}, z_{l, q} / \norm{\z_l} )$. 
  As in the proof of column gap, we can write 
  \begin{multline*}
    \P\left[ a_{\pi(p)} \bar{v}_{k, \pi(p)}^{2I-2} \in (1 \pm \delta_t) a_{\pi(q)} \bar{v}_{l, \pi(p)}^{2I-2} \right]
    = \P\left[ 
        \frac{a_{\pi(p)}}{a_{\pi(q)}} \left( \frac{z_{k, p}}{ z_{k, l} } \right)^{2I-2} 
        \in (1 \pm \delta_t) \left( \frac{ \norm{\z_k} }{\norm{\z_l}} \right)^{2I-2}
      \right] \\
    \le \P\left[ 
      \frac{a_{\pi(p)}}{a_{\pi(q)}} \left( \frac{z_{k, p}}{ z_{k, l} } \right)^{2I-2} 
        \in 1 \pm 3 \delta_t 
      \right] 
      + \P\left[
        \left( \frac{ \norm{\z_k} }{\norm{\z_l}} \right)^{2I-2}
        \notin 1 \pm \delta_t 
      \right].
  \end{multline*}
  By the proof of the row gap and the column gap, the last two terms are bounded by 
  $\frac{6 \delta_t}{\pi}$ and $4 \exp\left( - \frac{d}{4} \frac{\delta_t^2}{100 (I - 1)^2} \right)$, respectively.
  Note that this is the same as the bounds in the column gap proof (up to changing $\delta_c$ to $\delta_t$). 
  Thus, we have 
  \[
    \P\left[ \exists k \ne l \in [m], p \in [P], \bar{v}_{k, \pi(p)}^{2I-2} \in (1 \pm \delta_c) \bar{v}_{l, \pi(p)}^{2I-2} \right]
    \le \delta_{\P}, 
  \]
  provided that 
  \[
    \delta_t \le \frac{\delta_{\P} \pi}{12 m^2 P}, \quad 
    d \ge \frac{400 (I - 1)^2}{\delta_t^2} \log\left( \frac{2 \pi}{3 \delta_t} \right).
  \]
  To complete the proof, note that by the definition of the greedy maximum selection process, we have 
  $a_{\pi(P_*)} \bar{v}_{P_*, \pi(P_*)}^{2I-2} \ge a_{\pi(q)} \bar{v}_{k, \pi(q)}^{2I-2}$ for all $P_* < k \le m$
  and $P_* < q \le P$. 
\end{proof}

\begin{proof}[Proof of Lemma~\ref{lemma: initialization} (regularity conditions)]
  First, we consider the upper bound. 
  Let $\z_1, \dots, \z_m$ be independent $\Gaussian{0}{\Id_d}$ random vectors. We have 
  $( \bar{\v}_k )_k \overset{d}{=} ( \z_k / \norm{\z_k} )_k$. By the standard Gaussian concentration results,
  we have $\P( \max_{k\in[m]} \norm{\z_k}_\infty \ge z ) \le 2 m d e^{-z^2/2}$ and 
  $\P( \max_{k \in [m]} \left| \norm{\z_k} / \E \norm{\z_1} - 1 \right| \ge \eps ) 
  \le 2 m e^{-\eps^2 d / 3}$. Therefore, we have $\max_k \norm{\bar{\v}_k}_\infty^2 \le \log^2 d / d$ 
  with probability at least $1 - O(\delta_{\P})$.

  Now, we consider the lower bound. Let $K$ be a parameter to be determined later. Our goal is to show that 
  with high probability, $a_{\pi(p)} \bar{v}_{p, \pi(p)}^2$ is at least the $K$-th largest entry of the $\pi(p)$-th 
  column of the greedy maximum selection matrix. In other words, at most the first $K-1$ largest entries can be 
  covered by the earlier neurons. 

  For any $k \ne l \in [m]$, the events that the $k$-th and $l$-th neurons are used by some earlier are independent. 
  In addition, by symmetry, the probability that the $k$-th row is used by some other neuron is at most 
  $P_* / (m - P_*)$, as we always have at least $m - P_*$ neurons remained. Meanwhile, since the coordinates of 
  $\bar{\v}_k$ are negatively correlated, conditioned on that $\bar{v}_{k, \pi(p)}^2$ is among the $K$ largest
  entries of that column, the probability that that row gets used is still upper bounded by $P_* / (m - P_*)$.
  Thus, 
  \[
    \P\left[ \text{all first $K$ largest entries of the $\pi(p)$-th column are used} \right]
    \le \left( \frac{P_*}{m - P_*} \right)^K.
  \]
  By union bound, the probability that one of $\{ \bar{v}_{p, \pi(p)}^2 \}_{p \in [P_*]}$ is not at least 
  the $K$-th largest in that column is upper bounded by $P_* \left( \frac{P_*}{m - P_*} \right)^K$. 
  For this to be upper bounded by $\delta_{\P}$, it suffices to have 
  \[
    P_* \left( \frac{P_*}{m - P_*} \right)^K
    \le \delta_{\P}
    \quad\Leftarrow\quad 
    K 
    \ge \frac{ \log\left( P_* / \delta_{\P} \right) }{ \log\left( (m - P_*) / P_* \right) }
    \quad\Leftarrow\quad 
    \left\{
      \begin{aligned}
        & K = \log\left( P_* / \delta_{\P} \right), \\
        & m \ge 4  P_* \log(P_* / \delta_{\P}). 
      \end{aligned}
    \right.
  \]
  Finally, by Lemma~\ref{lemma: lower bound on the Kth largest Gaussian}, provided that\footnote{Note that the second
  condition is stronger, so it suffices to keep the second one.} 
  \[
    \frac{m }{\log m}  \ge 128 \pi \log^2( P_* / \delta_{\P} ) 
    \quad\text{and}\quad 
    \frac{m }{\log^3 m} \ge 512 \log^2(P_*/\delta_{\P}), 
  \]
  we have with probability at least $1 - \delta_{\P}$ that 
  \[
    \bar{v}_{p, \pi(p)}^2(0)
    \ge \frac{1}{d} \log\left( \frac{m}{\log(P_*/\delta_{\P})} \right)
    \ge \frac{\log P_*}{d}, 
    \quad \forall p \in [P_*].
  \]
  We conclude by establishing the last regularity condition. For fixed $j, q$, the PDF of $Z := \bar v_{j, q}$ is $p_Z(z) = \frac{\Gamma(\frac{d}{2})}{\sqrt{\pi}\Gamma(\frac{d-1}{2})}(1 - z^2)^{\frac{d-3}{2}}$, and therefore

  \begin{align*}
      \P(\bar v_{j, q}^2 \le \frac{1}{d}) \le \frac{2}{\sqrt{d}}\cdot \frac{\Gamma(\frac{d}{2})}{\sqrt{\pi}\Gamma(\frac{d-1}{2})}\le \frac{2}{\sqrt{d}}\cdot \frac{\sqrt{d/2}}{\sqrt{\pi}} \le \sqrt{\frac{2}{\pi}} \le 0.8,
  \end{align*}
  where the first inequality upper bounds the PDF by $p_Z(0)$, and the second is Gautschi's inequality. Therefore
  \begin{align*}
      \P(\max_{j > P_*} \bar v_{j, q}^2 \le 1/d) \le \P(\sum_{j \in [m]}\indi(\bar v_{j, q}^2 \ge 1/d) \le P_*).
  \end{align*}
  Note that $\sum_{j \in [m]}\indi(\bar v_{j, q}^2 \ge 1/d)$ is subGaussian with variance proxy $\le m$. Therefore for $m \ge 10P_*$
  \begin{align*}
      \P(\sum_{j \in [m]}\indi(\bar v_{j, q}^2 \ge 1/d) \le P_*) \le \exp(-(P_* - 0.2m)^2/m) = \exp(-m/100).
  \end{align*}
  Union bounding over all $q \in [P]$, we get
  \begin{align*}
      \P(\min_{q \in [P]}\max_{j > P_*} \bar v_{j, q}^2 \le 1/d) \le P\exp(-m/100) \le \delta_{\P}
  \end{align*}
  for $m \ge 100\log(P/\delta_{\P})$.
\end{proof}

\begin{lemma}
  \label{lemma: lower bound on the Kth largest Gaussian}
  Let $Z_1, \dots, Z_m$ be independent $\Gaussian{0}{1}$ variables. Suppose that 
  \[
    \frac{m }{\log m}  \ge 128 \pi \log^2( 1 / \delta_{\P} ) 
    \quad\text{and}\quad 
    \frac{m }{\log^3 m} \ge 512 \pi K^2. 
  \]
  Then, with probability at least $1 - \delta_{\P}$, the $K$-th largest among $Z_1, \dots, Z_m$ is at least 
  $\sqrt{\log(m/K)}$.
\end{lemma}
\begin{proof}
  Let $\Phi$ denote the CDF of $\Gaussian{0}{1}$. Then, the CDF $F_K$ of the $K$-th largest element among
  $Z_1, \dots, Z_m$ is 
  \[
    F_K(z)
    = \sum_{k=1}^{K-1} \binom{m}{k} ( 1 - \Phi(z) )^k \Phi^{m-k}(z) 
  \]
  It is well-known that the mill's ratio of $\Gaussian{0}{1}$ satisfies 
  \[
    \frac{1}{\sqrt{2\pi}} \frac{z}{1 + z^2} e^{-z^2/2} 
    \le 1 - \Phi(z) 
    \le \frac{1}{\sqrt{2\pi}} \frac{1}{z} e^{-z^2/2}. 
  \]
  Meanwhile, we have $\binom{m}{k} \le m^k e^k / k^k$. As a result,
  \begin{align*}
    F_K(z)
    &\le \sum_{k=1}^{K-1} \left( \frac{m e}{k} \right)^k 
      \left( \frac{1}{\sqrt{2\pi}} \frac{1}{z} e^{-z^2/2} \right)^k 
      \left( 1 - \frac{1}{\sqrt{2\pi}} \frac{z}{1 + z^2} e^{-z^2/2}  \right)^{m-k} \\
    &\le \sum_{k=1}^{K-1} \left( \frac{m e}{k} \frac{1}{\sqrt{2\pi}} \frac{1}{z} \right)^k 
      \exp\left( - \frac{k z^2}{2} \right)\exp\left( - \frac{m - k}{\sqrt{2\pi}} \frac{z}{1 + z^2} e^{-z^2/2}  \right).
  \end{align*}
  Choose $z = \sqrt{ (1 - \eps) 2 \log (m / K) }$ for some $\eps \in (0, 1)$. Then, we have 
  $e^{-z^2/2} = (K/m)^{1-\eps}$ and 
  \[
    F_K(z)
    \le \sum_{k=1}^{K-1} \left( 
        \frac{m e}{k} \frac{1}{\sqrt{2\pi}} \frac{1}{z} 
        \left( \frac{K}{m} \right)^{1-\eps}
      \right)^k 
      \exp\left( - \frac{m - k}{\sqrt{2\pi}} \frac{z}{1 + z^2} \left( \frac{K}{m} \right)^{1-\eps}  \right).
  \]
  Choose $\eps = 1/2$ and suppose that $K \le m / 2$. Then, we have 
  \begin{align*}
    F_K(z)
    \le \sum_{k=1}^{K-1} \left(  m^{1/2} K^{1/2} \right)^k 
      \exp\left( - \frac{1}{4 \sqrt{2\pi}} \frac{m^{1/2}}{z} \right) 
    &\le \sum_{k=1}^{K-1} 
      \exp\left(
        \frac{k}{2} \log(  m K ) 
        - \frac{1}{4 \sqrt{2\pi}} \frac{m^{1/2}}{z} 
      \right) \\
    &\le \exp\left( 2 K \log m - \frac{1}{4 \sqrt{2\pi}} \frac{m^{1/2}}{\sqrt{ \log m }} \right). 
  \end{align*} 
  To merge the first term into the second term, it suffices to require
  \[
    2 K \log m \le \frac{1}{8 \sqrt{2\pi}} \frac{m^{1/2}}{\sqrt{ \log m }}
    \quad\Leftarrow\quad
    \frac{m }{\log^3 m} \ge 512 \pi K^2.
  \]
  Finally, we compute 
  \begin{align*}
    \exp\left( - \frac{1}{8 \sqrt{2\pi}} \frac{m^{1/2}}{\sqrt{ \log m }} \right)
    \le \delta_{\P}
    \quad\Leftarrow\quad 
    \frac{m }{\log m} 
    \ge 128 \pi \log^2( 1 / \delta_{\P} )
  \end{align*}

\end{proof}

\bigskip
\section{Gradient Flow Analysis}\label{sec:gradient-flow}

In this section, we analyze the gradient flow dynamics and show that gradient flow implements the greedy maximum 
selection scheme.  We will assume the following on the initialization. 

\begin{assumption}[Initialization]
  \label{assumption: gf init}
  Suppose $P_* \le \min\{ P, m \}$. 
  We assume that the following hold at initialization. 
  \begin{enumerate}[(a)]
    \item \label{assumption-itm: init: row gap}
      (Row gap) For any $p \in [P_*]$ and $p < q \in [P]$, we have $a_{\pi(p)} \bar{v}_{p, \pi(p)}^{2I-2} 
      \ge (1 + \delta_r) a_{\pi(q)} \bar{v}_{p, \pi(q)}^{2I-2}$.
    \item \label{assumption-itm: init: col gap}
      (Column gap) For any $p \in [P_*]$ and $p < k \in [m]$, we have $\bar{v}_{p, \pi(p)}^{2I-2} 
      \ge (1 + \delta_c) \bar{v}_{k, \pi(p)}^{2I-2}$.
    \item \label{assumption-itm: init: threshold gap}
      (Threshold gap) For any $P_* < k \in [m]$ and $P_* < q \in [P]$, we have $a_{\pi(P_*)} \bar{v}_{P_*, \pi(P_*)}^{2I-2}
      \ge (1 + \delta_t) a_{\pi(q)} \bar{v}_{k, \pi(q)}^{2I-2}$.
    \item \label{assumption-itm: init: regularity conditions}
      (Regularity conditions) $\max_{k \in [m]} \norm{\bar{\v}_k}_\infty^2 \le \log^2 d / d$ and 
      $\min_{p \in [P_*]} \bar{v}_{p, \pi(p)}^2 \ge 1/d$. 
  \end{enumerate}
\end{assumption}
\begin{remark*}
  By Lemma~\ref{lemma: initialization}, this assumption hold with high probability with 
  $\delta_c, \delta_r, \delta_t = 1 / \poly(P)$.
\end{remark*}

Now, we formally state the main theorem for gradient flow. The proof is deferred to the end of this section 
(cf.~Section~\ref{subsec: gf: deferred proofs}). In the statement, we hide the constants that depend only on 
$\sigma$. 
\begin{theorem}[Main theorem for gradient flow]
  \label{thm: gf}
  Assume Assumption~\ref{assumption: gf init} holds at initialization.
  Let $\eps_D, \eps_R$ be our target accuracies and $\delta_T$ be the target error in time. Put $\delta_{r, t} := \delta_r \wedge \delta_t$. Suppose that\footnote{
    Note that the lower bounds are $1 / \poly(d)$, and we know from Lemma~\ref{lemma: initialization} that 
    $\delta_c,\delta_r,\delta_r$ are $1/\poly(P)$. Hence, the range from which $\eps_D, \eps_R, \delta_T$ can 
    be chosen is not restrictive. 
  }
  \begin{gather*}
    \eps_D 
    \gtrsim_\sigma \frac{\norm{\a}_1}{a_{\min_*}} \frac{1}{d^{I - 1/4}}, \quad 
    \frac{1}{d^{I-1/4}}
    \lesssim_\sigma \eps_R 
    \lesssim_\sigma \frac{ a_{\min_*}^2 \delta_c}{(\log^2 d)^{I-1}}, \quad 
    \frac{\norm{\a}_1}{a_{\min_*}} \frac{1}{d^{1/4} } 
    \lesssim_\sigma \delta_T 
    \lesssim_\sigma \delta_c \wedge \delta_r \wedge \delta_t, \\
    \frac{d}{(\log^2 d)^{4I}} 
    \gtrsim_\sigma 
      \delta_{r, t}^{-8}
      \vee 
      \left( 
        \frac{ a_{\min_*} }{ \norm{\a}_1 }
        \delta_{r, t} 
      \right)^{-4}
      \vee 
      \left( 
      \frac{ a_{\min_*}^2 \delta_c}{\norm{\a}_1}
    \right)^{-4}.
  \end{gather*}
  Choose the initialization scale to be 
  \[
    \sigma_0^2 
    \approx_\sigma
      \frac{\bar\eps^{ 8  / (I \hat\sigma_{2I}^2) }}{m}
      \left(
        a_{\min_*} \eps_D 
        \wedge  \frac{a_{\min_*}  \delta_T }{d^{I-1/2}}
        \wedge \eps_R 
        \wedge \frac{ a_{\min_*} \delta_{r, t} }{ (\log d)^{2I-2} d^{I-1/2} } 
        \wedge \frac{ a_{\min_*}^2 \delta_c}{(\log^2 d)^{I-1}}
        \frac{1}{d^{I-1/2}}
      \right),
  \]
  where 
  \(
    \bar\eps 
    =_\sigma
      \eps_D^2 d^{2(I-1)}
      \wedge 
        \frac{\delta_T^2 \delta_{r, t}^2 }{d (\log d)^{4(I-1)}}  
      \wedge
      \frac{\eps_R}{a_{\min_*}}
      \wedge 
        \frac{\delta_{r, t}^4}{ d (\log d)^{4(I-1)} }
      \wedge 
      \frac{ a_{\min_*}^2 \delta_c^2}{(\log^2 d)^{2I-2}}
      \frac{\delta_{r, t}^2}{d (\log d)^{4(I-1)}}.
  \)
  For each $p \in [P_*]$, define 
  \[
    T_p
    := \frac{1}{ 4I (I - 1) \hat\sigma_{2I}^2 a_{\pi(p)} \bar{v}_{p, \pi(p)}^{2I-2}(0) }
    = \Theta\left(
          \frac{1}{ a_{\pi(p)} \bar{v}_{p, \pi(p)}^{2I-2}(0) }
        \right)
      = \tilde\Theta\left(
        \frac{1}{ a_{\pi(p)} d^{I-1} }
      \right).
  \]
  Then, we have the following over time interval $[0, (1 + 20 \delta_T)T_{P_*}]$:
  \begin{enumerate}[(a)]
    \item \textbf{(Unused neurons)} $\norm{\v_k}^2 \le \sigma_1^2$ for all $k > P_*$.
    \item \textbf{(Learning)} For any $p\in[P_*]$, $\bar{v}_{p, \pi(p)}^2 \ge 1 - \eps_D$ and $\norm{\v_p}^2 = a_{\pi(p)} \pm \eps_R$ 
      for all $t \ge (1 + 20 \delta_T) T_p$.
    \item \textbf{(Sharp transition)} For any $p\in[P_*]$, $\bar{v}_{p, \pi(p)}^2 \le \left( \frac{4 }{ \delta_T} \right)^{\frac{1}{I-1}} \frac{\log^2 d}{d}$ and 
      $\norm{\v_p}^2 \le \sigma_1^2$ for all $t \le (1 - 10 \delta_T) T_p$.
  \end{enumerate}
  In words, for each $p \in [P_*]$, $\bar{\v}_p$ converges to $\e_{\pi(p)}$ and fit $a_{\pi(p)}$ at time 
  $(1 \pm o(1)) T_p$, and all other neurons stay small throughout training.  
\end{theorem}

Our proof will be a large (continuous) induction argument. Namely, we assume a collection of induction hypotheses, 
analyze the dynamics under these conditions, derive the convergence guarantees, and show that these induction hypotheses
hold throughout training. One may refer to, for example, Section~A.1 of \cite{ge_understanding_2021} or Chapter~1.3
of \cite{tao2006nonlinear} for details on this method. 

We will maintain the following induction hypothesis. 
\begin{inductionH}
  \label{inductionH: gf}
  Let $\sigma_1 > \sigma_0$, $\bar\eps \le \eps_0, \gamma$ be $o(1)$ parameters. We say this induction
  hypothesis holds at a time point if the following hold at that time point. 
  \begin{enumerate}[(a)]
    \item Define $L := \braces{ k \in [m] \,:\, \norm{\v_k} \ge \sigma_1 }$. 
      For any $p \in [m]$, $\v_p \in L$ implies $p \le P_*$ and $\bar{v}_{p, \pi(p)}^2 \ge 1 - \bar\eps$. 
      \label{inductH-itm: large norm => converged}
    \item For any $(k, \pi(q))$ that is not in $\{ (p, \pi(p)) \,:\, p \in [P_*] \}$, we have 
      $\bar{v}_{k, \pi(q)}^2 \le \eps_0 := d^{-(1-\gamma)}$. 
      \label{inductH-itm: bound on the failed coordinates}
    \item We have $\norm{\v_p}^2 \le 2 a_l$ for any $p \in [P \wedge m]$
        and $\bar{v}_{p, \pi(p)}^2 \ge 1/d$ for any $p \in [P_*]$.
      \label{inductH-itm: regularity conditions}
  \end{enumerate}
\end{inductionH}
\begin{remark*}
  Condition~\ref{inductH-itm: large norm => converged} states that the norm of a neuron is large (when compared to 
  $\sigma_0$) only if it is close to one ground-truth direction. 
  Condition~\ref{inductH-itm: bound on the failed coordinates} means that all irrelevant coordinates stay small 
  throughout training. Condition~\ref{inductH-itm: regularity conditions} includes some basic regularity conditions. 
\end{remark*}

\newcommand{\IHGF}{Induction~Hypothesis~\ref{inductionH: gf}}

Before proceeding to the proofs, we state the following lemma that controls the interaction between different learner
neurons. The proof is deferred to Section~\ref{subsec: gf: deferred proofs}.

\begin{lemma}
  \label{lemma: gf: tangent dynamics (crude)}
  Suppose that \IHGF{} is true at time $t$. Then, at time $t$, for any $k \in [m]$ and $q \in [P]$, we have 
  \begin{align*}
    \frac{\rd}{\rd t} \bar{v}_{k, \pi(q)}^2
    &= 2 \bar{v}_{k, \pi(q)}^2
      \sum_{i=I}^\infty 2i \hat\sigma_{2i}^2 \left( 
        a_{\pi(q)} \bar{v}_{k, \pi(q)}^{2i-2} 
        - \sum_{r=1}^P a_{\pi(r)} \bar{v}_{k, \pi(r)}^{2i} 
      \right)  \\
      &\quad
      - \indi\braces{k \ne q, q \in L}
      2 \norm{\v_q}^2 \left( 1 - \bar{v}_{k, \pi(q)}^2 \right)
      \sum_{i=I}^\infty 2 i \hat\sigma_{2i}^2 \bar{v}_{k, \pi(q)}^{2i} \\
      &\quad
      \pm I 2^{3I+6} C_\sigma^2 \abs{\bar{v}_{k, \pi(q)} } 
      \braces{
        a_{\pi(q)} \bar\eps^{1/2} \eps_0^{I-1} 
        \vee m \sigma_1^2
        \vee \norm{\a}_1 \eps_0^I 
      }.
  \end{align*}
  In addition, for any target $\delta > 0$, we have 
  \begin{equation}
    \label{eq: tangent: error <= delta}
    a_{\pi(q)} \bar\eps^{1/2} \eps_0^{I-1} 
    \vee m \sigma_1^2
    \vee \norm{\a}_1 \eps_0^I 
    \le \delta 
    \quad\Leftarrow\quad
    \left\{
      \begin{aligned}
        & \bar\eps 
          \le \left( \frac{\delta }{a_{\pi(q)}} \right)^2 d^{2(1-\gamma)(I-1)}, \\
        & m \sigma_1^2 \le \delta, \\
        & d \ge \left( \frac{\delta}{\norm{\a}_1} \right)^{-\frac{1}{(1 - \gamma) I}}.
      \end{aligned}
    \right.
  \end{equation}
\end{lemma}

The rest of this section is organized as follows. In Section~\ref{sec: gf: convergence}, we 
assume~\IHGF{} and show that $\v_p$ ($p \in [P_*]$) converges to $\e_{\pi(p)}$ and fits $a_{\pi(p)}$ at time 
$(1 \pm o(1))T_p$. Then, in Section~\ref{sec: gf: induction hypotheses}, we verify \IHGF{}. Finally, 
in Section~\ref{subsec: gf: deferred proofs}, we prove Lemma~\ref{lemma: gf: tangent dynamics (crude)} 
and Theorem~\ref{thm: gf}.

\subsection{Convergence Guarantees}

\label{sec: gf: convergence}

In this subsection, we show under \IHGF{} that $\v_p$ ($p \in [P_*]$) converges to $\e_{\pi(p)}$ and fits $a_{\pi(p)}$ at 
time $(1 \pm o(1))T_p$. We will first consider the dynamics of $\bar{\v}_p$ and then $\norm{\v_p}^2$. 
Our main result is the following, whose proof can be found at the end of this subsection. 

\begin{restatable*}[Convergence]{ccorollary}{GfCorConvergence}
  \label{cor: convergence of one direction}
  Let $\eps_D, \eps_R$ be our target accuracy in the tangent and radial directions, and $\delta_T$ the target
  error in time. Suppose that $\gamma < 1/(2I)$, $\delta_v' = 1/3$,
  \begin{gather*}
    \eps_D 
    \ge \frac{ 2^{3I+7} C_\sigma^2 }{(\delta_v')^I \hat\sigma_{2I}^2 } 
      \frac{\norm{\a}_1}{a_{\min_*}} 
      \frac{1}{d^{(1-\gamma) I} }, \quad 
    \eps_R 
    \ge  12 \norm{\a}_1 2^{2I} d^{-(1-\gamma) I}, \quad 
    \delta_T 
    \ge \frac{ 2^{3I+4} C_\sigma^2  }{ \hat\sigma_{2I}^2 } \frac{\norm{\a}_1}{a_{\min_*}} \frac{1}{d^{1/2-\gamma I} }, \\
    m \sigma_1^2 
    \le 
      \frac{ \hat\sigma_{2I}^2 a_{\min_*}}{ 2^{3I+7} C_\sigma^2 }  
      \left(
        (\delta_v')^I  \eps 
        \wedge 
        \frac{\delta_T }{d^{I-1/2}}
      \right)
      \wedge \frac{\eps_R}{12}, \\
    \bar\eps 
    \le \left( \frac{ (\delta_v')^I \hat\sigma_{2I}^2 }{ 2^{3I+7} C_\sigma^2 } \right)^2 
        \eps_D^2 d^{2 (1-\gamma)(I-1)}
        \wedge 
        \left( \delta_T \frac{ \hat\sigma_{2I}^2  }{ 2^{3I+4} C_\sigma^2 } \right)^2
            \frac{1}{d^{1+ 2 \gamma (I - 1)} }
        \wedge
        \frac{\eps_R}{12 C_\sigma^2  a_{\pi(p)}} . 
  \end{gather*}
  Then, for any $p \in [P_*]$, we have 
  \begin{align*}
    &\bar{v}_{p, \pi(p)}^2 \ge 1 - \eps_D, \quad \norm{\v_p}^2 = a_{\pi(p)} \pm \eps_R, 
      && \quad \forall t \ge (1 + 20 \delta_T) T_p, \\
    &\bar{v}_{p, \pi(p)}^2 \le \left( \frac{4 }{ \delta_T} \right)^{\frac{1}{I-1}} \frac{\log^2 d}{d}, \quad 
      \norm{\v_p}^2 \le \sigma_1^2, 
      && \quad \forall t \le (1 - 10\delta_T) T_p,
  \end{align*}
  where 
  \[
    T_p
    := \frac{1}{ 4I (I - 1) \hat\sigma_{2I}^2 a_{\pi(p)} \bar{v}_{p, \pi(p)}^{2I-2}(0) }
    = \Theta\left(
          \frac{1}{ a_{\pi(p)} \bar{v}_{p, \pi(p)}^{2I-2}(0) }
        \right)
      = \tilde\Theta\left(
        \frac{1}{ a_{\pi(p)} d^{I-1} }
      \right).
  \]
\end{restatable*}

\subsubsection{Tangent Dynamics}

Here, we analyze the diagonal entries $\{ \bar{v}_{p, \pi(p)}^2 \}_{p \in [P_*]}$. 
Let $p \in [P_*]$ be fixed. For $\delta \in (0, 1)$, let $T_\delta$ denote the time $\bar{v}_{p,\pi(p)}^2$ reaches $\delta$. 
We split the training process into $[0, T_{\delta_v}]$, $[T_{\delta_v}, T_{\delta_v'}]$ and $[T_{\delta_v'}, T_{1-\eps}]$,
where $\delta_v = o(1)$ and $\delta_v' = O(1)$ are two parameters to be chosen later. 
Our goal is to show that $\bar{v}_{p, \pi(p)}^2$ will converge to close to $1$ around time $(1 \pm O(\delta_T)) T_p$,
where $T_p$ is the time indicated by the idealized process and $\delta_T$ is a parameter measuring the error.

\begin{lemma}[Dynamics of the diagonal entries (Stage 1)]
  \label{lemma: tangent: dynamics of the diagonal entries (stage 1)}
  Suppose that at time $t \in [0, T_{\delta_v}]$, \IHGF{} is true and the following hold: 
  \begin{gather*}
    \delta_v \le \frac{\delta_T}{2} \frac{ 2I \hat\sigma_{2I}^2  } { C_\sigma^2 }, \quad 
    \gamma < \frac{1}{2 I}, \quad 
    m \sigma_1^2
    \le \delta_T \frac{ \hat\sigma_{2I}^2 a_{\min_*} }{ 2^{3I+4} C_\sigma^2 d^{I-1/2} }, \\
    \bar\eps 
      \le \left( \delta_T \frac{ \hat\sigma_{2I}^2  }{ 2^{3I+4} C_\sigma^2 } \right)^2
        \frac{1}{d^{1+ 2 \gamma (I - 1)} }, \quad 
    d \ge \left( 
        \frac{ \hat\sigma_{2I}^2 } { 2^{3I+4} C_\sigma^2  } \frac{a_{\min_*}}{\norm{\a}_1} \delta_T
      \right)^{-\frac{2}{1 - 2 \gamma I}}.
  \end{gather*}
  Then, at time $t \in [0, T_{\delta_v}]$, for any $p \in [P_*]$, we have 
  \[
    \frac{\rd}{\rd t} \bar{v}_{p, \pi(p)}^2 
    = \left( 1 \pm 3 \delta_T \right)
      \times 4I \hat\sigma_{2I}^2 a_{\pi(p)} \bar{v}_{p, \pi(p)}^{2I}.
  \]
\end{lemma}
\begin{proof}
  First, by Lemma~\ref{lemma: gf: tangent dynamics (crude)}, we have 
  \begin{align*}
    \frac{\rd}{\rd t} \bar{v}_{p, \pi(p)}^2
    &= 2 \bar{v}_{p, \pi(p)}^2
      2I \hat\sigma_{2I}^2 \left( 
        a_{\pi(p)} \bar{v}_{p, \pi(p)}^{2I-2} 
        - \sum_{r=1}^P a_{\pi(r)} \bar{v}_{p, \pi(r)}^{2I} 
      \right)  \\
      &\quad
      + 2 \bar{v}_{p, \pi(p)}^2
      \sum_{i=I+1}^\infty 2i \hat\sigma_{2i}^2 \left( 
        a_{\pi(p)} \bar{v}_{p, \pi(p)}^{2i-2} 
        - \sum_{r=1}^P a_{\pi(r)} \bar{v}_{p, \pi(r)}^{2i} 
      \right)  \\
      &\quad 
      \pm I 2^{3I+6} C_\sigma^2 \abs{\bar{v}_{p, \pi(p)} } 
      \braces{
        a_{\pi(p)} \bar\eps^{1/2} \eps_0^{I-1} 
        \vee m \sigma_1^2
        \vee \norm{\a}_1 \eps_0^I 
      }  \\
    &=: \Term_1\left( \frac{\rd}{\rd t} \bar{v}_{p, \pi(p)}^2 \right)
      + \Term_2\left( \frac{\rd}{\rd t} \bar{v}_{p, \pi(p)}^2 \right)
      + \Term_3\left( \frac{\rd}{\rd t} \bar{v}_{p, \pi(p)}^2 \right).
  \end{align*}
  For the signal term $\Term_1$, by \IHGF\ref{inductH-itm: bound on the failed coordinates}, we have 
  \begin{align*}
    \Term_1 
    &= 4I \hat\sigma_{2I}^2 \left( 
        a_{\pi(p)} \left( 1 - \bar{v}_{p, \pi(p)}^2 \right) \bar{v}_{p, \pi(p)}^{2I-2} 
        - \sum_{r: r \ne p} a_{\pi(r)} \bar{v}_{p, \pi(r)}^{2I} 
      \right)
      \bar{v}_{p, \pi(p)}^2 \\
    &= 4I \hat\sigma_{2I}^2 \left( 
        a_{\pi(p)} \left( 1 \pm \delta_v \right) \bar{v}_{p, \pi(p)}^{2I-2} 
        \pm \eps_0^I \norm{\a}_1 
      \right)
      \bar{v}_{p, \pi(p)}^2 \\
    &= \left( 1 \pm \delta_v \pm \frac{\eps_0^I \norm{\a}_1 }{a_{\pi(p)} \bar{v}_{p, \pi(p)}^{2I-2} } \right)
      \times 4I \hat\sigma_{2I}^2 a_{\pi(p)} \bar{v}_{p, \pi(p)}^{2I}.
  \end{align*}
  We want the error terms in the coefficient to be bounded by $\delta_T$. For this to happen, we first require
  $\delta_v \le \delta_T/2$. Then, recall from \IHGF\ref{inductH-itm: regularity conditions} that 
  $\bar{v}_{p, \pi(p)}^2 \ge 1/d$. Also recall $\eps_0 = d^{-(1-\gamma)}$. Hence, we have 
  \[
    \frac{\eps_0^I \norm{\a}_1 }{a_{\pi(p)} \bar{v}_{p, \pi(p)}^{2I-2} }
    \le \frac{\delta_T}{2}
    \quad\Leftarrow\quad
    d^{I\gamma - 1 } 
    \le \frac{ a_{\min_*} } { \norm{\a}_1 }\frac{\delta_T}{2}
    \quad\Leftarrow\quad
    \gamma < 1/I, \quad 
    d \ge \left( \frac{ a_{\min_*} } { \norm{\a}_1 }\frac{\delta_T}{2} \right)^{\frac{-1}{1-I\gamma}}.
  \]
  When the above conditions hold, we have 
  \[
    \Term_1 = \left( 1 \pm \delta_T \right) \times 4I \hat\sigma_{2I}^2 a_{\pi(p)} \bar{v}_{p, \pi(p)}^{2I}.
  \]
  Then, consider $\Term_2$. We have 
  \begin{align*}
    \abs{\Term_2}
    &\le 2 C_\sigma^2 \bar{v}_{p, \pi(p)}^2
      \left( 
        a_{\pi(p)} \bar{v}_{p, \pi(p)}^{2I} 
        + \norm{\a}_1 \eps_0^I
      \right) \\
    &\le 
      \left( 
        a_{\pi(p)} \bar{v}_{p, \pi(p)}^{2I} 
        + \norm{\a}_1 \eps_0^I
      \right) 
      \frac{ C_\sigma^2 }{ 2I \hat\sigma_{2I}^2 a_{\pi(p)} \bar{v}_{p, \pi(p)}^{2I-2} } 
      \times 4I \hat\sigma_{2I}^2 a_{\pi(p)} \bar{v}_{p, \pi(p)}^{2I}. 
  \end{align*}
  Again, for the coefficient to be bounded by $\delta_T$, it suffices to require
  \begin{align*} 
    \frac{ C_\sigma^2 a_{\pi(p)} \bar{v}_{p, \pi(p)}^{2I}}{ 2I \hat\sigma_{2I}^2 a_{\pi(p)} \bar{v}_{p, \pi(p)}^{2I-2} } 
    \le \frac{\delta_T}{2}
    & \quad\Leftarrow\quad
    \frac{ C_\sigma^2  \bar{v}_{p, \pi(p)}^2}{ 2I \hat\sigma_{2I}^2  } 
    \le \frac{\delta_T}{2}
    \quad\Leftarrow\quad
    \delta_v 
    \le \frac{\delta_T}{2} \frac{ 2I \hat\sigma_{2I}^2  } { C_\sigma^2 }, \\
    \frac{ C_\sigma^2 \norm{\a}_1 \eps_0^I}{ 2I \hat\sigma_{2I}^2 a_{\pi(p)} \bar{v}_{p, \pi(p)}^{2I-2} } 
    \le \frac{\delta_T}{2}
    & \quad\Leftarrow\quad
    \eps_0^I d^{I-1}
    \le \frac{\delta_T}{2} \frac{ 2I \hat\sigma_{2I}^2 } { C_\sigma^2  } \frac{a_{\min_*}}{\norm{\a}_1} \\
    &\quad\Leftarrow\quad
    \gamma < 1 / I,  \quad 
    d \ge \left(
      \frac{\delta_T}{2} \frac{ 2I \hat\sigma_{2I}^2 } { C_\sigma^2  } \frac{a_{\min_*}}{\norm{\a}_1}
    \right)^{\frac{-1}{1 - \gamma I}}.
  \end{align*}
  Finally, consider $\Term_3$. We have 
  \begin{align*}
    \abs{\Term_3}
    &\le I 2^{3I+6} C_\sigma^2 \abs{\bar{v}_{p, \pi(p)} } 
      \braces{
        a_{\pi(p)} \bar\eps^{1/2} \eps_0^{I-1} 
        \vee m \sigma_1^2
        \vee \norm{\a}_1 \eps_0^I 
      } \\
    &= 
      \braces{
        a_{\pi(p)} \bar\eps^{1/2} \eps_0^{I-1} 
        \vee m \sigma_1^2
        \vee \norm{\a}_1 \eps_0^I 
      }
      \frac{ 2^{3I+4} C_\sigma^2 d^{I-1/2} }{ \hat\sigma_{2I}^2 a_{\pi(p)} } 
      \times 4I \hat\sigma_{2I}^2 a_{\pi(p)} \bar{v}_{p, \pi(p)}^{2I}. 
  \end{align*}
  By \eqref{eq: tangent: error <= delta}, for $a_{\pi(q)} \bar\eps^{1/2} \eps_0^{I-1} 
  \vee m \sigma_1^2
  \vee \norm{\a}_1 \eps_0^I 
  \le \frac{ \hat\sigma_{2I}^2 a_{\pi(p)} } { 2^{3I+4} C_\sigma^2 d^{I-1/2} } \delta_T$
  to hold, it suffices to have 
  \begin{gather*}
    m \sigma_1^2 \le \frac{ \hat\sigma_{2I}^2 a_{\min_*} } { 2^{3I+4} C_\sigma^2 d^{I-1/2} } \delta_T, \quad 
    \bar\eps
    \le \left( \frac{ \hat\sigma_{2I}^2 } { 2^{3I+4} C_\sigma^2 } \delta_T \right)^2 
      \frac{1}{d^{1 + 2\gamma(I-1)}}, \quad
    d \ge \left( 
        \frac{1}{\norm{\a}_1} 
        \frac{ \hat\sigma_{2I}^2 a_{\pi(p)} } { 2^{3I+4} C_\sigma^2 d^{I-1/2} } \delta_T
      \right)^{-\frac{1}{(1 - \gamma) I}}.
  \end{gather*}
  Note that the last condition has $d$ on both sides. Rearrange terms and it becomes
  \[
    d^{1 - \frac{I-1/2}{(1 - \gamma) I}} 
    \ge \left( 
      \frac{ \hat\sigma_{2I}^2 } { 2^{3I+4} C_\sigma^2  } \frac{a_{\min_*}}{\norm{\a}_1} \delta_T
    \right)^{-\frac{1}{(1 - \gamma) I}}
    \quad\Leftarrow\quad  
    \gamma < \frac{1}{2I}, \quad 
    d 
    \ge \left( 
      \frac{ \hat\sigma_{2I}^2 } { 2^{3I+4} C_\sigma^2  } \frac{a_{\min_*}}{\norm{\a}_1} \delta_T
    \right)^{-\frac{2}{1 - 2 \gamma I}}.
  \]

  Combining the above bounds, we get 
  \[
    \frac{\rd}{\rd t} \bar{v}_{p, \pi(p)}^2
    = \left( 1 \pm 3 \delta_T \right) \times 4I \hat\sigma_{2I}^2 a_{\pi(p)} \bar{v}_{p, \pi(p)}^{2I},
  \]
  as long as the following conditions are true:
  \begin{align*}
    \Term_1:\quad &
      \delta_v \le \frac{\delta_T}{2}, \quad
      \gamma < 1/I, \quad 
      d \ge \left( \frac{ a_{\min_*} } { \norm{\a}_1 }\frac{\delta_T}{2} \right)^{\frac{-1}{1-I\gamma}}, \\
    \Term_2:\quad &
      \delta_v 
      \le \frac{\delta_T}{2} \frac{ 2I \hat\sigma_{2I}^2  } { C_\sigma^2 }, \quad 
      \gamma < 1 / I,  \quad 
      d \ge \left(
        \frac{\delta_T}{2} \frac{ 2I \hat\sigma_{2I}^2 } { C_\sigma^2  } \frac{a_{\min_*}}{\norm{\a}_1}
      \right)^{\frac{-1}{1 - \gamma I}}, \\
    \Term_3:\quad &
      m \sigma_1^2
      \le \delta_T \frac{ \hat\sigma_{2I}^2 a_{\min_*} }{ 2^{3I+4} C_\sigma^2 d^{I-1/2} }, \quad 
      \bar\eps 
      \le \left( \delta_T \frac{ \hat\sigma_{2I}^2  }{ 2^{3I+4} C_\sigma^2 } \right)^2
        \frac{1}{d^{1+ 2 \gamma (I - 1)} }, \\ 
      & \gamma < \frac{1}{2I}, \quad 
      d \ge \left( 
        \frac{ \hat\sigma_{2I}^2 } { 2^{3I+4} C_\sigma^2  } \frac{a_{\min_*}}{\norm{\a}_1} \delta_T
      \right)^{-\frac{2}{1 - 2 \gamma I}}.
  \end{align*}
  Clear that the second set of conditions is stronger than the first set. In addition, 
  since $\frac{1}{1 - \gamma I} \le \frac{2}{1 - 2 \gamma I}$, the last condition on $d$ is stronger than the 
  first one. Hence, we can prune the above as 
  \begin{gather*}
    \delta_v \le \frac{\delta_T}{2} \frac{ 2I \hat\sigma_{2I}^2  } { C_\sigma^2 }, \quad 
    \gamma < \frac{1}{2 I}, \quad 
    m \sigma_1^2
    \le \delta_T \frac{ \hat\sigma_{2I}^2 a_{\min_*} }{ 2^{3I+4} C_\sigma^2 d^{I-1/2} }, \\
    \bar\eps 
      \le \left( \delta_T \frac{ \hat\sigma_{2I}^2  }{ 2^{3I+4} C_\sigma^2 } \right)^2
        \frac{1}{d^{1+ 2 \gamma (I - 1)} }, \quad 
    d \ge \left( 
        \frac{ \hat\sigma_{2I}^2 } { 2^{3I+4} C_\sigma^2  } \frac{a_{\min_*}}{\norm{\a}_1} \delta_T
      \right)^{-\frac{2}{1 - 2 \gamma I}}.
  \end{gather*}
\end{proof}

We will see that the time needed for Stage~1 is much larger than all other stages combined, which allows the estimations to be looser in later stages.

\begin{lemma}[Dynamics of the diagonal entries (Stage 2)]
  \label{lemma: tangent: dynamics of the diagonal entries (stage 2)}
  Suppose that at time $t \in [T_{\delta_v}, T_{\delta_v'}]$, \IHGF{} is true. 
  In addition, suppose that the conditions of Lemma~\ref{lemma: tangent: dynamics of the diagonal entries (stage 1)} 
  holds and $\delta_{v}' \le 1/3$.
  Then, at time $t \in [T_{\delta_v}, T_{\delta_v'}]$, for any $p \in [P_*]$, we have 
  \[
    \frac{\rd}{\rd t} \bar{v}_{p, \pi(p)}^2 
    \ge \frac{1}{2} \times 4I \hat\sigma_{2I}^2 a_{\pi(p)} \bar{v}_{p, \pi(p)}^{2I}.
  \]
\end{lemma}
\begin{proof}
  Similar to the previous proof, by Lemma~\ref{lemma: gf: tangent dynamics (crude)}, we have 
  \begin{align*}
    \frac{\rd}{\rd t} \bar{v}_{p, \pi(p)}^2
    &= 2 \bar{v}_{p, \pi(p)}^2
      2I \hat\sigma_{2I}^2 \left( 
        a_{\pi(p)} \bar{v}_{p, \pi(p)}^{2I-2} 
        - \sum_{r=1}^P a_{\pi(r)} \bar{v}_{p, \pi(r)}^{2I} 
      \right)  \\
      &\quad
      + 2 \bar{v}_{p, \pi(p)}^2
      \sum_{i=I+1}^\infty 2i \hat\sigma_{2i}^2 \left( 
        a_{\pi(p)} \bar{v}_{p, \pi(p)}^{2i-2} 
        - \sum_{r=1}^P a_{\pi(r)} \bar{v}_{p, \pi(r)}^{2i} 
      \right)  \\
      &\quad 
      \pm I 2^{3I+6} C_\sigma^2 \abs{\bar{v}_{p, \pi(p)} } 
      \braces{
        a_{\pi(p)} \bar\eps^{1/2} \eps_0^{I-1} 
        \vee m \sigma_1^2
        \vee \norm{\a}_1 \eps_0^I 
      }  \\
    &=: \Term_1\left( \frac{\rd}{\rd t} \bar{v}_{p, \pi(p)}^2 \right)
      + \Term_2\left( \frac{\rd}{\rd t} \bar{v}_{p, \pi(p)}^2 \right)
      + \Term_3\left( \frac{\rd}{\rd t} \bar{v}_{p, \pi(p)}^2 \right).
  \end{align*}
  Since $\bar{v}_{p, \pi(p)}^2$ is larger this time, under the same conditions of Lemma~\ref{lemma: tangent: dynamics of 
  the diagonal entries (stage 1)}, we have 
  \[
    \abs{\Term_3}
    \le \delta_T \times 4I \hat\sigma_{2I}^2 a_{\pi(p)} \bar{v}_{p, \pi(p)}^{2I}.
  \]
  In addition, we have 
  \begin{align*}
    \Term_2 
    \ge - 2 \bar{v}_{p, \pi(p)}^2
      \sum_{i=I+1}^\infty 2i \hat\sigma_{2i}^2 \sum_{r: r \ne P} a_{\pi(r)} \bar{v}_{p, \pi(r)}^{2i} 
    &\ge - 2 C_\sigma^2 \bar{v}_{p, \pi(p)}^2 \norm{\a}_1 \eps_0^{I+1} \\
    &= - \frac{
      C_\sigma^2 \norm{\a}_1 \eps_0^{I+1} 
    }{2 I \hat\sigma_{2I}^2 a_{\pi(p)} \bar{v}_{p, \pi(p)}^{2I-2}}
    \times 4I \hat\sigma_{2I}^2 a_{\pi(p)} \bar{v}_{p, \pi(p)}^{2I}.
  \end{align*}
  For the same reason, under the conditions of Lemma~\ref{lemma: tangent: dynamics of 
  the diagonal entries (stage 1)}, the coefficient is bounded by $\delta_T$. Hence
  \[
    \frac{\rd}{\rd t} \bar{v}_{p, \pi(p)}^2
    \ge \Term_1\left( \frac{\rd}{\rd t} \bar{v}_{p, \pi(p)}^2 \right)
      - 2 \delta_T \times 4I \hat\sigma_{2I}^2 a_{\pi(p)} \bar{v}_{p, \pi(p)}^{2I}.
  \]
  Finally, we lower bound $\Term_1$. To this end, we compute 
  \begin{align*}
    \Term_1
    &= 2 \bar{v}_{p, \pi(p)}^2
      2I \hat\sigma_{2I}^2 \left( 
        a_{\pi(p)} \left( 1 - \bar{v}_{p, \pi(p)}^2 \right) \bar{v}_{p, \pi(p)}^{2I-2} 
        - \sum_{r: r \ne p} a_{\pi(r)} \bar{v}_{p, \pi(r)}^{2I} 
      \right) \\
    &\ge 
      2 \bar{v}_{p, \pi(p)}^2
      2I \hat\sigma_{2I}^2 \left( 
        a_{\pi(p)} \left( 1 - \delta_v' \right) \bar{v}_{p, \pi(p)}^{2I-2} 
        - \norm{\a}_1 \eps_0^I
      \right) \\
    &= \left( 
        1 - \delta_v' 
        - \frac{\norm{\a}_1 \eps_0^I}{a_{\pi(p)} \bar{v}_{p, \pi(p)}^{2I-2}}
      \right)
      \times 4 I \hat\sigma_{2I}^2 a_{\pi(p)} \bar{v}_{p, \pi(p)}^{2I}.
  \end{align*}
  We will see that since the initial $\bar{v}_{p, \pi(p)}^2$ in Stage~2 is much larger than $1/d$, Stage~2 is much
  shorter than Stage~1, whence we only need the error in the coefficient to be smaller than a constant, say, $1/2$. 
  To this end, it suffices to require
  $\delta_v' \le 1/3$ and $\frac{\norm{\a}_1 \eps_0^I}{a_{\pi(p)} \bar{v}_{p, \pi(p)}^{2I-2}} \le \frac{1}{3}$, and 
  the second condition is again implied by the conditions of Lemma~\ref{lemma: tangent: dynamics of the diagonal entries 
  (stage 1)}. 
\end{proof}

\begin{lemma}[Dynamics of the diagonal entries (Stage 3)]
  \label{lemma: tangent: dynamics of the diagonal entries (stage 3)}
  Suppose that at time $t \in [T_{\delta_v'}, T_{1 - \eps}]$, \IHGF{} is true. 
  In addition, suppose that the conditions of Lemma~\ref{lemma: tangent: dynamics of the diagonal entries (stage 1)} 
  holds and $\eps 
  \ge \frac{ 2^{3I+7} C_\sigma^2 }{ (\delta_v')^I \hat\sigma_{2I}^2 }
    \braces{
      \bar\eps^{1/2} \eps_0^{I-1} 
      \vee \frac{m \sigma_1^2}{a_{\min_*}}
      \vee \frac{\norm{\a}_1}{a_{\min_*}}  \eps_0^I
    }.$\footnote{Note that the order of the RHS is higher than $1$. This allows $\eps$ to be smaller than 
    $\eps_0$ and $\bar\eps$.}
  Then, at time $t \in [T_{\delta_v'}, T_{1 - \eps}]$, for any $p \in [P_*]$, we have 
  \[
    \frac{\rd}{\rd t} \bar{v}_{p, \pi(p)}^2
    \ge \left( \delta_v' \right)^I  I \hat\sigma_{2I}^2  a_{\pi(p)} \left( 1 - \bar{v}_{p, \pi(p)}^2 \right). 
  \]
\end{lemma}
\begin{proof}
  By the proof of Lemma~\ref{lemma: tangent: dynamics of the diagonal entries (stage 2)}, we have 
  \[
    \frac{\rd}{\rd t} \bar{v}_{p, \pi(p)}^2
    = \Term_1\left( \frac{\rd}{\rd t} \bar{v}_{p, \pi(p)}^2 \right)
      + \Term_2\left( \frac{\rd}{\rd t} \bar{v}_{p, \pi(p)}^2 \right)
      + \Term_3\left( \frac{\rd}{\rd t} \bar{v}_{p, \pi(p)}^2 \right), 
  \]
  where 
  \begin{align*}
    \Term_1 
    &\ge 2 \bar{v}_{p, \pi(p)}^2
      2I \hat\sigma_{2I}^2 \left( 
        a_{\pi(p)} \left( 1 - \bar{v}_{p, \pi(p)}^2 \right) \bar{v}_{p, \pi(p)}^{2I-2} 
        - \norm{\a}_1 \eps_0^I
      \right), \\
    \Term_2
    &\ge - 2 C_\sigma^2  \norm{\a}_1 \eps_0^{I+1}, \\
    \abs{\Term_3}
    &\le I 2^{3I+6} C_\sigma^2 
      \braces{
        a_{\pi(p)} \bar\eps^{1/2} \eps_0^{I-1} 
        \vee m \sigma_1^2
        \vee \norm{\a}_1 \eps_0^I 
      }.
  \end{align*}
  For the first term, we compute 
  \[
    \Term_1
    \ge \delta_v'
      \left( 
        \left( \delta_v' \right)^{I-1} 
        - \frac{\norm{\a}_1 \eps_0^I}{a_{\pi(p)} \eps }
      \right)
      \times 4 I \hat\sigma_{2I}^2  a_{\pi(p)} \left( 1 - \bar{v}_{p, \pi(p)}^2 \right) 
  \]
  When $\eps \ge \frac{2 \norm{\a}_1 \eps_0^I}{a_{\min_*} \left( \delta_v' \right)^{I-1} }$, we can further 
  rewrite the above as 
  \[
    \Term_1
    \ge \frac{\left( \delta_v' \right)^I }{2}
      \times 4 I \hat\sigma_{2I}^2  a_{\pi(p)} \left( 1 - \bar{v}_{p, \pi(p)}^2 \right). 
  \]
  When $\bar{v}_{p, \pi(p)}^2 \le 1 - \eps$, the RHS is lower bounded by 
  $\frac{\left( \delta_v' \right)^I }{2} \times 4 I \hat\sigma_{2I}^2  a_{\pi(p)} \eps$. 
  Our goal now is to show ensure $\Term_2$ and $\Term_3$ are both bounded by 
  $\frac{\left( \delta_v' \right)^I }{8} \times 4 I \hat\sigma_{2I}^2  a_{\pi(p)} \eps$.
  For $\Term_2$, we compute 
  \[
    -\Term_2 
    \le 2 C_\sigma^2  \norm{\a}_1 \eps_0^{I+1} 
    \le \frac{(\delta_v')^I}{8} \times 4 I \hat\sigma_{2I}^2  a_{\pi(p)} \eps
    \quad\Leftarrow\quad
    \eps
    \ge \frac{ 4 C_\sigma^2 }{ (\delta_v')^I I \hat\sigma_{2I}^2 }
      \frac{\norm{\a}_1}{a_{\min_*}} \eps_0^{I+1}.  
  \]
  Then, for $\Term_3$, by \eqref{eq: tangent: error <= delta}, we 
  \begin{multline*}
    a_{\pi(p)} \bar\eps^{1/2} \eps_0^{I-1} 
    \vee m \sigma_1^2
    \vee \norm{\a}_1 \eps_0^I 
    \le 
      \frac{\left( \delta_v' \right)^I \hat\sigma_{2I}^2}{2 2^{3I+6} C_\sigma^2 } 
      a_{\pi(p)} \eps 
    \\
    \Leftarrow\quad 
    \bar\eps 
    \le \left( \frac{\left( \delta_v' \right)^I \hat\sigma_{2I}^2}{2 2^{3I+6} C_\sigma^2 } \eps \right)^2 
      d^{2(1-\gamma)(I-1)}, \quad 
    m \sigma_1^2 \le \frac{\left( \delta_v' \right)^I \hat\sigma_{2I}^2}{2 2^{3I+6} C_\sigma^2 } a_{\min_*} \eps, \\
    d \ge \left( 
            \frac{1}{\norm{\a}_1} 
            \frac{\left( \delta_v' \right)^I \hat\sigma_{2I}^2}{2 2^{3I+6} C_\sigma^2 } 
            a_{\pi(p)} \eps
          \right)^{-\frac{1}{(1 - \gamma) I}}.
  \end{multline*}
  Then, rearrange terms so that they become conditions on $\eps$: 
  \[
    \eps 
    \ge \frac{2^{3I+7} C_\sigma^2 }{\left( \delta_v' \right)^I \hat\sigma_{2I}^2}
      \left(
        \bar\eps^{1/2} \eps_0^{I-1}
        \vee \frac{m \sigma_1^2}{a_{\min_*}} 
        \vee \frac{\norm{\a}_1} {a_{\min_*} } \eps_0^I
      \right).
  \]
  Combine the above results, and we obtain
  \[
    \frac{\rd}{\rd t} \bar{v}_{p, \pi(p)}^2
    \ge \frac{\left( \delta_v' \right)^I }{4}
      \times 4 I \hat\sigma_{2I}^2  a_{\pi(p)} \left( 1 - \bar{v}_{p, \pi(p)}^2 \right),
  \]
  provided that 
  \[
    \eps 
    \ge \frac{2 \norm{\a}_1 \eps_0^I}{a_{\min_*} \left( \delta_v' \right)^{I-1} }
      \vee \frac{ 4 C_\sigma^2 }{ (\delta_v')^I I \hat\sigma_{2I}^2 } \frac{\norm{\a}_1}{a_{\min_*}} \eps_0^{I+1}
      \vee \frac{2^{3I+7} C_\sigma^2 }{\left( \delta_v' \right)^I \hat\sigma_{2I}^2}
        \left(
          \bar\eps^{1/2} \eps_0^{I-1}
          \vee \frac{m \sigma_1^2}{a_{\min_*}} 
          \vee \frac{\norm{\a}_1} {a_{\min_*} } \eps_0^I
        \right).
  \]
  Note that (the last condition of) the third condition dominate the first two conditions. Hence, we can simplify
  the above condition to be 
  \[
    \eps 
    \ge \frac{2^{3I+7} C_\sigma^2 }{\left( \delta_v' \right)^I \hat\sigma_{2I}^2}
        \left(
          \bar\eps^{1/2} \eps_0^{I-1}
          \vee \frac{m \sigma_1^2}{a_{\min_*}} 
          \vee \frac{\norm{\a}_1} {a_{\min_*} } \eps_0^I
        \right).
  \]
\end{proof}

Now, we combine the previous lemmas and estimate the convergence rate of $\bar{\v}_p$. 

\begin{lemma}[Directional convergence]
  \label{lemma: directional convergence}
  Inductively assume \IHGF. Let $\eps$ be the target accuracy and $\delta_T$ the target error in time.
  Suppose that 
  \begin{gather*}
    \gamma < \frac{1}{2 I}, \quad \delta_v' = \frac{1}{3}, \\
    \eps \ge  
    \exp\left(
      - \frac{4 C_\sigma^2 }{ I \hat\sigma_{2I}^2  } 
      \frac{(\delta_v')^I}{ 8 I }
      \left(\frac{d}{\log^2 d} \right)^{I + 1/I - 2}
    \right), \quad 
    m \sigma_1^2 
    \le 
      \frac{ \hat\sigma_{2I}^2 a_{\min_*}}{ 2^{3I+7} C_\sigma^2 }  
      \left(
        (\delta_v')^I  \eps 
        \wedge 
        \frac{\delta_T }{d^{I-1/2}}
      \right), \\
    d 
    \ge \left( 
        \frac{ 2^{3I+7} C_\sigma^2 }{(\delta_v')^I \hat\sigma_{2I}^2 } 
        \frac{\norm{\a}_1}{a_{\min_*}} 
        \frac{1}{\eps}
      \right)^{\frac{1}{(1-\gamma) I}}
      \vee 
      \left( 
        \frac{ \hat\sigma_{2I}^2 } { 2^{3I+4} C_\sigma^2  } \frac{a_{\min_*}}{\norm{\a}_1} \delta_T
      \right)^{-\frac{2}{1 - 2 \gamma I}} , \\
    \bar\eps 
    \le \left( \frac{ (\delta_v')^I \hat\sigma_{2I}^2 }{ 2^{3I+7} C_\sigma^2 } \right)^2 
        \eps^2 d^{2 (1-\gamma)(I-1)}
        \wedge 
        \left( \delta_T \frac{ \hat\sigma_{2I}^2  }{ 2^{3I+4} C_\sigma^2 } \right)^2
            \frac{1}{d^{1+ 2 \gamma (I - 1)} }.
  \end{gather*} 
  Then, for any $p \in [P_*]$, the time needed for $\bar{v}_{p, \pi(p)}^2$ to reach $1 - \eps$ satisfies
  \[
    T_{1-\eps}
    = \frac{1 \pm 10 \delta_T }{ 4I (I - 1) \hat\sigma_{2I}^2 a_{\pi(p)} \bar{v}_{p, \pi(p)}^{2I-2}(0) }
    = \Theta\left(
        \frac{1}{ a_{\pi(p)} \bar{v}_{p, \pi(p)}^{2I-2}(0) }
      \right)
    = \tilde\Theta\left(
      \frac{1}{ a_{\pi(p)} d^{I-1} }
    \right). 
  \]
  Moreover, the requirements on $d$ can be removed if we choose\footnote{Note that this condition on $\eps$ is 
  stronger than the existing one.} 
  \begin{align*}
    \eps 
    &\ge \frac{ 2^{3I+7} C_\sigma^2 }{(\delta_v')^I \hat\sigma_{2I}^2 } 
      \frac{\norm{\a}_1}{a_{\min_*}} 
      \frac{1}{d^{(1-\gamma) I} }
    = \Theta\left( \frac{\norm{\a}_1}{a_{\min_*}} \frac{1}{d^{(1-\gamma) I} } \right), \\
    \delta_T 
    &\ge \frac{ 2^{3I+4} C_\sigma^2  }{ \hat\sigma_{2I}^2 } \frac{\norm{\a}_1}{a_{\min_*}} \frac{1}{d^{1/2-\gamma I} }
    = \Theta\left( \frac{\norm{\a}_1}{a_{\min_*}} \frac{1}{d^{1/2-\gamma I} } \right). 
  \end{align*}
\end{lemma}
\begin{proof}[Proof (Part I): convergence rate]
  By Lemma~\ref{lemma: tangent: dynamics of the diagonal entries (stage 1)}, for any $t \in [0, T_{\delta_v}]$,
  we have 
  \begin{multline*}
    \frac{\rd}{\rd t} \bar{v}_{p, \pi(p)}^2 
    = \left( 1 \pm 3 \delta_T \right)
      \times 4I \hat\sigma_{2I}^2 a_{\pi(p)} \left( \bar{v}_{p, \pi(p)}^2 \right)^I \\
    \Rightarrow\quad 
    \bar{v}_{p, \pi(p)}^2(t) 
    = \bar{v}_{p, \pi(p)}^{2I-2}(0)
      \left(
      1
      - \left( 1 \pm 3 \delta_T \right) 4I (I - 1) \hat\sigma_{2I}^2 a_{\pi(p)} \bar{v}_{p, \pi(p)}^{2I-2}(0) t
    \right)^{-\frac{1}{I-1}}.
  \end{multline*}
  This implies 
  \[
    \frac{1 - 4 \delta_T }{4I (I - 1) \hat\sigma_{2I}^2 a_{\pi(p)} \bar{v}_{p, \pi(p)}^{2I-2}(0)}
    \left( 1 - \left( \frac{\bar{v}_{p, \pi(p)}^{2I-2}(0)}{\delta_v} \right)^{I-1} \right)
    \le T_{\delta_v}
    \le \frac{1 + 4 \delta_T }{ 4I (I - 1) \hat\sigma_{2I}^2 a_{\pi(p)} \bar{v}_{p, \pi(p)}^{2I-2}(0) }. 
  \]
  For the lower bound, note that 
  \[
    \left( \frac{\bar{v}_{p, \pi(p)}^{2I-2}(0)}{\delta_v} \right)^{I-1}
    \le \delta_T 
    \quad\Leftarrow\quad
    \delta_v  
    \ge \delta_T^{\frac{-1}{I-1}} \bar{v}_{p, \pi(p)}^{2I-2}(0)
    \quad\Leftarrow\quad
    \delta_v  
    \ge \left( \frac{ \log^2 d }{ d \delta_T } \right)^{I-1}.
  \]
  When the above condition holds, we have 
  \[
    T_{\delta_v}
    = \frac{1 \pm 6 \delta_T }{ 4I (I - 1) \hat\sigma_{2I}^2 a_{\pi(p)} \bar{v}_{p, \pi(p)}^{2I-2}(0) }.
  \]
  For Stage~2, by Lemma~\ref{lemma: tangent: dynamics of the diagonal entries (stage 2)}, we have 
  \begin{multline*}
    \frac{\rd}{\rd t} \bar{v}_{p, \pi(p)}^2 
    \ge 2 I \hat\sigma_{2I}^2 a_{\pi(p)} \left( \bar{v}_{p, \pi(p)}^2 \right)^I  \\
    \Rightarrow\quad 
    \bar{v}_{p, \pi(p)}^2(t) 
    \ge \delta_v
      \left(
        1 - 2 I (I - 1) \hat\sigma_{2I}^2 a_{\pi(p)} \delta_v^{I-1} (t - T_{\delta_v})
      \right)^{-\frac{1}{I-1}} \\
    \Rightarrow\quad 
    T_{\delta_v'} - T_{\delta_v}
    \le \frac{1}{2 I (I - 1) \hat\sigma_{2I}^2 a_{\pi(p)} \delta_v^{I-1}}
    \le \frac{4 \bar{v}_{p, \pi(p)}^{2I-2}(0)}{ \delta_v^{I-1}}  T_{\delta_v}.
  \end{multline*}
  For the coefficient to be smaller than $\delta_T$, it suffices to require 
  \[
    \frac{4 \bar{v}_{p, \pi(p)}^{2I-2}(0)}{ \delta_v^{I-1}} 
    \le \delta_T 
    \quad\Leftarrow\quad
    \delta_v 
    \ge \left( \frac{4 \bar{v}_{p, \pi(p)}^{2I-2}(0)}{ \delta_T} \right)^{\frac{1}{I-1}}
    \quad\Leftarrow\quad
    \delta_v 
    \ge \left( \frac{4 }{ \delta_T} \right)^{\frac{1}{I-1}} \frac{\log^2 d}{d}.
  \]
  Finally, for Stage~3, by Lemma~\ref{lemma: tangent: dynamics of the diagonal entries (stage 3)}, we have 
  \begin{multline*}
    \frac{\rd}{\rd t} \left( 1 - \bar{v}_{p, \pi(p)}^2 \right)
    \le - \left( \delta_v' \right)^I  I \hat\sigma_{2I}^2  a_{\pi(p)} \left( 1 - \bar{v}_{p, \pi(p)}^2 \right) \\
    \Rightarrow\quad
    1 - \bar{v}_{p, \pi(p)}^2(t)
    \le \exp\left( - \left( \delta_v' \right)^I  I \hat\sigma_{2I}^2  a_{\pi(p)} t \right) \\
    \Rightarrow\quad
    T_{1-\eps} - T_{\delta_v}
    \le \frac{\log(1/\eps)}{ \left( \delta_v' \right)^I  I \hat\sigma_{2I}^2  a_{\pi(p)} }
    \le \frac{
        8 I \bar{v}_{p, \pi(p)}^{2I-2}(0)
        \log(1/\eps)
      }{ \left( \delta_v' \right)^I  } 
      T_{\delta_v}.
  \end{multline*}
  Again, for the coefficient to be smaller than $\delta_T$, it suffices to require
  \begin{multline*}
    \frac{
      8 I \bar{v}_{p, \pi(p)}^{2I-2}(0)
      \log(1/\eps)
    }{ \left( \delta_v' \right)^I  } 
    \le \delta_T 
    \quad\Leftarrow\quad
    \eps
    \ge \exp\left( - \frac{ \delta_T \left( \delta_v' \right)^I }{ 8 I \bar{v}_{p, \pi(p)}^{2I-2}(0) } \right) \\
    \Leftarrow\quad
    \eps
    \ge \exp\left( - \frac{ \delta_T \left( \delta_v' \right)^I }{ 8 I } \left( \frac{d}{\log ^2 d} \right)^{I-1} \right).
  \end{multline*}
  Combine the above results, and we obtain 
  \[
    T_{1-\eps}
    = T_{\delta_v} \pm 2 \delta_T T_{\delta_v}
    = \frac{1 \pm 10 \delta_T }{ 4I (I - 1) \hat\sigma_{2I}^2 a_{\pi(p)} \bar{v}_{p, \pi(p)}^{2I-2}(0) },
  \]
  provided that the conditions of Lemma~\ref{lemma: tangent: dynamics of the diagonal entries (stage 1)},
  \ref{lemma: tangent: dynamics of the diagonal entries (stage 2)}, \ref{lemma: tangent: dynamics of the diagonal entries (stage 3)}
  hold and 
  \[
    \delta_v  
    \ge \left( \frac{ \log^2 d }{ d \delta_T } \right)^{I-1}
      \vee \left( \frac{4 }{ \delta_T} \right)^{\frac{1}{I-1}} \frac{\log^2 d}{d}
    \quad\text{and}\quad
    \eps
    \ge \exp\left( - \frac{ \delta_T \left( \delta_v' \right)^I }{ 8 I } \left( \frac{d}{\log ^2 d} \right)^{I-1} \right).
  \]
\end{proof}
\begin{proof}[Proof (Part II): resolving the conditions]
  We now resolve the needed conditions. For easier reference, we list the requirements of Lemma~\ref{lemma: tangent: dynamics of the diagonal entries (stage 1)},
  \ref{lemma: tangent: dynamics of the diagonal entries (stage 2)}, \ref{lemma: tangent: dynamics of the diagonal entries (stage 3)},
  and this lemma below: 
  \begin{equation}
    \label{eq: directional convergence conditions (proof)}
    \begin{gathered}
      \delta_v \le \frac{\delta_T}{2} \frac{ 2I \hat\sigma_{2I}^2  } { C_\sigma^2 }, \quad 
      \gamma < \frac{1}{2 I}, \quad 
      m \sigma_1^2
      \le \delta_T \frac{ \hat\sigma_{2I}^2 a_{\min_*} }{ 2^{3I+4} C_\sigma^2 d^{I-1/2} }, \\
      \bar\eps 
        \le \left( \delta_T \frac{ \hat\sigma_{2I}^2  }{ 2^{3I+4} C_\sigma^2 } \right)^2
          \frac{1}{d^{1+ 2 \gamma (I - 1)} }, \quad 
      d \ge \left( 
          \frac{ \hat\sigma_{2I}^2 } { 2^{3I+4} C_\sigma^2  } \frac{a_{\min_*}}{\norm{\a}_1} \delta_T
        \right)^{-\frac{2}{1 - 2 \gamma I}}, \\
      \delta_v' \le 1/3, \\
      \eps 
      \ge \frac{ 2^{3I+7} C_\sigma^2 }{ (\delta_v')^I \hat\sigma_{2I}^2 }
        \braces{
          \bar\eps^{1/2} \eps_0^{I-1} 
          \vee \frac{m \sigma_1^2}{a_{\min_*}}
          \vee \frac{\norm{\a}_1}{a_{\min_*}}  \eps_0^I
        }, \\
      \delta_v  
      \ge \left( \frac{ \log^2 d }{ d \delta_T } \right)^{I-1}
        \vee \left( \frac{4 }{ \delta_T} \right)^{\frac{1}{I-1}} \frac{\log^2 d}{d}, 
      \quad 
      \eps
      \ge \exp\left( - \frac{ \delta_T \left( \delta_v' \right)^I }{ 8 I } \left( \frac{d}{\log ^2 d} \right)^{I-1} \right).
    \end{gathered}
  \end{equation}
  
  We proceed under the following principle. First, $\eps$ is a given parameter, so we should have minimal restrictions
  on it. $\delta_T$ should be interpreted as the final output of the lemma. In other parts of the proof, we only need 
  to be $1/\poly P$ small, and it is relatively easy to obtain contains of form $\delta_T \ge 1/d^c$. Hence, we will 
  try to change condition on other parameters to conditions on $\delta_T$. Finally, $\delta_v, \delta_v'$ are only 
  used in this proof, so it suffices to ensure the existence of them. 

  We start with the conditions on $\eps$, which are 
  \[
    \eps 
    \ge \frac{ 2^{3I+7} C_\sigma^2 }{ (\delta_v')^I \hat\sigma_{2I}^2 }
      \braces{
        \bar\eps^{1/2} \eps_0^{I-1} 
        \vee \frac{m \sigma_1^2}{a_{\min_*}}
        \vee \frac{\norm{\a}_1}{a_{\min_*}}  \eps_0^I
      }
      \vee 
      \exp\left( - \frac{ \delta_T \left( \delta_v' \right)^I }{ 8 I } \left( \frac{d}{\log ^2 d} \right)^{I-1} \right).
  \]
  This can be translated into 
  \begin{gather*}
    \eps_0^I
    \le \frac{ (\delta_v')^I \hat\sigma_{2I}^2 }{ 2^{3I+7} C_\sigma^2 } \frac{a_{\min_*}}{\norm{\a}_1} \eps , 
    \quad 
    m \sigma_1^2 
    \le \frac{ (\delta_v')^I \hat\sigma_{2I}^2 }{ 2^{3I+7} C_\sigma^2 } a_{\min_*} \eps , 
    \quad 
    \bar\eps^{1/2} \eps_0^{I-1} 
    \le \frac{ (\delta_v')^I \hat\sigma_{2I}^2 }{ 2^{3I+7} C_\sigma^2 } \eps, \\
    \delta_T 
    \ge \frac{ 8 I }{(\delta_v')^I} \left( \frac{\log ^2 d}{d} \right)^{I-1} \log\left(\frac{1}{\eps}\right).
  \end{gather*}
  Then, consider $\delta_v, \delta_v'$. We choose $\delta_v' = 1/3$. For the existence of $\delta_v$, it suffices to 
  require (cf.~the first and second last conditions of \eqref{eq: directional convergence conditions (proof)})
  \begin{multline*}
    \left( \frac{ \log^2 d }{ d \delta_T } \right)^{I-1}
        \vee \left( \frac{4 }{ \delta_T} \right)^{\frac{1}{I-1}} \frac{\log^2 d}{d}
    \le \frac{\delta_T}{2} \frac{ 2I \hat\sigma_{2I}^2  } { C_\sigma^2 } \\
    \Leftarrow\quad 
    \delta_T 
    \ge \left( \frac{ C_\sigma^2 }{ I \hat\sigma_{2I}^2 } \right)^{1/I} 
      \left( \frac{ \log^2 d }{ d } \right)^{1-1/I} 
      \vee 
      \left(
        \frac{4 C_\sigma^2 }{ I \hat\sigma_{2I}^2  } 
        \frac{\log^2 d}{d}
      \right)^{1 - 1/I} \\
    \Leftarrow\quad 
    \delta_T 
    \ge \frac{4 C_\sigma^2 }{ I \hat\sigma_{2I}^2  } \left(\frac{\log^2 d}{d} \right)^{1 - 1/I}.
  \end{multline*}
  This condition will also be stronger than the previous one, as long as 
  \begin{multline*}
    \frac{4 C_\sigma^2 }{ I \hat\sigma_{2I}^2  } \left(\frac{\log^2 d}{d} \right)^{1 - 1/I}
    \ge \frac{ 8 I }{(\delta_v')^I} \left( \frac{\log ^2 d}{d} \right)^{I-1} \log\left(\frac{1}{\eps}\right) \\
    \Leftarrow\quad 
    \eps \ge  
      \exp\left(
        - \frac{4 C_\sigma^2 }{ I \hat\sigma_{2I}^2  } 
        \frac{(\delta_v')^I}{ 8 I }
        \left(\frac{d}{\log^2 d} \right)^{I + 1/I - 2}
      \right).
  \end{multline*}
  While this is a restriction on $\eps$, it is very mild as the RHS is super-polynomially small. Now, we have 
  replaced \eqref{eq: directional convergence conditions (proof)} with 
  \begin{gather*}
    \eps \ge  
    \exp\left(
      - \frac{4 C_\sigma^2 }{ I \hat\sigma_{2I}^2  } 
      \frac{(\delta_v')^I}{ 8 I }
      \left(\frac{d}{\log^2 d} \right)^{I + 1/I - 2}
    \right), \quad 
    m \sigma_1^2 
    \le 
      \frac{ \hat\sigma_{2I}^2 a_{\min_*}}{ 2^{3I+7} C_\sigma^2 }  
      \left(
        (\delta_v')^I  \eps 
        \wedge 
        \frac{\delta_T }{d^{I-1/2}}
      \right), \\
    \delta_T 
    \ge \frac{4 C_\sigma^2 }{ I \hat\sigma_{2I}^2  } \left(\frac{\log^2 d}{d} \right)^{1 - 1/I}, 
      \\
    \eps_0^I
    \le \frac{ (\delta_v')^I \hat\sigma_{2I}^2 }{ 2^{3I+7} C_\sigma^2 } \frac{a_{\min_*}}{\norm{\a}_1} \eps , 
    \quad 
    \bar\eps^{1/2} \eps_0^{I-1} 
    \le \frac{ (\delta_v')^I \hat\sigma_{2I}^2 }{ 2^{3I+7} C_\sigma^2 } \eps, \\
    \gamma < \frac{1}{2 I}, \quad 
    \bar\eps 
    \le \left( \delta_T \frac{ \hat\sigma_{2I}^2  }{ 2^{3I+4} C_\sigma^2 } \right)^2
        \frac{1}{d^{1+ 2 \gamma (I - 1)} }, \quad 
    d \ge \left( 
        \frac{ \hat\sigma_{2I}^2 } { 2^{3I+4} C_\sigma^2  } \frac{a_{\min_*}}{\norm{\a}_1} \delta_T
      \right)^{-\frac{2}{1 - 2 \gamma I}}. 
  \end{gather*} 
  Consider the last two lines. For the second last line, we compute 
  \begin{align*}
    \eps_0^I
    \le \frac{ (\delta_v')^I \hat\sigma_{2I}^2 }{ 2^{3I+7} C_\sigma^2 } \frac{a_{\min_*}}{\norm{\a}_1} \eps
    &\quad\Leftarrow\quad 
    d 
    \ge \left( 
        \frac{ 2^{3I+7} C_\sigma^2 }{(\delta_v')^I \hat\sigma_{2I}^2 } 
        \frac{\norm{\a}_1}{a_{\min_*}} 
        \frac{1}{\eps}
      \right)^{\frac{1}{(1-\gamma) I}}, \\
    \bar\eps^{1/2} \eps_0^{I-1} 
      \le \frac{ (\delta_v')^I \hat\sigma_{2I}^2 }{ 2^{3I+7} C_\sigma^2 } \eps
    &\quad\Leftarrow\quad 
    \bar\eps 
    \le \left( \frac{ (\delta_v')^I \hat\sigma_{2I}^2 }{ 2^{3I+7} C_\sigma^2 } \right)^2 
        \eps^2 d^{2 (1-\gamma)(I-1)}.
  \end{align*}
  For the last line, we convert the conditions into conditions on $\delta_T$: 
  \begin{align*}
    \bar\eps 
    \le \left( \delta_T \frac{ \hat\sigma_{2I}^2  }{ 2^{3I+4} C_\sigma^2 } \right)^2
          \frac{1}{d^{1+ 2 \gamma (I - 1)} }
    &\quad\Leftrightarrow\quad 
    \delta_T 
    \ge \frac{ 2^{3I+4} C_\sigma^2 }{ \hat\sigma_{2I}^2  } 
      \sqrt{ \bar\eps d^{1+ 2 \gamma (I - 1)}  }, 
    \\
    d \ge \left( 
          \frac{ \hat\sigma_{2I}^2 } { 2^{3I+4} C_\sigma^2  } \frac{a_{\min_*}}{\norm{\a}_1} \delta_T
        \right)^{-\frac{2}{1 - 2 \gamma I}}
    &\quad\Leftrightarrow\quad 
    \delta_T 
    \ge \frac{ 2^{3I+4} C_\sigma^2  } { \hat\sigma_{2I}^2 }  \frac{\norm{\a}_1} {a_{\min_*}} d^{-1/2 + \gamma I} .
  \end{align*}
  Thus, the conditions are 
  \begin{gather*}
    \eps \ge  
    \exp\left(
      - \frac{4 C_\sigma^2 }{ I \hat\sigma_{2I}^2  } 
      \frac{(\delta_v')^I}{ 8 I }
      \left(\frac{d}{\log^2 d} \right)^{I + 1/I - 2}
    \right), \quad 
    m \sigma_1^2 
    \le 
      \frac{ \hat\sigma_{2I}^2 a_{\min_*}}{ 2^{3I+7} C_\sigma^2 }  
      \left(
        (\delta_v')^I  \eps 
        \wedge 
        \frac{\delta_T }{d^{I-1/2}}
      \right), \\
    d 
    \ge \left( 
        \frac{ 2^{3I+7} C_\sigma^2 }{(\delta_v')^I \hat\sigma_{2I}^2 } 
        \frac{\norm{\a}_1}{a_{\min_*}} 
        \frac{1}{\eps}
      \right)^{\frac{1}{(1-\gamma) I}}, \quad
    \bar\eps 
    \le \left( \frac{ (\delta_v')^I \hat\sigma_{2I}^2 }{ 2^{3I+7} C_\sigma^2 } \right)^2 
        \eps^2 d^{2 (1-\gamma)(I-1)}.
    \\
    \gamma < \frac{1}{2 I}, \quad 
    \delta_T 
    \ge 
      \frac{4 C_\sigma^2 }{ I \hat\sigma_{2I}^2  } \left(\frac{\log^2 d}{d} \right)^{1 - 1/I}
      \vee 
      \frac{ 2^{3I+4} C_\sigma^2 }{ \hat\sigma_{2I}^2  }  \sqrt{ \bar\eps d^{1+ 2 \gamma (I - 1)}  } 
      \vee 
      \frac{ 2^{3I+4} C_\sigma^2  } { \hat\sigma_{2I}^2 }  \frac{\norm{\a}_1} {a_{\min_*}} d^{-1/2 + \gamma I}.
  \end{gather*} 
  Note that $1/2 - \gamma I \le 1/2 \le 1 - 1/I$ when $I \ge 2$. Hence, the condition on $\delta_T$ is equivalent to 
  \[
    \delta_T 
    \ge 
      \frac{ 2^{3I+4} C_\sigma^2 }{ \hat\sigma_{2I}^2  }  \sqrt{ \bar\eps d^{1+ 2 \gamma (I - 1)}  } 
      \vee 
      \frac{ 2^{3I+4} C_\sigma^2  } { \hat\sigma_{2I}^2 }  \frac{\norm{\a}_1} {a_{\min_*}} d^{-1/2 + \gamma I}.
  \]
  To complete the proof, it suffices to revert the above conditions to conditions on $\bar\eps$ and $\delta_T$. 
\end{proof}

\subsubsection{Radial Dynamics}

Now, we estimate the time needed for a neuron to fit the ground truth after it converges in direction. 

\begin{lemma}[Dynamics of the norm (converged)]
  \label{lemma: dynamics of the norm (converged)}
  Suppose that \IHGF{} is true at time $t$. Then, at time $t$, for any $p \in [P_*]$ with $\bar{v}_{p, \pi(p)}^2 
  \ge 1 - \bar\eps$, we have 
  \[
    \frac{\rd}{\rd t} \norm{\v_p}^2 
    = 4 \norm{\v_p}^2 \left( 
        a_{\pi(p)} 
        - \norm{\v_p}^2 
        \pm \left(
          2 C_\sigma^2  a_{\pi(p)}  \bar\eps
          + 2 \norm{\a}_1 2^{2I} \eps_0^I 
          + 2 m \sigma_1^2 
        \right)
      \right). 
  \]
\end{lemma}
\begin{proof}
  By Lemma~\ref{lemma: population and per-sample gradients}, we have 
  \begin{align*}
    \frac{1}{2} \frac{\rd}{\rd t} \norm{\v_p}^2  
    &= 2 \norm{\v_p}^2 \sum_{i=I}^\infty \hat\sigma_{2i}^2 \sum_{q=1}^P a_{\pi(q)} \bar{v}_{p, \pi(q)}^{2i}
      - 2 \norm{\v_p}^2 \sum_{i=I}^\infty \hat\sigma_{2i}^2 \sum_{l=1}^m \norm{\v_l}^2 \inprod{\bar{\v}_p}{\bar{\v}_l}^{2i}  \\
    &= 2 \norm{\v_p}^2 \sum_{i=I}^\infty \hat\sigma_{2i}^2 \left(
        \sum_{q=1}^P a_{\pi(q)} \bar{v}_{p, \pi(q)}^{2i}
        - \norm{\v_p}^2
      \right)
      - 2 \norm{\v_p}^2 \sum_{i=I}^\infty \hat\sigma_{2i}^2 \sum_{l: l \ne p} \norm{\v_l}^2 \inprod{\bar{\v}_p}{\bar{\v}_l}^{2i}  \\
    &=: \Term_1\left( \frac{1}{2} \frac{\rd}{\rd t} \norm{\v_p}^2  \right)
      + \Term_2\left( \frac{1}{2} \frac{\rd}{\rd t} \norm{\v_p}^2  \right).
  \end{align*}
  First, for $\Term_1$, first recall from Assumption~\ref{assumption: link function} that 
  $\sum_{i=I}^{\infty} \hat\sigma_{2i}^2 = 1$, and 
  $\sum_{i=I}^\infty 2i \hat\sigma_{2i}^2 \le \sum_{i=I}^\infty i^2 \hat\sigma_{2i}^2 \le C_\sigma^2$.
  Also note that for any small $\delta \in (0, 1)$ and integer $N$, we have 
  \begin{align*}
    (1 - \delta)^N
    &= 1 - N \delta + \delta^2 \sum_{k=0}^{N-2} \binom{N}{k+2} (-\delta)^k \\
    &= 1 - N \delta \pm N^2 \delta^2 \sum_{k=0}^{N-2} \binom{N-2}{k} (-\delta)^k 
    = 1 - N \delta \pm N^2 \delta^2. 
  \end{align*}
  Hence, we can write 
  \begin{align*}
    \Term_1 
    &= 2 \norm{\v_p}^2 \sum_{i=I}^\infty \hat\sigma_{2i}^2 \left( a_{\pi(p)} - \norm{\v_p}^2 \right) \\
      &\qquad
      + 2 \norm{\v_p}^2 \sum_{i=I}^\infty \hat\sigma_{2i}^2 a_{\pi(p)} \left( \bar{v}_{p, \pi(p)}^{2i} - 1 \right)
      + 2 \norm{\v_p}^2 \sum_{i=I}^\infty \hat\sigma_{2i}^2 \sum_{q: q \ne p} a_{\pi(q)} \bar{v}_{p, \pi(q)}^{2i} \\
    &= 2 \norm{\v_p}^2 \left( a_{\pi(p)} - \norm{\v_p}^2 \right) 
      \pm 4 C_\sigma^2 \norm{\v_p}^2 a_{\pi(p)}  \bar\eps
      \pm 2 \norm{\v_p}^2 \norm{\a}_1 \eps_0^I. 
  \end{align*}
  Meanwhile, for $\Term_2$, by the proof of Lemma~\ref{lemma: gf: tangent dynamics (crude)}, we have 
  \begin{align*}
    \abs{\Term_2} 
    &\le 2 \norm{\v_p}^2 \sum_{i=I}^\infty \hat\sigma_{2i}^2 \sum_{l \in L\setminus\{p\} } \norm{\v_l}^2 \inprod{\bar{\v}_p}{\bar{\v}_l}^{2i}
      + 2 \norm{\v_p}^2 \sum_{i=I}^\infty \hat\sigma_{2i}^2 \sum_{l \notin L\cup\{p\} } \norm{\v_l}^2 \inprod{\bar{\v}_p}{\bar{\v}_l}^{2i} \\
    &\le 4 \norm{\v_p}^2 \sum_{i=I}^\infty \hat\sigma_{2i}^2 
      \sum_{l \in L\setminus\{p\} } a_{\pi(l)} \left( \sqrt{\eps_0} + \sqrt{2\bar\eps} \right)^{2i}
      + 2 \norm{\v_p}^2 m \sigma_1^2 \\
    &\le 4 \norm{\v_p}^2 \norm{\a}_1 2^{2I} \eps_0^I 
      + 2 \norm{\v_p}^2 m \sigma_1^2. 
  \end{align*}
  As a result, we have 
  \begin{align*}
    \frac{\rd}{\rd t} \norm{\v_p}^2 
    &= 4 \norm{\v_p}^2 \left( a_{\pi(p)} - \norm{\v_p}^2 \right)  \\
      &\qquad
      \pm 8 \norm{\v_p}^2 \left(
        C_\sigma^2  a_{\pi(p)}  \bar\eps
        + \norm{\a}_1 \eps_0^I
        + \norm{\a}_1 2^{2I} \eps_0^I 
        + m \sigma_1^2
      \right) \\
    &= 4 \norm{\v_p}^2 \left( 
        a_{\pi(p)} 
        - \norm{\v_p}^2 
        \pm \left(
          2 C_\sigma^2  a_{\pi(p)}  \bar\eps
          + 2 \norm{\a}_1 2^{2I} \eps_0^I 
          + 2 m \sigma_1^2 
        \right)
      \right). 
  \end{align*}
\end{proof}

\begin{lemma}[Fitting the signal]
  \label{lemma: fitting the signal}
  Inductively assume \IHGF. Consider $p \in [P_*]$ and 
  $\eps \ge 4 \left(
    C_\sigma^2  a_{\pi(p)}  \bar\eps
    +\norm{\a}_1 2^{2I} \eps_0^I 
    + m \sigma_1^2 
  \right)$.
  Then,  after $\bar{v}_{p, \pi(p)}^2$ reaches $1 - \bar\eps$, it takes at most 
  $\frac{3 \log\left( a_{\pi(p)}^2 / ( \sigma_0^2 \eps ) \right)}{a_{\pi(p)}}$ amount of time for $\norm{\v_p}^2$ to reach 
  $a_{\pi(p)} \pm \eps$. In addition, once it enters this range, it will stay there. 
\end{lemma}
\begin{proof}
  Let $T_0$ be the time $\bar{v}_{p, \pi(p)}^2$ reaches $1 - \bar\eps$.
  By the proof of Lemma~\ref{lemma: directional convergence}, $\bar{v}_{p, \pi(p)}^2$ will stay above $1 - \bar\eps$
  after time $T_0$. By Lemma~\ref{lemma: dynamics of the norm (converged)} and our hypothesis on $\eps$, we have 
  \[
    \frac{\rd}{\rd t} \norm{\v_p}^2 
    = 4 \norm{\v_p}^2 \left( 
        a_{\pi(p)} 
        - \norm{\v_p}^2 
        \pm \frac{\eps}{2}
      \right). 
  \]
  In particular, this implies that once $\norm{\v_p}^2$ reaches $a_{\pi(p)} \pm \eps$, it will stay in this range. 
  Let $T_{R, 1/2}$ and $T_{R, 1-\eps}$ be the time $\norm{\v_p}^2$ reaches $a_{\pi(p)}/2$ and $1 - \eps$, respectively. 
  For any $t \le T_{R, 1/2}$, we have 
  \begin{align*}
    \frac{\rd}{\rd t} \norm{\v_p}^2 
    \ge \frac{4 a_{\pi(p)} }{3} \norm{\v_p}^2 
    &\quad\Rightarrow\quad 
    \norm{\v_p(t)}^2 
    \ge \sigma_0^2 \exp\left( \frac{4 a_{\pi(p)}}{3} (t - T_0) \right) \\
    &\quad\Rightarrow\quad 
    T_{R, 1/2} - T_0
    \le \frac{3 \log\left( a_{\pi(p)} / \sigma_0^2 \right)}{ a_{\pi(p)}}. 
  \end{align*}
  After $T_{R, 1/2}$ and before $T_{R, 1-\eps}$, we have 
  \begin{multline*}
    \frac{\rd}{\rd t} \norm{\v_p}^2 
    \ge a_{\pi(p)} \left( 
        a_{\pi(p)} 
        - \norm{\v_p}^2 
        \pm \frac{\eps}{2}
      \right) 
    \ge \frac{a_{\pi(p)}}{2} \left( a_{\pi(p)} - \norm{\v_p}^2 \right)  \\
    \Rightarrow\quad 
    a_{\pi(p)}(t) - \norm{\v_p}^2
    \le \frac{a_{\pi(p)}}{2} \exp\left( - a_{\pi(p)} (t - T_{R, 1/2}) / 2  \right) \\
    \Rightarrow\quad 
    T_{R, 1-\eps} - T_{R, 1/2}
    \le \frac{3 \log\left(  a_{\pi(p)}  / \eps \right)}{a_{\pi(p)}}.  
  \end{multline*}
  As a result, we have 
  \[
    T_{R, 1-\eps} - T_0 
    \le \frac{3}{a_{\pi(p)}} \left(
        \log\left( a_{\pi(p)} / \sigma_0^2 \right)
        + \log\left(  a_{\pi(p)} / \eps \right)
      \right)
    = \frac{3 \log\left( a_{\pi(p)}^2 / ( \sigma_0^2 \eps ) \right)}{a_{\pi(p)}}.
  \]
\end{proof}

We are now ready to prove the main result of this subsection, which we restate below. 

\GfCorConvergence
\begin{proof}
  First, by Lemma~\ref{lemma: tangent: dynamics of the diagonal entries (stage 1)} (and the proof of Lemma~\ref{lemma: 
  directional convergence}), we have 
  \[
    \bar{v}_{p, \pi(p)}^2(t) 
    \le \delta_v 
    := \left( \frac{4 }{ \delta_T} \right)^{\frac{1}{I-1}} \frac{\log^2 d}{d}, 
    \quad \forall t \le \frac{1 - 10 \delta_T }{ 4I (I - 1) \hat\sigma_{2I}^2 a_{\pi(p)} \bar{v}_{p, \pi(p)}^{2I-2}(0) }
  \]
  Meanwhile, by Lemma~\ref{lemma: directional convergence}, we have $\bar{v}_{p, \pi(p)}^2 \le \delta_v$ 
  $\bar{v}_{p, \pi(p)}^2 \ge 1 - \eps_D$
  after time 
  \[
    T_T
    = \frac{1 \pm 10 \delta_T }{ 4I (I - 1) \hat\sigma_{2I}^2 a_{\pi(p)} \bar{v}_{p, \pi(p)}^{2I-2}(0) }
    = \Theta\left(
        \frac{1}{ a_{\pi(p)} \bar{v}_{p, \pi(p)}^{2I-2}(0) }
      \right),
  \]
  as long as $\gamma < 1/(2I)$, $\delta_v' = 1/3$, and 
  \begin{gather*}
    \eps_D 
    \ge \frac{ 2^{3I+7} C_\sigma^2 }{(\delta_v')^I \hat\sigma_{2I}^2 } 
      \frac{\norm{\a}_1}{a_{\min_*}} 
      \frac{1}{d^{(1-\gamma) I} }, \quad 
    \delta_T 
    \ge \frac{ 2^{3I+4} C_\sigma^2  }{ \hat\sigma_{2I}^2 } \frac{\norm{\a}_1}{a_{\min_*}} \frac{1}{d^{1/2-\gamma I} }, \\
    m \sigma_1^2 
    \le 
      \frac{ \hat\sigma_{2I}^2 a_{\min_*}}{ 2^{3I+7} C_\sigma^2 }  
      \left(
        (\delta_v')^I  \eps 
        \wedge 
        \frac{\delta_T }{d^{I-1/2}}
      \right), \\
    \bar\eps 
    \le \left( \frac{ (\delta_v')^I \hat\sigma_{2I}^2 }{ 2^{3I+7} C_\sigma^2 } \right)^2 
        \eps_D^2 d^{2 (1-\gamma)(I-1)}
        \wedge 
        \left( \delta_T \frac{ \hat\sigma_{2I}^2  }{ 2^{3I+4} C_\sigma^2 } \right)^2
            \frac{1}{d^{1+ 2 \gamma (I - 1)} }.
  \end{gather*} 
  By Lemma~\ref{lemma: fitting the signal}, fitting $a_{\pi(p)}$ to $\pm \eps_R$ takes $T_R$ amount of time, where
  \[
    T_R := \frac{3 \log\left( a_{\pi(p)}^2 / ( \sigma_0^2 \eps_R ) \right)}{a_{\pi(p)}}.
  \]
  Since $\delta_T \ge \frac{ 2^{3I+4} C_\sigma^2  }{ \hat\sigma_{2I}^2 } \frac{\norm{\a}_1}{a_{\min_*}} \frac{1}{d^{1/2-\gamma I} }$,
  we have 
  \begin{align*}
    T_R \le \delta_T T_T 
    &\quad\Leftarrow\quad 
    \log\left( a_{\pi(p)}^2 / (\sigma_0^2 \eps_R) \right)  
    \le \frac{\delta_T d^{I-1}}{ 24 I (I - 1) \hat\sigma_{2I}^2 (\log d)^{2I-2} } \\
    &\quad\Leftarrow\quad 
    \eps_R
    \ge  \frac{a_{\pi(p)}^2 }{\sigma_0^2} \exp\left( - \frac{d^{(1-\gamma)I-1/2}}{ \hat\sigma_{2I}^2 (\log d)^{2I-2} } \right). 
  \end{align*}
  Again, this condition is mild as the RHS decays exponentially fast. To meet the conditions of 
  Lemma~\ref{lemma: fitting the signal}, it suffices to require
  \begin{gather*}
    \bar\eps \le \frac{\eps_R}{12 C_\sigma^2  a_{\pi(p)}}, \quad 
    m \sigma_1^2 \le \frac{\eps_R}{12}, \quad 
    \eps_R 
    \ge  12 \norm{\a}_1 2^{2I} d^{-(1-\gamma) I}. 
  \end{gather*}
  Note that last condition on $\eps_R$ is stronger than the previous condition on $\eps_R$. 
\end{proof}

\subsection{Maintaining the Induction Hypotheses}

\label{sec: gf: induction hypotheses}

In this subsection, we show \IHGF{} is true throughout training. 
Recall the meaning and requirements of $\eps_D, \eps_R, \delta_T$ from Corollary~\ref{cor: convergence of one direction}.

\subsubsection{Upper Bounds on the Irrelevant Coordinates}

\begin{lemma}[Upper triangular entries (case I)]
  \label{lemma: tangent: upper triangular entries (case I)}
  Consider $p \in [P_*]$ and $p < q \in [P]$ with $a_{\pi(q)} \ge a_{\min_*} / (2 (\log d)^{2I-2})$. 
  Assume the conditions of Corollary~\ref{cor: convergence of one direction} and 
  \begin{gather*}
    \bar\eps 
    \le \left( 
      \frac{ \hat\sigma_{2I}^2 }{2^{3I+4} C_\sigma^2 } 
      \frac{\delta_r}{24}
    \right)^2 
    \frac{1}{d^{1+2\gamma(I-1)}}, \quad 
    m \sigma_1^2 
    \le \frac{ \hat\sigma_{2I}^2 }{2^{3I+4} C_\sigma^2 } 
      \frac{ a_{\min_*} }{ 2 (\log d)^{2I-2} d^{I-1/2} }
      \frac{\delta_r}{24}, \\
    \frac{d}{(\log^2 d)^{1/\gamma}} 
      \ge \left( \frac{\delta_r}{4} \right)^{-\frac{1}{\gamma(I-1)}}, \;
      \frac{d}{(\log^2 d)^{\frac{I-1}{1/2 - \gamma I}}}  
      \ge \left( 
          \frac{ \hat\sigma_{2I}^2 }{2^{3I+4} C_\sigma^2 } 
          \frac{ a_{\min_*} }{ \norm{\a}_1 2^{2I-2}  }
          \frac{\delta_r}{24}
        \right)^{
          - \frac{1}{1/2 - \gamma I}
        }
        , \; 
      \delta_T \le \frac{\delta_r}{240}.
  \end{gather*}
  Then, $\bar{v}_{p, \pi(q)}^2 \le \eps_0$ throughout training. 
\end{lemma}
\begin{remark*}
  Recall from Lemma~\ref{lemma: directional convergence} that we only need $\delta_T \ge \tilde\Theta( 1 / d^{1/2-\gamma I} )$
  and by Lemma~\ref{lemma: initialization}, $\delta_r = 1 / \poly(P)$. Hence, the last condition can hold 
  as long as $d$ is large.
\end{remark*}
\begin{proof}
  \def\currentprefix{proof: upper triangular entries. aodsfm}
  First, by Corollary~\ref{cor: convergence of one direction}, we know 
  $\bar{v}_{p, \pi(p)}^2 \ge 1 - \bar\eps$ after time
  \[
    T_p 
    := \frac{1 \pm 20 \delta_T }{ 4I (I - 1) \hat\sigma_{2I}^2 a_{\pi(p)} \bar{v}_{p, \pi(p)}^{2I-2}(0) }.
  \]
  This automatically implies $\bar{v}_{p, \pi(p)}^2 \le \bar\eps \le \eps_0$ after time $T_p$. Hence, it suffices
  to consider the time before $T_p$. By Lemma~\ref{lemma: gf: tangent dynamics (crude)} and the choice 
  $\eps_0 \ge \bar\eps$, we have 
  \begin{align*}
    \frac{\rd}{\rd t} \bar{v}_{p, \pi(q)}^2
    &\le 2 \sum_{i=I}^\infty 2i \hat\sigma_{2i}^2  a_{\pi(q)} \bar{v}_{p, \pi(q)}^{2i} 
      + I 2^{3I+6} C_\sigma^2 \abs{\bar{v}_{p, \pi(q)} } 
      \braces{
        a_{\pi(q)} \bar\eps^{1/2} \eps_0^{I-1} 
        \vee m \sigma_1^2
        \vee \norm{\a}_1 \eps_0^I 
      } \\
    &=: \Term_1\left( \frac{\rd}{\rd t} \bar{v}_{p, \pi(q)}^2 \right)
      + \Term_2\left( \frac{\rd}{\rd t} \bar{v}_{p, \pi(q)}^2 \right).
  \end{align*}
  Since our goal is to upper bound $\bar{v}_{p, \pi(q)}^2$, we may assume w.l.o.g.~that $\bar{v}_{p, \pi(p)}^2 
  \ge 1/d$, as we only need to track those $t$. Then, for $\Term_2$, we have 
  \[
    \Term_2
    \le I 2^{3I+6} C_\sigma^2 d^{I-1/2}
      \braces{
        a_{\pi(q)} \bar\eps^{1/2} \eps_0^{I-1} 
        \vee m \sigma_1^2
        \vee \norm{\a}_1 \eps_0^I 
      }
      \bar{v}_{p, \pi(q)}^{2I}. 
  \]
  Meanwhile, for $\Term_1$, we have 
  \begin{align*}
    \Term_1
    &= 4I \hat\sigma_{2I}^2  a_{\pi(q)} \bar{v}_{p, \pi(q)}^{2I}
      + 2 \sum_{i=I+1}^\infty 2i \hat\sigma_{2i}^2  a_{\pi(q)} \bar{v}_{p, \pi(q)}^{2i}  \\
    &\le 4I \hat\sigma_{2I}^2  a_{\pi(q)} \bar{v}_{p, \pi(q)}^{2I}
      + 2 a_{\pi(q)} \bar{v}_{p, \pi(q)}^{2I} \eps_0 \sum_{i=I+1}^\infty 2i \hat\sigma_{2i}^2  \\
    &\le 4I \hat\sigma_{2I}^2  a_{\pi(q)} \bar{v}_{p, \pi(q)}^{2I}
      + 2 C_\sigma^2 a_{\pi(q)} \bar{v}_{p, \pi(q)}^{2I} \eps_0 .
  \end{align*}
  Combining the above two bounds, we obtain 
  \begin{align*}
    \frac{\rd}{\rd t} \bar{v}_{p, \pi(q)}^2
    &\le 4I \hat\sigma_{2I}^2  a_{\pi(q)} \bar{v}_{p, \pi(q)}^{2I}
      + 2 C_\sigma^2 a_{\pi(q)} \bar{v}_{p, \pi(q)}^{2I} \eps_0 \\
      &\qquad
      + I 2^{3I+6} C_\sigma^2 d^{I-1/2}
      \braces{
        a_{\pi(q)} \left( \bar\eps^{1/2} \eps_0^{I-1} \vee \bar\eps^{I-1/2} \right)
        \vee m \sigma_1^2
        \vee \norm{\a}_1  \eps_0^I 
      }
      \bar{v}_{p, \pi(q)}^{2I} \\
    &\le 
      \left(  1 +  \delta_{\Tmp} \right)
      4I \hat\sigma_{2I}^2  a_{\pi(q)} 
      \bar{v}_{p, \pi(q)}^{2I},
  \end{align*}
  where 
  \begin{align*}
    \delta_{\Tmp}
    &= \frac{ 2 C_\sigma^2 a_{\pi(q)} \eps_0 }{ 4I \hat\sigma_{2I}^2  a_{\pi(q)} }
      + \frac{
        I 2^{3I+6} C_\sigma^2 d^{I-1/2}
        \braces{
          a_{\pi(q)} \left( \bar\eps^{1/2} \eps_0^{I-1} \vee \bar\eps^{I-1/2} \right)
          \vee m \sigma_1^2
          \vee \norm{\a}_1 \eps_0^I 
        }
      }{
        4I \hat\sigma_{2I}^2  a_{\pi(q)}
      } \\
    &\le \frac{ C_\sigma^2 \eps_0 }{ 2 I \hat\sigma_{2I}^2 }
      + \frac{2^{3I+4} C_\sigma^2 d^{I-1/2}}{ \hat\sigma_{2I}^2  a_{\pi(q)} }
      \braces{
        a_{\pi(q)} \bar\eps^{1/2} \eps_0^{I-1} 
        \vee m \sigma_1^2 
        \vee \norm{\a}_1 \eps_0^I 
      } \\
    &=: \delta_{\Tmp, 1} + \delta_{\Tmp, 2}.
  \end{align*}
  As a result, for any $t \le T_p$, we have 
  \[
    \bar{v}_{p, \pi(q)}^2(t) 
    \le \bar{v}_{p, \pi(q)}^2(0) 
      \left(
        1 
        - (I - 1) \left(  1 +  \delta_{\Tmp} \right) 4I \hat\sigma_{2I}^2  a_{\pi(q)} \bar{v}_{p, \pi(q)}^{2I-2}(0) t 
      \right)^{-\frac{1}{I-1}}
  \]
  In particular, this implies
  \begin{align*}
    \bar{v}_{p, \pi(q)}^2(t) 
    &\le \bar{v}_{p, \pi(q)}^2(0) 
      \left(
        1 
        -  
        \left(  1 +  \delta_{\Tmp} \right) \left( 1 + 20 \delta_T \right)
        \frac{ a_{\pi(q)} \bar{v}_{p, \pi(q)}^{2I-2}(0)  }{a_{\pi(p)} \bar{v}_{p, \pi(p)}^{2I-2}(0) } 
      \right)^{-\frac{1}{I-1}} \\
    &\le \bar{v}_{p, \pi(q)}^2(0) 
      \left(
        1 - \frac{
          \left(  1 +  \delta_{\Tmp} \right) \left( 1 + 20 \delta_T \right)
        }{1 + \delta_r}
      \right)^{-\frac{1}{I-1}} \\
    &\le \bar{v}_{p, \pi(q)}^2(0) 
      \left(
        \frac{\delta_r}{2}
        - 2 \delta_{\Tmp} 
        - 20 \delta_T
      \right)^{-\frac{1}{I-1}}, 
  \end{align*}
  where the second line comes from Assumption~\ref{assumption: gf init}\ref{assumption-itm: init: row gap}. 
  Now, we find conditions under which the last term is upper bounded by $\eps_0 = d^{-(1-\gamma)}$.
  We will first find conditions under which $2 \delta_{\Tmp} + 20 \delta_T \le \delta_r / 4$ and then upper bound 
  $\bar{v}_{p, \pi(q)}^2(0) \left( \delta_r / 4 \right)^{-\frac{1}{I-1}}$.

  We compute 
  \begin{align*}
    2 \delta_{\Tmp, 1} \le \frac{\delta_r}{12} 
    &\quad\Leftarrow\quad 
    d \ge \left( \frac{ I \hat\sigma_{2I}^2 }{ C_\sigma^2  } \frac{\delta_r}{12} \right)^{-\frac{1}{1 - \gamma}}, \\
    20 \delta_T \le \frac{\delta_r}{12} 
    &\quad\Leftarrow\quad 
    \delta_T \le \frac{\delta_r}{240}, 
  \end{align*}
  and by \eqref{eq: tangent: error <= delta}, 
  \begin{align*}
    2 \delta_{\Tmp, 2} \le \frac{\delta_r}{12}  
    &\quad\Leftarrow\quad 
    a_{\pi(q)} \bar\eps^{1/2} \eps_0^{I-1} 
    \vee m \sigma_1^2 
    \vee \norm{\a}_1 \eps_0^I 
    \le \frac{ \hat\sigma_{2I}^2 }{2^{3I+4} C_\sigma^2 } 
      \frac{ a_{\pi(q)} }{ d^{I-1/2} }
      \frac{\delta_r}{24} \\
    &\quad\Leftarrow\quad 
    \bar\eps 
    \le \left( 
      \frac{ \hat\sigma_{2I}^2 }{2^{3I+4} C_\sigma^2 } 
      \frac{\delta_r}{24}
    \right)^2 
    \frac{1}{d^{1+2\gamma(I-1)}}, \quad 
    m \sigma_1^2 
    \le \frac{ \hat\sigma_{2I}^2 }{2^{3I+4} C_\sigma^2 } 
      \frac{ a_{\pi(q)} }{ d^{I-1/2} }
      \frac{\delta_r}{24}, \\
    &\quad\quad\qquad 
    d 
    \ge \left( 
        \frac{ \hat\sigma_{2I}^2 }{2^{3I+4} C_\sigma^2 } 
        \frac{ a_{\pi(q)} }{ \norm{\a}_1 }
        \frac{\delta_r}{24}
      \right)^{
        - \frac{1}{1/2 - \gamma I}
      } .
  \end{align*}
  The above conditions ensure $\delta_r/4 \ge 2 \delta_{\Tmp} + 20 \delta_T$. 
  By Assumption~\ref{assumption: gf init}\ref{assumption-itm: init: regularity conditions}, $\bar{v}_{p, \pi(p)}^2(0)
  \le \log^2 d / d$. Hence, in order for $\bar{v}_{p, \pi(q)}^2(0) (\delta_r/4)^{-1/(I-1)}$ to be smaller than $\eps_0$, 
  it suffices to have 
  \[
    \frac{\log^2 d}{d} \left( \frac{\delta_r}{4} \right)^{-\frac{1}{I-1}}
    \le d^{-(1-\gamma)}
    \quad\Leftarrow\quad 
    \frac{d^{\gamma}}{\log^2 d} 
    \ge \left( \frac{\delta_r}{4} \right)^{-\frac{1}{I-1}}.
  \]
  We now clean up the conditions required by this lemma, which are the conditions of Corollary~\ref{cor: convergence of one 
  direction} and 
  \begin{gather*} 
    \bar\eps 
    \le \left( 
      \frac{ \hat\sigma_{2I}^2 }{2^{3I+4} C_\sigma^2 } 
      \frac{\delta_r}{24}
    \right)^2 
    \frac{1}{d^{1+2\gamma(I-1)}}, \quad 
    m \sigma_1^2 
    \le \frac{ \hat\sigma_{2I}^2 }{2^{3I+4} C_\sigma^2 } 
      \frac{ a_{\pi(q)} }{ d^{I-1/2} }
      \frac{\delta_r}{24},
    \\
    \frac{d}{(\log^2 d)^{1/\gamma}} 
    \ge \left( \frac{\delta_r}{4} \right)^{-\frac{1}{\gamma(I-1)}}, \;
    d 
    \ge \left( 
        \frac{ \hat\sigma_{2I}^2 }{2^{3I+4} C_\sigma^2 } 
        \frac{ a_{\pi(q)} }{ \norm{\a}_1 }
        \frac{\delta_r}{24}
      \right)^{
        - \frac{1}{1/2 - \gamma I}
      }
      \vee 
      \left( \frac{ I \hat\sigma_{2I}^2 }{ C_\sigma^2  } \frac{\delta_r}{12} \right)^{-\frac{1}{1 - \gamma}}, \; 
    \delta_T \le \frac{\delta_r}{240}.
  \end{gather*}
  For the condition on $d$, since $1/2 - \gamma I \le 1/2 \le 1 - \gamma$, the first part of it is stronger. 
  Finally, we use the hypothesis $a_{\pi(q)} \ge a_{\min_*} / (2 (\log d)^{2I-2})$ to replace (the first part of) 
  the second condition with 
  \[
    d 
    \ge \left( 
        \frac{ \hat\sigma_{2I}^2 }{2^{3I+4} C_\sigma^2 } 
        \frac{ a_{\min_*} }{ \norm{\a}_1 2^{2I-2}  }
        \frac{\delta_r}{24}
      \right)^{
        - \frac{1}{1/2 - \gamma I}
      }
      (\log^2 d)^{\frac{I-1}{1/2 - \gamma I}}.
  \]
\end{proof}

\begin{lemma}[Upper triangular entries (case II)]
  \label{lemma: tangent: upper triangular entries (case II)}
  Consider $p \in [P_*]$ and $p < q \in [P]$ with $a_{\pi(q)} \le a_{\min_*} / (2 \log^{2I-2} d)$. 
  Suppose that the hypotheses of Lemma~\ref{lemma: tangent: upper triangular entries (case I)} are true. 
  Then, $\bar{v}_{p, \pi(q)}^2 \le \eps_0$ throughout training. 
\end{lemma}
\begin{proof}
  By the proof of Lemma~\ref{lemma: tangent: upper triangular entries (case I)}, we have 
  \begin{align*}
    \frac{\rd}{\rd t} \bar{v}_{p, \pi(q)}^2
    &\le 4I \hat\sigma_{2I}^2  a_{\pi(q)} \bar{v}_{p, \pi(q)}^{2I}
      + 2 C_\sigma^2 a_{\pi(q)} \bar{v}_{p, \pi(q)}^{2I} \eps_0 \\
      &\qquad
      + I 2^{3I+6} C_\sigma^2 d^{I-1/2}
      \braces{
        a_{\pi(q)} \bar\eps^{1/2} \eps_0^{I-1} 
        \vee m \sigma_1^2
        \vee \norm{\a}_1  \eps_0^I 
      }
      \bar{v}_{p, \pi(q)}^{2I}.
  \end{align*}
  Suppose that $a_{\pi(q)} \le a_{\min_*} / M$ for some $M \ge 1$ to be determined later. Then, we have 
  \begin{align*}
    \frac{\rd}{\rd t} \bar{v}_{p, \pi(q)}^2
    &\le 4I \hat\sigma_{2I}^2  \frac{a_{\min_*}}{M} \bar{v}_{p, \pi(q)}^{2I}
      + 2 C_\sigma^2 \frac{a_{\min_*}}{M} \bar{v}_{p, \pi(q)}^{2I} \eps_0 \\
      &\qquad
      + I 2^{3I+6} C_\sigma^2 d^{I-1/2}
      \braces{
        \frac{a_{\min_*}}{M} \bar\eps^{1/2} \eps_0^{I-1} 
        \vee m \sigma_1^2
        \vee \norm{\a}_1  \eps_0^I 
      }
      \bar{v}_{p, \pi(q)}^{2I} \\
    &\le \left( 1 + \delta_{\Tmp} \right)
      4I \hat\sigma_{2I}^2 \frac{a_{\min_*}}{M} \bar{v}_{p, \pi(q)}^{2I}, 
  \end{align*}
  where 
  \begin{align*}
    \delta_{\Tmp}
    &= \frac{ C_\sigma^2 \eps_0 }{ 2  I \hat\sigma_{2I}^2 }
      + \frac{ 2^{3I+6} C_\sigma^2 }{ 4 \hat\sigma_{2I}^2 }
        d^{I-1/2}
        \frac{M}{a_{\min_*}}
        \braces{
          \frac{a_{\min_*}}{M} \bar\eps^{1/2} \eps_0^{I-1} 
          \vee  m \sigma_1^2
          \vee \eps_0^I 
        } \\
    &=: \delta_{\Tmp, 1} + \delta_{\Tmp, 2}.
  \end{align*}
  As a result, for any $t \le T_p$, we have 
  \[
    \bar{v}_{p, \pi(q)}^2(t) 
    \le \bar{v}_{p, \pi(q)}^2(0) 
      \frac{a_{\min_*}}{M}
      \left(
        1 
        -  
        \left(  1 +  \delta_{\Tmp} \right) \left( 1 + 20 \delta_T \right)
        \frac{ a_{\min_*} \bar{v}_{p, \pi(q)}^{2I-2}(0)  }{M a_{\pi(p)} \bar{v}_{p, \pi(p)}^{2I-2}(0) } 
      \right)^{-\frac{1}{I-1}}. 
  \]
  Recall from Assumption~\ref{assumption: gf init} that $\bar{v}_{p, \pi(p)}^2(0) \ge 1/d$
  and $\bar{v}_{p, \pi(q)}^2(0) \le \log^2 d / d$. Hence, with $M = 2 \log^{2I-2} d$, we have 
  \[
    \frac{ a_{\min_*} \bar{v}_{p, \pi(q)}^{2I-2}(0)  }{M a_{\pi(p)} \bar{v}_{p, \pi(p)}^{2I-2}(0) } 
    \le \frac{ a_{\min_*}  }{a_{\pi(p)}  } \frac{\log^{2I-2} d}{M}
    \le \frac{1}{2}. 
  \]
  Hence, 
  \[
    \bar{v}_{p, \pi(q)}^2(t) 
    \le \bar{v}_{p, \pi(q)}^2(0) 
      \left(
        1 -  \frac{\left(  1 +  \delta_{\Tmp} \right) \left( 1 + 10 \delta_T \right)}{2}
      \right)^{-\frac{1}{I-1}}. 
  \]
  As a result, to ensure $\bar{v}_{p, \pi(q)}^2 \le \eps_0$ throughout training, it suffices to have 
  $\delta_{\Tmp} \le 0.1$ and $\delta_T \le 0.01$. The second condition clear holds under the hypotheses of 
  Lemma~\ref{lemma: tangent: upper triangular entries (case I)}. For the same reason, we have $\delta_{\Tmp, 1}
  \le 0.05$ and the first term in $\delta_{\Tmp, 2}$ will also be sufficiently small. Finally, we compute 
  \begin{multline*}
    d^{I-1/2}
    \braces{
      \frac{M}{a_{\min_*}} m \sigma_1^2
      \vee \frac{M \norm{\a}_1}{a_{\min_*}}  \eps_0^I 
    }
    \le \frac{1}{20} \frac{ 4 \hat\sigma_{2I}^2 }{ 2^{3I+6} C_\sigma^2 } \\
    \Leftarrow\quad 
    m \sigma_1^2
    \le \frac{1}{20} \frac{ 4 \hat\sigma_{2I}^2 }{ 2^{3I+6} C_\sigma^2 }
      \frac{a_{\min_*}}{2 \log^{2I-2} d} \frac{1}{d^{I-1/2}}, \quad 
    \frac{d}{\log^{\frac{4I}{1 - 2 \gamma I}} d}  
    \ge \left(
        \frac{1}{40} \frac{ 4 \hat\sigma_{2I}^2 }{ 2^{3I+6} C_\sigma^2 }
        \frac{a_{\min_*}}{ \norm{\a}_1}
      \right)^{-\frac{2}{ 1 - 2 \gamma I }},
  \end{multline*}
  which are also covered by the conditions of Lemma~\ref{lemma: tangent: upper triangular entries (case I)}. 
  In fact, $M$ is chosen to balance the requirements of these two lemmas. 
\end{proof}

\begin{lemma}[Lower triangular entries]
  \label{lemma: tangent: lower triangular entries}
  Consider $p \in [P_*]$ and $p < k \in [m]$. Assume the conditions of Corollary~\ref{cor: convergence of one direction}
  and 
  \begin{gather*}
    \delta_T \le \frac{\delta_c}{240}, \quad 
    \eps_R 
    \le \frac{1}{6} \frac{ a_{\min_*}^2 \delta_c}{8 (\log^2 d)^{I-1}}, \quad 
    \bar\eps
    \le 
      \left(  
        \frac{1}{48} \frac{ 4 \hat\sigma_{2I}^2   }{ 2^{3I+6} C_\sigma^2  }
      \right)^2 
      \frac{ a_{\min_*}^2 \delta_c^2}{(\log^2 d)^{2I-2}}
      \frac{1}{d^{1+2\gamma(I-1)}}, \\
    m \sigma_1^2 
    \le \frac{1}{48} \frac{ \hat\sigma_{2I}^2   }{ 2^{3I+4} C_\sigma^2  }
      \frac{ a_{\min_*}^2 \delta_c}{(\log^2 d)^{I-1}}
      \frac{1}{d^{I-1/2}}, \quad 
    \frac{ d }{ (\log^2 d)^{\frac{I-1}{1/2-\gamma I}} } 
    \ge \left( 
        \frac{1}{6} \frac{ 4 \hat\sigma_{2I}^2   }{ 2^{3I+6} C_\sigma^2  }
        \frac{ a_{\min_*}^2 \delta_c}{8 \norm{\a}_1}
      \right)^{-\frac{1}{1/2 - \gamma I}}. 
  \end{gather*}
  Then, we have $\bar{v}_{k, \pi(p)}^2 \le \eps_0$ throughout training.
\end{lemma}
\begin{proof}
  \def\currentprefix{proof: lower triangular entries. aadsfl}
  First, by Lemma~\ref{lemma: gf: tangent dynamics (crude)}, we have 
  \begin{align*}
    \frac{\rd}{\rd t} \bar{v}_{k, \pi(p)}^2
    &= 2 \bar{v}_{k, \pi(p)}^2
      \sum_{i=I}^\infty 2i \hat\sigma_{2i}^2 \left( 
        a_{\pi(p)} \bar{v}_{k, \pi(p)}^{2i-2} 
        - \sum_{r=1}^P a_{\pi(r)} \bar{v}_{k, \pi(r)}^{2i} 
      \right)  \\
      &\quad
      - \indi\braces{p \in L}
      2 \norm{\v_p}^2 \left( 1 - \bar{v}_{k, \pi(p)}^2 \right)
      \sum_{i=I}^\infty 2 i \hat\sigma_{2i}^2 \bar{v}_{k, \pi(p)}^{2i} \\
      &\quad
      \pm I 2^{3I+6} C_\sigma^2 \abs{\bar{v}_{k, \pi(p)} } 
      \braces{
        a_{\pi(p)} \bar\eps^{1/2} \eps_0^{I-1} 
        \vee m \sigma_1^2
        \vee \norm{\a}_1 \eps_0^I 
      } \\
    &\le 2 \left( 1 - \bar{v}_{k, \pi(p)}^2  \right)  
      \sum_{i=I}^\infty 2i \hat\sigma_{2i}^2  
        \left( a_{\pi(p)} - \indi\braces{p \in L} \norm{\v_p}^2  \right)
        \bar{v}_{k, \pi(p)}^{2i} 
      \\
      &\quad
      + I 2^{3I+6} C_\sigma^2 \abs{\bar{v}_{k, \pi(p)} } 
      \braces{
        a_{\pi(p)} \bar\eps^{1/2} \eps_0^{I-1} 
        \vee m \sigma_1^2
        \vee \norm{\a}_1 \eps_0^I 
      } \\
    &=: \Term_1\left( \frac{\rd}{\rd t} \bar{v}_{k, \pi(p)}^2 \right)
      + \Term_2\left( \frac{\rd}{\rd t} \bar{v}_{k, \pi(p)}^2 \right).
  \end{align*}
  Similar to the proof of Lemma~\ref{lemma: tangent: upper triangular entries (case I)}, we assume w.l.o.g.~that 
  $\bar{v}_{k, \pi(p)}^2 \ge 1/d$ and write 
  \[
    \Term_2 
    \le I 2^{3I+6} C_\sigma^2 
      \bar{v}_{k, \pi(p)}^{2I}
      d^{I-1/2}
      \braces{
        a_{\pi(p)} \bar\eps^{1/2} \eps_0^{I-1} 
        \vee m \sigma_1^2
        \vee \norm{\a}_1 \eps_0^I 
      }.
  \]
  For the first term, we have 
  \begin{align*}
    \Term_1 
    &\le 4 I \hat\sigma_{2I}^2  
        \abs{ a_{\pi(p)} - \indi\braces{p \in L} \norm{\v_p}^2  }
        \bar{v}_{k, \pi(p)}^{2I}
      + 2 \sum_{i=I+1}^\infty 2i \hat\sigma_{2i}^2  a_{\pi(p)} \bar{v}_{k, \pi(p)}^{2i} \\
    &\le 4 I \hat\sigma_{2I}^2  
      \abs{ a_{\pi(p)} - \indi\braces{p \in L} \norm{\v_p}^2  }
      \bar{v}_{k, \pi(p)}^{2I}
      + 2 C_\sigma^2 a_{\pi(p)}  \bar{v}_{k, \pi(p)}^{2I} \eps_0 \\
    &\le 
      \left(
        \abs{ 1 - \frac{\indi\braces{p \in L} \norm{\v_p}^2}{a_{\pi(p)}}  } 
        + \frac{C_\sigma^2 \eps_0}{2 I \hat\sigma_{2I}^2  }
      \right)
      \times 
      4 I \hat\sigma_{2I}^2  a_{\pi(p)} \bar{v}_{k, \pi(p)}^{2I}.
  \end{align*}
  Therefore, 
  \[
    \frac{\rd}{\rd t} \bar{v}_{k, \pi(p)}^2
    \le \left(
        \abs{ 1 - \frac{\indi\braces{p \in L} \norm{\v_p}^2}{a_{\pi(p)}}  } 
        + \delta_{\Tmp}
      \right)
      \times 
      4 I \hat\sigma_{2I}^2  a_{\pi(p)} \bar{v}_{k, \pi(p)}^{2I},
  \]
  where 
  \begin{align*}
    \delta_{\Tmp}
    &:= \frac{C_\sigma^2 \eps_0}{2 I \hat\sigma_{2I}^2  }
      + \frac{ 2^{3I+6} C_\sigma^2  }{ 4 \hat\sigma_{2I}^2   }
        d^{I-1/2}
        \braces{
          \bar\eps^{1/2} \eps_0^{I-1} 
          \vee \frac{m \sigma_1^2}{a_{\min_*}}
          \vee \frac{\norm{\a}_1}{a_{\min_*}} \eps_0^I 
        }\\
    &=: \delta_{\Tmp, 1} + \delta_{\Tmp, 2}.
  \end{align*}
  By Corollary~\ref{cor: convergence of one direction}, we know $p \in L$ and $\norm{\v_p}^2 = a_{\pi(p)} \pm \eps_R$ 
  for $\eps_R$ satisfying the condition in Corollary~\ref{cor: convergence of one direction} after time 
  \[
    T_p := \frac{1 \pm 20 \delta_T }{ 4I (I - 1) \hat\sigma_{2I}^2 a_{\pi(p)} \bar{v}_{p, \pi(p)}^{2I-2}(0) }.
  \]
  
  We now analyze the stages $[0, T_p]$ and $[T_p, T_{P_*}]$, separately. Let $\eps_0' \le \eps_0$ be a parameter to be 
  chosen later. We want to show that $\bar{v}_{k, \pi(p)}^2$ is upper bounded by $\eps_0'$ in the first stage 
  and by $\eps_0$ in the second stage. 

  First, for $t \le T_p$, we have $\frac{\rd}{\rd t} \bar{v}_{k, \pi(p)}^2
  \le \left( 1 + \delta_{\Tmp} \right)
    \times 
    4 I \hat\sigma_{2I}^2  a_{\pi(p)} \bar{v}_{k, \pi(p)}^{2I}$ and therefore
  \begin{align*}
    \bar{v}_{k, \pi(p)}^2(t)
    &\le \bar{v}_{k, \pi(p)}^2(0)
      \left(
        1 - (I - 1) (1 + \delta_{\Tmp}) 4 I \hat\sigma_{2I}^2  a_{\pi(p)} \bar{v}_{k, \pi(p)}^{2I-2}(0) t
      \right)^{-\frac{1}{I-1}} \\
    &\le \bar{v}_{k, \pi(p)}^2(0)
      \left(
        1 - (1 + \delta_{\Tmp}) (1 + 20 \delta_T) 
        \frac{a_{\pi(p)} \bar{v}_{k, \pi(p)}^{2I-2}(0) }{ a_{\pi(p)} \bar{v}_{p, \pi(p)}^{2I-2}(0) }
      \right)^{-\frac{1}{I-1}} \\
    &\le \bar{v}_{k, \pi(p)}^2(0)
      \left( \frac{\delta_c}{2} -  2 \delta_{\Tmp, 1} - 2 \delta_{\Tmp, 2} - 20 \delta_T \right)^{-\frac{1}{I-1}}, 
  \end{align*}
  where the last line comes from Assumption~\ref{assumption: gf init}\ref{assumption-itm: init: col gap}.
  By the proof of Lemma~\ref{lemma: tangent: upper triangular entries (case I)}, we have 
  \[
    2 \delta_{\Tmp, 1} + 2 \delta_{\Tmp, 2} + 20 \delta_T
    \le \frac{\delta_c}{4},
  \]
  provided that
  \begin{equation}
    \locallabel{eq: conditions for stage 1}
    \begin{gathered}
      \delta_T \le \frac{\delta_c}{240}, \quad 
      d 
      \ge \left( 
          \frac{ \hat\sigma_{2I}^2 }{2^{3I+4} C_\sigma^2 } 
          \frac{ a_{\min_*} }{ \norm{\a}_1 }
          \frac{\delta_c}{24}
        \right)^{
          - \frac{1}{1/2 - \gamma I}
        } , \\
      \bar\eps 
      \le \left( 
        \frac{ \hat\sigma_{2I}^2 }{2^{3I+4} C_\sigma^2 } 
        \frac{\delta_c}{24}
      \right)^2 
      \frac{1}{d^{1+2\gamma(I-1)}}, \quad 
      m \sigma_1^2 
      \le \frac{ \hat\sigma_{2I}^2 }{2^{3I+4} C_\sigma^2 } 
        \frac{ a_{\min_*} }{ d^{I-1/2} }
        \frac{\delta_c}{24}.
    \end{gathered}
  \end{equation} 
  Then, we compute 
  \begin{equation}
    \locallabel{eq: eps0' >=}
    \bar{v}_{k, \pi(p)}^2(t)
    \le \eps_0'
    \quad\Leftarrow\quad 
    \eps_0'
    \ge \frac{\log^2 d}{d} \left( \frac{\delta_c}{4} \right)^{-\frac{1}{I-1}}
    \quad\Leftarrow\quad 
    \eps_0'
    = \frac{\log^2 d}{d} \left( \frac{\delta_c}{4} \right)^{-\frac{1}{I-1}}.
  \end{equation}
  
  Now, consider the second stage. For $t \ge T_p$, we have 
  \begin{multline*}
    \frac{\rd}{\rd t} \bar{v}_{k, \pi(p)}^2
    \le \left(
        \frac{ \eps_R }{a_{\pi(p)}}  
        + \delta_{\Tmp}
      \right)
      \times 
      4 I \hat\sigma_{2I}^2  a_{\pi(p)} \bar{v}_{k, \pi(p)}^{2I} \\
    \Rightarrow\quad 
    \bar{v}_{k, \pi(p)}^2(t)
    \le \eps_0'
      \left(
        1 
        - 
        \left(
          \frac{ \eps_R }{a_{\pi(p)}}  
          + \delta_{\Tmp}
        \right) 4 I (I - 1) \hat\sigma_{2I}^2  a_{\pi(p)} (\eps_0')^{I-1}
        (t - T_p)
      \right)^{-\frac{1}{I-1}} . 
  \end{multline*}
  Also, recall that the training process ends before time 
  \[
    T_{P_*}
    = \frac{1 \pm 20 \delta_T }{ 4I (I - 1) \hat\sigma_{2I}^2 a_{\pi(P_*)} \bar{v}_{P_*, \pi(P_*)}^{2I-2}(0) }.
  \]
  For any $t \in [T_p, T_{P_*}]$, we have 
  \begin{align*}
    \bar{v}_{k, \pi(p)}^2 
    &\le \eps_0'
      \left(
        1 
        - 
        \left(
          \frac{ \eps_R }{a_{\pi(p)}}  
          + \delta_{\Tmp}
        \right) 
        \left( 1 + 20 \delta_T \right)
        \frac{a_{\pi(p)} (\eps_0')^{I-1}}{   a_{\pi(P_*)} \bar{v}_{P_*, \pi(P_*)}^{2I-2}(0) }
      \right)^{-\frac{1}{I-1}} \\
    &\le \eps_0'
      \left(
        1 
        - 
        \left(
          \frac{ \eps_R }{ a_{\min_*} }  
          + \delta_{\Tmp}
        \right) 
        \frac{2 (d \eps_0')^{I-1}}{ a_{\min_*} }
      \right)^{-\frac{1}{I-1}} \\
    &\le \eps_0'
      \left(
        1 
        - 
        \left(
          \frac{ \eps_R }{ a_{\min_*} }  
          + \delta_{\Tmp}
        \right)  
        \frac{8 (\log^2 d)^{I-1}}{ a_{\min_*} \delta_c}
      \right)^{-\frac{1}{I-1}}, 
  \end{align*}
  where the last line comes form choosing (\localref{eq: eps0' >=}).
  For the last term to be bounded by $\eps_0$, it suffices to require
  \[
    \left(
      \frac{ \eps_R }{ a_{\min_*} }  
      + \delta_{\Tmp}
    \right)  
    \frac{8 (\log^2 d)^{I-1}}{ a_{\min_*} \delta_c}
    \le \frac{1}{2}
    \quad\Leftarrow\quad
    \frac{ \eps_R }{ a_{\min_*} }  
    + \delta_{\Tmp, 1}
    + \delta_{\Tmp, 2}
    \le \frac{1}{2} \frac{ a_{\min_*} \delta_c}{8 (\log^2 d)^{I-1}}, 
  \]
  which is implied by 
  \[
    \eps_R 
    \le \frac{1}{6} \frac{ a_{\min_*}^2 \delta_c}{8 (\log^2 d)^{I-1}}, \quad 
    d 
    \ge \left(
      \frac{1}{6} \frac{2 I \hat\sigma_{2I}^2  }{C_\sigma^2 } \frac{ a_{\min_*} \delta_c}{8 (\log^2 d)^{I-1}}
    \right)^{-\frac{1}{1-\gamma}},
  \]
  and by \eqref{eq: tangent: error <= delta}, 
  \begin{gather*}
    m \sigma_1^2 
    \le \frac{1}{6} \frac{ 4 \hat\sigma_{2I}^2   }{ 2^{3I+6} C_\sigma^2  }
      \frac{ a_{\min_*}^2 \delta_c}{8 (\log^2 d)^{I-1}}
      \frac{1}{d^{I-1/2}}, \quad 
    \bar\eps 
    \le 
      \left(  
        \frac{1}{6} \frac{ 4 \hat\sigma_{2I}^2   }{ 2^{3I+6} C_\sigma^2  }
        \frac{ a_{\min_*} \delta_c}{8 (\log^2 d)^{I-1}}
      \right)^2 \frac{1}{d^{1+2\gamma(I-1)}}, \\
    \frac{ d }{ (\log^2 d)^{\frac{I-1}{1/2-\gamma I}} } 
    \ge \left( 
        \frac{1}{6} \frac{ 4 \hat\sigma_{2I}^2   }{ 2^{3I+6} C_\sigma^2  }
        \frac{ a_{\min_*}^2 \delta_c}{8 \norm{\a}_1}
      \right)^{-\frac{1}{1/2 - \gamma I}} 
  \end{gather*}
  Combining the above conditions with (\localref{eq: conditions for stage 1}), we conclude that 
  $\bar{v}_{k, \pi(p)}^2 \le \eps_0$ throughout training, as long as 
  the conditions of Corollary~\ref{cor: convergence of one direction} and the following conditions are true:
  \begin{gather*}
    \delta_T \le \frac{\delta_c}{240}, \quad 
    \eps_R 
    \le \frac{1}{6} \frac{ a_{\min_*}^2 \delta_c}{8 (\log^2 d)^{I-1}}, \\
    m \sigma_1^2 
    \le \frac{ \hat\sigma_{2I}^2 }{2^{3I+4} C_\sigma^2 } 
      \frac{ a_{\min_*} }{ d^{I-1/2} }
      \frac{\delta_c}{24} 
      \wedge 
      \frac{1}{6} \frac{ 4 \hat\sigma_{2I}^2   }{ 2^{3I+6} C_\sigma^2  }
      \frac{ a_{\min_*}^2 \delta_c}{8 (\log^2 d)^{I-1}}
      \frac{1}{d^{I-1/2}}, \\
    \bar\eps
    \le \left( 
        \frac{ \hat\sigma_{2I}^2 }{2^{3I+4} C_\sigma^2 } 
        \frac{\delta_c}{24}
      \right)^2 
      \frac{1}{d^{1+2\gamma(I-1)}}
      \wedge
      \left(  
        \frac{1}{6} \frac{ 4 \hat\sigma_{2I}^2   }{ 2^{3I+6} C_\sigma^2  }
        \frac{ a_{\min_*} \delta_c}{8 (\log^2 d)^{I-1}}
      \right)^2 \frac{1}{d^{1+2\gamma(I-1)}}, \\
    d 
    \ge \left( 
        \frac{ \hat\sigma_{2I}^2 }{2^{3I+4} C_\sigma^2 } 
        \frac{ a_{\min_*} }{ \norm{\a}_1 }
        \frac{\delta_c}{24}
      \right)^{
        - \frac{1}{1/2 - \gamma I}
      } 
      \vee 
      \left(
        \frac{1}{6} \frac{2 I \hat\sigma_{2I}^2  }{C_\sigma^2 } \frac{ a_{\min_*} \delta_c}{8 (\log^2 d)^{I-1}}
      \right)^{-\frac{1}{1-\gamma}} , \\
    \frac{ d }{ (\log^2 d)^{\frac{I-1}{1/2-\gamma I}} } 
    \ge \left( 
        \frac{1}{6} \frac{ 4 \hat\sigma_{2I}^2   }{ 2^{3I+6} C_\sigma^2  }
        \frac{ a_{\min_*}^2 \delta_c}{8 \norm{\a}_1}
      \right)^{-\frac{1}{1/2 - \gamma I}}. 
  \end{gather*}
  To complete the proof, it suffices to keep only the stronger one in each of the conditions on $m \sigma_1^2$,
  $\bar\eps$, and $d$. 
\end{proof}

\begin{lemma}[Lower right block]
  \label{lemma: tangent: lower right block}
  Consider $k \in [m], q \in [P]$ with $k, q > P_*$. Assume the conditions of Corollary~\ref{cor: convergence of 
  one direction} and the conditions of Lemma~\ref{lemma: tangent: upper triangular entries (case I)}, with 
  $\delta_r$ replaced by $\delta_t$. Then, we have $\bar{v}_{k, \pi(q)}^2 \le \eps_0$ throughout training.
\end{lemma}
\begin{proof}
  By Lemma~\ref{lemma: gf: tangent dynamics (crude)}, we have 
  \begin{align*}
    \frac{\rd}{\rd t} \bar{v}_{k, \pi(q)}^2
    &\le 2 \sum_{i=I}^\infty 2i \hat\sigma_{2i}^2 a_{\pi(q)} \bar{v}_{k, \pi(q)}^{2i} \\
      &\quad
      \pm I 2^{3I+6} C_\sigma^2 \abs{\bar{v}_{k, \pi(q)} } 
      \braces{
        a_{\pi(q)} \bar\eps^{1/2} \eps_0^{I-1} 
        \vee m \sigma_1^2
        \vee \norm{\a}_1 \eps_0^I 
      } \\
    &=: \Term_1\left( \frac{\rd}{\rd t} \bar{v}_{k, \pi(q)}^2 \right)
      + \Term_2\left( \frac{\rd}{\rd t} \bar{v}_{k, \pi(q)}^2 \right).
  \end{align*}
  For the first term, we have 
  \begin{align*}
    \Term_1 
    &= 4 I \hat\sigma_{2I}^2 a_{\pi(q)} \bar{v}_{k, \pi(q)}^{2I}
      + 2 \sum_{i=I+1}^\infty 2i \hat\sigma_{2i}^2 a_{\pi(q)} \bar{v}_{k, \pi(q)}^{2i} \\
    &\le 4 I \hat\sigma_{2I}^2 a_{\pi(q)} \bar{v}_{k, \pi(q)}^{2I}
      + 2 C_\sigma^2 a_{\pi(q)} \eps_0 \bar{v}_{k, \pi(q)}^{2I}  \\
    &= \left( 1 + \frac{ C_\sigma^2 \eps_0 }{ 2 I \hat\sigma_{2I}^2 } \right)
    \times 4 I \hat\sigma_{2I}^2 a_{\pi(q)} \bar{v}_{k, \pi(q)}^{2I}.
  \end{align*}
  Similar to the previous proofs, we may assume w.l.o.g.~that $\bar{v}_{k, \pi(q)}^2 \ge 1/d$. Then, for the second 
  term, we have 
  \begin{align*}
    \Term_2 
    &\le I 2^{3I+6} C_\sigma^2 
      d^{I+1/2}
      \braces{
        a_{\pi(q)} \bar\eps^{1/2} \eps_0^{I-1} 
        \vee m \sigma_1^2
        \vee \norm{\a}_1 \eps_0^I 
      }
      \bar{v}_{k, \pi(q)}^{2I} \\
    &= \frac{ 2^{3I+6} C_\sigma^2 }{ 4 \hat\sigma_{2I}^2 }
      d^{I+1/2}
      \braces{
        \bar\eps^{1/2} \eps_0^{I-1} 
        \vee \frac{m \sigma_1^2}{a_{\pi(q)} }
        \vee \frac{\norm{\a}_1}{a_{\pi(q)}} \eps_0^I 
      }
      \times 4 I \hat\sigma_{2I}^2 a_{\pi(q)} \bar{v}_{k, \pi(q)}^{2I}. 
  \end{align*}
  As a result, we have 
  \begin{multline*}
    \frac{\rd}{\rd t} \bar{v}_{k, \pi(q)}^2
    \le \left(
        1 
        + \frac{ C_\sigma^2 \eps_0 }{ 2 I \hat\sigma_{2I}^2 }
        + \frac{ 2^{3I+6} C_\sigma^2 }{ 4 \hat\sigma_{2I}^2 }
          d^{I+1/2}
          \braces{
            \bar\eps^{1/2} \eps_0^{I-1} 
            \vee \frac{m \sigma_1^2}{a_{\pi(q)} }
            \vee \frac{\norm{\a}_1}{a_{\pi(q)}} \eps_0^I 
          }
      \right) \\
      \times 4 I \hat\sigma_{2I}^2 a_{\pi(q)} \bar{v}_{k, \pi(q)}^{2I}. 
  \end{multline*}
  Note that this is the same as the bound in the proof of Lemma~\ref{lemma: tangent: upper triangular entries (case I)}
  and Lemma~\ref{lemma: tangent: upper triangular entries (case II)}. Thus, to achieve $\bar{v}_{k, \pi(q)}^2  
  \le \eps_0$, it suffices to require the same conditions as in those two lemmas, with $\delta_r$ replaced by 
  $\delta_t$ (cf.~Assumption~\ref{assumption: gf init}). 
\end{proof}

\subsubsection{Upper Bound on the Norm Growth}

Here, we verify \IHGF\ref{inductH-itm: large norm => converged}. 

\begin{lemma}[Upper bound on unused neurons]
  \label{lemma: radial: unused neurons}
  Consider $k \in [m]$ with $k > P_*$. Suppose that 
  \[
    \gamma < \frac{1}{I}, \quad 
    d 
    \ge \left( \frac{a_{\min_*}}{\norm{\a}_1} \frac{ I (I - 1) \hat\sigma_{2I}^2 }{2 } \right)^{-\frac{1}{1 - \gamma I}} . 
  \]
  Then, we have $\norm{\v_k}^2 \le e \sigma_0^2$ throughout training.  
\end{lemma}
\begin{proof}
  First, by Lemma~\ref{lemma: population and per-sample gradients}, \IHGF\ref{inductH-itm: bound on the failed 
  coordinates} and Assumption~\ref{assumption: link function}, we have 
  \[
    \frac{\rd}{\rd t} \norm{\v_k}^2
    \le 4 \norm{\v_k}^2 \sum_{i=I}^\infty \hat\sigma_{2i}^2 \sum_{p=1}^P a_p \bar{v}_{k, p}^{2i}
    \le 4 \norm{\v_k}^2 \sum_{i=I}^\infty \hat\sigma_{2i}^2 \sum_{p=1}^P a_p \eps_0^i  
    \le 4 \norm{\a}_1 \eps_0^I \norm{\v_k}^2.
  \]
  Thus, by Gronwall's lemma, 
  we have $\norm{\v_k(t)}^2 \le \sigma_0^2 \exp\left( 4 \norm{\a}_1 \eps_0^I t \right) \le e \sigma_0^2$
  as long as $t \le ( 4 \norm{\a}_1 \eps_0^I )\inv$. By Lemma~\ref{lemma: directional convergence}
  and Lemma~\ref{lemma: fitting the signal}, the training process ends at time 
  \[
    T_{P_*} 
    \le \frac{2}{ 4I (I - 1) \hat\sigma_{2I}^2 a_{\pi(P_*)} \bar{v}_{P_*, \pi(P_*)}^{2I-2}(0) } 
    \le \frac{d^{I-1}}{ 2 I (I - 1) \hat\sigma_{2I}^2 a_{\min_*}} .
  \]
  Hence, it suffices to require
  \begin{align*}
    \frac{1}{ 4 \norm{\a}_1 \eps_0^I }
    \ge \frac{d^{I-1}}{ 2 I (I - 1) \hat\sigma_{2I}^2 a_{\min_*}}
    &\quad\Leftarrow\quad 
    d^{\gamma I - 1}
    \le \frac{a_{\min_*}}{\norm{\a}_1} \frac{ I (I - 1) \hat\sigma_{2I}^2 }{2 }   \\
    &\quad\Leftarrow\quad 
    \gamma < \frac{1}{I}, \quad 
    d 
    \ge \left( \frac{a_{\min_*}}{\norm{\a}_1} \frac{ I (I - 1) \hat\sigma_{2I}^2 }{2 } \right)^{-\frac{1}{1 - \gamma I}} . 
  \end{align*}
\end{proof}

Then, we consider $k = p \le P_*$. Unlike those unused neurons, since $\v_p$ will eventually converge to $\e_{\pi(p)}$, 
its norm cannot stay small. Our strategy here will be coupling its norm growth with the tangent movement. 

\begin{lemma}[Upper bound on $\norm{\v_p}^2$ with $p \le P_*$]
  \label{lemma: radial: upper bound on the norm}
  Consider $p \in [P_*]$. Suppose that the hypotheses of Lemma~\ref{lemma: radial: unused neurons} and 
  Lemma~\ref{lemma: directional convergence} hold. 
  Then, $\norm{\v_p}^2 \ge \sigma_1^2$ only if $\bar{v}_{p, \pi(p)}^2 \ge 1 - \bar\eps$, 
  where $\sigma_1^2 := 2 \sigma_0^2 e^{ 5 / \hat\sigma_{2I}^2 } \bar\eps^{- 8  / (I \hat\sigma_{2I}^2) }$.
\end{lemma}
\begin{proof}
  Again, by Lemma~\ref{lemma: population and per-sample gradients}, \IHGF\ref{inductH-itm: bound on the failed 
  coordinates} and Assumption~\ref{assumption: link function}, we have 
  \begin{align*}
    \frac{\rd}{\rd t} \norm{\v_p}^2
    \le 4 \norm{\v_p}^2 \sum_{i=I}^\infty \hat\sigma_{2i}^2 \sum_{q=1}^P a_{\pi(q)} \bar{v}_{p, \pi(q)}^{2i} 
    &\le 4 \norm{\v_p}^2 \sum_{i=I}^\infty \hat\sigma_{2i}^2 \left(
        a_{\pi(p)} \bar{v}_{p, \pi(p)}^{2i} + \norm{\a}_1 \eps_0^i
      \right)   \\
    &\le 4 \norm{\v_p}^2 a_{\pi(p)} \bar{v}_{p, \pi(p)}^{2I} 
      + 4 \norm{\v_p}^2 \norm{\a}_1 \eps_0^I .
  \end{align*}
  Hence, by Gronwall's lemma, we have 
  \[
    \norm{\v_p(t)}^2
    \le \sigma_0^2 \exp\left( 4 \norm{\a}_1 \eps_0^I t \right)
      \exp\left( 4 a_{\pi(p)} \int_0^t \bar{v}_{p, \pi(p)}^{2I}(s) \,\rd s \right) .
  \]
  Let $c_0 > 0$ be a small constant to be determined later and let $T_0$ be the time $\bar{v}_{p, \pi(p)}^2$ 
  reaches $1 - c_0/I$. By the proof of Lemma~\ref{lemma: tangent: dynamics of the diagonal entries (stage 1)}, we know 
  \[
    \frac{\rd}{\rd t} \bar{v}_{p, \pi(p)}^2 
    \ge \left( 1 - (1 - c_0/I) - o(1) \right) 4I \hat\sigma_{2I}^2 a_{\pi(p)} \bar{v}_{p, \pi(p)}^{2I}
    \ge c_0 2 \hat\sigma_{2I}^2 a_{\pi(p)} \bar{v}_{p, \pi(p)}^{2I}.
  \]
  Integrate both sides, and we obtain 
  \[
    1 
    \ge 1 - c_0/I - \bar{v}_{p, \pi(p)}^2(0)
    \ge c_0 2 \hat\sigma_{2I}^2 a_{\pi(p)} \int_0^{T_0} \bar{v}_{p, \pi(p)}^{2I}(s) \,\rd s.
  \]
  As a result, for $t \le T_0$, we have 
  \begin{align*}
    \norm{\v_p(t)}^2
    &\le \sigma_0^2 \exp\left( 4 \norm{\a}_1 \eps_0^I T_0 \right)
      \exp\left( \frac{4 a_{\pi(p)} }{c_0 2 \hat\sigma_{2I}^2 a_{\pi(p)}} \right) \\
    &\le \sigma_0^2 \exp\left( 4 \norm{\a}_1 \eps_0^I T_0 \right)
      \exp\left( \frac{2}{c_0 \hat\sigma_{2I}^2  } \right). 
  \end{align*}
  Clear that $T_0 \le T_{P_*}$ and under the conditions of Lemma~\ref{lemma: radial: unused neurons}, we have 
  $4 \norm{\a}_1 \eps_0^I T_{P_*} \le 1$. Therefore, 
  \[
    \norm{\v_p(t)}^2
    \le \sigma_0^2 \exp\left( 1 + \frac{2}{c_0 \hat\sigma_{2I}^2  } \right), \quad 
    \forall t \le T_0.
  \]
  Now, consider the $T_0 \le t \le T_1$, where $T_1$ is the time $\bar{v}_{p, \pi(p)}^2$ reaches $1 - \bar\eps$. 
  By Lemma~\ref{lemma: tangent: dynamics of the diagonal entries (stage 3)} (and the proof of Lemma~\ref{lemma: 
  directional convergence}), we know 
  \[
    \frac{\rd}{\rd t} \bar{v}_{p, \pi(p)}^2 
    \ge (1 - c_0) I \hat\sigma_{2I}^2 a_{\pi(p)} \left( 1 - \bar{v}_{p, \pi(p)}^2  \right)
    \quad\Rightarrow\quad 
    T_1 - T_0
    \le \frac{ \log\left( c_0 / \eps \right) }{ (1 - c_0) I \hat\sigma_{2I}^2 a_{\pi(p)} }.
  \]
  Thus, for $t \in [T_0, T_1]$, we have 
  \begin{align*}
    \norm{\v_p(t)}^2
    &\le \norm{\v_p(T_0)}^2
      \exp\left( 4 \norm{\a}_1 \eps_0^I (T_1 - T_0) \right)
      \exp\left( 4 a_{\pi(p)} (T_1 - T_0) \right) \\
    &\le \norm{\v_p(T_0)}^2
      \left( 1 + o(1) \right)
      \exp\left( 
        4 
        \frac{ \log\left( c_0 / \eps \right) }{ (1 - c_0) I \hat\sigma_{2I}^2 }
      \right) \\
    &\le \norm{\v_p(T_0)}^2 2 \left( \frac{c_0}{\bar\eps} \right)^{\frac{4}{ (1 - c_0) I \hat\sigma_{2I}^2 }}. 
  \end{align*}
  Choose $c_0 = 1/2$ and recall $\norm{\v_p(T_0)}^2 \le \sigma_0^2 \exp\left( 1 + \frac{2}{c_0 \hat\sigma_{2I}^2  } \right)$. 
  Then, we conclude that 
  \[
    \norm{\v_p(t)}^2 
    \le 2 \sigma_0^2 e^{ 5 / \hat\sigma_{2I}^2 } \bar\eps^{- 8  / (I \hat\sigma_{2I}^2) }
    =: \sigma_1^2,
  \]
  for all $t \le T_1$. Recall from Lemma~\ref{lemma: directional convergence} that once $\bar{v}_{p, \pi(p)}^2$
  reaches $1 - \bar\eps$, it will stay above $1 - \bar\eps$. Thus, this implies that $\norm{\v_p}^2 \ge \sigma_1^2$
  only if $\bar{v}_{p, \pi(p)}^2 \ge 1 - \bar\eps$.
\end{proof}

\subsection{Deferred Proofs}
\label{subsec: gf: deferred proofs}

\subsubsection{Proof of Lemma~\ref{lemma: gf: tangent dynamics (crude)}}

\begin{proof}[Proof of Lemma~\ref{lemma: gf: tangent dynamics (crude)}]
  Recall from Lemma~\ref{lemma: population and per-sample gradients} that 
  \begin{align*}
    - \frac{\left[ (\Id - \bar{\v}_k\bar{\v}_k\trans)\nabla_{\v_k} \Loss \right]_p}{\norm{\v_k}}
    &= \sum_{i=I}^\infty 2i \hat\sigma_{2i}^2 \left( a_p \bar{v}_{k, p}^{2i-2} - \sum_{r=1}^P a_r \bar{v}_{k, r}^{2i} \right) \bar{v}_{k, p}
      \\
      &\qquad
      - \sum_{i=I}^\infty 2 i \hat\sigma_{2i}^2 \sum_{l : l \ne k} \norm{\v_l}^2 
        \inprod{\bar{\v}_k}{\bar{\v}_l}^{2i-1} \inprod{(\Id - \bar{\v}_k\bar{\v}_k\trans) \bar{\v}_l}{\e_p}. 
  \end{align*}
  Re-index the summation as $\sum_{r=1}^P a_{\pi(r)} \bar{v}_{k, \pi(r)}^{2i}$, replace $p$ with $\pi(q)$, and 
  we obtain
  \begin{align*}
    \dot{\bar{v}}_{k, \pi(q)}
    &= \sum_{i=I}^\infty 2i \hat\sigma_{2i}^2 \left( 
        a_{\pi(q)} \bar{v}_{k, \pi(q)}^{2i-2} 
        - \sum_{r=1}^P a_{\pi(r)} \bar{v}_{k, \pi(r)}^{2i} 
      \right) \bar{v}_{k, \pi(q)}
      \\
      &\qquad
      - \sum_{i=I}^\infty 2 i \hat\sigma_{2i}^2 \sum_{l : l \ne k} \norm{\v_l}^2 
        \inprod{\bar{\v}_k}{\bar{\v}_l}^{2i-1} \inprod{(\Id - \bar{\v}_k\bar{\v}_k\trans) \bar{\v}_l}{\e_{\pi(q)}}. 
  \end{align*}
  Therefore, we have 
  \begin{align*}
    \frac{\rd}{\rd t} \bar{v}_{k, \pi(q)}^2 
    &= 2 \bar{v}_{k, \pi(q)}^2
      \sum_{i=I}^\infty 2i \hat\sigma_{2i}^2 \left( 
        a_{\pi(q)} \bar{v}_{k, \pi(q)}^{2i-2} 
        - \sum_{r=1}^P a_{\pi(r)} \bar{v}_{k, \pi(r)}^{2i} 
      \right) 
      \\ &\qquad
      - \indi\braces{k \ne q}
        2 \bar{v}_{k, \pi(q)} 
        \sum_{i=I}^\infty 
          2 i \hat\sigma_{2i}^2 
          \norm{\v_q}^2 
          \inprod{\bar{\v}_k}{\bar{\v}_q}^{2i-1} 
          \inprod{(\Id - \bar{\v}_k\bar{\v}_k\trans) \bar{\v}_q}{\e_{\pi(q)}}
      \\ &\qquad
      - 2 \bar{v}_{k, \pi(q)}
         \sum_{i=I}^\infty 
          2 i \hat\sigma_{2i}^2 
          \sum_{l \notin \{k, q\} } 
            \norm{\v_l}^2 
            \inprod{\bar{\v}_k}{\bar{\v}_l}^{2i-1} 
            \inprod{(\Id - \bar{\v}_k\bar{\v}_k\trans) \bar{\v}_l}{\e_{\pi(q)}} \\
    &=: \Term_1\left( \frac{\rd}{\rd t} \bar{v}_{k, \pi(q)}^2 \right)
      + \Term_2\left( \frac{\rd}{\rd t} \bar{v}_{k, \pi(q)}^2 \right)
      + \Term_3\left( \frac{\rd}{\rd t} \bar{v}_{k, \pi(q)}^2 \right).
  \end{align*}
  We keep $\Term_1$ as it is, and simplify $\Term_2$ and $\Term_3$ as follows. Consider $\Term_2$. 
  When $q \notin L$, we have $\norm{\v_q}^2 \le \sigma_1^2$, and therefore,
  \[
    \text{(When $q \notin L$)}\quad 
    \abs{ \Term_2  }
    \le 2 \abs{ \bar{v}_{k, \pi(q)} } \sum_{i=I}^\infty 2 i \hat\sigma_{2i}^2 \sigma_1^2
    \le 2 \abs{ \bar{v}_{k, \pi(q)} } C_\sigma^2 \sigma_1^2, 
  \]
  where the last inequality comes from Assumption~\ref{assumption: link function}. Now, suppose that $q \in L$. 
  In this case, we have $\bar{\v}_q \approx s_q \e_{\pi(q)}$ where $s_q := \sgn \bar{v}_{q, \pi(q)}$. This suggests
  writing 
  \begin{align*}
    \inprod{\bar{\v}_k}{\bar{\v}_q}^{2i-1} \inprod{(\Id - \bar{\v}_k\bar{\v}_k\trans) \bar{\v}_q}{\e_{\pi(q)}}
    &= \inprod{\bar{\v}_k}{\bar{\v}_q}^{2i-1} 
      \left(
        \inprod{\bar{\v}_q}{\e_{\pi(q)}}
        - \inprod{\bar{\v}_k}{\bar{\v}_q}  \inprod{ \bar{\v}_k}{\e_{\pi(q)}}
      \right) \\
    &= \inprod{\bar{\v}_k}{\bar{\v}_q}^{2i-1} \bar{v}_{q, \pi(q)}
        - \inprod{\bar{\v}_k}{\bar{\v}_q}^{2i} \bar{v}_{k, \pi(q)}.
  \end{align*}
  By \IHGF\ref{inductH-itm: large norm => converged}, we have $\bar{v}_{q, \pi(q)}^2 \ge 1 - \bar\eps$. 
  First, this implies $|\bar{v}_{q, \pi(q)}| \ge \sqrt{1 - \bar\eps} \ge 1 - \bar\eps$. Hence, 
  $\bar{v}_{q, \pi(q)} = s_q \pm \bar\eps$. In addition, we have 
  \[
    \norm{ s_q \e_{\pi(q)} -  \bar{\v}_q }
    = \sqrt{ 2 - 2 \inprod{ s_q \e_{\pi(q)} }{ \bar{\v}_q } }
    = \sqrt{ 2 - 2 s_q  (s_q \pm \bar\eps) }
    \le \sqrt{ 2 \bar\eps  }.
  \]
  As a result, we have 
  \[
    \inprod{\bar{\v}_k}{\bar{\v}_q}
    = \inprod{\bar{\v}_k}{s_q \e_{\pi(q)}} + \inprod{\bar{\v}_k}{ s_q \e_{\pi(q)} -  \bar{\v}_q}
    = s_q \bar{v}_{k, \pi(q)}
      \pm \norm{ s_q \e_{\pi(q)} -  \bar{\v}_q }
    = s_q \bar{v}_{k, \pi(q)} \pm \sqrt{ 2 \bar\eps  }.
  \]
  Combine these estimations with the previous identity, and we obtain 
  \begin{multline*}
    \inprod{\bar{\v}_k}{\bar{\v}_q}^{2i-1} \inprod{(\Id - \bar{\v}_k\bar{\v}_k\trans) \bar{\v}_q}{\e_{\pi(q)}}
    = \inprod{\bar{\v}_k}{\bar{\v}_q}^{2i-1} \bar{v}_{q, \pi(q)}
        - \inprod{\bar{\v}_k}{\bar{\v}_q}^{2i} \bar{v}_{k, \pi(q)} \\
    = \left( s_q \bar{v}_{k, \pi(q)} \pm \sqrt{ 2 \bar\eps  } \right)^{2i-1} \left( s_q \pm \bar\eps \right)  
        - \left( s_q \bar{v}_{k, \pi(q)} \pm \sqrt{ 2 \bar\eps  } \right)^{2i} \bar{v}_{k, \pi(q)}. 
  \end{multline*}
  Note that, for any $a, \delta \in \R$ and integer $N$, we have 
  \begin{align*}
    (a + \delta)^N
    = a^N + \sum_{n=1}^{N} \binom{N}{n} a^{N-n} \delta^n 
    &= a^N + \delta \sum_{n=0}^{N-1} \binom{N}{n+1} a^{N-n-1} \delta^n \\
    &= a^N + \delta \sum_{n=0}^{N-1} \binom{N-1}{n} \frac{N}{n+1} a^{(N-1)-n} \delta^n \\
    &= a^N \pm \delta N \left( |a| + |\delta| \right)^{N-1}  \\
    &= a^N \pm N 2^{N-1} \left( \delta |a|^{N-1} \vee |\delta|^N \right). 
  \end{align*}
  Thus, we can further rewrite the above as 
  \begin{align*}
    & \inprod{\bar{\v}_k}{\bar{\v}_q}^{2i-1} \inprod{(\Id - \bar{\v}_k\bar{\v}_k\trans) \bar{\v}_q}{\e_{\pi(q)}} \\
    =\;& \left( 
        s_q^{2i-1}  \bar{v}_{k, \pi(q)}^{2i-1}  
        \pm i 2^{3i}
          \left(
            \bar\eps^{1/2} \bar{v}_{k, \pi(q)}^{2i-2} 
            \vee \bar\eps^{i-1/2}
          \right) 
      \right) 
      \left( s_q \pm \bar\eps \right)   \\
      &\qquad
        - \left( 
          \bar{v}_{k, \pi(q)}^{2i} 
          \pm i 2^{3i} 
          \left(
            \bar\eps^{1/2} \abs{\bar{v}_{k, \pi(q)}}^{2i-1}
            \vee  \bar\eps^i 
          \right)
        \right) \bar{v}_{k, \pi(q)} \\
    =\;& \left( 1 - \bar{v}_{k, \pi(q)}^2 \right) \bar{v}_{k, \pi(q)}^{2i-1}   \\
      &\qquad
      \pm \bar{v}_{k, \pi(q)}^{2i-1}  \bar\eps 
      \pm 2 i 2^{3i}
          \left(
            \bar\eps^{1/2} \bar{v}_{k, \pi(q)}^{2i-2} 
            \vee \bar\eps^{i-1/2}
          \right) 
      \pm i 2^{3i} \bar{v}_{k, \pi(q)}
      \left(
        \bar\eps^{1/2} \abs{\bar{v}_{k, \pi(q)}}^{2i-1}
        \vee  \bar\eps^i 
      \right). 
  \end{align*}
  For the last three terms, clear that the second one is the largest as it has the smallest exponents on both 
  $\bar\eps$ and $\bar{v}_{k, \pi(q)}$. Also recall from \IHGF\ref{inductH-itm: bound on the failed coordinates} that 
  $|\bar{v}_{k, \pi(q)}| \le \eps_0$. Thus, we have 
  \[
    \inprod{\bar{\v}_k}{\bar{\v}_q}^{2i-1} \inprod{(\Id - \bar{\v}_k\bar{\v}_k\trans) \bar{\v}_q}{\e_{\pi(q)}}
    = \left( 1 - \bar{v}_{k, \pi(q)}^2 \right) \bar{v}_{k, \pi(q)}^{2i-1}
      \pm 3 i 2^{3i}
      \left(
        \bar\eps^{1/2} \eps_0^{i-1} 
        \vee \bar\eps^{i-1/2}
      \right).
  \]
  As a result, we have 
  \begin{align*}
    & \text{(When $q \in L$)}  \\
    \Term_2
    &= - \indi\braces{k \ne q}
        2 \bar{v}_{k, \pi(q)} 
        \sum_{i=I}^\infty 
          2 i \hat\sigma_{2i}^2 
          \norm{\v_q}^2 
          \left(
            \left( 1 - \bar{v}_{k, \pi(q)}^2 \right) \bar{v}_{k, \pi(q)}^{2i-1}
            \pm 3 i 2^{3i}
            \left(
              \bar\eps^{1/2} \eps_0^{i-1} 
              \vee \bar\eps^{i-1/2}
            \right)
          \right) \\
    &= - \indi\braces{k \ne q}
      2 \sum_{i=I}^\infty 
        2 i \hat\sigma_{2i}^2 
        \norm{\v_q}^2 
        \left( 1 - \bar{v}_{k, \pi(q)}^2 \right) \bar{v}_{k, \pi(q)}^{2i}
        \\
      &\qquad
      \pm 
      2 \bar{v}_{k, \pi(q)} 
      3 I 2^{3I}
      \left(
        \bar\eps^{1/2} \eps_0^{I-1} 
        \vee \bar\eps^{I-1/2}
      \right)
      \sum_{i=I}^\infty 
        2 i \hat\sigma_{2i}^2 
        \norm{\v_q}^2 
        \\
    &= - \indi\braces{k \ne q}
        2 \norm{\v_q}^2 \left( 1 - \bar{v}_{k, \pi(q)}^2 \right)
        \sum_{i=I}^\infty 2 i \hat\sigma_{2i}^2 \bar{v}_{k, \pi(q)}^{2i} \\
        &\qquad
        \pm 12 I 2^{3I} C_\sigma^2 a_{\pi(q)} \bar{v}_{k, \pi(q)} 
          \left(
            \bar\eps^{1/2} \eps_0^{I-1} 
            \vee \bar\eps^{I-1/2}
          \right).
  \end{align*}
  Combining the cases $q \in L$ and $q \notin L$, we obtain 
  \begin{align*}
    \Term_2 
    &= - \indi\braces{k \ne q, q \in L}
      2 \norm{\v_q}^2 \left( 1 - \bar{v}_{k, \pi(q)}^2 \right)
      \sum_{i=I}^\infty 2 i \hat\sigma_{2i}^2 \bar{v}_{k, \pi(q)}^{2i} \\
      &\qquad
      \pm 12 I 2^{3I} C_\sigma^2 a_{\pi(q)} \bar{v}_{k, \pi(q)} 
          \left(
            \bar\eps^{1/2} \eps_0^{I-1} 
            \vee \bar\eps^{I-1/2}
          \right)
      \pm 2 \abs{ \bar{v}_{k, \pi(q)} } C_\sigma^2 \sigma_1^2.
  \end{align*}
  Now, we estimate 
  \begin{align*}
    \Term_3 
    &:= - 2 \bar{v}_{k, \pi(q)}
      \sum_{i=I}^\infty 
      2 i \hat\sigma_{2i}^2 
      \sum_{l \notin \{k, q\} } 
        \norm{\v_l}^2 
        \inprod{\bar{\v}_k}{\bar{\v}_l}^{2i-1} 
        \inprod{(\Id - \bar{\v}_k\bar{\v}_k\trans) \bar{\v}_l}{\e_{\pi(q)}} \\
    &:= - 2 \bar{v}_{k, \pi(q)}
        \sum_{i=I}^\infty 
        2 i \hat\sigma_{2i}^2 
        \sum_{l \notin L\cup \{k, q\} } 
          \norm{\v_l}^2 
          \inprod{\bar{\v}_k}{\bar{\v}_l}^{2i-1} 
          \inprod{(\Id - \bar{\v}_k\bar{\v}_k\trans) \bar{\v}_l}{\e_{\pi(q)}} \\
      &\qquad
      - 2 \bar{v}_{k, \pi(q)}
        \sum_{i=I}^\infty 
        2 i \hat\sigma_{2i}^2 
        \sum_{l \in L \setminus \{k, q\} } 
          \norm{\v_l}^2 
          \inprod{\bar{\v}_k}{\bar{\v}_l}^{2i-1} 
          \inprod{(\Id - \bar{\v}_k\bar{\v}_k\trans) \bar{\v}_l}{\e_{\pi(q)}} \\
    &=: \Term_{3.1} + \Term_{3.2}. 
  \end{align*}
  Similar to the previous analysis, for $\Term_{3.1}$, we have 
  \[
    |\Term_{3.1}|
    \le 2 \abs{ \bar{v}_{k, \pi(q)} }
      \sum_{i=I}^\infty 
      2 i \hat\sigma_{2i}^2 
      \sum_{l \notin L\cup \{k, q\} } \sigma_1^2
    \le 2 C_\sigma^2
      \abs{ \bar{v}_{k, \pi(q)} }
      (m - 1) \sigma_1^2.
  \]
  Consider $\Term_{3.2}$. Note that by our previous analysis, for any $l \in L \setminus \{k, q\}$, we have 
  \begin{multline*}
    \abs{ \inprod{\bar{\v}_k}{\bar{\v}_l}^{2i-1} \inprod{(\Id - \bar{\v}_k\bar{\v}_k\trans) \bar{\v}_l}{\e_\pi(q)} } \\
    \le \abs{
        \left(
          s_l \bar{v}_{k, \pi(l)}
          \pm \sqrt{2 \bar\eps}
        \right)^{2i-1} 
        \bar{v}_{l, \pi(q)} 
      }
      + \abs{ 
          \left(
          s_l \bar{v}_{k, \pi(l)}
          \pm \sqrt{2 \bar\eps}
        \right)^{2i} \bar{v}_{k, \pi(q)} 
      } \\
    \le \left(
        \sqrt{\eps_0} + \sqrt{2 \bar\eps}
      \right)^{2i-1} 
      \sqrt{\eps_0}
      + \left(
        \sqrt{\eps_0} + \sqrt{2 \bar\eps}
      \right)^{2i}.
  \end{multline*}
  Note that $ \sqrt{\eps_0}^{2i} \vee \sqrt{\bar\eps}^{2i-1}\sqrt{\eps_0} \vee \sqrt{\bar\eps}^{2i} 
  = \eps_0^i \vee \bar\eps^i$. Hence, we can bound the last term as 
  \[
    \abs{ \inprod{\bar{\v}_k}{\bar{\v}_l}^{2i-1} \inprod{(\Id - \bar{\v}_k\bar{\v}_k\trans) \bar{\v}_l}{\e_\pi(q)} } 
    \le 2^{i+2} \left( \eps_0^i \vee \bar\eps^i \right). 
  \]
  Therefore, 
  \[
    \abs{\Term_{3.2}}
    \le 2 \bar{v}_{k, \pi(q)}
      \sum_{i=I}^\infty 
      2 i \hat\sigma_{2i}^2 
      \sum_{l \in L \setminus \{k, q\} } 
        \norm{\v_l}^2 
        2^{i+2} \left( \eps_0^i \vee \bar\eps^i \right) 
    \le 2^{I+5} C_\sigma^2 \norm{\a}_1 \abs{ \bar{v}_{k, \pi(q)} } \left( \eps_0^I \vee \bar\eps^I \right) .
  \]
  As a result, for $\Term_3$, we have 
  \[
    \abs{\Term_3}
    \le 2 C_\sigma^2 \abs{ \bar{v}_{k, \pi(q)} } (m - 1) \sigma_1^2
      + 2^{i+5} C_\sigma^2 \norm{\a}_1 \abs{ \bar{v}_{k, \pi(q)} } \left( \eps_0^I \vee \bar\eps^I \right).
  \]
  Combine our bounds for $\Term_2$ and $\Term_3$, and we get 
  \begin{align*}
    \frac{\rd}{\rd t} \bar{v}_{k, \pi(q)}^2
    &= 2 \bar{v}_{k, \pi(q)}^2
      \sum_{i=I}^\infty 2i \hat\sigma_{2i}^2 \left( 
        a_{\pi(q)} \bar{v}_{k, \pi(q)}^{2i-2} 
        - \sum_{r=1}^P a_{\pi(r)} \bar{v}_{k, \pi(r)}^{2i} 
      \right)  \\
      &\qquad
      - \indi\braces{k \ne q, q \in L}
      2 \norm{\v_q}^2 \left( 1 - \bar{v}_{k, \pi(q)}^2 \right)
      \sum_{i=I}^\infty 2 i \hat\sigma_{2i}^2 \bar{v}_{k, \pi(q)}^{2i} \\
      &\qquad
      \pm 12 I 2^{3I} C_\sigma^2 a_{\pi(q)} \bar{v}_{k, \pi(q)} 
        \left(
          \bar\eps^{1/2} \eps_0^{I-1} 
          \vee \bar\eps^{I-1/2}
        \right)
      \pm 2 \abs{ \bar{v}_{k, \pi(q)} } C_\sigma^2 \sigma_1^2 \\
      &\qquad
      \pm C_\sigma^2 \abs{ \bar{v}_{k, \pi(q)} } (m - 1) \sigma_1^2
      \pm 2^{I+5} C_\sigma^2 \norm{\a}_1 \abs{ \bar{v}_{k, \pi(q)} } \left( \eps_0^I \vee \bar\eps^I \right).
  \end{align*}
  For the last four error terms, clear that we can merge the second and the third terms, which leads 
  to $2 C_\sigma^2 \abs{ \bar{v}_{k, \pi(q)} } m \sigma_1^2$. Meanwhile, the largest coefficient is 
  $12 I 2^{3I} C_\sigma^2$. Thus, 
  \begin{align*}
    \frac{\rd}{\rd t} \bar{v}_{k, \pi(q)}^2
    &= 2 \bar{v}_{k, \pi(q)}^2
      \sum_{i=I}^\infty 2i \hat\sigma_{2i}^2 \left( 
        a_{\pi(q)} \bar{v}_{k, \pi(q)}^{2i-2} 
        - \sum_{r=1}^P a_{\pi(r)} \bar{v}_{k, \pi(r)}^{2i} 
      \right)  \\
      &\quad
      - \indi\braces{k \ne q, q \in L}
      2 \norm{\v_q}^2 \left( 1 - \bar{v}_{k, \pi(q)}^2 \right)
      \sum_{i=I}^\infty 2 i \hat\sigma_{2i}^2 \bar{v}_{k, \pi(q)}^{2i} \\
      &\quad
      \pm I 2^{3I+6} C_\sigma^2 \abs{\bar{v}_{k, \pi(q)} } 
      \braces{
        a_{\pi(q)} \left( \bar\eps^{1/2} \eps_0^{I-1} \vee \bar\eps^{I-1/2} \right)
        \vee m \sigma_1^2
        \vee \norm{\a}_1 \left( \eps_0^I \vee \bar\eps^I \right)
      }.
  \end{align*}
  Finally, recall that $\bar\eps \le \eps_0$. Hence, 
  $\bar\eps^{1/2} \eps_0^{I-1} \vee \bar\eps^{I-1/2} = \bar\eps^{1/2} \eps_0^{I-1}$
  and $\eps_0^I \vee \bar\eps^I = \eps_0^I$.
  
  Now, consider the second part of the lemma. In order for 
  $a_{\pi(q)} \bar\eps^{1/2} \eps_0^{I-1} \vee m \sigma_1^2 \vee \norm{\a}_1 \eps_0^I \le \delta$, clear that 
  we need $m \sigma_1^2 \le \delta$. Meanwhile, for the last condition, we have 
  \[
    \norm{\a}_1 \eps_0^I
    \le \delta
    \quad\Leftarrow\quad
    d^{-(1-\gamma) I} 
    \le \frac{\delta}{\norm{\a}_1}
    \quad\Leftarrow\quad
    d 
    \ge \left( \frac{\delta}{\norm{\a}_1} \right)^{-\frac{1}{(1 - \gamma) I}}.
  \]
  For the first condition, we have 
  \[
    a_{\pi(q)} \bar\eps^{1/2} \eps_0^{I-1}
    \le \delta 
    \quad\Leftarrow\quad 
    \bar\eps 
    \le \left( \frac{\delta }{a_{\pi(q)}} \right)^2 d^{2(1-\gamma)(I-1)} .
  \]
\end{proof}

\subsubsection{Proof of Theorem~\ref{thm: gf}}

\begin{proof}[Proof of Theorem~\ref{thm: gf}]
  By Corollary~\ref{cor: convergence of one direction}, Lemma~\ref{lemma: tangent: upper triangular entries 
  (case I)}, \ref{lemma: tangent: upper triangular entries (case II)}, \ref{lemma: tangent: lower triangular entries},
  \ref{lemma: tangent: lower right block}, \ref{lemma: radial: unused neurons}, and \ref{lemma: radial: upper bound on 
  the norm}. \IHGF{} holds throughout training and the conclusions of Theorem~\ref{thm: gf} are true, provided
  that all the conditions of these lemmas are met. 

  For easier reference, we collect the conditions of all above lemmas below:
  \begin{gather*}
    \gamma < 1/(2I), \quad \delta_v' = 1/3, \quad \delta_{r, t} = \delta_r \wedge \delta_t, \\
    \eps_D 
    \ge \frac{ 2^{3I+7} C_\sigma^2 }{(\delta_v')^I \hat\sigma_{2I}^2 } 
      \frac{\norm{\a}_1}{a_{\min_*}} 
      \frac{1}{d^{(1-\gamma) I} }, \quad 
    \eps_R 
    \ge  12 \norm{\a}_1 2^{2I} d^{-(1-\gamma) I}, \quad 
    \delta_T 
    \ge \frac{ 2^{3I+4} C_\sigma^2  }{ \hat\sigma_{2I}^2 } \frac{\norm{\a}_1}{a_{\min_*}} \frac{1}{d^{1/2-\gamma I} }, \\
    m \sigma_1^2 
    \le 
      \frac{ \hat\sigma_{2I}^2 a_{\min_*}}{ 2^{3I+7} C_\sigma^2 }  
      \left(
        (\delta_v')^I  \eps 
        \wedge 
        \frac{\delta_T }{d^{I-1/2}}
      \right)
      \wedge \frac{\eps_R}{12}, \\
    \bar\eps 
    \le \left( \frac{ (\delta_v')^I \hat\sigma_{2I}^2 }{ 2^{3I+7} C_\sigma^2 } \right)^2 
        \eps_D^2 d^{2 (1-\gamma)(I-1)}
        \wedge 
        \left( \delta_T \frac{ \hat\sigma_{2I}^2  }{ 2^{3I+4} C_\sigma^2 } \right)^2
            \frac{1}{d^{1+ 2 \gamma (I - 1)} }
        \wedge
        \frac{\eps_R}{12 C_\sigma^2  a_{\pi(p)}}, \\
    \bar\eps 
    \le \left( 
      \frac{ \hat\sigma_{2I}^2 }{2^{3I+4} C_\sigma^2 } 
      \frac{\delta_{r, t}}{24}
    \right)^2 
    \frac{1}{d^{1+2\gamma(I-1)}}, \quad 
    m \sigma_1^2 
    \le \frac{ \hat\sigma_{2I}^2 }{2^{3I+4} C_\sigma^2 } 
      \frac{ a_{\min_*} }{ 2 (\log d)^{2I-2} d^{I-1/2} }
      \frac{\delta_{r, t}}{24}, \\
    \frac{d}{(\log^2 d)^{1/\gamma}} 
      \ge \left( \frac{\delta_{r, t}}{4} \right)^{-\frac{1}{\gamma(I-1)}}, \;
      \frac{d}{(\log^2 d)^{\frac{I-1}{1/2 - \gamma I}}}  
      \ge \left( 
          \frac{ \hat\sigma_{2I}^2 }{2^{3I+4} C_\sigma^2 } 
          \frac{ a_{\min_*} }{ \norm{\a}_1 2^{2I-2}  }
          \frac{\delta_{r, t}}{24}
        \right)^{
          - \frac{1}{1/2 - \gamma I}
        }
        , \; 
      \delta_T \le \frac{\delta_{r, t}}{240}, \\
    \delta_T \le \frac{\delta_c}{240}, \quad 
    \eps_R 
    \le \frac{1}{6} \frac{ a_{\min_*}^2 \delta_c}{8 (\log^2 d)^{I-1}}, \quad 
    \bar\eps
    \le 
      \left(  
        \frac{1}{48} \frac{ 4 \hat\sigma_{2I}^2   }{ 2^{3I+6} C_\sigma^2  }
      \right)^2 
      \frac{ a_{\min_*}^2 \delta_c^2}{(\log^2 d)^{2I-2}}
      \frac{1}{d^{1+2\gamma(I-1)}}, \\
    m \sigma_1^2 
    \le \frac{1}{48} \frac{ \hat\sigma_{2I}^2   }{ 2^{3I+4} C_\sigma^2  }
      \frac{ a_{\min_*}^2 \delta_c}{(\log^2 d)^{I-1}}
      \frac{1}{d^{I-1/2}}, \quad 
    \frac{ d }{ (\log^2 d)^{\frac{I-1}{1/2-\gamma I}} } 
    \ge \left( 
        \frac{1}{6} \frac{ 4 \hat\sigma_{2I}^2   }{ 2^{3I+6} C_\sigma^2  }
        \frac{ a_{\min_*}^2 \delta_c}{8 \norm{\a}_1}
      \right)^{-\frac{1}{1/2 - \gamma I}}. 
  \end{gather*}
  In the following, for notational simplicity, we will use $\lesssim_\sigma$ and $\gtrsim_\sigma$ to hide 
  constant that can only depend on $\sigma$. First, we consider the conditions on $\gamma$, which are 
  \[
    \gamma < \frac{1}{2I} 
    \quad\text{and}\quad 
    \frac{d}{(\log^2 d)^{1/\gamma}} 
    \ge \left( \frac{\delta_{r, t}}{4} \right)^{-\frac{1}{\gamma(I-1)}}.
  \]
  For concreteness, we will require $\gamma \le 1/(4I)$ and choose $\gamma$ such that 
  \[
    \frac{d^\gamma}{\log^2 d} 
    = \left( \frac{\delta_{r, t}}{4} \right)^{-\frac{1}{I-1}}.
  \]
  For such a $\gamma$ to exist, it suffices to have 
  \[
    \frac{d^{1/(4I)}}{\log^2 d} 
    \ge \left( \frac{\delta_{r, t}}{4} \right)^{-\frac{1}{I-1}}
    \quad\Leftarrow\quad 
    \frac{d}{\log^{8I} d} 
    \gtrsim \delta_{r, t}^{-8}.
  \]
  First, for the conditions on the target accuracy $\eps_D, \eps_R$
  and error in time $\delta_T$, we need 
  \[
    \eps_D 
    \gtrsim_\sigma \frac{\norm{\a}_1}{a_{\min_*}} \frac{1}{d^{I - 1/4}}, \quad 
    \frac{1}{d^{I-1/4}}
    \lesssim_\sigma \eps_R 
    \lesssim_\sigma \frac{ a_{\min_*}^2 \delta_c}{(\log^2 d)^{I-1}}, \quad 
    \frac{\norm{\a}_1}{a_{\min_*}} \frac{1}{d^{1/4} } 
    \lesssim_\sigma \delta_T 
    \lesssim_\sigma \delta_c \wedge \delta_r \wedge \delta_t.
  \]
  Then, for $\bar\eps$, we choose 
  \[
    \bar\eps 
    =_\sigma
      \eps_D^2 d^{2(I-1)}
      \wedge 
        \frac{\delta_T^2 \delta_{r, t}^2 }{d (\log d)^{4(I-1)}}  
      \wedge
      \frac{\eps_R}{a_{\min_*}}
      \wedge 
        \frac{\delta_{r, t}^4}{ d (\log d)^{4(I-1)} }
      \wedge 
      \frac{ a_{\min_*}^2 \delta_c^2}{(\log^2 d)^{2I-2}}
      \frac{\delta_{r, t}^2}{d (\log d)^{4(I-1)}}.
  \]
  The condition on $m \sigma_1^2$ is 
  \[
    m \sigma_1^2 
    \lesssim_\sigma
      a_{\min_*} \eps_D 
      \wedge  \frac{a_{\min_*}  \delta_T }{d^{I-1/2}}
      \wedge \eps_R 
      \wedge \frac{ a_{\min_*} \delta_{r, t} }{ (\log d)^{2I-2} d^{I-1/2} } 
      \wedge \frac{ a_{\min_*}^2 \delta_c}{(\log^2 d)^{I-1}}
      \frac{1}{d^{I-1/2}}
  \]
  Since $\sigma_1^2 := 2 \sigma_0^2 e^{ 5 / \hat\sigma_{2I}^2 } \bar\eps^{- 8  / (I \hat\sigma_{2I}^2) }$, this is 
  equivalent to 
  \[
    \sigma_0^2 
    \lesssim_\sigma
      \frac{\bar\eps^{ 8  / (I \hat\sigma_{2I}^2) }}{m}
      \left(
        a_{\min_*} \eps_D 
        \wedge  \frac{a_{\min_*}  \delta_T }{d^{I-1/2}}
        \wedge \eps_R 
        \wedge \frac{ a_{\min_*} \delta_{r, t} }{ (\log d)^{2I-2} d^{I-1/2} } 
        \wedge \frac{ a_{\min_*}^2 \delta_c}{(\log^2 d)^{I-1}}
        \frac{1}{d^{I-1/2}}
      \right).
  \]
  Finally, the conditions on $d$ are 
  \[
    \frac{d}{\log^{8I} d} 
    \gtrsim \delta_{r, t}^{-8}, \quad 
    \frac{d}{(\log^2 d)^{4(I-1)}}  
    \ge \left( 
          \frac{ a_{\min_*} }{ \norm{\a}_1 }
          \delta_{r, t} 
        \right)^{
          - \frac{1}{1/4}
        }
        \vee 
        \left( 
        \frac{ a_{\min_*}^2 \delta_c}{\norm{\a}_1}
      \right)^{-\frac{1}{1/4}},
  \]
  which can be merged into 
  \[
    \frac{d}{(\log^2 d)^{4I}} 
    \gtrsim_\sigma 
      \delta_{r, t}^{-8}
      \vee 
      \left( 
        \frac{ a_{\min_*} }{ \norm{\a}_1 }
        \delta_{r, t} 
      \right)^{-4}
      \vee 
      \left( 
      \frac{ a_{\min_*}^2 \delta_c}{\norm{\a}_1}
    \right)^{-4},
  \]
\end{proof}

\bigskip
\section{Online SGD Dynamics}\label{sec:online-sgd-proofs}

Our goal in this section is to prove Theorem \ref{thm:online-sgd}, which we restate below for convenience:
\onlinesgd*

Similarly to the gradient flow setting, our proof will proceed by maintaining Induction Hypothesis \ref{inductionH: gf} with high probability throughout training. We will additionally maintain the following induction hypothesis on the growth of $\norm{\v_p}^2$.

\begin{inductionH}\label{inductH:learning-time}
The neuron $\v_p$ learns at time $(1 \pm o(1))T_p$; that is
\begin{enumerate}[(a)]
\item $\bar v_{p, \pi(p)}^2(t) \ge 1 - \eps_D$ for all $t \in \left[(1 + \frac{\Delta}{8})T_p, T_{\max}\right]$. \label{inductH-itm:directional-convergence}
\item $\abs{\norm{\v_p}^2 - a_p} \le \eps_R$ for all $t \in \left[(1 + \frac{\Delta}{4})T_p, T_{\max}\right]$ \label{inductH-itm:norm-convergence}
\end{enumerate}
\end{inductionH}

To maintain these induction hypotheses, we rely on the following stochastic induction argument from \cite{ren2024learning}.
Suppose that the goal is to show a stochastic process $X_t$ stays close to its deterministic counterpart $x_t$ with 
high probability. First, we assume $X_t \approx x_t$ and use this induction hypothesis to obtain estimations 
on the related quantities, such as the variance of the noises. Then, using these estimations, we show that when 
$X_t$ is still close to $x_t$, the probability that $X_t$ will drift away from $x_t$ is small. This argument can be 
viewed as the stochastic counterpart of the continuity argument, and can be made rigorous by considering the stopping 
time $\tau$ that $X_t$ is no longer close to $x_t$ and analyzing the stopped process $(X_{t \wedge \tau})_t$. 
One may refer to Section~F.2 of \cite{ren2024learning} for more details on this technique. Finally, we remark that 
this argument can be easily generalized to cases with multiple induction hypotheses by considering the stopping time 
that any of them is violated.

\subsection{Preliminaries}

The following lemma decomposes the online SGD dynamics into the update on the radial component $\norm{\v_k(t)}^2$ and the tangent component $\bar v_{k,p}^2(t+1)$.

\begin{lemma}
  \label{lemma: radial and tangent dynamics}
    Fix $k \in [m]$, $p \in [P]$ and $t > 0$. Let $\delta_{\P, \xi} \in (0, 1)$ be target failure probability at this 
  step. Let $C > 0$ be a large universal constant. 
  Suppose that $\eta \le 2\left(C \norm{\a}_1d \log^{\tilde Q/2}(m d/ \delta_{\P}) \right)\inv$ and 
  let $\H_k(t+1) := \hat\nabla_{\v_k} l - \nabla_{\v_k} \Loss$ denote the difference between the mini-batch gradient and 
  the population at this step. Then, we have (denoting $\v_k := \v_k(t)$):
  \begin{align*}
    \norm{\v_k(t+1)}^2 
    &= \norm{\v_k}^2
      + 4 \eta \left(
        \sum_{i=I}^\infty \hat \sigma_{2i}^2\sum_{p=1}^P a_p \bar{v}_{k, p}^{2i} 
        - \sum_{i=I}^\infty \hat \sigma_{2i}^2\sum_{l=1}^m \norm{\v_l}^2 \inprod{\bar{\v}_k}{\bar{\v}_l}^{2i}
      \right) \norm{\v_k}^2\\
    & \qquad 
      - 2 \eta \inprod{\v_k}{ \H_k }
      + \xi_{k, R}(t+1), \\
    \bar{v}_{k, p}^2(t+1)
    &= \bar{v}_{k, p}^2 + 2\eta \bar{v}_{k, p}^2\cdot \sum_{i=I}^\infty 2i \hat\sigma_{2i}^2 \left( a_p \bar{v}_{k, p}^{2i-2} - \sum_{q=1}^P a_q \bar{v}_{k, q}^{2i} \right)
      \\
      &\qquad
      -2\eta \bar{v}_{k, p} \sum_{i=I}^\infty 2 i \hat\sigma_{2i}^2 \sum_{l : l \ne k} \norm{\v_l}^2 
        \inprod{\bar{\v}_k}{\bar{\v}_l}^{2i-1} \inprod{(\Id - \bar{\v}_k\bar{\v}_k\trans) \bar{\v}_l}{\e_p}\\
      &\qquad
      - 2 \eta \bar{v}_{k, p} \frac{ \inprod{(\Id - \bar{\v}_k \bar{\v}_k\trans) \H_k}{\e_p} }{\norm{\v_k}}
      + \xi_{k, p}(t+1), 
  \end{align*}
  where $\xi_{k, R}(t+1)$ and $\xi_{k, p}(t+1)$ satisfy 
  \[
    |\xi_{k, R}(t+1)|
    \le C \eta^2 d \norm{\a}_1^2\log^{\tilde Q}\left( \frac{m d}{ \delta_{\P, \xi} } \right)  \norm{\v_k}^2, \quad 
    |\xi_{k, p}(t+1)| 
    \le 
        C\eta^2 \left( 1 \vee \bar{v}_{k, p}^2 d  \right) \norm{\a}_1^2\log^{\tilde Q}\left( \frac{m d}{ \delta_{\P, \xi} } \right)  
  \]
  with probability at least $1 - \delta_{\P, \xi}$. 
\end{lemma}
\begin{proof}
  Let $k \in [m]$ be fixed and $t > 0$. We write 
  \[
    \hat\nabla_{\v_k} l 
    = \nabla_{\v_k} \Loss + \left( \hat\nabla_{\v_k} l - \nabla_{\v_k} \Loss  \right) 
    =: \nabla_{\v_k} \Loss + \H_k, 
  \]
  where $\hat\nabla$ denotes the mini-batch gradient. 
  First, consider the dynamics of $\norm{\v_k}^2$. By Lemma \ref{lemma: population and per-sample gradients}, we have that
  \begin{align*}
    \norm{\v_k(t+1)}^2 
    &= \norm{\v_k - \eta \hat\nabla_{\v_k} l }^2  \\
    &= \norm{\v_k}^2
      - 2 \eta \inprod{\v_k}{ \nabla_{\v_k} \Loss }
      - 2 \eta \inprod{\v_k}{ \H_k }
      + \eta^2 \norm{\hat \nabla_{\v_k} l}^2 \\
    &= \norm{\v_k}^2
      + 4 \eta \left(
        \sum_{i=I}^\infty \hat \sigma_{2i}^2\sum_{p=1}^P a_p \bar{v}_{k, p}^{2i} 
        - \sum_{i=I}^\infty \hat \sigma_{2i}^2\sum_{l=1}^m \norm{\v_l}^2 \inprod{\bar{\v}_k}{\bar{\v}_l}^{2i}
      \right) \norm{\v_k}^2\\ 
     & \qquad - 2 \eta \inprod{\v_k}{ \H_k }
      + \eta^2 \norm{\hat\nabla_{\v_k} l}^2. 
  \end{align*}
  By the tail bound in Lemma~\ref{lemma: population and per-sample gradients}, for any given direction $\u \in \S^{d-1}$, with 
  probability at least $1 - \delta_{\P}$, we have 
  \(
    \left| \inprod{\hat\nabla_{\v_k} l}{\u} \right|
    \le C \norm{\a}_1\log^{\tilde Q/2}(m/ \delta_{\P}) \norm{\v_k},
  \)
  for some universal constant $C > 0$. Take $\u$ to be $\v_k$ and $\e_1, \dots, \e_d$, and replace $\delta_{\P}$ with 
  $\delta_{\P}/(2d)$. Then, we obtain 
  \[
    \left| \inprod{\v_k}{\hat\nabla_{\v_k} l} \right|
    \le C \norm{\a}_1\log^{\tilde Q/2}(m d/ \delta_{\P}) \norm{\v_k}^2, \quad 
    \norm{\hat\nabla_{\v_k} l}^2 
    \le C^2 d \norm{\a}_1^2\log^{\tilde Q}(m d/ \delta_{\P}) \norm{\v_k}^2, 
  \]
  for some universal constant $C > 0$ with probability at least $1 - \delta_{\P}$. Plugging in the bound for $\norm{\hat\nabla_{\v_k} l}^2$ yields the desired update for $\norm{\v_k(t+1)}^2$.
  
  We next analyze the dynamics of $\bar{v}_{k, p}^2$ where $p \in [P]$. 
  To this end, first we estimate $1/\norm{\v_k(t+1)}^2$. 
  With probability $1 - \delta_{\P}$ we have that, 
  \begin{align*}
    \norm{\v_k(t+1)}^2 &= \norm{\v_k}^2 - 2\eta \langle \nabla_{\v_k}l, \v_k\rangle + \eta^2\norm{\nabla_{\v_k}l}^2\\
    & = \norm{\v_k}^2 \left(
      1 
      \pm 2 \eta C \norm{\a}_1\log^{\tilde Q/2}(m d/ \delta_{\P}) 
      \pm C^2 \eta^2 d \norm{\a}_1^2\log^{\tilde Q}(m d/ \delta_{\P}) 
    \right). 
  \end{align*}
  When $\eta \le 2\left(C \norm{\a}_1d \log^{\tilde Q/2}(m d/ \delta_{\P}) \right)\inv$, we have 
  \[
    C^2 \eta^2 d \norm{\a}_1^2\log^{\tilde Q}(m d/ \delta_{\P}) 
    \le 2 \eta C \norm{\a}_1\log^{\tilde Q/2}(m d/ \delta_{\P})  
    \le \frac{1}{4}. 
  \]
  Hence, we can use the identity 
  \begin{equation}
    \label{eq: 1/(1+delta) approx 1 - delta}
    \frac{1}{1 + \delta}
    = 1 - \delta \pm 2 \delta^2, \quad 
    \forall\; |\delta| \le 1/2, 
  \end{equation}
  to obtain 
  \begin{align*}
    \frac{1}{\norm{\v_k(t+1)}^2}
    &= \frac{1}{\norm{\v_k}^2} 
      \left(
        1 
        + \frac{2 \eta \inprod{\v_k}{\hat\nabla_{\v_k} l}}{\norm{\v_k}^2}  
        + \frac{\eta^2 \norm{\hat\nabla_{\v_k} l}^2}{\norm{\v_k}^2} 
        \pm 8 C^2 \eta^2 \norm{\a}_1^2\log^{\tilde Q}\left( \frac{m d}{ \delta_{\P} } \right) 
      \right) \\
    &= \frac{1}{\norm{\v_k}^2} 
      \left(
        1 
        + \frac{2 \eta \inprod{\v_k}{\hat\nabla_{\v_k} l}}{\norm{\v_k}^2}  
        \pm 2 C^2 \eta^2 d \norm{\a}_1^2\log^{\tilde Q}\left( \frac{m d}{ \delta_{\P} } \right) 
      \right). 
  \end{align*}
  Therefore the update for $\v_{k, p}(t+1)$ is 
  \begin{align*}
    \bar{v}_{k, p}^2(t+1)
    &= \frac{ 
        v_{k, p}^2 
        - 2 \eta v_{k, p} \inprod{\hat\nabla_{\v_k} l}{\e_p}
        + \eta^2 \inprod{\hat\nabla_{\v_k} l}{\e_p}^2
      }{ \norm{\v_k(t+1)}^2 } \\
    &= \left(
        \bar{v}_{k, p}^2 
        - 2 \eta \bar{v}_{k, p} \frac{\inprod{\hat\nabla_{\v_k} l}{\e_p}}{\norm{\v_k}}
        \pm C^2 \eta^2 \norm{\a}_1^2\log^{\tilde Q}\left( \frac{m d}{ \delta_{\P} } \right) 
      \right) 
      \\
      &\qquad\times
        \left(
          1 
          + \frac{2 \eta \inprod{\v_k}{\hat\nabla_{\v_k} l}}{\norm{\v_k}^2}  
          \pm 2 C^2 \eta^2 d \norm{\a}_1^2\log^{\tilde Q}\left( \frac{m d}{ \delta_{\P} } \right) 
        \right) \\
      &= \bar{v}_{k, p}^2 
          - 2 \eta \bar{v}_{k, p} \frac{\inprod{\hat\nabla_{\v_k} l}{\e_p}}{\norm{\v_k}}
          + \frac{2 \eta \inprod{\v_k}{\hat\nabla_{\v_k} l}}{\norm{\v_k}^2}  \bar{v}_{k, p}^2 
        \pm O\left(
          \eta^2 \left( 1 \vee \bar{v}_{k, p}^2 d  \right) \norm{\a}_1^2\log^{\tilde Q}\left( \frac{m d}{ \delta_{\P} } \right)
        \right)  \\
    &= \bar{v}_{k, p}^2 
      - 2 \eta \bar{v}_{k, p} \frac{ \inprod{(\Id - \bar{\v}_k \bar{\v}_k\trans) \hat\nabla l}{\e_p}
      }{\norm{\v_k}}
      \pm O\left(
        \eta^2 \left( 1 \vee \bar{v}_{k, p}^2 d  \right) \norm{\a}_1^2\log^{\tilde Q}\left( \frac{m d}{ \delta_{\P} } \right)
      \right). 
  \end{align*}
  Finally, write $\hat\nabla_{\v_k} l = \nabla\Loss + \H_k$, use our previous formula from Lemma \ref{lemma: population and per-sample gradients} for the tangent term of 
  $\nabla\Loss$, and we obtain 
  \begin{align*}
    \bar{v}_{k, p}^2(t+1)
    &= \bar{v}_{k, p}^2 + 2\eta \bar{v}_{k, p}^2\cdot \sum_{i=I}^\infty 2i \hat\sigma_{2i}^2 \left( a_p \bar{v}_{k, p}^{2i-2} - \sum_{q=1}^P a_q \bar{v}_{k, q}^{2i} \right)
      \\
      &\qquad
      -2\eta \bar{v}_{k, p} \sum_{i=I}^\infty 2 i \hat\sigma_{2i}^2 \sum_{l : l \ne k} \norm{\v_l}^2 
        \inprod{\bar{\v}_k}{\bar{\v}_l}^{2i-1} \inprod{(\Id - \bar{\v}_k\bar{\v}_k\trans) \bar{\v}_l}{\e_p}\\
      &\qquad
      - 2 \eta \bar{v}_{k, p} \frac{ \inprod{(\Id - \bar{\v}_k \bar{\v}_k\trans) \H_k}{\e_p} }{\norm{\v_k}}
      \pm O\left(
        \eta^2 \left( 1 \vee \bar{v}_{k, p}^2 d  \right) \norm{\a}_1^2\log^{\tilde Q}\left( \frac{m d}{ \delta_{\P} } \right)
      \right). 
  \end{align*}
\end{proof}

For notational convenience, we will define the quantity $\Delta := \min(\delta_c, \delta_r, \delta_t)$.

\subsection{Convergence Guarantees}
In this subsection, we show under Induction Hypothesis~\ref{inductionH: gf} that for all $p \in [P_*]$, $\bar v_{p, \pi(p)}^2$ reaches 1 in time $(1 \pm o(1))T_p$.

\subsubsection{Tangent Dynamics}

We begin by tracking  the growth of the signal term $\bar v_{p, \pi(p)}^2$, for $p \in [P_*]$. Our goal is to prove the following lemma.
\begin{lemma}[Directional Convergence]\label{thm:directional-convergence-proof}
    Let $p \in [P_*]$. Inductively assume Induction Hypothesis \ref{inductionH: gf}, and that the conditions on Lemma \ref{lemma: tangent: dynamics of the diagonal entries (stage 1)} hold. Let the target accuracy $\eps_D$ satisfy $\eps_D 
  \ge \frac{ 2^{3I+7}3^I C_\sigma^2 }{\hat\sigma_{2I}^2}
    \braces{
      \bar\eps^{1/2} \eps_0^{I-1} 
      \vee \frac{m \sigma_1^2}{a_{\min_*}}
      \vee \frac{\norm{\a}_1}{a_{\min_*}}  \eps_0^I
    },$ the dimension $d$ satisfy
    \begin{align*}
        \frac{d}{\log^4d} \ge 2^{20}I^2\Delta^{-2}, \quad d \ge \frac{C^2I^2C_{\sigma}^4\Delta^{-4}}{\hat\sigma^4_{2I}},
    \end{align*}
    the learning rate $\eta$ satisfy
    \begin{align*}
        \eta \le \frac{a_{\pi(p)}\hat\sigma_{2I}^2\norm{\a}_1^{-2}\delta_{\P}}{C\log(512I/\Delta)\log^{\tilde Q}\left(\frac{md}{\delta_{\P,\xi}}\right)}\min(d^{-I}\Delta^2, 3^{-I}d^{-1}\eps_D, 3^{-I}\eps_D^2)
    \end{align*}
    for sufficiently large constant $C$. Then, with probability $1 - T_{max}\delta_{\P,\xi} - \delta_{\P}\cdot\log\log d$, we have 
    \begin{align*}
        \bar v_{p, \pi(p)}^2(t) \le \frac{1}{\sqrt{d}}, &\quad \forall t \le \frac{1 - \Delta/256}{4I(I-1)\hat\sigma_{2I}^2\eta a_{\pi(p)}\bar v_{p, \pi(p)}^{2I-2}(0)}\\
        \bar v_{p, \pi(p)}^2(t) \ge 1 - \eps_D, &\quad \forall \frac{1 + \Delta/8}{4I(I-1)\hat\sigma_{2I}^2\eta a_{\pi(p)}\bar v_{p, \pi(p)}^{2I-2}(0)} \le t \le T_{max}.
    \end{align*}
\end{lemma}

The proof of Lemma \ref{thm:directional-convergence-proof} is split into stages based on the size of $\bar v_{p, \pi(p)}^2$. We first consider the case when $\bar v_{p, \pi(p)}^2$ is small. The update is given by the following:
\begin{lemma}\label{lem:signal-update}
Assume that Induction Hypothesis \ref{inductionH: gf} holds, and moreover that $\bar v^2_{p, \pi(p)} \le \delta_{\bar v}$ for some $\delta_{\bar v} > 0$. Let $\delta_T \ge \frac{C_{\sigma}^2\delta_{\bar v}}{I\hat\sigma_{2I}^2}$. Then, under the same conditions as Lemma \ref{lemma: tangent: dynamics of the diagonal entries (stage 1)}, we have
\begin{align*}
\bar v^2_{p, \pi(p)}(t+1) = \bar v^2_{p, \pi(p)}(t) + 4I\hat\sigma_{2I}^2\eta a_{\pi(p)}\bar v_{p, \pi(p)}^{2I}(t) + Z(t+1) + \xi(t+1),
\end{align*}
where $\E[Z(t+1) \mid \mathcal{F}_t] \lesssim \eta^2\norm{\a}_1^2\bar v_{p, \pi(p)}^2$, and with probability $1 - \delta_{\P, \xi}$.
\begin{align*}
    \abs{\xi(t+1)} &\lesssim \eta^2(1 \lor \bar v_{p,\pi(p)}^2d)\norm{\a}_1^2\log^{\tilde Q}\left(\frac{md}{\delta_{\P,\xi}}\right) + \eta \delta_TI\hat\sigma_{2I}^2 a_{\pi(p)}\bar v_{p, \pi(p)}^{2I}
\end{align*}
\end{lemma}
\begin{proof}
    This follows directly from Lemma \ref{lemma: tangent: dynamics of the diagonal entries (stage 1)} and Lemma \ref{lemma: radial and tangent dynamics}.
\end{proof}

This motivates the following stochastic induction helper lemma, with proof deferred to Appendix \ref{subsec: online sgd: deferred proofs}

\begin{lemma}\label{lem:stochastic induction signal term}
    Let $(X_t)_t$ satisfy
    \begin{align}\label{eq:stochastic-induction-signal}
        X_{t+1} = X_t + \alpha X_t^I + \xi_{t+1} + Z_{t+1}, \quad X_0 = x_0,
    \end{align}
    where $(\xi_t)_t$ is an adapted process and $(Z_t)_t$ is a martingale difference sequence. Define the processes $(x_t^+)_t, (x_t^-)_t$ by
    \begin{align*}
        x^+_{t+1} &= \left(1 + \alpha \left(x^+_t\right)^{I-1}\right)x^+_t, \quad x^+_0 = (1 + \eps)x_0\\
        x^-_{t+1} &= \left(1 + \alpha \left(x^-_t\right)^{I-1}\right)x^-_t, \quad x^-_0 = (1 - \eps)x_0.
    \end{align*}
    Suppose that when $X_t \in [x_t^-, x_t^+]$ we have $\abs{\xi_{t+1}} \le X_t^I\Xi_1 + X_t\Xi_2 + \Xi_3$ with probability $1 - \delta_{\P, \xi}$, and $\mathbb{E}[Z_{t+1} \mid \mathcal{F}_t] \le X_t\sigma_Z^2$. Then, if
    \begin{align*}
        \Xi_1 \le \frac{\eps x_0}{6 \sum_{t=0}^{T-1} \hat x_t^{I}},\quad\Xi_2 \le \frac{\eps x_0}{6 \sum_{t=0}^{T-1} \hat x_t},\quad \Xi_3 \le \frac{\eps x_0}{6T},~~\text{and}~~\sigma_Z^2 \le \frac{x_0^2\eps^2\delta_{\P}}{4\sum_{t=0}^{T-1}\hat x_t},
    \end{align*}
    we have $X_t \in [x^-_t, x^+_t]$ for all $t \le T$, with probability $1 - T\delta_{\P, \xi} - \delta_{\P}$.
\end{lemma}

We can use this lemma to bound the time it takes for $\bar v^2_{p, \pi(p)}$ to reach some $\omega(1/d)$ quantity. 

\begin{lemma}[Weak Recovery]\label{lem:weak-recovery}
    Assume that the learning rate $\eta$ satisfies
    \begin{align*}
        \eta \ll \frac{a_{\pi(p)}\hat\sigma_{2I}^2d^{-I}\norm{\a}_1^{-2}\Delta^2\delta_{\P}}{\log(512 I/\Delta)\log^{\tilde Q}\left(\frac{md}{\delta_{\P,\xi}}\right)}.
    \end{align*}
    Moreover, assume that the conditions of Lemma \ref{lemma: tangent: dynamics of the diagonal entries (stage 1)} hold for $\delta_v = d^{-1/2}, \delta_T = \frac{\Delta^2}{CI^2}$ for sufficiently large constant $C$, and also that
    \begin{align*}
        \frac{d}{\log^4d} \ge 2^{20}I^2\Delta^{-2}, \quad d \ge \frac{C^2C_{\sigma}^4I^2\Delta^{-4}}{\hat\sigma^4_{2I}}.
    \end{align*}
    Define $T^+$ by
    \begin{align*}
        T^+ := (1 - \Delta/256)T_p = \frac{1 - \Delta/256}{4I(I-1)\hat\sigma_{2I}^2\eta a_{\pi(p)}\bar v_{p, \pi(p)}^{2I-2}(0)}.
    \end{align*}
    Then with probability $1 - T^+\delta_{\P, \xi} - \delta_{\P}$, 
    \begin{align*}
        \sup_{t \le T^+}v^2_{p, \pi(p)}(t) \le \frac{1}{\sqrt{d}}~~\text{and}~~
        (2/\Delta)^{\frac{1}{I-1}} \cdot \bar v_{p, \pi(p)}^2(0)\le \bar v^2_{p, \pi(p)}(T^+).
    \end{align*}
\end{lemma}

\begin{proof}We will apply Lemma \ref{lem:stochastic induction signal term} to the process with $X_t = \bar v^{2}_{p, \pi(p)}(t)$, $\alpha = 4I\hat\sigma_{2I}^2\eta a_{\pi(p)}$, $\eps = \frac{\Delta}{256I}$. By Lemma \ref{lem:continuous-upper-bound}, the process $(x_t^+)_t$ satisfies 
\begin{align*}
    x_t^+ \le \frac{(1 + \eps)\bar v_{p, \pi(p)}^2(0)}{\left(1 - 4I(I-1)\hat\sigma_{2I}^2\eta a_{\pi(p)} (1 + \eps)^{I-1}\bar v_{p, \pi(p)}^{2I-2}(0) \cdot t\right)^{\frac{1}{I-1}}}
\end{align*} Therefore for 
\begin{align*}
    t \le T^+ \le \frac{1 - I\eps}{4I(I-1)\hat\sigma_{2I}^2\eta a_{\pi(p)}\bar v_{p, \pi(p)}^{2I-2}(0)},
\end{align*}
    we have
\begin{align*}
    (I-1)\alpha \left(\hat x_0^+\right)^{I-1} \cdot t &= 4I(I-1)\hat\sigma_{2I}^2\eta a_{\pi(p)} (1 + \eps)^{I-1}\bar v_{p, \pi(p)}^{2I-2}(0) \cdot t\\
    &\le (1 + \eps)^{I-1}(1 - I\eps)\\
    &\le \exp(-\eps)\\
    &\le 1 - \eps/2.
\end{align*}
Altogether, we can upper bound $\hat x_t^+$ as   
\begin{align*}
    x_t^+ \le \frac{(1 + \eps)\bar v^2_{p, \pi(p)}(0)}{\left(\eps/2\right)^{\frac{1}{I-1}}} \le 4\eps^{-1}\bar v_{p, \pi(p)}^2(0) \le \frac{1}{\sqrt{d}},
\end{align*}
as long as $\frac{d}{\log^4 d} \ge 2^{20}I^2\Delta^{-2}$. As such, if $X_t \le x_t^+$ at time $t$, then the update in Lemma \ref{lem:signal-update} holds for $\delta_{v} = 1/\sqrt{d}$. This update is indeed of the form \eqref{eq:stochastic-induction-signal}; we must now verify that the conditions on $\sigma_Z^2, \Xi_1, \Xi_2, \Xi_3$ indeed hold. Recall that
\begin{align*}
    1 - (I-1)\alpha \left(x_0^+\right)^{I-1} T \ge \eps/2 = \frac{\Delta}{512I}.
\end{align*}
We therefore have that
\begin{align*}
    \sum_{t=0}^{T-1}x_t^+ &\le \int_0^T\frac{x_0^+}{\left(1 - \alpha(I-1) \left(x_0^+\right)^{I-1}t\right)^{\frac{1}{I-1}}} dt\\
    &\le 
    \begin{cases}
        \alpha^{-1}\log\left(\frac{1}{1 - \alpha x_0^+ T}\right) & I = 2\\
        \frac{1}{(I-2)\alpha \left(x_0^+\right)^{I-2}}\left[1 - (1 - \alpha(I-1) \left(x_0^+\right)^{I-1}T^+)^{\frac{I-2}{I-1}}\right] & I > 2
    \end{cases}\\
    &\le 
    \begin{cases}
        \alpha^{-1}\log(512I/\Delta) & I = 2\\
        (I-2)^{-1}\alpha^{-1}(x_0^+)^{2 - I} & I > 2
    \end{cases}.
\end{align*}
and
\begin{align*}
    \sum_{t=0}^{T-1}(x_t^+)^I &\le \int_0^T\frac{\left(x_0^+\right)^I}{\left(1 - \alpha(I-1) \left(x_0^+\right)^{I-1}t\right)^{\frac{I}{I-1}}}dt\\
    &= x_0^+\alpha^{-1}\left(\frac{1}{\left(1 - \alpha(I-1) \left(x_0^+\right)^{I-1}T\right)^{\frac{1}{I-1}}} - 1\right) \\
    & \le x_0^+\alpha^{-1}(\eps/2)^{-\frac{1}{I-1}}.
\end{align*}
The condition on $\sigma_Z^2$ is
\begin{align*}
    \sigma_Z^2 \le \frac{x_0^2\eps^2\delta_{\P}}{4\sum_{t=0}^{T-1}x^+_t} \Longleftarrow \sigma_Z^2 \lesssim x_0^I\Delta^2I^{-2}\delta_{\P}\alpha \cdot \left(\frac{1}{\log(512I/\Delta)} \lor (I-2)\right)
\end{align*}
Since $\sigma_Z^2 \lesssim \eta^2\norm{\a}_1^2$, this is satisfied if we take
\begin{align*}
    \eta \lesssim \frac{a_{\pi(p)}\hat\sigma_{2I}^2d^{-I}\norm{\a}_1^{-2}\Delta^2\delta_{\P}}{\log(512I/\Delta)}.
\end{align*}
Next, observe that $\Xi_1 \lesssim \delta_T \cdot \eta a_{\pi(p)} I \hat\sigma_{2I}^2$. We observe that
\begin{align*}
    \frac{\eps x_0}{6 \sum_{t=0}^{T-1} {x^+_t}^I} \gtrsim \frac{\eps^{\frac{I}{I-1}} x_0 \alpha}{x_0^+} \gtrsim \Delta^{\frac{I}{I-1}} I^{-\frac{I}{I-1}} \cdot \eta a_{\pi(p)} I \hat\sigma_{2I}^2 \gg \Xi_1,
\end{align*}
and thus the condition on $\Xi_1$ is satisfied since $\delta_T = \frac{\Delta^2}{CI^2}$ for a sufficiently large constant $C$. Next, we see that $\Xi_2 = \eta^2d\norm{\a}_1^2\log^{\tilde Q}\left(\frac{md}{\delta_{\P,\xi}}\right)$, and thus we require
\begin{align*}
    \Xi_2 \le \frac{\eps x_0}{6\sum_{t=1}^T\hat x_t} &\Longleftarrow \Xi_2 \lesssim \frac{\Delta I^{-1} x_0^{I-1} \alpha }{\log(512I/\Delta)}\\
    &\Longleftarrow \eta^2d\norm{\a}_1^2\log^{\tilde Q}\left(\frac{md}{\delta_{\P,\xi}}\right) \ll \frac{\Delta d^{-(I-1)} \eta a_{\pi(p)}\hat\sigma_{2I}^2}{\log(512 I/\Delta)}\\
    &\Longleftarrow \eta \ll \frac{a_{\pi(p)}\hat\sigma_{2I}^2d^{-I}\norm{\a}_1^{-2}\Delta}{\log(512I/\Delta)\log^{\tilde Q}\left(\frac{md}{\delta_{\P,\xi}}\right)},
\end{align*}
which is indeed satisfied from our choice of $\eta$. Finally, we see that $\Xi_3 = \eta^2\norm{\a}_1^2\log^{\tilde Q}\left(\frac{md}{\delta_{\P,\xi}}\right)$, and thus we require
\begin{align*}
    \Xi_3 \le \frac{\eps x_0}{6 T} &\Longleftarrow \eta^2\norm{\a}_1^2\log^{\tilde Q}\left(\frac{md}{\delta_{\P,\xi}}\right) \lesssim \Delta (I-1)\hat\sigma_{2I}^2 \eta a_{\pi(p)}x_0^I\\
    &\Longleftarrow \eta \ll \frac{a_{\pi(p)}(I-1)\hat\sigma_{2I}^2d^{-I}\norm{\a}_1^{-2}\Delta}{\log^{\tilde Q}\left(\frac{md}{\delta_{\P,\xi}}\right)}.
\end{align*}
which is again satisfied by our choice of $\eta$. Therefore the conditions of Lemma \ref{lem:stochastic induction signal term} are satisfied, and so with probability $1 - T^+\delta_{\P, \xi} - \delta_{\P}$ we have $X_t \in [x_t^-, x_t^+]$ for all $t \le T^+$.

We conclude by lower bounding $x_t^{-}$. By Lemma \ref{lem:discrete-gronwall-LB},
    \begin{align*}
        x_t^- \ge \frac{x_0^-}{\left(1 - \alpha (I-1) \exp(-\alpha I) \left(x_0^-\right)^{I-1}t\right)^{\frac{1}{I-1}}}.
    \end{align*}
Plugging in $\alpha = 4I\hat\sigma_{2I}^2 \eta a_{p, \pi(p)} \le \eps$, we see that
\begin{align*}
    \alpha (I-1)\exp(-\alpha I)\left(x_0^-\right)^{I-1}T^+ \ge \exp(-\alpha I)\left(\frac{x_0^-}{x_0}\right)^{I-1} \ge \exp(-\alpha I)(1 - \eps)^I \ge 1 - 2I\eps,
\end{align*}
and therefore
\begin{align*}
    x_{T^+}^- \ge \frac{(1 - \eps)x_0}{\left(2I\eps\right)^{\frac{1}{I-1}}} \ge x_0 \cdot \frac{\exp(-\Delta/(128I))}{(\Delta/128)^{\frac{1}{I-1}}} \ge (64/\Delta)^{\frac{1}{I-1}}x_0,
\end{align*}
as desired.

\end{proof}

Next, we bound the time that $\bar v_{p,\pi(p)}^2(t)$ grows to $1/3$. We first introduce the following helper lemma, with proof deferred to Appendix \ref{subsec: online sgd: deferred proofs}.
\begin{lemma}\label{lem:stochastic-induction-LB}
    Let $(X_t)_t$ satisfy
    \begin{align*}
        X_{t+1} \ge X_t + \alpha X_t^I + \xi_{t+1} + Z_{t+1}, \quad X_0 > x_0.
    \end{align*}
    where $(\xi_t)_t$ is an adapted process and $(Z_t)_t$ is a martingale difference sequence. Define the process $\hat x_t$ by 
    \begin{align*}
        \hat x_{t+1} = (1 + \alpha \hat x_t^{I-1})\hat x_t, \quad \hat x_0 = x_0/2.
    \end{align*}
    Suppose that when $\hat x_t \le X_t \le \delta$, we have $\abs{\xi_{t+1}} \le \Xi$ with probability $1 - \delta_{\P, \xi}$ and $\E[Z_{t+1} \mid \mathcal{F}_t] \le \sigma_Z^2$. Then if
    \begin{align*}
        \Xi \le \frac{x_0}{4T},\quad \text{and}~~\sigma_Z^2 \le \frac{x_0^2\delta_{\P}}{16T},
    \end{align*}
    we with probability $1 - T\delta_{\P, \xi} - \delta_{\P}$ either have $X_t \ge \hat x_t$ for all $t \le T$, or $\sup_{t \le T} X_t > \delta$.
\end{lemma}

The following lemma bounds the time it takes for $\bar v_{p, \pi(p)}(t)$ to grow slightly.

\begin{lemma}[Intermediate growth]\label{lem:intermediate growth}
    Let $\delta > 1$. Assume that for some $T_{\delta/d}$, $\bar v_{p, \pi(p)}^2(T_{\delta/d}) \ge \delta/d$. Assume that the learning rate $\eta$ satisfies
    \begin{align*}
        \eta \ll \frac{a_{\pi(p)}I(I-1)\hat\sigma_{2I}^2d^{-I}\norm{\a}_1^{-2}\delta_{\P}}{\log^{\tilde Q}\left(\frac{md}{\delta_{\P, \xi}}\right)}.
    \end{align*}
    Moreover, assume that Induction Hypothesis \ref{inductionH: gf} and the same conditions as Lemma \ref{lemma: tangent: dynamics of the diagonal entries (stage 1)} hold. Then, with probability $1 - T^*_\delta \delta_{\P,\xi} - \delta_P$, there exists some $t \le \frac{d^{I-1}}{2I(I-1)\hat\sigma_{2I}^2\eta a_{\pi(p)}\delta^{I-1}} =: T^*_\delta$ such that
    \begin{align*}
        \bar v_{p, \pi(p)}^2(T_{\delta/d} + t) > \min \left(\frac{\delta^I}{d}, \frac13 \right)
    \end{align*}
\end{lemma}
\begin{proof}
    Define $X_t = \bar v_{p, \pi(p)}^2(T_{\delta/d}+t)$, so that $X_0 \ge \delta/d =: x_0$. For notational convenience, let us define $\bar \delta := \min(\delta^I/d, \frac13)$. Let $T$ be the last time at which $\hat x_t \le \bar \delta$. For $t \le T$, if $X_t \le \bar \delta$, then by Lemma \ref{lemma: gf: tangent dynamics (crude)} and Lemma \ref{lemma: radial and tangent dynamics}, we have
    \begin{align*}
        \bar v_{p, \pi(p)}^2(t+1) \ge \bar v_{p, \pi(p)}^2(t) + 2\eta a_{\pi(p)}I\hat\sigma_{2I}^2\bar v_{p, \pi(p)}^{2I}(t) + Z(t+1) + \xi(t+1),
    \end{align*}
    where  $\E[Z(t+1) \mid \mathcal{F}_t] \lesssim \bar\delta \eta^2\norm{\a}_1^2$ and $\abs{\xi(t+1)} \lesssim \eta^2d\bar\delta\norm{\a}_1^2 \log^{\tilde Q}\left(\frac{md}{\delta_{\P, \xi}}\right)$. We would like to apply Lemma \ref{lem:stochastic-induction-LB} with $\alpha = 2\eta a_{\pi(p)}I\hat\sigma_{2I}^2$.

    By Lemma \ref{lem:discrete-gronwall-LB},
    \begin{align*}
        \bar\delta \ge \hat x_T \ge \frac{\hat x_0}{\left(1 - \alpha (I-1) \exp(-\alpha I) \hat x_0^{I-1}T\right)^{\frac{1}{I-1}}},
    \end{align*}
    and thus
    \begin{align*}
        T \le \frac{\exp(\alpha I)}{\alpha(I - 1)\hat x_0^{I-1}} \le \frac{d^{I-1}}{2I(I-1)\hat\sigma_{2I}^2\eta a_{\pi(p)}\delta^{I-1}}.
    \end{align*}
    We next verify the conditions of the lemma. We first require $\sigma_Z^2 \le \frac{x_0^2\delta_{\P}}{16T}$, or equivalently
    \begin{align*}
        \eta^2\norm{\a}_1^2\bar \delta \lesssim \frac{\delta^2\delta_{\P}}{d^2T} &\Longleftarrow \eta^2\bar \delta \lesssim d^{-(I+1)}\delta^{I+1}\norm{\a}_1^{-2}\delta_{\P} \cdot I(I-1)\hat\sigma_{2I}^2\eta a_{\pi(p)}\\
        &\Longleftarrow \eta \lesssim \bar\delta^{-1}d^{-(I+1)}\delta^{I+1}\norm{\a}_1^{-2}\delta_{\P} \cdot I(I-1)\hat\sigma_{2I}^2\eta a_{\pi(p)}\\
        &\Longleftarrow \eta \lesssim a_{\pi(p)}I(I-1)\hat\sigma_{2I}^2 d^{-I}\delta \norm{\a}_1^{-2} \delta_{\P} 
    \end{align*}
    We additionally require $\Xi \le \frac{x_0}{4T}$. Plugging in $\Xi, x_0, T$, it suffices to take
    \begin{align*}
        &\eta^2d\norm{\a}_1^2\bar\delta \log^{\tilde Q}\left(\frac{md}{\delta_{\P, \xi}}\right) \ll  \delta^Id^{-I} I(I-1)\hat\sigma_{2I}^2 \eta a_{\pi(p)}\\
        &\Longleftarrow \eta \ll \frac{a_{\pi(p)}I(I-1)\hat\sigma_{2I}^2d^{-I}\norm{\a}_1^{-2}}{\log^{\tilde Q}\left(\frac{md}{\delta_{\P, \xi}}\right)},
    \end{align*}
    where we have used the fact that $\bar\delta \le \delta^I/d$. Therefore by Lemma \ref{lem:stochastic-induction-LB}, with high probability we have $X_t \ge \hat x_t$ for all $t \le T$. But this implies that we actually must have $X_t > \bar \delta$ for some $t \le T$, as desired.
\end{proof}

Putting everything together, we can now bound the total time it takes for $\bar v_{p, \pi(p)}^2(t)$ to reach $1/3$.

\begin{lemma}\label{lem:grow-to-12}
    Assume that the conditions of Lemma \ref{lem:weak-recovery} hold. Then, with high probability, there exists some $t \le T = \frac{1 + \Delta/16}{4I(I-1)\hat\sigma_{2I}^2 \eta a_{\pi(p)}\bar v_{p, \pi(p)}^{2I-2}(0)}$ such that $\bar v_{p, \pi(p)}^2(t) \ge \frac13$.
\end{lemma}

\begin{proof}
    On the event that Lemma \ref{lem:weak-recovery} holds, at time $T^+$, we have the bound
    \begin{align*}
        \bar v_{p, \pi(p)}^2(T^+) \ge (64/\Delta)^{\frac{1}{I-1}}\bar v_{p, \pi(p)}^2(0) =: \delta_0/d,
    \end{align*}
    for $\delta_0 := (64/\Delta)^{\frac{1}{I-1}} d\bar v_{p, \pi(p)}^2(0) $. By Lemma \ref{lem:intermediate growth}, with probability $1 - T_\delta^*\delta_{\P,\xi} - \delta_{\P}$, $v_{p, \pi(p)}^2(t)$ grows to a value of $\delta_0^I/d$ in time $t \le \frac{d^{I-1}}{2I(I-1)\hat\sigma_{2I}^2\eta a_{\pi(p)}\delta_0^{I-1}}$. Repeatedly applying this lemma for at most $\log\log d$ iterations we get that $v_{p, \pi(p)}^2(t)$ grows to be at least $\frac13$ in time
    \begin{align*}
        \sum_{k=0}^{\infty}\frac{d^{I-1}}{2I(I-1)\hat\sigma_{2I}^2\eta a_{\pi(p)}\delta_0^{(I-1)I^k}} &= \frac{d^{I-1}}{2I(I-1)\hat\sigma_{2I}^2\eta a_{\pi(p)}}\sum_{k=0}^\infty \delta_0^{-(I-1)I^k}\\
        &\le \frac{d^{I-1}}{I(I-1)\hat\sigma_{2I}^2\eta a_{\pi(p)}\delta_0^{I-1}}\\
        &= \frac{\Delta/64}{I(I-1)\hat\sigma_{2I}^2 \eta a_{\pi(p)}\bar v_{p, \pi(p)}^{2I-2}(0)}\\
        &\le \frac{\Delta/16}{4I(I-1)\hat\sigma_{2I}^2 \eta a_{\pi(p)}\bar v_{p, \pi(p)}^{2I-2}(0)}
    \end{align*}
    with total failure probability at most $T\delta_{\P, \xi} + \delta_{\P}\log\log d $.
\end{proof}

    Finally, we can lower bound the time it takes for $\bar v_{p, \pi(p)}^2$ to grow from $\frac12$ to $1 - \eps_D$. The proof of the following is deferred to Appendix \ref{subsec: online sgd: deferred proofs}.

\begin{lemma}\label{lem:strong-recovery-helper}
    Let $(X_t)_t \ge 0$ satisfy
    \begin{align*}
        X_{t+1} \le (1 - \alpha)X_t + \xi_{t+1} + Z_{t+1}, \quad X_0 = x_0
    \end{align*}
    where $(\xi_t)_t$ is an adapted process and $(Z_t)_t$ is a martingale difference sequence, and with probability $1 - \delta_{\P,\xi}$ we have $\abs{\xi_{t+1}} \le \Xi$ and $\E[Z_{t+1} \mid \mathcal{F}_t] \le \sigma_Z^2$ when $X_t \le 1.5x_0$. Then, if
    \begin{align*}
        \Xi \le \frac{\eps\alpha}{4}, \quad \sigma_Z^2 \le \frac{\eps^2\alpha\delta_{\P}}{16}
    \end{align*}
    we have with probability $1 - T \delta_{\P, \xi} - \delta_{\P}$.
    \begin{align*}
        X_t \le (1 - \alpha)^t x_0 + \eps/2 \le 1.5 x_0
    \end{align*}
    for all $t \le T$.
\end{lemma}

\begin{lemma}[Strong Recovery]\label{lem:strong-convergence}
    Let us assume that Lemma \ref{lem:grow-to-12} holds, i.e for some time $T_{1/3}$, $\bar v_{p, \pi(p)}^2(T_{1/3}) \ge \frac13$. Let the target accuracy $\eps_D$ satisfy the same condition as in Lemma \ref{lemma: tangent: dynamics of the diagonal entries (stage 3)}. Choose $\eta$ so that
    \begin{align*}
        \eta \ll \frac{a_{\pi(p)}I\hat\sigma_{2I}^23^{-I}\norm{\a}_1^{-2}\delta_{\P}}{\log^{\tilde Q}\left(\frac{md}{\delta_{\P, \xi}}\right)}\min(d^{-1}\eps_D, \eps_D^2)
    \end{align*}
    Then with probability $1 - T\delta_{\P,\xi} - \delta_{\P}$, we have
    \begin{align*}
        \bar v_{p, \pi(p)}^2(t) \ge 1 - \eps_D, \quad \forall \frac{3^I}{I\hat\sigma_{2I}^2 \eta a_{\pi(p)}}\log(2/\eps_D) \le t \le T.
    \end{align*}
\end{lemma}

\begin{proof}
    By Lemma \ref{lemma: radial and tangent dynamics} and Lemma \ref{lemma: tangent: dynamics of the diagonal entries (stage 3)}, when $\bar v_{p, \pi(p)}^2(t) \ge \frac13$ we have
    \begin{align*}
        \bar v_{p, \pi(p)}^2(t+1) &\ge \bar v_{p, \pi(p)}^2(t) + 3^{-I}I\hat\sigma_{2I}^2\eta a_{\pi(p)}(1 - \bar v^2_{p, \pi(p)}(t))+ \xi_{t+1} + Z_{t+1}
    \end{align*}
    where
    \begin{align*}
        \abs{\xi_{t+1}} \lesssim \eta^2\norm{\a}_1^2d \log^{\tilde Q}\left(\frac{md}{\delta_{\P, \xi}}\right),\quad \E[Z_{t+1}^2 \mid \mathcal{F}_t] \lesssim \eta^2\norm{\a}_1^2.
    \end{align*}
    We would like to apply Lemma \ref{lem:strong-recovery-helper}, with $\alpha = 3^{-I}I\hat\sigma_{2I}^2\eta a_{\pi(p)}$ and $X_t = 1 - \bar v_{p, \pi(p)}^2(T_{1/2} + t)$, $\eps = \eps_D$. We first require $\Xi \le \frac{\eps\alpha}{4}$, which is satisfied by taking
    \begin{align*}
        &\eta^2\norm{\a}_1^2 d \log^{\tilde Q}\left(\frac{md}{\delta_{\P, \xi}}\right) \lesssim 3^{-I}I\hat\sigma_{2I}^2\eta a_{\pi(p)}\eps\\
        &\Longleftarrow \eta \ll \frac{a_{\pi(p)}I\hat\sigma_{2I}^23^{-I}d^{-1}\norm{\a}_1^{-2}\eps }{\log^{\tilde Q}\left(\frac{md}{\delta_{\P, \xi}}\right)}.
    \end{align*}
    Next, we require $\sigma_Z^2 \le \eps^2\alpha\delta_{\P}/16$, which is obtained by taking
    \begin{align*}
        \eta^2\norm{\a}_1^2 \lesssim \eps^23^{-I}I\hat\sigma_{2I}^2\eta a_{\pi(p)} \delta_{\P} \Longleftarrow \eta \lesssim a_{\pi(p)}I\hat\sigma_{2I}^23^{-I}\norm{\a}_1^{-2}\eps^2\delta_{\P}.
    \end{align*}
    Altogether, with high probability,
    \begin{align*}
        1 - \bar v_{p, \pi(p)}^2(T_{1/2} + t) \le (1 - \alpha)^t\cdot\frac12 + \eps/2 \le \eps
    \end{align*}
    for $t \ge \alpha^{-1}\log(2/\eps) = \frac{3^I}{I\hat\sigma_{2I}^2 \eta a_{\pi(p)}}\log(2/\eps)$.
\end{proof}

\begin{proof}[Proof of Theorem \ref{thm:directional-convergence-proof}]
This follows directly from combining Lemma \ref{lem:weak-recovery}, Lemma \ref{lem:grow-to-12}, and Lemma \ref{lem:strong-convergence}, and noting that
\begin{align*}
    \frac{3^I}{I\hat\sigma_{2I}^2\eta a_{\pi(p)}} \le \frac{\Delta/16}{4I(I-1)\hat\sigma_{2I}^2\eta a_{\pi(p)}\bar v_{p, \pi(p)}^{2I-2}(0)}.
\end{align*}
    
\end{proof}

\subsubsection{Radial Dynamics}

In this subsection, we analyze the dynamics of $\norm{\v_p}^2$, when $\bar v_{p, \pi(p)}^2(t) \ge 1 - \bar\eps$. In this regime, the update on the norm is given by the following.

\begin{lemma}\label{lem:norm-update-large}
    Assume that $\bar v^2_{p, \pi(p)}(t) \ge 1 - \bar\eps$. Then
    \begin{align*}
        \norm{\v_p(t+1)}^2 = \norm{\v_p(t)}^2 + 4\eta\norm{\v_p(t)}^2\left(a_{\pi(p)} - \norm{\v_p(t)}^2\right) + Z_{t+1} + \xi_{t+1}
    \end{align*}
    where with probability $1 - \delta_{\P, \xi}$
    \begin{align*}
        \E[Z_{t+1}^2 \mid \mathcal{F}_t] &\lesssim \eta^2\norm{\a}_1^2\norm{\v_p(t)}^4\\
        \abs{\xi_{t+1}} &\lesssim \left(\eta^2d\norm{\a}_1^2\log^{\tilde Q}(md/\delta_{\P,\xi}) + \eta(C_{\sigma}^2a_{\pi(p)}\bar\eps + \norm{\a}_12^{2I}\eps_0^I + m\sigma_1^2)\right)\norm{\v_p(t)}^2.
    \end{align*}
\end{lemma}
\begin{proof}
    This follows directly from Lemma \ref{lemma: dynamics of the norm (converged)} and Lemma \ref{lemma: radial and tangent dynamics}.
\end{proof}

We would like to prove that Inductive Hypothesis~\ref{inductH:learning-time}\ref{inductH-itm:norm-convergence} holds, assuming that \ref{inductH:learning-time}\ref{inductH-itm:directional-convergence} holds. This is given by the following result.
\begin{lemma}\label{lem:norm-growth-complete}
    Assume that Inductive Hypothesis~\ref{inductionH: gf} and Inductive Hypothesis~\ref{inductH:learning-time}\ref{inductH-itm:directional-convergence} hold. Let $T_{1 - \bar\eps} \le \frac{1 + \Delta/8}{4I(I-1)\hat\sigma_{2I}^2\eta a_{\pi(p)}v_{p, \pi(p)}^{2I-2}(0)}$ be some time at which $\bar v_{p,\pi(p)}^2 \ge 1 - \bar\eps$. Let the learning rate $\eta$ and target accuracy $\eps_R$ satisfy
    \begin{align*}
        \eta \lesssim \frac{\norm{\a}_1^{-2}}{\log(2a_k/\sigma_0^2)}\min\left(\frac{a_{\min_*}d^{-1}\eps_R}{\log^{\tilde Q}(md/\delta_{\P,\xi})}, \eps_R^2\delta_{\P}\right),\quad \eps_R \gtrsim \log(2a_k/\sigma_0^2)\left(C_{\sigma}^2a_{\pi(p)}\bar\eps + \norm{\a}_12^{2I}\eps_0^I + m\sigma_1^2\right),
    \end{align*}
    Then, with probability $1 - T_{max}\delta_{\P,\xi} - \delta_{\P}$,
    \begin{align*}
        \abs{\norm{\v_p(t)}^2 - a_k} \le \eps_R,\quad \forall~~T_{1 - \bar\eps} + \frac{\Delta/8}{4I(I-1)\hat\sigma_{2I}^2\eta a_{\pi(p)}\bar v_{p, \pi(p)}^{2I-2}(0)} \le t \le T_{1 - \bar\eps} + T_{max}.
    \end{align*}
\end{lemma}

To prove this lemma, we first lower bound the time it takes for $\norm{\v_{p}(t)}^2$ to reach $\delta a_{\pi(p)}$ for some small quantity $\delta a_p$. We start by proving the following helper lemma, which resembles Lemma F.6 from \cite{ren2024learning} and whose proof is deferred to Appendix \ref{subsec: online sgd: deferred proofs}.

\begin{lemma}\label{lem:stochastic-gronwall-linear}
    Let $(X_t)_t$ satisfy
    \begin{align*}
        X_{t+1} = (1 + \alpha)X_t + \xi_{t+1} + Z_{t+1}, \quad X_0 = x_0 > 0,
    \end{align*}
    where $(\xi_t)_t$ is an adapted process and $(Z_t)_t$ is a martingale difference sequence. Define $x_t = (1 + \alpha)^tx_0$. Suppose that if $X_t = (1 \pm 0.5)x_t$, then $\abs{\xi_{t+1}} \le x_t\Xi$ with probability $1 - \delta_{\P,\xi}$ and $\E[Z_{t+1}^2 \mid \mathcal{F}_t] \le x_t^2\sigma_Z^2$. Then, if
    \begin{align*}
        \Xi \le \frac{1}{4T}, \quad \sigma_Z^2 \le \frac{\delta_{\P}}{16 T}.
    \end{align*}
    then we have with probability $1 - T\delta_{\P,\xi} - \delta_{\P}$  that $X_t = (1 \pm 0.5)x_t$ for all $t \le T$.
\end{lemma}

The following lemma then lower bounds the escape time.

\begin{lemma}\label{lem:norm-growth-small}
    Let $\delta = \frac{1}{C\log(2 a_{\pi(p)}/\sigma_0^2)}$, for sufficiently large constant $C$. Define $T = \frac{\log(2\delta a_{\pi(p)}/\sigma_0^2)}{4\eta a_{\pi(p)}} \le \frac{\delta^{-1}}{4C\eta a_{\pi(p)}}$. Let the learning rate satisfy $\eta \lesssim \frac{a_{\pi(p)}d^{-1}\norm{\a}_1^{-2}\delta_{\P}\delta}{\log^{\tilde Q}(md/\delta_{\P,\xi})}$. With probability $1 - \delta_{\P,\xi} - T\delta_{\P}$, we have $\sup_{t \le T} \norm{\v_p(T_{1-\bar\eps} + t)}^2 \ge \delta a_{\pi(p)}$.
\end{lemma}
\begin{proof}
    When $\norm{\v_p(t)}^2 \le \delta a_{\pi(p)}$, we can bound
    \begin{align*}
        \norm{\v_p(t+1)}^2 = \norm{\v_p(t)}^2 + 4\eta a_{\pi(p)}\norm{\v_p(t)}^2 + Z_{t+1} + \xi_{t+1},
    \end{align*}
    where
    \begin{align*}
        \abs{\xi_{t+1}} &\lesssim \eta\delta a_{\pi(p)}\norm{\v_p(t)}^2,\quad \E[Z_{t+1}^2 \mid \mathcal{F}_t] \lesssim \eta^2\norm{\a}_1^2\norm{\v_p(t)}^4,
    \end{align*}
    provided that
    \begin{align*}
        &\delta \gtrsim a_{\min_*}^{-1}\left(\eta d \norm{\a}_1^2\log^{\tilde Q}(md/\delta_{\P, \xi}) + C_{\sigma}^2a_{\pi(p)}\bar\eps + \norm{\a}_12^{2I}\eps_0^I + m\sigma_1^2\right)\\
        &\Longleftarrow \eta \lesssim \frac{a_{\min_*}d^{-1}\delta \norm{\a}_1^{-2}}{\log^{\tilde Q}(md/\delta_{\P,\xi})}, \quad \delta^{-1}\left(C_{\sigma}^2a_{\pi(p)}\bar\eps + \norm{\a}_12^{2I}\eps_0^I + m\sigma_1^2\right) \lesssim 1.
    \end{align*}
    Define the process $X_t = \norm{\v_p(T_{1 - \bar\eps} + t)}^2$, where $x_0 = \norm{\v_p(T_{1 - \bar\eps})}^2$ and $\alpha = 4\eta a_{\pi(p)}$. Assume that $\sup_{t \le T}X_t < \delta a_{\pi(p)}$. We can thus apply Lemma \ref{lem:stochastic-gronwall-linear}, since the conditions on $\sigma_Z^2, \Xi$ are indeed met:
    \begin{align*}
        \sigma_Z^2 \le \frac{\delta_{\P}}{16 T} &\Longleftarrow \eta \ll a_{\pi(p)}\norm{\a}_1^{-2}\delta_{\P}\delta\\
        \Xi \le \frac{1}{4T} &\Longleftarrow 1 \ll C.
    \end{align*}
    But recall that for the process $x_t = (1 + \alpha)^tx_0$, for $T = \alpha^{-1}\log(2\delta a_{\pi(p)}/x_0)$ we have $x_T \ge 2\delta a_{\pi(p)}$ and thus $X_T > \delta a_{\pi(p)}$, a contradiction. Therefore there exists $t \le T$ such that $X_t \ge \delta a_{\pi(p)}$, as desired.
\end{proof}

We next introduce the following helper lemma, with proof deferred to Appendix \ref{subsec: online sgd: deferred proofs}.
\begin{lemma}\label{lem:norm-convergence-helper}
    Let $(X_t)_t$ satisfy
    \begin{align*}
        X_{t+1} = (1 - \alpha(X_t))X_t + \xi_{t+1} + Z_{t+1},
    \end{align*}
    where $(\xi_t)_t$ is an adapted process and $(Z_t)_t$ is a martingale difference sequence, and with probability $1 - \delta_{\P,\xi}$ we have $\alpha(X_t) \in [\alpha_-, \alpha_+]$, $\abs{\xi_{t+1}} \le \Xi$ and $\E[Z_{t+1}^2 \mid \mathcal{F}_t] \le \sigma_Z^2$ when $X_t \in [-\eps/2, x_0 + \eps/2]$. Then, if for some $\eps \in (0, x_0)$
    \begin{align*}
        \Xi \le \frac{\eps\alpha_-}{4}, \quad \sigma_Z^2 \le \frac{\eps^2\alpha_-\delta_{\P}}{16},
    \end{align*}
    we have with probability $1 - T\delta_{\P,\xi} - \delta_{\P}$ that
    \begin{align*}
        (1 - \alpha_+)^tx_0 - \eps/2 \le X_t \le (1 - \alpha_-)^tx_0 + \eps/2
    \end{align*}
    for all $t \le T$.
\end{lemma}

The following lemma bounds the time it takes for the norm to grow from $\delta a_{\pi(k)}$ to approximately $a_{\pi(k)}$, and furthermore establishes that it stays close to $a_{\pi(k)}$
\begin{lemma}\label{lem:norm-growth-large}
    Inductively assume that Induction Hypothesis \ref{inductionH: gf} and Induction Hypothesis \ref{inductH:learning-time}\ref{inductH-itm:directional-convergence} are true. Pick $\delta > 0$, and let $T^*$ be some time at which $\norm{\v_k(T^*)}^2 \in [\delta a_k, a_k/2]$. Let $\eps_R > 0$ be the target accuracy. If
    \begin{align*}
        \eta \lesssim \eps_R^2\norm{\a}_1^{-2}\delta_{\P}\delta \land  \frac{\eps_Rd^{-1}\norm{\a}_1^{-2}\delta}{\log^{\tilde Q}(md/\delta_{\P, \xi})}, \quad \eps_R \gtrsim \delta^{-1}\left(C_{\sigma}^2a_{\pi(p)}\bar\eps + \norm{\a}_12^{2I}\eps_0^I + m\sigma_1^2\right),
    \end{align*}
    then we have with probability $1 - T_{max}\delta_{\P,\xi} - \delta_{\P}$ that
    \begin{align*}
        \norm{\v_p(t)}^2 \in [a_{\pi(p)} - \eps_R, a_{\pi(p)} + \eps_R]~~\forall~ T^* + \frac{2\log(a_{\pi(p)}/\eps_R)}{\delta\eta a_{\pi(p)}}\le t \le T^* + T_{max}.
    \end{align*}
\end{lemma}
\begin{proof}
    Assume that the inductive hypothesis holds at time $t$. By Lemma \ref{lem:norm-update-large}, we have that
    \begin{align*}
        \norm{\v_p(t+1)}^2 &= \norm{\v_p(t)}^2 + 4\eta \norm{\v_p(t)}^2(a_{\pi(p)} - \norm{\v_p(t)}^2) + Z_{t+1} + \xi_{t+1}
    \end{align*}
    for $\E[Z_{t+1}^2 \mid \mathcal{F}_t] \lesssim \eta^2\norm{\a}_1^2\norm{\v_p(t)}^4 \lesssim \eta^2\norm{\a}_1^2a_{\pi(p)}^2$ and $$\abs{\xi_{t+1}} \lesssim \left(\eta^2d\norm{\a}_1^2\log^{\tilde Q}(md/\delta_{\P,\xi}) + \eta(C_{\sigma}^2a_{\pi(p)}\bar\eps + \norm{\a}_12^{2I}\eps_0^I + m\sigma_1^2)\right)a_{\pi(p)}.$$ Therefore
    \begin{align*}
        a_{\pi(p)} - \norm{\v_p(t+1)}^2 = \left(1 - 4\eta \norm{\v_p(t)}^2\right)\left(a_{\pi(p)} - \norm{\v_p(t)}^2\right) + Z_{t+1} + \xi_{t+1}.
    \end{align*}
    We thus would like to apply Lemma \ref{lem:norm-convergence-helper} to the process $X_t = a_{\pi(p)} - \norm{\v_p(t+T^*)}^2$, with $\eps = \eps_R$. We see that $x_0 \in [a_{\pi(p)}/2, (1-\delta)a_{\pi(p)}]$, so for $X_t \in [-\eps_R/2, (1 - \delta/2)a_{\pi(p)}]$ we can bound
    \begin{align*}
        \frac{\delta a_{\pi(p)}}{2}\le \norm{\v_p(t)}^2 \le 2a_{\pi(p)}.
    \end{align*}
    Therefore the conditions of Lemma \ref{lem:norm-convergence-helper} are indeed satisfied. It thus suffices to take
    \begin{align*}
        \Xi \le \frac{\eps\alpha_-}{4} &\Longleftarrow \left(\eta^2d\norm{\a}_1^2\log^{\tilde Q}(md/\delta_{\P,\xi}) + \eta(C_{\sigma}^2a_{\pi(p)}\bar\eps + \norm{\a}_12^{2I}\eps_0^I + m\sigma_1^2)\right)a_{\pi(p)} \lesssim \eta \delta a_{\pi(p)}\eps_R\\
        &\Longleftarrow \eta \lesssim \frac{\eps_Rd^{-1}\norm{\a}_1^{-2}\delta}{\log^{\tilde Q}(md/\delta_{\P, \xi})}, \quad \eps_R \gtrsim \delta^{-1}\left(C_{\sigma}^2a_{\pi(p)}\bar\eps + \norm{\a}_12^{2I}\eps_0^I + m\sigma_1^2\right).
    \end{align*}
    as well as
    \begin{align*}
        \sigma_Z^2 \le \frac{\eps_R^2\alpha_- \delta_{\P}}{16} &\Longleftarrow \eta^2 \norm{\a}_1^2a_{\pi(p)}^2 \lesssim \eps_R^2 \eta a_{\pi(p)}\delta \delta_{\P}\\
        &\Longleftarrow \eta \lesssim \frac{\eps_R^2\norm{\a}_1^{-2}\delta_{\P}\delta}{a_{\pi(p)}}.
    \end{align*}
    Altogether, by Lemma \ref{lem:norm-convergence-helper} with high probability we have
    \begin{align*}
        (1 - \alpha_+)^tx_0 - \eps_R/2 \le X_t \le (1 - \alpha_-)^tx_0 + \eps_R/2
    \end{align*}
    Naively, we have the bound $X_t \ge -\eps_R/2$, which implies $\norm{\v_p(t)}^2 \le a_{\pi(p)} + \eps_R/2$. Moreover, for $t \ge \frac{2\log(a_p/\eps_R)}{\delta\eta a_{\pi(p)}} \ge \alpha_-^{-1}\log(2x_0/\eps_R)$, we have $X_t \le \eps_R$.
\end{proof}

Putting everything together, we can prove Lemma \ref{lem:norm-growth-complete}.

\begin{proof}[Proof of Lemma \ref{lem:norm-growth-complete}]
    We apply Lemma \ref{lem:norm-growth-small} and Lemma \ref{lem:norm-growth-large} with $\delta = \frac{1}{C\log(2a_k/\delta_0^2)}$. The conditions on $\eta, \bar\eps$ are indeed satisfied, and moreover $\norm{\v_p(t)}^2$ reaches the interval $[a_{\pi(p)} - \eps_R, a_{\pi(p)} + \eps_R]$ within a time of 
    \begin{align*}
        \frac{\log(2\delta a_{\pi(p)}/\sigma_0^2)}{4\eta a_{\pi(p)}} + \frac{2\log(a_{\pi(p)}/\eps_R)}{\delta \eta a_{\pi(p)}} &\le \frac{\log(a_{\pi(p)}/\sigma_0^2) + 2C\log(2a_{\pi(p)}/\sigma_0^2)\log(a_k/\eps_R)}{\eta a_{\pi(p)}}\\
        &\ll \frac{\Delta/8}{4I(I-1)\hat\sigma_{2I}^2\eta a_{\pi(p)}\bar v_{p, \pi(p)}^{2I-2}(0)}.
    \end{align*}
\end{proof}

\subsection{Maintaining the Induction Hypotheses}

\subsubsection{Upper Bounds on the Irrelevant Coordinates}

We first track the growth of a failed coordinate $\bar v_{k, \pi(q)}$ for $(k, \pi(q)) \not \in \{(p, \pi(p)\}_{p \in [P_*]}$. The update on $\bar v_{k, \pi(q)}(t)$ is given by the following.
\begin{lemma}\label{lem:growth failed coords}
   Assume that Induction Hypothesis \ref{inductionH: gf} holds at time $t$. Then
    \begin{align*}
        \bar v_{k, \pi(q)}^2(t+1) &\le \bar v^2_{k, \pi(q)}(t) + 4I\hat \sigma_{2I}^2\eta \bar v_{k, \pi(q)}^{2I}\abs{a_{\pi(q)} - \indi(q \in [m], q \in L)\norm{\v_q}^2} + Z(t+1) + \xi(t+1),
    \end{align*}
    where $\E[Z(t+1) \mid \mathcal{F}_t] \lesssim \eta^2\norm{\a}_1^2\bar v^2_{k, \pi(q)}(t)$, and 
    \begin{align*}
        \abs{\xi(t+1)} \lesssim \eta^2(1 + \bar v^2_{k, \pi(q)}(t)d)\norm{\a}_1^2\log^{\tilde Q}(md/\delta_{\P,\xi}) + C_{\sigma}^2\eta a_{\pi(q)}\bar v_{k, \pi(q)}^{2I}\eps_0 + \eta \abs{\bar v_{k, \pi(q)}(t)}\delta_{\mathrm{error}},
    \end{align*}
    where
    \begin{align*}
        \delta_{\mathrm{error}} := I2^{3I + 6}C_{\sigma}^2\left(a_{\pi(q)}\bar\eps^{1/2}\eps_0^{I-1}  \lor m \sigma^2_1 \lor \norm{\a}_1\eps_0^I\right)
    \end{align*}
\end{lemma}
\begin{proof}
    From the proof of Lemma \ref{lemma: tangent: upper triangular entries (case I)}, we have that
    \begin{align*}
        \frac{d}{dt}\bar v_{k, \pi(q)}^2 \le 4I\hat \sigma_{2I}^2a_{\pi(q)} \bar v_{k, \pi(q)}^{2I} + 2C_{\sigma}^2a_{\pi(q)}\bar v_{k, \pi(q)}^{2I}\eps_0 + \abs{\bar v_{k, \pi(q)}}\delta_{\mathrm{error}},
    \end{align*}
    and so the desired result follows directly from combining the above with Lemma \ref{lemma: radial and tangent dynamics}.
\end{proof}

We will next require the following stochastic induction helper lemma, with proof deferred to Appendix \ref{subsec: online sgd: deferred proofs}.
\begin{lemma}\label{lem: failed coords stochastic induction}
    Suppose that $(X_t)_t \ge 0$ satisfies
    \begin{align}\label{eq:simple-stochastic-update}
        X_{t+1} \le X_t + \alpha X_t^{I} + \xi_{t+1} + Z_{t+1}, \quad X_0 \le x_0,
    \end{align}
    where $(\xi_t)_t$ is an adapted process and $(Z_t)_t$ is a martingale difference sequence. Let $\hat x_{t}$ be a solution to the recurrence
    \begin{align*}
        \hat x_{t+1} = \hat x_t + \alpha \hat x_t^{I}, \quad \hat x_0 = (1 + \epsilon)x_0
    \end{align*}
    Suppose that when $X_t \le \hat x_t$, we have $\abs{\xi_{t+1}} \le X_t^{1/2}\Xi_1 + X_t\Xi_2 + X_t^I\Xi_3 + \Xi_4$ with probability $1 - \delta_{\P,\xi}$ and $\E[Z(t+1) \mid \mathcal{F}_t] \le X_t\sigma_Z^2$. Then if
    \begin{align}\label{eq:conditions-stochastic-update}
        \Xi_1 \le \frac{\epsilon x_0}{8 \sum_{t=0}^{T-1} \hat x_t^{1/2}},\quad \Xi_2 \le \frac{\epsilon x_0}{8 \sum_{t=0}^{T-1} \hat x_t},\quad \Xi_3 \le \frac{\epsilon x_0}{8 \sum_{t=0}^{T-1} \hat x_t^I},\quad \Xi_4 \le \frac{\epsilon x_0}{8T},~~\text{and}~~\sigma_Z^2 \le \frac{x_0^2\epsilon^2\delta_{\P}}{4\sum_{t=0}^{T-1}\hat x_t},
    \end{align}
    we have $X_t \le \hat x_t$ for all $t \le T$ with probability $1 - T\delta_{\P, \xi} - \delta_{\P}$.
\end{lemma}

We can now control the growth of $\bar v_{k, \pi(q)}$ by applying Lemma \ref{lem: failed coords stochastic induction} with $X_t = \bar v^2_{k, \pi(q)}(t)$. For $(k, \pi(q)) \not\in \{(p, \pi(p))\}_{p \in [P_*]}$, define the time $T_{(k, \pi(q))}$ by
\begin{align*}
    T_{(k, \pi(q))} := \begin{cases}
        T_k & k < q, k \in [P^*]\\
        T_q & q < k, q \in [P^*]\\
        T_{P_*} & k, q > P^*
    \end{cases}.
\end{align*}
By Assumption \ref{assumption: gf init}\ref{assumption-itm: init: regularity conditions}, we have that 
\begin{align*}
    T_{(k, \pi(q))} \le  \frac{(1 + \Delta/4)d^{I-1}}{4I(I-1)\hat \sigma_{2I}^2\eta a_{\min_*}}.
\end{align*}

\begin{lemma}[Total growth of failed coordinates]\label{lem:total-growth-failed-coords}
Let $(k, \pi(q)) \not\in \{(p, \pi(p))\}_{p \in [P_*]}$. Assume that the learning rate $\eta$ satisfies
\begin{align*}
    \eta \le \frac{a_{\min_*}I\hat \sigma_{2I}^2 d^{-I}\norm{\a}_1^{-2}\Delta^2\delta_{\P}}{I\log(4/\Delta)\log^{\tilde Q}(md/\delta_{\P, \xi})}
\end{align*}
for some sufficiently large constant $C$. Furthermore, suppose that
\[
        \bar\eps^{1/2}\eps_0^{I-1}  \lor m \sigma^2_1 \lor \norm{\a}_1\eps_0^I \ll \frac{a_{\min_*}\hat\sigma_{2I}^2d^{-I+1/2}\Delta}{I2^{3I + 6}C_{\sigma}^2}, \quad \eps_R \lesssim \frac{a_{\min_*}}{1.5^Id^{\gamma(I-1)}}, \quad \frac{d}{\log^{2/\gamma} d} \ge 2^{1/\gamma}(4/\Delta)^{\frac{1}{\gamma(I-1)}}.
  \]
Then, with probability $1 - T_{P_*}\delta_{\P,\xi} - 2\delta_{\P}$, we have that $\bar v^2_{k, \pi(q)} \le \eps_0$ (and hence Induction Hypothesis \ref{inductionH: gf}\ref{inductH-itm: bound on the failed coordinates} is true) for all $t \le T_{P_*}$.
\end{lemma}

\begin{proof} First, we will show that $\bar v_{k, \pi(q)}^2(t) \le \eps_0/2$ up to time $T_{(k, \pi(q))}$. Next, we will show that $v_{k, \pi(q)}^2(t)$ does not grow too much more in the interval $[T_{(k, \pi(q))}, T_{P_*}]$.

    \paragraph{Part 1 $(t \le T_{(k, \pi(q))})$.}
    
    Our goal will be to apply Lemma \ref{lem: failed coords stochastic induction} up to time $T = T_{(k, \pi(q))}$, to the process $X_t = \bar v^2_{k, \pi(q)}(t)$, with $\alpha = 4I\hat\sigma_{2I}^2\eta a_{\pi(q)}$, $\eps = \frac{\Delta}{4I}$, and $x_0 = \max(\frac{1}{2d}, \bar v_{k, \pi(q)}^2)$. 
    
    We first aim to bound the quantity $\alpha(I-1) \hat x_0^{I-1} T$. We begin by considering the upper triangular entries, i.e those where $k < q$ and $k \in [P^*]$, in which case $T_{(k, \pi(q))} = T_k$. We have that
    \begin{align*}
        \alpha(I-1) \hat x_0^{I-1} T &= 4I(I-1)\hat\sigma_{2I}^2\eta a_{\pi(q)}\cdot (1 + \eps)^{I-1} \max\left(\frac{1}{2d}, \bar v_{k, \pi(q)}^2\right)^{I-1} \cdot \frac{1 + \Delta/4}{4I(I-1)\hat\sigma_{2I}^2\eta a_{\pi(k)}\bar v^{2I-2}_{k, \pi(k)}(0)}\\
        &\le (1 + \eps)^{I-1}(1 + \Delta/4) \frac{a_{\pi(q)}\max\left(\left(\frac{1}{2d}\right)^{I-1}, \bar v_{k, \pi(q)}^{2I-2}\right)}{a_{\pi(k)}\bar v^{2I-2}_{k, \pi(k)}(0)}.
    \end{align*}
    By the bound on the row gap in Assumption \ref{assumption: gf init}\ref{assumption-itm: init: row gap}, we have that $\frac{a_{\pi(q)}v_{k, \pi(q)}^{2I-2}}{a_{\pi(k)}\bar v^{2I-2}_{k, \pi(k)}(0)} \le \frac{1}{1 + \Delta}$. Moreover, by the definition of the greedy maximum selection process along with Assumption \ref{assumption: gf init}\ref{assumption-itm: init: regularity conditions}, $a_{\pi(k)}\bar v^{2I-2}_{k, \pi(k)}(0) \ge a_{\pi(q)}\max_{j > k}\bar v^{2I-2}_{j, \pi(q)}(0) \ge  a_{\pi(q)}/d^{I-1}$, and thus $\frac{a_{\pi(q)}\cdot 1/(2d)^{I-1}}{a_{\pi(k)}\bar v^2_{k, \pi(k)}(0)} \le \frac{1}{2^{I-1}} \le \frac{1}{1 + \Delta}$.
    Altogether,
    \begin{align*}
        \alpha(I-1) \hat x_0^{I-1} T \le \frac{(1 + \eps)^{I-1}(1 + \Delta/4)}{1 + \Delta} \le \frac{\exp(\Delta/2)}{1 + \Delta} \le 1 - \Delta/4,
    \end{align*}
    since $\eps = \frac{\Delta}{4I}$ and $\Delta \le 1/2$.

    Next, consider the lower triangular entries, with $q < k, q \in [P^*]$. We have that $T_{(k, \pi(q))} = T_q$, and thus
    \begin{align*}
        \alpha(I-1) \hat x_0^{I-1} T = (1 + \eps)^{I-1}(1 + \Delta/4) \frac{\max\left(\frac{1}{2d}, \bar v_{k, \pi(q)}^2(0)\right)^{I-1}}{\bar v_{q, \pi(q)}^{2I-2}(0)}.
    \end{align*}
    By the bound on the column gap in Assumption \ref{assumption: gf init}\ref{assumption-itm: init: col gap}, we have $\frac{ \bar v_{k, \pi(q)}^{2I-2}(0)}{\bar v_{q, \pi(q)}^{2I-2}(0)} \le \frac{1}{1 + \Delta}$. Moreover, by Assumption \ref{assumption: gf init}\ref{assumption-itm: init: regularity conditions}, we have $\frac{1/(2d)}{\bar v_{q, \pi(q)}^2(0)} \le \frac12 \le \frac{1}{1 + \Delta}$. Therefore $\alpha(I-1) \hat x_0^{I-1} T \le \frac{(1 + \eps)^{I-1}(1 + \Delta/4)}{1 + \Delta} \le 1 - \Delta/4$ as well.

    Finally, we consider the lower right block, with $k, q > P_*$, in which case $T_{(k, p)} = T_{P_*}$. We see that
    \begin{align*}
        \alpha(I-1) \hat x_0^{I-1} T = (1 + \eps)(1 + \Delta/4)\frac{a_{\pi(q)}\max(\frac{1}{2d}, \bar v_{k, \pi(q)}^2(0))}{a_{\pi(P_*)}\bar v_{P_*, \pi(P_*)}^2(0)}.
    \end{align*}
    By the bound on the threshold gap in \ref{assumption: gf init}\ref{assumption-itm: init: threshold gap}, we have $\frac{ \bar v_{k, \pi(q)}^{2I-2}(0)}{\bar v_{P_*, \pi(P_*)}^{2I-2}(0)} \le \frac{1}{1 + \Delta}$. Moreover, by the definition of the greedy maximum selection process along with Assumption \ref{assumption: gf init}\ref{assumption-itm: init: regularity conditions}, we have that $a_{\pi(P_*)}\bar v_{P_*, \pi(P_*)}^2(0) \ge a_{\pi(q)}\max_{j > P_*} \bar v_{j, \pi(q)}^2(0) \ge a_{\pi(q)}/d$, and thus $\frac{a_{\pi(q)}\cdot 1/(2d)^{I-1}}{a_{\pi(P_*)}\bar v_{P_*, \pi(P_*)}^{2I-2}(0)} \le \frac{1}{2^{I-1}} \le \frac{1}{1 + \Delta}$. Altogether, $\alpha(I-1) \hat x_0^{I-1} T \le \frac{(1 + \eps)^{(I-1)}(1 + \Delta/4)}{1 + \Delta} \le 1 - \Delta/4$.
    
    In all cases, we have $\alpha(I-1) \hat x_0^{I-1} T\le 1 - \Delta/4$. Thus by Lemma \ref{lem:continuous-upper-bound}, we can bound $\hat x_T$ by
    \begin{align*}
        \hat x_T &\le \frac{\hat x_0}{\left(1 - \alpha(I-1) \hat x_0^{I-1} T\right)^{\frac{1}{I-1}}} \le \hat x_0(\Delta/4)^{-\frac{1}{I-1}} \le d^{-1 + \gamma}/2 =: \eps_0/2,
    \end{align*}
    provided that $\frac{d}{\log^{2/\gamma} d} \ge 2^{1/\gamma}(4/\Delta)^{\frac{1}{\gamma(I-1)}}$.

    Therefore by Lemma \ref{lem:growth failed coords}, the update for $\bar v_{k, \pi(q)}^2(t)$ is 
    \begin{align*}
        \bar v_{k, \pi(q)}^2(t+1) &\le \bar v^2_{k, \pi(q)}(t) + 4I\hat \sigma_{2I}^2\eta a_{\pi(q)} \bar v_{k, \pi(q)}^{2I} + Z(t+1) + \xi(t+1),
    \end{align*}
    which is indeed of the form \eqref{eq:simple-stochastic-update} for $\sigma^2_Z \lesssim \eta^2\norm{\a}_1^2$ and $\Xi_1 \lesssim \eta \delta_{\mathrm{error}}, \Xi_2 \lesssim \eta^2d\norm{\a}_1^2\log^{\tilde Q}(md/\delta_{\P,\xi}), \Xi_3 \lesssim C_{\sigma}^2\eta a_{\pi(q)}\eps_0, \Xi_4 \lesssim \eta^2\norm{\a}_1^2\log^{\tilde Q}(md/\delta_{\P,\xi})$.

    Next, we verify that the conditions on $\Xi, \sigma_Z^2$, in \eqref{eq:conditions-stochastic-update} hold. We first bound the quantity $\sum_{t=0}^{T-1}\hat x_t$.
    \begin{align*}
        \sum_{t=0}^{T-1}\hat x_t &\le \int_0^T\frac{x_0}{\left(1 - \alpha(I-1) \left(\hat x_0\right)^{I-1}t\right)^{\frac{1}{I-1}}} dt\\
    &\le 
    \begin{cases}
        \alpha^{-1}\log\left(\frac{1}{1 - \alpha \hat x_0 T}\right) & I = 2\\
        \frac{1}{(I-2)\alpha \left(\hat x_0\right)^{I-2}}\left[1 - (1 - \alpha(I-1) \hat x_0^{I-1}T^+)^{\frac{I-2}{I-1}}\right] & I > 2
    \end{cases}\\
    &\le 
        \begin{cases}
        \alpha^{-1}\log(4/\Delta) & I = 2\\
        (I-2)^{-1}\alpha^{-1}\hat x_0^{2 - I} & I > 2
    \end{cases}.
    \end{align*}
    Therefore
    \begin{align}\label{eq:sum x-hat}
        \sum_{t=0}^{T-1}\hat x_t \le \alpha^{-1}\hat x_0^{2-I}\min((I-2)^{-1},\log(4/\Delta))
    \end{align}
    Next, we can bound the quantity $\sum_{t=0}^{T-1}\hat x_t^{1/2}$:
    \begin{align*}
        \sum_{t=0}^{T-1}\hat x_t^{1/2} &\le \int_0^T \frac{\hat x_0^{1/2}}{\left(1 - \alpha(I-1) \left(\hat x_0\right)^{I-1}t\right)^{\frac{1}{2(I-1)}}}dt\\
        &=  \frac{2\hat x_0^{1/2}}{\alpha(2I-3)\hat x_0^{I-1}}\left(1 - \left(1 - \alpha(I-1)\hat x_0^{I-1}T\right)^{\frac{2I - 3}{2(I-1)}}\right)\\
        &\le 2\hat x_0^{1/2}T
    \end{align*}
    Finally, we can bound the quantity $\sum_{t=0}^{T-1}\hat x_t^I$
    \begin{align*}
    \sum_{t=0}^{T-1}\hat x_t^I &\le \int_0^T\frac{\hat x_t^I}{\left(1 - \alpha(I-1) \hat x_t^{I-1}t\right)^{\frac{I}{I-1}}}dt\\
    &= \hat x_t\alpha^{-1}\left(\frac{1}{\left(1 - \alpha(I-1) \hat x_t^{I-1}T\right)^{\frac{1}{I-1}}} - 1\right) \\
    & \le x_0^+\alpha^{-1}(\Delta/4)^{-\frac{1}{I-1}}.
    \end{align*}
   Let us consider the $\sigma_Z^2$ condition. Plugging in \eqref{eq:sum x-hat}, it suffices to take
    \begin{align*}
        \sigma_Z^2 \le \frac{x_0^I\eps^2\delta_{\P} \alpha }{\log(4/\Delta)} = \frac{x_0^I\eps^2\delta_{\P} \cdot 4 I\hat \sigma_{2I}^2 \eta a_{\pi(q)}}{\log(4/\Delta)}
    \end{align*}
    Plugging in $\sigma_Z^2 \lesssim \eta^2\norm{\a}_1^2$, and noting $x_0 \ge \frac{1}{2d}$, this is satisfied if we take
    \begin{align*}
        \eta \lesssim  \frac{a_{\min_*}\hat \sigma_{2I}^2 d^{-I}  \Delta^2\norm{\a}_1^{-2}\delta_{\P}}{I\log(4/\Delta)}.
    \end{align*}
    Next, for the $\Xi_1$ constraint, we require
    \begin{align*}
        \Xi_1 \le \frac{\eps x_0}{8 \sum_{t=0}^{T-1}\hat x^{1/2}_t} &\Longleftarrow \eta \delta_{\mathrm{error}} \lesssim \frac{\eps x_0}{T \hat x_0^{1/2}}\\
        &\Longleftarrow \eta \delta_{\mathrm{error}} \lesssim a_{\min_*}\hat\sigma_{2I}^2\eta d^{-(I-1)}\Delta x_0^{1/2}\\
        &\Longleftarrow \delta_{\mathrm{error}} \lesssim a_{\min_*}\hat\sigma_{2I}^2d^{-I+1/2}\Delta
    \end{align*}
    For $\Xi_2$, plugging in \eqref{eq:sum x-hat} we require
    \begin{align*}
        &\Xi_2 \le \frac{\eps x_0^{I-1}\alpha}{8\log(4/\Delta)} = \frac{\eps x_0^{I-1}\cdot I\hat\sigma_{2I}^2\eta a_{\pi(q)}}{2\log(4/\Delta)}\\
        &\Longleftarrow \eta^2d\norm{\a}_1^2\log^{\tilde Q}(md/\delta_{\P, \xi}) \ll \frac{\eps x_0^{I-1}\cdot I\hat\sigma_{2I}^2\eta a_{\pi(q)}}{2\log(4/\Delta)}\\
        &\Longleftarrow \eta \ll \frac{a_{\min_*}\hat \sigma_{2I}^2 d^{-I}\norm{\a}_1^{-2}\Delta}{\log(4/\Delta)\log^{\tilde Q}(md/\delta_{\P, \xi})}
    \end{align*}
    For $\Xi_3$, we require
    \begin{align*}
        \Xi_3 \le \frac{\eps x_0}{8\sum_{t=0}^{T-1}\hat x_t^I} &\Longleftarrow C_{\sigma}^2\eta a_{\pi(q)}\eps_0 \lesssim (\Delta/I)^{\frac{I}{I-1}}\alpha\\
        &\Longleftarrow C_{\sigma}^2\eta a_{\pi(q)}\eps_0 \lesssim \Delta^{\frac{I}{I-1}}I^{-\frac{1}{I-1}}\hat\sigma_{2I}^2\eta a_{\pi(q)}\\
        &\Longleftarrow \eps_0 \lesssim C_{\sigma}^{-2}\Delta^{\frac{I}{I-1}}\hat\sigma_{2I}^2,
    \end{align*}
    which is indeed true since $\eps_0 \le d^{-1/2} \ll  C_{\sigma}^{-2}\Delta^{\frac{I}{I-1}}\hat\sigma_{2I}^2$. Finally, for $\Xi_4$, we require
    \begin{align*}
        \Xi_4 \le \frac{\epsilon x_0}{8T}&\Longleftarrow \eta^2\norm{\a}_1^2\log^{\tilde Q}(md/\delta_{\P,\xi}) \ll \Delta d^{-1}\cdot \eta a_{\min_*}(I-1)\hat\sigma_{2I}^2 d^{-(I-1)}\\
        &\Longleftarrow \eta \ll \frac{a_{\min_*} (I-1)\hat\sigma_{2I}^2d^{-I}\norm{\a}_1^{-2}\Delta }{\log^{\tilde Q}(md/\delta_{\P,\xi})}.
    \end{align*}
    Therefore the conditions of Lemma \ref{lem: failed coords stochastic induction} are satisfied, and so with probability $1 - T\delta_{\P, \xi} - \delta_{\P}$ we have $X_t \le \hat x_t \le \eps_0/2$ for all $t \le T$.

    \paragraph{Part 2 $(T_{(k, \pi(q))} \le t \le T_{P_*})$} We now show that $\bar v^2_{k, \pi(q)}$ doesn't increase too much in the time interval $[T_{(k, \pi(q))}, T_{P_*}]$. The case where $k, q > P_*$ is trivially true.
    
    Consider the case when $q < k, q \in [P_*]$, so that $T = T_q$. By Induction Hypothesis \ref{inductH:learning-time}\ref{inductH-itm:norm-convergence}, when $t \ge T_q$, we have that $\norm{\v_q(t)}^2 = a_{\pi(q)} \pm \eps_R$. When $\bar v_{k, \pi(q)}^2(t) \le \eps_0$, we have that
    \begin{align*}
        \bar v_{k, \pi(q)}^2(t+1) \le \bar v^2_{k, \pi(q)}(t) + 4I\hat \sigma_{2I}^2\eta \eps_R \bar v_{k, \pi(q)}^{2I} + Z(t+1) + \xi(t+1),
    \end{align*}
    where $\E[Z(t+1) \mid \mathcal{F}_t] \lesssim \eta^2\norm{\a}_1^2v_{k, \pi(q)}^2(t)$, and 
    \begin{align*}
        \abs{\xi(t+1)} \lesssim \eta^2 \norm{\a}_1^2 \eps_0 d\log^{\tilde Q}(md/\delta_{\P,\xi}) + C_{\sigma}^2\eta a_{\pi(q)}\eps_0^{I+1} + \eta\eps_0^{1/2}\delta_{\mathrm{error}}.
    \end{align*}   
    We would like to apply Lemma \ref{lem: failed coords stochastic induction} to the process $X_t = \bar v_{k, \pi(q)}^2(t + T_q)$ up to time $T_{P^*}$, with $\alpha = 4I\hat \sigma_{2I}^2\eta \eps_R , \eps = 0.5$. We see that $X_0 \le \frac{\eps_0}{2} := x_0$, and so setting $\hat x_0 = 1.5 x_0$, we have that
    \begin{align*}
        t \le T_{P_*} \le \frac{d^{I-1}}{2I(I-1)\hat\sigma_{2I}^2\eta a_{\min_*}},
    \end{align*}
    and thus as long as $\eps_R \lesssim \frac{a_{\min_*}}{8\cdot1.5^{I-1}d^{\gamma(I-1)}}$, we have
    \begin{align*}
        &\alpha(I-1)\hat x_0^{I-1}t \le 2\cdot 1.5^{I-1}\eps_0^{I-1}d^{I-1}\eps_R a_{\min_*}^{-1} = 2\cdot 1.5^{I-1} d^{\gamma(I-1)}\eps_R a_{\min_*}^{-1}\le 1/4 \le 1 - (3/4)^{I-1}\\
        \Longrightarrow~~&\hat x_t \le \frac{\hat x_0}{\left(1 - \alpha(I-1) \hat x_0^{I-1} t\right)^{\frac{1}{I-1}}} = \frac{1.5 x_0}{\left(1 - \alpha(I-1) \hat x_0^{I-1} t\right)^{\frac{1}{I-1}}} \le 2x_0 = \eps_0
    \end{align*}
    We next verify that the conditions of Lemma \ref{lem: failed coords stochastic induction} hold . We first require
    \begin{align*}
        &\Xi_4 \le \frac{x_0}{16T}\\
        &\Longleftarrow \eta^2 \norm{\a}_1^2 \eps_0 d\log^{\tilde Q}(md/\delta_{\P,\xi}) + C_{\sigma}^2\eta a_{\pi(q)}\eps_0^{I+1} + \eta\eps_0^{1/2}\delta_{\mathrm{error}} \lesssim I(I-1)\hat\sigma_{2I}^2 \eta a_{\min_*}d^{-(I-1)} \eps_0\\
        &\Longleftarrow \eta \lesssim \frac{a_{\min_*}d^{-I}I(I-1)\hat\sigma_{2I}^2\norm{\a}_1^{-2}}{\log^{\tilde Q}(md/\delta_{\P,\xi})}, \quad \eps_0^I \ll I(I-1)\hat\sigma_{2I}^2 a_{\min_*}d^{-(I-1)},\\
        &\quad\quad\quad\text{and}\quad\delta_{\mathrm{error}} \ll \frac{a_{\min_*}I(I-1)\hat\sigma_{2I}^2\eps_0^{1/2}}{d^{I-1}}.
    \end{align*}
    Clearly the condition on $\eta$ is satisfied. Next, plugging in $\eps_0 = d^{-(1-\gamma)}$, we require
    \begin{align*}
        d \gg \left(I(I-1)\hat\sigma_{2I}^2a_{\min_*}\right)^{-\frac{1}{1 - I\gamma}}.
    \end{align*}
    Finally, the condition on $\delta_{\mathrm{error}}$ is indeed satisfied, since we already have
    \begin{align*}
        \delta_{\mathrm{error}} \lesssim a_{\min_*}I(I-1)\hat\sigma_{2I}^2d^{-I + 1/2}\Delta \ll a_{\min_*}I(I-1)\hat\sigma_{2I}^2d^{-I + 1}\eps_0^{1/2}
    \end{align*}
    Additionally, since we can bound $\sum_{t=1}^T\hat x_t \le T\eps_0$, we require 
    \begin{align*}
        \sigma_Z^2 \lesssim \frac{x^2_0\delta_{\P}}{T\eps_0} &\Longleftarrow \eta^2\norm{\a}_1^2 \lesssim \eps_0\delta_{\P} I(I-1)\hat\sigma_{2I}^2\eta a_{\min_*}d^{-(I-1)}\\
        &\Longleftarrow \eta \lesssim a_{\min_*}\eps_0d^{-(I-1)} I(I-1)\hat\sigma_{2I}^2\norm{\a}_1^{-2}\delta_{\P},
    \end{align*}
    which is again satisfied by our choice of $\eta$. Altogether, we have $X_t \le \hat x_t \le \eps_0$ for all $t \le T_{P_*}$.

    Finally, consider the case when $k < q, k \in [P_*]$, so that $T = T_k$. By Induction Hypothesis \ref{inductH:learning-time}\ref{inductH-itm:directional-convergence}, when $t \ge T_k$, we have that $\bar v_{k, \pi(k)}(t)^2 \ge 1 - \bar\eps $, and thus $\bar v_{k, \pi(q)}^2(t) \le \bar\eps \le \eps_0$, as desired.
\end{proof}

\subsubsection{Upper Bounds on the Norm Growth}

We start with an upper bound on the norm of the unused neurons, i.e., $\v_k$ with $k > P_*$. 
\begin{lemma}[Bound on the unused neurons]\label{lem:unused neurons SGD}
  Inductively assume that Induction Hypothesis~\ref{inductionH: gf}\ref{inductH-itm: bound on the failed coordinates} is true. 
  Suppose that we choose 
  \[
    \eta \lesssim \frac{a_{\min_*}d^{-I}I(I-1)\hat\sigma_{2I}^2\norm{\a}^{-2}\delta_{\P}}{\log^{\tilde Q}\left( Tmd/ \delta_{\P} \right)  }. 
  \]
  Then, for any $k \in [m]$ with $k > P_*$, with probability at least $1 - \delta_{\P}$ we have 
  $\norm{\v_k}^2 \le O(\sigma_0^2) \ll \sigma_1^2$ throughout training. 
\end{lemma}
\begin{proof}
  By the proof of Lemma \ref{lemma: radial: unused neurons} along with Lemma \ref{lemma: radial and tangent dynamics}, we have 
  \begin{align*}
    \norm{\v_k(t+1)}^2 
    &\le \left( 1 + 4 \eta \eps_0^I\norm{\a}_1 \right) \norm{\v_k}^2
      - 2 \eta \inprod{\v_k}{ \H_k(t+1) }
      + \xi_{k, R}(t+1)
  \end{align*}
  The total running time of SGD is $T = \frac{1 + \Delta/4}{4I(I-1)\hat\sigma_{2I}^2\eta a_{\min_*} \bar v_{P_*, \pi(P_*)}^{2I-2}(0)} \lesssim \frac{d^{I-1}}{4I(I-1)\hat\sigma_{2I}^2a_{\min_*}\eta}$. Therefore
  \begin{align*}
      4 \eta \eps_0^I\norm{\a}_1 \cdot T \lesssim \frac{\eps_0^I\norm{\a}_1d^{I-1}}{I(I-1)\hat\sigma_{2I}^2a_{\min_*}} = \frac{d^{-(1 - \gamma I)}\norm{\a}_1}{I(I-1)\hat\sigma_{2I}^2a_{\min_*}} \ll 1,
  \end{align*}
  since $d \gtrsim \left(\frac{\norm{\a}_1}{I(I-1)\hat\sigma_{2I}^2a_{\min_*}}\right)^{\frac{1}{1 - \gamma I}}$. Thus $\left( 1 + 4 \eta \eps_0^I\norm{\a}_1 \right)^T \lesssim 1$. In addition, by Lemma~\ref{lemma: population and per-sample gradients}, we have 
  \[
    \Var\left( 2 \eta \norm{\v_k} \inprod{\bar{\v}_k}{ \H_k(t+1) } \right)
    \lesssim \eta^2 \norm{\a}_1^2 \norm{\v_k}^4.
  \]
  Hence, using the language of Lemma~F.6 of \cite{ren2024learning}, we have 
  \begin{align*}
    \alpha &= \Theta\left( \eta\eps_0^I\norm{\a}_1 \right), \quad 
    \sigma_Z^2 = O\left( \eta^2\norm{\a}_1^2 \sigma_0^4\right)\\
    \Xi &= O\left( \eta^2 d \norm{\a}_1^2\log^{\tilde Q}\left( \frac{T m d}{ \delta_{\P} } \right)  \sigma_0^2 \right), \quad 
    T = O\left( \frac{d^{I-1}}{4I(I-1)\hat\sigma_{2I}^2a_{\min_*}\eta} \right).
  \end{align*}
  To satisfy the condition of that lemma, it suffices to choose 
  \begin{align*}
    \sigma_Z^2 \lesssim \alpha \delta_{\P} \sigma_0^4 
    &\quad\Leftarrow\quad
    \eta
    \lesssim \eps_0^{-I}\norm{\a}^{-1}\delta_{\P}\\
    \Xi \lesssim \frac{\sigma_0^2}{T}
    &\quad\Leftarrow\quad
    \eta 
    \lesssim \frac{a_{\min_*}d^{-I}I(I-1)\hat\sigma_{2I}^2\norm{\a}^{-2}}{\log^{\tilde Q}\left( Tmd/ \delta_{\P} \right)  }. 
  \end{align*}
  
\end{proof}

Then, we consider $k = p \le P_*$. Unlike those unused neurons, since $\v_p$ will eventually converge to $\e_{\pi(p)}$, 
its norm cannot stay small. Our strategy here will be coupling its norm growth with the tangent movement. 
We will use the following extension to Lemma~F.11 of \cite{ren2024learning}. The proof of this lemma can be found 
in Section~\ref{subsec: online sgd: deferred proofs}. 

\begin{lemma} 
  \label{lemma: general nonlinear stochastic induction}
  Suppose that $(X_t)_t$ satisfies 
  \[
    X_{t+1} = X_t + \alpha_t(X_t) X_t + \xi_{t+1} + Z_{t+1}, \quad X_0 = x_0 > 0, 
  \]
  where $\alpha_t: \R \to \R_{\ge 0}$ is an $\cF_t$-measurable non-decreasing function, $(\xi_t)_t$ is an adapted 
  process, and $(Z_t)_t$ is a martingale difference sequence. Let $\eps > 0$ be given and define the process 
  \[
    \hat{X}_{t+1} = \hat{X}_t + \alpha_t(\hat{X}_t), \quad \hat{X}_0 = (1 + \eps) x_0. 
  \]
  Fix $T > 0$, $\delta_{\P} \in (0, 1)$. Suppose that there exists $\Xi, \sigma_Z > 0$ and $\delta_{\P, \xi} \in (0, 1)$
  such that when $X_t \le \hat{X}_t$, we have $|\xi_{t+1}| \le \Xi$ with probability at least $1 - \delta_{\P, \xi}$,
  and $\E[Z_{t+1} \mid \cF_t] \le \sigma_Z^2$. Then, if 
  \[
    \Xi \le \eps_0 x_0 / (2 T) 
    \quad\text{and}\quad 
    \sigma_Z^2 \le \eps^2 x_0^2 \delta_{\P} / (4 T),
  \]
  we have $X_t \le \hat{X}_t$ for all $t \le T$. 
\end{lemma}

The following lemma verifies Induction Hypothesis \ref{inductionH: gf}\ref{inductH-itm: large norm => converged} for $\sigma_1 = O(\sigma_0\bar\eps^{-C/2})$ for some constant $C$.

\begin{lemma}[Bound on $\norm{\v_p}^2$]\label{lem:bound-norm}
  Suppose that $d \gg \left(\frac{\norm{\a}_1}{I(I-1)\hat\sigma_{2I}^2a_{\min_*}}\right)^{\frac{1}{1 - I\gamma}}$ and $\eta \lesssim \frac{a_{\min_*}I(I-1)\hat\sigma_{2I}^2d^{-I}\norm{\a}_1^{-2}}{\log^{\tilde Q}(md/\delta_{\P,\xi})}$. 
  Then there exists a constant $C_{\exp}$ such that $\norm{\v_p}^2 \le O\left( \sigma_0^2 \bar\eps^{-C_{\exp}} \right)$ 
  as long as $\bar{v}_{p, \pi(p)}^2$ has not reached $1 - \bar{\eps}$. 
\end{lemma}
\begin{proof}
  By the proof of Lemma \ref{lemma: radial: upper bound on the norm}, when Induction Hypothesis~\ref{inductionH: gf}\ref{inductH-itm: bound on the failed coordinates} holds we have 
    \begin{align*}
        \norm{\v_p(t+1)}^2 \le \norm{\v_p}^2 + 4\eta(a_{\pi(p)}\bar v_{p, \pi(p)}^{2I} + \norm{\a}_1\eps_0^I)\norm{\v_p}^2 - Z_{p, R}(t+1) + \xi_{p, R}(t+1),
    \end{align*}
  where, by Lemma~\ref{lemma: population and per-sample gradients} and Lemma~\ref{lemma: radial and tangent 
  dynamics}, the conditional variance of $Z_{p, R}$ is bounded by $O\left( \eta^2\norm{\a}_1^2 \norm{\v_p}^4 \right)$ and we 
  have 
  \[
    |\xi_{p, R}(t+1)|
    \lesssim \eta^2 d \norm{\a}_1^2\log^{\tilde Q}\left( \frac{md}{ \delta_{\P} } \right)  \norm{\v_p}^2 
    \quad 
    \text{with probability at least $1 - \delta_{\P}$}. 
  \]
  First, consider the situation where $\bar{v}_{p, \pi(p)}^2 \le 0.9$. We prove by stochastic induction that 
  $\norm{\v_p}^2 \le O(\sigma_0^2)$. 
  Under this induction hypothesis, using the language of Lemma~\ref{lemma: general 
  nonlinear stochastic induction} with $\eps = 0.5$, we have 
  \[
    \sigma_Z^2 = O( \eta^2\norm{\a}_1^2 \sigma_0^4  ), \quad 
    \Xi = O\left( \eta^2 d \norm{\a}_1^2 \log^{\tilde Q}\left( \frac{md}{ \delta_{\P} } \right) \sigma_0^2 \right), \quad 
    T = O\left( \frac{d^{I-1}}{I(I-1)\hat\sigma_{2I}^2a_{\min_*}\eta} \right). 
  \]
  Hence, to meet the condition of Lemma~\ref{lemma: general nonlinear stochastic induction}, it suffices to choose 
  \begin{align*}
    \sigma_Z^2 \lesssim \frac{\sigma_0^4 \delta_{\P}}{T}
    &\quad\Leftarrow\quad
    \eta \lesssim a_{\min_*}d^{-(I-1)}I(I-1)\hat\sigma_{2I}^2\norm{\a}^{-2}\delta_{\P}, \\
    \Xi \lesssim \frac{\sigma_0^2}{T}
    &\quad\Leftarrow\quad
    \eta \le \frac{a_{\min_*}d^{-(I-1)}I(I-1)\hat\sigma_{2I}^2\norm{\a}^{-2}}{\log^{\tilde Q}\left( \frac{md}{ \delta_{\P} } \right)  }. 
  \end{align*}
  When these hold, then we have with probability at least $1 - O(\delta_{\P})$ that $\norm{\v_p(t)}^2
  = (1 \pm 0.5) N^2(t)$ for any $t \le T$, where $N^2$ is defined via 
  \[
    N^2(t+1)
    := N^2(t) + 4 \eta \left( a_{\pi(p)} \bar{v}_{p, \pi(p)}^{2I}(t) + \norm{\a}_1\eps_0^I \right) N^2(t), \quad 
    N^2(0) = 1.5 \norm{\v_p(0)}^2. 
  \]
  Now, we analyze the process $N^2$. First, note that 
  \begin{align*}
    N^2(t) 
    &\le N^2(0) \prod_{s=0}^{t-1} \left( 1 + 4 \eta \left( a_{\pi(p)} \bar{v}_{p, \pi(p)}^{2I}(s) + \norm{\a}_1\eps_0^I\right) \right) \\
    &\le 1.5 \sigma_0^2 \exp\left( 4 \eta T \norm{\a}_1\eps_0^I \right) 
      \exp\left(
        4 \eta  a_{\pi(p)} \sum_{s=0}^{t} \bar{v}_{p, \pi(p)}^{2I}(s)
      \right). 
  \end{align*}
First, we see that
\begin{align*}
    4\eta T \norm{\a}_1 \eps_0^I \le \frac{d^{I-1}\norm{\a}_1\eps_0^I}{I(I-1)\hat\sigma_{2I}^2a_{\min_*}} = \frac{d^{I\gamma - 1}\norm{\a}_1}{I(I-1)\hat\sigma_{2I}^2a_{\min_*}} \ll 1,
\end{align*}
since $d \gg \left(\frac{\norm{\a}_1}{I(I-1)\hat\sigma_{2I}^2a_{\min_*}}\right)^{\frac{1}{1 - I\gamma}}$.
  
Next, By the proof of Lemma \ref{lem:intermediate growth}, when $\bar{v}_{p, \pi(p)}^2 \le 0.9$, we have 
\begin{align*}
    \bar{v}_{p, \pi(p)}^2(t+1)
    &\ge \bar{v}_{p, \pi(p)}^2(t) 
      + 2 \eta a_{\pi(p)} I\hat\sigma_{2I}^2\bar{v}_{p, \pi(p)}^{2I}(t) + Z_{t+1} + \xi_{t+1},
\end{align*}
where with probability $1 - \delta_{\P,\xi}$ we have $\abs{\xi_{t+1}} \lesssim \eta^2d\norm{\a}_1^2\log^{\tilde Q}(md/\delta_{\P,\xi})$, and the martingale term $Z_{t+1}$ satisfies $\E[Z_{t+1}^2 \mid \mathcal{F}_t] \lesssim \eta^2\norm{\a}_1^2$. Therefore
\begin{align*}
    \bar{v}_{p, \pi(p)}^2(t+1) \ge \bar v_{p, \pi(p)}^2(0) + 2 \eta a_{\pi(p)} I\hat\sigma_{2I}^2 \sum_{s=0}^{t} \bar{v}_{p, \pi(p)}^{2I}(s) + \sum_{s=0}^t\xi_{s+1} + \sum_{s=0}^tZ_{s+1}.
\end{align*}
We first have
\begin{align*}
    \abs{\sum_{s=0}^t\xi_{s+1}} \lesssim T\eta^2d\norm{\a}_1^2\log^{\tilde Q}(md/\delta_{\P,\xi}).
\end{align*}
Since $\eta T \le O(\frac{d^{I-1}}{I(I-1)\hat\sigma_{2I}^2a_{\min_*}})$, we thus have $\abs{\sum_{s=0}^t\xi_{s+1}} \le 1$ whenever $\eta \lesssim \frac{a_{\min_*}I(I-1)\hat\sigma_{2I}^2d^{-I}\norm{\a}_1^{-2}}{\log^{\tilde Q}(md/\delta_{\P,\xi})}$. Next, by Doob's submartingale inequality, we have
\begin{align*}
    \mathbb{P}\left[\sup_{r \le t}\abs{\sum_{s=1}^rZ_{s}} \ge 1 \right] \lesssim T\eta^2\norm{\a}_1^2 \lesssim \eta \cdot \frac{d^{I-1}\norm{\a}_1^2}{I(I-1)\hat\sigma_{2I}^2a_{\min_*}}
\end{align*}
and thus if $\eta \lesssim \frac{a_{\min_*}I(I-1)\hat\sigma_{2I}^2\norm{\a}_1^{-2}\delta_P}{d^{I-1}}$ we have that $\sup_{r \le t}\abs{\sum_{s=1}^rZ_{s}} \le 1$ with probability $1 - \delta_{\P}$. Altogether, on these events we have that 
\begin{align*}
    \eta a_{\pi(p)} I \hat\sigma_{2I}^2 \sum_{s=0}^{t} \bar{v}_{p, \pi(p)}^{2I}(s) \le 1.5
\end{align*}
   As a result, 
  \[
    N^2(t) 
    \le 1.5 \sigma_0^2 \exp\left( 4 \eta T \norm{\a}_1\eps_0^I \right) \exp\left(\frac{6}{I\hat\sigma_{2I}^2} \right) 
    = O(\sigma_0^2),
  \]
In other words, we have $\norm{\v_p}^2 = O(\sigma_0^2)$ when $\bar{v}_{p, \pi(p)}^2 \le 0.9$. 

  Now, consider the situation where $\bar{v}_{p, \pi(p)}^2 \in [0.9, 1 - \bar\eps]$. By the proof of 
  Lemma~\ref{lem:strong-convergence}, it takes at most $\frac{3^I\log(2/\bar\eps)}{I\hat\sigma_{2I}^2\eta a_{\pi(p)}}$ iterations for 
  $\bar{v}_{p, \pi(p)}^2$ to grow from $0.9$ to $1 - \bar\eps$. In this stage, we have 
  \[
    \norm{\v_p(t+1)}^2 
    \le \norm{\v_p}^2 + 4.1 \eta a_{\pi(p)} \norm{\v_p}^2 
      - Z_{p, R}(t+1)
      + \xi_{p, R}(t+1).  
  \]
  Let the corresponding deterministic process be $M^2(t+1) = M^2(t) + 4.1 \eta a_{\pi(p)}  M^2(t)$ with 
  $M^2(T_0) = O(\sigma_0^2)$ where $T_0$ is the time $\bar{v}_{p, \pi(p)}^2$ reaches $0.9$. Using the language
  of Lemma~F.6 of \cite{ren2024learning}, we have 
  \[
    \alpha 
    = 4.1 \eta a_{\pi(p)}, \quad 
    \sigma_Z^2 = O( \eta^2 \norm{\a}_1^2\sigma_0^4  ), \quad 
    \Xi = O\left( \eta^2 d \norm{\a}_1^2\log^{\tilde Q}\left( \frac{md}{ \delta_{\P} } \right) \sigma_0^2 \right). 
  \]
 Therefore, to meet the condition of Lemma~F.6 of \cite{ren2024learning}, it suffices to require
  \begin{align*}
    \Xi \lesssim \frac{x_0}{T}
    &\quad\Leftarrow\quad 
    \eta \lesssim \frac{a_{\pi(p)}I\hat \sigma_{2I}^2\norm{\a}_1^{-2}}{d 3^I\log^{\tilde Q}\left( \frac{md}{ \delta_{\P} } \right) \log(2/\bar\eps) }, \\
    \sigma_Z^2 \lesssim \delta_{\P} \alpha x_0^2 
    &\quad\Leftarrow\quad 
    \eta \lesssim a_{\pi(p)} \delta_{\P} \norm{\a}_1^{-2}  .
  \end{align*}
  Meanwhile, we have 
  \[
    M^2(T_1)
    \le M^2(T_0) \exp\left( (T_1 - T_0)\cdot 4.1 \eta a_{\pi(p)} \right)
    \le O\left( \sigma_0^2 \bar\eps^{-C_{\exp}} \right),
  \]
  for $C_{\exp} = \frac{4.1\cdot 3^I}{I\hat\sigma_{2I}^2}$.
\end{proof}

\subsection{Proof of Theorem \ref{thm:online-sgd}}

\begin{proof}
First, by Lemma \ref{lemma: initialization}, with probability $1 - \delta_{\P^*}/2$, Assumption \ref{assumption: gf init} holds at initialization, with $\Delta := \min(\delta_r, \delta_c, \delta_t) = O(\frac{\delta_{\P^*}}{mP\max(m, P)})$.

Define $T_{\max} = \max_{p \in [P_*]}(1 + \Delta/4)T_p \lesssim \frac{d^{I-1}}{I(I-1)\hat\sigma_{2I}^2\eta a_{\min_*}}$. We will show that, with probability $1 - \delta^*_{\P}/2$, that Induction Hypotheses \ref{inductionH: gf} and \ref{inductH:learning-time} hold for all $t \le T_{\max}$ with choice of parameters

We do so by union bounding over the consequence of the following lemmas:
    \begin{itemize}
        \item (Directional convergence) Lemma \ref{thm:directional-convergence-proof} for all $p \in [P_*]$, with $\delta_{\P} = \frac{\delta^*_{\P}}{16P_* \log\log d}, \delta_{\P,\xi} = \frac{\delta^*_{\P}}{16 T_{\max} P_*}$. This implies the first half of part (b).

        \item (Convergence of norm) Lemma \ref{lem:norm-growth-complete} for all $p \in [P_*]$, with $\delta_{\P} = \frac{\delta_{\P}^*}{16P_*}, \delta_{\P, \xi} = \frac{\delta_{\P}^*}{16T_{\max} P_*}$. This implies the second half of part (b).

        \item (Bound on the failed coordinates) Lemma \ref{lem:total-growth-failed-coords} for all $(k, \pi(q)) \not\in \{(p, \pi(p))\}_{p \in [P_*]}$, with $\delta_{\P} = \frac{\delta^*_{\P}}{16mP}, \delta_{\P,\xi} = \frac{\delta^*_{\P}}{16T_{\max} mP}$. This verifies that Induction Hypothesis \ref{inductionH: gf}\ref{inductH-itm: bound on the failed coordinates} holds throughout training.
    
        \item (Bound on unused neurons) Lemma \ref{lem:unused neurons SGD} for all $k \in [m] \setminus [P_*]$ with $\delta_{\P} = \frac{\delta_{\P}^*}{16 m}$.  This implies part (a).

        \item (Upper bound on norm growth) Lemma \ref{lem:bound-norm} for all $p \in [P_*]$, with $\delta_{\P} = \frac{\delta_{\P}^*}{16P_*}$. This implies part (c).

    \end{itemize}

    Next, we verify that our choice of $\eps_0, \bar\eps, \sigma_0, \sigma_1$ indeed satisfy the conditions of the lemmas. First, Lemma \ref{thm:directional-convergence-proof} requires the conditions on \ref{lemma: tangent: dynamics of the diagonal entries (stage 1)} to hold. Recall that we have chosen $\delta_T = \frac{\Delta^2}{CI^2}$ for sufficiently large constant $C$, and we will select $\gamma \le \frac{1}{4I}$. We thus require
    \begin{align*}
        d^{-1/2} &\le \frac{\delta_T I\hat\sigma_{2I}^2}{C_\sigma^2} \Longleftarrow d \gtrsim \frac{C_{\sigma}^4I^2}{\hat\sigma_{2I}^4\Delta^4}, \quad \frac{d}{\log^4d} \gtrsim I^2\Delta^{-2}\\
        m\sigma_1^2 &\lesssim \frac{\Delta^2\hat\sigma_{2I}^2a_{\min_*}}{I^22^{3I}C_{\sigma}^2d^{I-1/2}}\\
        \bar\eps &\lesssim \left(\frac{\Delta^2\hat\sigma_{2I}^2}{I^22^{3I + 4}C_\sigma^2}\right)^2\cdot \frac{1}{d^{1 + 2\gamma(I-1)}} \\
        d &\gtrsim \left(\frac{\hat\sigma_{2I}^2 a_{\min_*} \Delta^2}{I^22^{3I}C_{\sigma}^2\norm{\a}_1}\right)^{-\frac{2}{1-2\gamma I}}\\
        \eps_D
  &\ge \frac{ 2^{3I+7}3^I C_\sigma^2 }{\hat\sigma_{2I}^2}
    \braces{
      \bar\eps^{1/2} \eps_0^{I-1} 
      \vee \frac{m \sigma_1^2}{a_{\min_*}}
      \vee \frac{\norm{\a}_1}{a_{\min_*}}  \eps_0^I} \Longleftarrow \begin{cases}
          \bar\eps &\le (\frac{\hat\sigma_{2I}^2}{ 2^{3I+7}3^I C_\sigma^2 })^2\eps_D^2d^{2(1-\gamma)(I-1)}\\
          m \sigma_1^2 &\le \frac{\hat\sigma_{2I}^2a_{\min_*}\eps_D}{ 2^{3I+7}3^I C_\sigma^2 }\\
          \eps_D & \ge \frac{ 2^{3I+7}3^I C_\sigma^2 \norm{\a}_1}{\hat\sigma_{2I}^2a_{\min_*}}d^{-I(1 - \gamma)}
      \end{cases}\\
      \eta &\le \frac{a_{\pi(p)}\hat\sigma_{2I}^2\norm{\a}_1^{-2}\delta_{\P}}{C\log(512I/\Delta)\log^{\tilde Q}(md/\delta_{\P,\xi})}\min(d^{-I}\Delta^2, 3^{-I}\eps_D^2).
    \end{align*}
    Next, Lemma \ref{lem:norm-growth-complete} requires
    \begin{align*}
        \eta &\lesssim \frac{\norm{\a}_1^{-2}}{\log(2a_k/\sigma_0^2)}\min\left(\frac{a_{\min_*}d^{-1}\eps_R}{\log^{\tilde Q}(md/\delta_{\P,\xi})}, \eps_R^2\delta_{\P}\right)\\
        \eps_R &\gtrsim \log(2a_k/\sigma_0^2)\left(C_{\sigma}^2a_{\pi(p)}\bar\eps + \norm{\a}_12^{2I}\eps_0^I + m\sigma_1^2\right)
    \end{align*}
    Next, Lemma \ref{lem:total-growth-failed-coords} requires the conditions on Lemma \ref{cor: convergence of one direction} and Lemma \ref{lemma: tangent: upper triangular entries (case I)} to hold, which are
      \begin{gather*}
    \eps_D 
    \ge \frac{ 2^{3I+7}3^I C_\sigma^2 }{\hat\sigma_{2I}^2 } 
      \frac{\norm{\a}_1}{a_{\min_*}} 
      \frac{1}{d^{(1-\gamma) I} }, \quad 
    \eps_R 
    \ge  12 \norm{\a}_1 2^{2I} d^{-(1-\gamma) I}, \quad 
    \Delta^2 
    \ge \frac{CI^22^{3I+4} C_\sigma^2  }{ \hat\sigma_{2I}^2 } \frac{\norm{\a}_1}{a_{\min_*}} \frac{1}{d^{1/2-\gamma I} }, \\
    m \sigma_1^2 
    \le 
      \frac{ \hat\sigma_{2I}^2 a_{\min_*}}{ 2^{3I+7} C_\sigma^2 }  
      \left(
        3^{-I}  \eps_D
        \wedge 
        \frac{\Delta^2 }{CI^2d^{I-1/2}}
      \right)
      \wedge \frac{\eps_R}{12}, \\
    \bar\eps 
    \le \left( \frac{ \hat\sigma_{2I}^2 }{ 2^{3I+7}3^I C_\sigma^2 } \right)^2 
        \eps_D^2 d^{2 (1-\gamma)(I-1)}
        \wedge 
        \left( \frac{ \Delta^2 \hat\sigma_{2I}^2  }{CI^22^{3I+4} C_\sigma^2 } \right)^2
            \frac{1}{d^{1+ 2 \gamma (I - 1)} }
        \wedge
        \frac{\eps_R}{12 C_\sigma^2  a_{\pi(p)}} . 
  \end{gather*}
  and
  \begin{gather*}
    \bar\eps 
    \le \left( 
      \frac{ \hat\sigma_{2I}^2 }{2^{3I+4} C_\sigma^2 } 
      \frac{\Delta}{24}
    \right)^2 
    \frac{1}{d^{1+2\gamma(I-1)}}, \quad 
    m \sigma_1^2 
    \le \frac{ \hat\sigma_{2I}^2 }{2^{3I+4} C_\sigma^2 } 
      \frac{ a_{\min_*} }{ 2 (\log d)^{2I-2} d^{I-1/2} }
      \frac{\Delta}{24}, \\
    \frac{d}{(\log^2 d)^{1/\gamma}} 
      \ge \left( \frac{\Delta}{4} \right)^{-\frac{1}{\gamma(I-1)}}, \;
      \frac{d}{(\log^2 d)^{\frac{I-1}{1/2 - \gamma I}}}  
      \ge \left( 
          \frac{ \hat\sigma_{2I}^2 }{2^{3I+4} C_\sigma^2 } 
          \frac{ a_{\min_*} }{ \norm{\a}_1 2^{2I-2}  }
          \frac{\Delta}{24}
        \right)^{
          - \frac{1}{1/2 - \gamma I}
        }
        , \; 
      \frac{\Delta^2}{CI^2} \le \frac{\Delta}{240}.
  \end{gather*}
  Moreover \ref{lem:total-growth-failed-coords} additionally requires
    \begin{align*}
        \bar\eps^{1/2}\eps_0^{I-1} \lor m\sigma_1^2 \lor \norm{\a}_1\eps_0^I &\ll \frac{a_{\min_*}\hat\sigma_{2I}^2d^{-I + 1/2}\Delta}{I2^{3I + 6}C_\sigma^2} \Longleftarrow \begin{cases}
            \bar\eps &\lesssim \left(\frac{a_{\min_*}\hat\sigma_{2I}^2\Delta}{I2^{3I + 6}C_\sigma^2}\right)^2\frac{1}{d^{1 + 2\gamma(I-1)}}\\
            m\sigma_1^2 &\lesssim \frac{a_{\min_*}\hat\sigma_{2I}^2d^{-I + 1/2}\Delta}{I2^{3I + 6}C_\sigma^2}\\
            d &\gtrsim \left(\frac{a_{\min_*}\hat\sigma_{2I}^2\Delta}{I2^{3I + 6}C_\sigma^2\norm{\a}_1}\right)^{-\frac{2}{1 - 2I\gamma}}
        \end{cases}\\
        \eps_R &\lesssim \frac{a_{\min_*}}{1.5^Id^{\gamma(I-1)}}\\
        \frac{d}{\log^{2/\gamma}d} &\ge 2^{1/\gamma}(4/\Delta)^{\frac{1}{\gamma(I-1)}}
    \end{align*}
    Finally, Lemma \ref{lem:bound-norm} requires
    \begin{align*}
        d &\gg \left(\frac{\norm{\a}_1}{\hat\sigma_{2I}^2 a_{\min_*}}\right)^{\frac{1}{1 - I\gamma}}\\
        \eta &\lesssim \frac{a_{\min_*}I(I-1)\hat\sigma_{2I}^2d^{-I}\norm{\a}_1^{-2}}{\log^{\tilde Q}(md/\delta_{\P,\xi})}\\
        \sigma_1^2 &\gtrsim \sigma_0^2\bar\eps^{-C_{\exp}}.
    \end{align*}
    
    Assume that $\frac{d}{\log^{8I}d} \ge 2^{4I}(4/\Delta)^{\frac{4I}{I-1}}$. Then by choosing $\gamma$ to be the solution to $\frac{d}{\log^{2/\gamma}} = 2^{1/\gamma}(4/\Delta)^{\frac{1}{\gamma(I-1)}}$, we know that $\gamma \le \frac{1}{4I}$. The constraints on $d$ then become:
    \begin{align*}
        &d \gtrsim_{\sigma} \Delta^{-4} \lor \log^4d \Delta^{-2} \lor \left(\norm{\a}_1\Delta^{-2}a_{\min_*}^{-1}\right)^{4} \lor \log^{8(I-1)}(d)\left(\norm{\a}_1\Delta^{-1}a_{\min_*}^{-1}\right)^{4}\\
        &\Longleftarrow \frac{d}{\log^{8I}d} \gtrsim_{\sigma} \norm{\a}_1^4\Delta^{-8}a_{\min_*}^{-4}.
    \end{align*}
    The conditions on the target accuracies $\eps_R, \eps_D$ become
    \begin{align*}
        \eps_D &\gtrsim_\sigma \frac{\norm{\a}_1}{a_{\min_*}}\frac{1}{d^{I - 1/4}}\\
        \eps_R &\gtrsim_\sigma \frac{\norm{\a}_1}{d^{I - 1/4}}\\
        \eps_R &\lesssim_{\sigma} \frac{a_{\min_*}}{d^{\gamma(I-1)}} =_\sigma \frac{a_{\min_*}\Delta}{\log^{2(I-1)}d}
    \end{align*}
    Next, the constraints on $\bar\eps$ become (substituting $d^\gamma = 2\log^2 d (4/\Delta)^{\frac{1}{I-1}}$):
    \begin{align*}
        \bar\eps \lesssim_{\sigma} \frac{\Delta^6}{d\log^{4(I-1)}d} \land \frac{\eps_D^2d^{2(I-1)}}{\log^{4(I-1)}d}\Delta^2 \land \frac{\eps_R}{\log(1/\sigma_0^2)},
    \end{align*}
    where we note we must also have $\bar\eps \ge \eps_D$. We can therefore choose $\bar\eps = \eps_D$, and observe that the conditions become
    \begin{align*}
        \frac{\Delta^6}{d \log^{4(I-1)}d} \gtrsim_\sigma \eps_D \gtrsim_\sigma \frac{\norm{\a}_1}{a_{\min_*}d^{I - 1/4}}\\
        \frac{a_{\min_*}\Delta}{\log^{2(I-1)d}} \gtrsim_\sigma \eps_R \gtrsim_\sigma \eps_D\log(1/\sigma_0^2) \lor \frac{\norm{\a}_1}{d^{I- 1/4}}
    \end{align*}
    The condition on $m\sigma_1^2$ becomes
    \begin{align*}
        m\sigma_1^2 \lesssim_\sigma \frac{a_{\min_*}\Delta^2}{d^{I - 1/2}} \land a_{\min_*}\eps_D \land \frac{\eps_R}{\log(1/\sigma_0^2)} \land \frac{a_{\min_*}\Delta}{ d^{I-1/2}\log^{2I-2}d}.
    \end{align*}
    We additionally require $\sigma_0^2 \lesssim_{\sigma}\sigma_1^2\eps_D^{C_{\exp}}$. Therefore it suffices to pick $\sigma_0 = d^{-C}, \sigma_1 = d^{-C'}$, where $C > C' > 0$ are sufficiently large constants depending only on $I, \sigma$.

    Next, we choose the learning rate $\eta$. It suffices to set $\eta$ as
    \begin{align*}
    \eta \lesssim_{\sigma} \frac{a_{\min_*}\norm{\a}_1^{-2}m^{-1}P^{-1}\delta^*_{\P}}{\log(512I/\Delta)\log^{\tilde Q}\left(\frac{md}{\delta_{\P,\xi}}\right)}\min(\Delta^2d^{-I}, \eps_D^2).
    \end{align*}

    Finally, we prove part (d), and bound the population loss $\mathcal{L}$ at time $t$. Recall that $\mathcal{L} = \sum_{i \ge I}\hat\sigma_{2i}^2\mathcal{L}_i$, where
    \begin{align*}
        \mathcal{L}_i := \frac12 \norm{\a}^2 - \sum_{p=1}^P\sum_{k=1}^m a_p\norm{\v_k}^2\bar v_{k,p}^{2i} + \frac12\sum_{k, l = 1}^m \norm{\v_k}^2\norm{\v_l}^2\langle \bar \v_k, \bar \v_l\rangle^{2i}.
    \end{align*}
    Recall that $L := \{p \in [P] : \norm{\v_p} \ge \sigma_1\}$. By parts (b) and (c), we must have $L = [k_*]$ for some integer $k_*$, and $\bar v_{p, \pi(p)}^2 \ge 1 - \bar\eps$ for $\v_p \in L$. We can decompose the loss as follows:
    \begin{align*}
        \mathcal{L}_i &= \frac12\norm{\a}^2 - \sum_{k \in L}\sum_{p \in [P]}a_{p} \norm{\v_k}^2\bar v_{k, p}^{2i} + \frac12\sum_{k \in L}\norm{\v_k}^4\\
        &\quad + \sum_{k, j \in L, k \neq j}\norm{\v_k}^2\norm{\v_j}^2\langle \bar \v_k, \bar \v_j \rangle^{2i} - \sum_{k\not\in L}\sum_{p \in [P]}a_p\norm{\v_k}^2\bar v_{k, p}^{2i} + \frac12\sum_{k \not\in L}\sum_{j=1}^m \norm{\v_k}^2\norm{\v_j}^2\langle \bar \v_k, \bar \v_j \rangle^{2i}
    \end{align*}
    The terms with $k \not\in L$ are straightforward to bound, as
    \begin{align*}
        &\sum_{k\not\in L}\sum_{p \in [P]}a_p\norm{\v_k}^2\bar v_{k, p}^{2i} \le m\sigma_1^2\norm{\a}_1\\
        &\frac12\sum_{k \not\in L}\sum_{j=1}^m \norm{\v_k}^2\norm{\v_j}^2\langle \bar \v_k, \bar \v_j \rangle^{2i} \le \frac12 m \sigma_1^2 \sum_{j=1}^m \norm{\v_j}^2 \le m\sigma_1^2\norm{\a}_1.
    \end{align*}
    Next, for $k\neq j \in L$, $\langle \bar \v_k, \bar \v_j \rangle^{2i} \le \bar\eps^i$, and thus
    \begin{align*}
        \sum_{k, j \in L, k \neq j}\norm{\v_k}^2\norm{\v_j}^2\langle \bar \v_k, \bar \v_j \rangle^{2i} \le \bar\eps^i(\sum_{k \in [m]}\norm{\v_k}^2)^2 \le 4 \norm{\a}_1^2\bar\eps^i.
    \end{align*}
    Finally, we track the dominant loss term. We have
    \begin{align*}
        &\frac12\norm{\a}^2 - \sum_{k \in L}\sum_{p \in [P]}a_{p} \norm{\v_k}^2\bar v_{k, p}^{2i} + \frac12\sum_{k \in L}\norm{\v_k}^4\\
        &\qquad= \frac12\sum_{k \not\in L}a_{\pi(k)}^2 + \frac12\sum_{k \in L}\left(a_{\pi(k)}^2 - \norm{\v_k}^2\sum_{p \in P}a_p \bar v_{k, p}^{2i} + \norm{\v_k}^4\right)
    \end{align*}
    We can bound
    \begin{align*}
        \norm{\v_k}^2\sum_{p \neq \pi(k)}a_p \bar v_{k, p}^{2i} \le \norm{\v_k}^2\sum_{p \neq \pi(k)}a_p \bar\eps^i \le \bar\eps^i\norm{\v_k}^2\norm{\a}_1.
    \end{align*}
    Moreover, $1 - \bar v_{k, p}^{2i} \le 2i\bar\eps$. Altogether,
    \begin{align*}
        \mathcal{L}_i = \frac12\sum_{k \not\in L}a_{\pi(k)}^2 + \frac12\sum_{k \in L}\left(a_{\pi(k)} - \norm{\v_k}^2\right)^2 \pm O(\bar\eps),
    \end{align*}
    and since $\sum_l \hat \sigma_{2l}^2 = 1$, we have
    \begin{align*}
        \mathcal{L} = \frac12\sum_{k \not\in L}a_{\pi(k)}^2 + \frac12\sum_{k \in L}\left(a_{\pi(k)} - \norm{\v_k}^2\right)^2 \pm O(\bar\eps)
    \end{align*}
    as well. Next, if $t \le (1 - \Delta/4)T_p$, then $p \not\in L$, and thus
    \begin{align*}
        \mathcal{L} \ge \frac12\sum_{k \not\in L}a_{\pi(k)}^2 - O(\bar\eps) \ge \frac12 - \frac12\sum_{p \in P_*}a_{\pi(p)}^2\cdot \indi\left(t \ge (1 - \Delta/4)T_p\right) - O(\bar\eps).
    \end{align*}
    On the other hand, if $t \ge (1 + \Delta/4)T_p$, then $p \in L$ and $\abs{a_{\pi(p)} - \norm{\v_{p}^2}} \le \eps_R$, and thus
    \begin{align*}
        \mathcal{L} \le \frac12 - \frac12\sum_{p \in P_*}a_{\pi(p)}^2\cdot \indi\left(t \ge (1 + \Delta/4)T_p\right) + O(P_*\eps_R^2 + \bar\eps),
    \end{align*}
    where the desired claim follows by additionally choosing $\eps_R^2 \le P_*^{-1}\eps_D$.
\end{proof}

\subsection{Deferred Proofs}
\label{subsec: online sgd: deferred proofs}

\begin{proof}[Proof of Lemma \ref{lem:stochastic induction signal term}]
    Assume WLOG that the bounds on $X_t$ always hold. Inductively unroll the recursion as 
    \begin{align*}
        X_{t} = X_0P_{0, t} + \sum_{s=1}^t P_{s, t}\xi_{s} + \sum_{s=1}^t P_{s, t}Z_{s},
    \end{align*}
    where $P_{s, t} := \prod_{r=s}^{t-1}(1 + \alpha X^{I-1}_r) \ge 1$. As such,
    \begin{align*}
        P_{0, t}^{-1}X_t = X_0 +  \sum_{s=1}^tP_{0, s}^{-1}\xi_{s}+  \sum_{s=1}^tP_{0, s}^{-1}Z_{s}.
    \end{align*}
    The error term gets bounded as
    \begin{align*}
        \abs{\sum_{s=1}^tP_{0, s}^{-1}\xi_{s}} \le \sum_{s=1}^t\abs{\xi_{s}} \le \Xi_1\sum_{t=0}^{T-1}\left({x_t^+}\right)^I + \Xi_2\sum_{t=0}^{T-1}{x_t^+} + T\Xi_3
    \end{align*}
    with high probability for all $t$. We can bound each term by $x_0\eps/6$. The martingale term can be controlled by Doob's inequality,
    \begin{align*}
        \P\left[\sup_{r \le t}\abs{\sum_{s=1}^tP_{0, s}^{-1}Z_{s}}\ge M\right] \le M^{-2}\sum_{s=1}^t \E[Z_{s}^2] \le M^{-2}\sigma^2_Z\sum_{t=0}^{T-1}{x_t^+} \le \delta_{\P},
    \end{align*}
    when we take $M = x_0\eps/2$. Altogether, we have that $P_{0, t}^{-1}X_t \ge X_0 - x_0\eps$, and thus
    \begin{align*}
        X_t &\ge P_{0, t}(1 - \eps)x_0 = \prod_{s = 1}^{t-1}(1 + \alpha X^{I-1}_r)x_0^- \ge \prod_{s = 1}^{t-1}\left(1 + \alpha \left(x_r^-\right)^{I-1}\right)x_0^- = x_t^-.
    \end{align*}
    Similarly, we have $P_{0, t}^{-1}X_t \ge X_0 + x_0\eps$, and thus
    \begin{align*}
        X_t &\le P_{0, t}(1 + \epsilon)x_0 = \prod_{s = 1}^{t-1}(1 + \alpha X^{I-1}_r)x_0^+ \le \prod_{s = 1}^{t-1}\left(1 + \alpha \left(x_r^+\right)^{I-1}\right)x_0^+ = x_t^+,
    \end{align*}
    as desired.
\end{proof}

\begin{proof}[Proof of Lemma \ref{lem:stochastic-induction-LB}]
    Assume that the bounds on $X_t$ always hold. If $\sup_{s \le t} X_s > \delta$ then we are done; otherwise, unroll the recursion as
    \begin{align*}
        X_t = X_0P_{0, t} + \sum_{s=1}^tP_{s, t}Z_{s} + \sum_{s=1}^tP_{s, t}\xi_{s},
    \end{align*}
    where $P_{s, t} := \prod_{r=s}^{t-1}(1 + \alpha X^{I-1}_r) \ge 1$. As such,
    \begin{align*}
        P_{0, t}^{-1}X_t = X_0 +  \sum_{s=1}^tP_{0, s}^{-1}\xi_{s}+  \sum_{s=1}^tP_{0, s}^{-1}Z_{s}.
    \end{align*}
    The error term is bounded as
    \begin{align*}
        \abs{\sum_{s=1}^tP_{0, s}^{-1}\xi_{s}} \le \sum_{s=1}^t\abs{\xi_{s}} \le \Xi T \le \frac{x_0}{4}
    \end{align*}
    for high probability for all $t \le T$. Next, we bound the martingale term by Doob's inequality:
    \begin{align*}
        \P\left[\sup_{r \le t}\abs{\sum_{s=1}^tP_{0, s}^{-1}Z_{s}}\ge M\right] \le M^{-2}\sum_{s=1}^t \E[Z_{s}^2] \le M^{-2}\sigma_Z^2 T \le \delta_{\P},
    \end{align*}
    when we take $M = x_0/4$. Altogether,
    \begin{align*}
        X_t \ge P_{0, t}x_0/2 \ge \hat x_t,
    \end{align*}
    as desired.
\end{proof}

\begin{proof}[Proof of Lemma \ref{lem:strong-recovery-helper}]
 Expanding the recursion,
 \begin{align*}
     X_{t} \le (1 - \alpha)^t X_0 + \sum_{s=0}^{t-1} (1 - \alpha)^s\xi_{t-s} + \sum_{s=0}^{t-1} (1 - \alpha)^sZ_{t-s}.
 \end{align*}
 We can bound the error term by
 \begin{align*}
     \abs{\sum_{s=0}^{t-1} (1 - \alpha)^s\xi_{t-s}} \le \Xi \sum_{s=0}^{t-1} (1 - \alpha)^s \le \Xi \alpha^{-1} \le \frac{\eps}{4}
 \end{align*}
 and by Doob's inequality bound the martingale by
     \begin{align*}
        \P\left[\sup_{r \le t}\abs{\sum_{s=0}^{t-1} (1 - \alpha)^sZ_{t-s}}\ge M\right] \le M^{-2}\sum_{s=0}^{t-1} (1 - \alpha)^{-2s}\E[Z_{t-s}^2] \le M^{-2}\sigma^2_Z\alpha^{-1} \le \delta_{\P},
    \end{align*}
    since we take $M = \eps/4$. Therefore
    \begin{align*}
        X_{t} \le (1 - \alpha)^t X_0 + \eps/2 \le (1 - \alpha)^t x_0 + \eps/2.
    \end{align*}
 
\end{proof}

\begin{proof}[Proof of Lemma \ref{lem:stochastic-gronwall-linear}]
    Expanding the recursion,
    \begin{align*}
        X_t &= (1 + \alpha)^tX_0 + \sum_{s=0}^{t-1}(1 + \alpha)^s\xi_{t-s} + \sum_{s=0}^{t-1}(1 + \alpha)^sZ_{t-s}\\
        \Longrightarrow (1 + \alpha)^{-t}X_t &= X_0 + \sum_{s=1}^t(1 + \alpha)^{-s}\xi_s + \sum_{s=1}^t(1 + \alpha)^{-s}Z_s.
    \end{align*}
    We can bound the error term by
    \begin{align*}
        \abs{\sum_{s=1}^t(1 + \alpha)^{-s}\xi_s} \le \Xi\sum_{s=1}^t(1 + \alpha)^{-s}\cdot (1 + \alpha)^sx_0 = \Xi T x_0 \le \frac{x_0}{4}.
    \end{align*}
    By Doob's inequality, we can bound the martingale term by
    \begin{align*}
        \P\left[\sup_{t \le T}\abs{\sum_{s=0}^{t-1}(1 + \alpha)^{-s}Z_{s}} \ge M\right] \le M^{-2}\sigma_Z^2\sum_{s=0}^{t-1}(1 + \alpha)^{-2s}\cdot(1 + \alpha)^{2s}x_0^2 = M^{-2}\sigma_Z^2Tx_0^2 \le \delta_{\P},
    \end{align*}
    since we chose $M = x_0/4$. Altogether,
    \begin{align*}
        (1 + \alpha)^{-t}X_t = x_0 \pm 0.5 x_0 \Longrightarrow X_t = (1 \pm 0.5)x_t,
    \end{align*}
    as desired.
\end{proof}

\begin{proof}[Proof of Lemma \ref{lem:norm-convergence-helper}]
    Define $P_{s, t} := \prod_{r=s}^{t-1}(1 + \alpha(X_r))$.  Expanding the recursion, 
    \begin{align*}
        X_t = P_{0, t} X_0 + \sum_{s=0}^{t-1}P_{t-s, t} \xi_{t-s} + \sum_{s=0}^{t-1}P_{t-s, t}Z_{t-s}.
    \end{align*}
    We can bound the error term by
    \begin{align*}
        \abs{\sum_{s=0}^{t-1}P_{t-s, t} \xi_{t-s}} \le \Xi\sum_{s=0}^{t-1}(1 - \alpha_-)^s \le \Xi\alpha_-^{-1} \le \eps/4.
    \end{align*}
    By Doob's inequality, we can bound the martingale term by
    \begin{align*}
        \P\left[\sup_{t \le T}\abs{\sum_{s=0}^{t-1}P_{t-s, t}Z_{t-s}} \ge M\right] \le M^{-2}\sigma_Z^2\sum_{s=0}^{t-1}(1 - \alpha_-)^{2s} \le M^{-2}\sigma_Z^2\alpha_- \le \delta_{\P},
    \end{align*}
    since we chose $M = \eps/4$. Therefore
    \begin{align*}
        X_t &\le P_{0, t} X_0 + \eps/2 \le (1 - \alpha_-)^tx_0 + \eps/2\\
        X_t &\ge P_{0, t} X_0 - \eps/2 \ge (1 - \alpha_+)^tx_0 - \eps/2,
    \end{align*}
    as desired.
\end{proof}

\begin{proof}[Proof of Lemma \ref{lem: failed coords stochastic induction}]
    Assume WLOG that the bounds on $X_t$ always hold. $(X_t)_t$ is stochastically dominated by the process where $X_{t+1} = X_t + \alpha X_t^I + \xi_{t+1} + Z_{t+1}$, so we can WLOG track this latter process. Expanding out the recursion, we get that
    \begin{align*}
        X_t = X_0P_{0, t} + \sum_{s=1}^t P_{s, t}(\xi_{s} + Z_{s}),
    \end{align*}
    where $P_{s, t} := \prod_{r=s}^{t-1}(1 + \alpha X^{I-1}_r)$. Since $X_r \ge 0, P_{0, s} \ge 1$ and thus
    \begin{align*}
        P_{0, t}^{-1}X_t = X_0 + \sum_{s=1}^tP_{0, s}^{-1}\xi_{s} + \sum_{s=1}^t P_{0, s}^{-1}Z_{s}.
    \end{align*}
    The error term gets bounded as
    \begin{align*}
        \abs{\sum_{s=1}^tP_{0, s}^{-1}\xi_{s}} \le \sum_{s=1}^t \abs{\xi_{s}} \le \Xi_1\sum_{t=0}^{T-1} \hat x_t^{1/2} + \Xi_2\sum_{t=0}^{T-1} \hat x_t + \Xi_3\sum_{t=0}^{T-1}\hat x_t^I + T\Xi_4
    \end{align*}
    with high probability for all $t$. We can bound each term by $x_0\eps/8$. The martingale term can be controlled by Doob's inequality:
    \begin{align*}
        \P\left[\sup_{r \le t}\abs{\sum_{s=1}^tP_{0, s}^{-1}Z_{s}}\ge M\right] \le M^{-2}\sum_{s=1}^T\E[Z_{s}^2] \le M^{-2}\sigma_Z^2\sum_{t=0}^{T-1} \hat x_t \le \delta_{\P},
    \end{align*}
    when we take $M = x_0\epsilon/2$. Altogether, we get
    \begin{align*}
        X_t \le P_{0, t}x_0(1 + \epsilon) = P_{0, t}\hat x_0 \le \hat x_t,
    \end{align*}
    as desired.
\end{proof}

\begin{proof}[Proof of Lemma~\ref{lemma: general nonlinear stochastic induction}]
  We may assume w.l.o.g.~that the bounds on $\xi_t$ and the conditional variance of $Z_{t+1}$ always hold. Define
  \begin{align*}
    P_{s, t}(X) := \begin{cases}
      \prod_{r=s}^{t-1} (1 + \alpha_r(X_r)), & t > s, \\
      1, & t = s. 
    \end{cases}
  \end{align*}
  Note that since $\alpha_r > 0$, we have $P_{s, t} \ge 1$. Then, we can unroll the recurrence relationship as 
  \[
    X_t = X_0 P_0(X) + \sum_{s=1}^{t} P_{s, t}(X) \left( \xi_{s-1} + Z_{s-1} \right). 
  \]
  Divide both sides with $P_{0, t}$, and we obtain 
  \[
    P_{0, t}\inv(X) X_0 
    = X_0 + \sum_{s=1}^{t} P_{0, s}\inv(X) \xi_{s-1} + \sum_{s=1}^{t} P_{0, s}\inv(X) Z_{s-1}. 
  \]
  For the second term, we have 
  \[
    \left| \sum_{s=1}^{t} P_{0, s}\inv(X) \xi_{s-1} \right|
    \le  \sum_{s=1}^{t} | \xi_{s-1} |
    \le T \Xi, 
  \]
  for all $t \le T$ with probability at least $1 - T \delta_{\P, \xi}$. For the RHS to be bounded by 
  $\eps x_0/2$, it suffices to choose $\Xi \le \eps_0 x_0 / (2 T)$. Meanwhile, by Doob's submartingale inequality, 
  for any $M > 0$, we have 
  \[
    \P\left[
      \sup_{r \le t} \left|\sum_{s=1}^{t} P_{0, s}\inv Z_{s-1} \right|
      \ge M 
    \right]
    \le M^{-2} \sum_{s=1}^{t} \E\left[ P_{0, s}^{-2} Z_{s-1}^2 \right]
    \le \frac{\sigma_Z^2 T}{M^2}. 
  \]
  Choose $M = \eps x_0/2$. Then, the RHS becomes $\frac{4 \sigma_Z^2 T}{ \eps^2 x_0 }$. For it to be bounded by 
  $\delta_{\P}$, we need $\sigma_Z^2 \le \eps^2 x_0^2 \delta_{\P} / (4 T)$. The above two results imply that 
  with the conditions on $\xi$ and $Z$ stated in the lemma, we have, with probability at least 
  $1 - \delta_{\P} - T \delta_{\P, \xi}$, that 
  \[
    X_t 
    = P_{0, t}(X) (1 \pm \eps) x_0 
    \le P_{0, t}(X) \hat{X}_0
    \le P_{0, t}(\hat{X}) \hat{X}_0
    \le \hat{X}_t,
  \]
  where the second inequality comes from the monotonicity of $x \mapsto \alpha_t(x)$. 
\end{proof}

\begin{lemma}\label{lem:continuous-upper-bound}
    Let $(x_t)_t \in [0, 1]$ follow the update
    \begin{align*}
        \hat x_{t+1} = \hat x_t + \alpha \hat x_t^I.
    \end{align*}
    Then 
    \begin{align*}
        \hat x_t \le \frac{\hat x_0}{\left(1 - \alpha (I-1)\hat x_0^{I-1} t\right)^{\frac{1}{I-1}}}.
    \end{align*}
\end{lemma}
\begin{proof}
    Define the continuous time process $x(t)$ be the ODE $\dot x(t) = \alpha x(t)^I$ with initial condition $\hat x_0 = x(0)$. We prove by induction that $\hat x_t \le x(t)$. Observe that both processes are monotonically increasing. Therefore
    \begin{align*}
        \hat x_{t+1} = \hat x_t + \alpha \hat x_t^I \le x(t) + \alpha x(t)^I \le x(t) +  \int_{t}^{t+1}\alpha x(s)^I ds = x(t+1).
    \end{align*}
    The desired result is obtained by solving the ODE for $x(t)$ with initial condition $x(0) = x_0$.
\end{proof}

\begin{lemma}\label{lem:discrete-gronwall-LB}
    Let $(x_t)_t \in [0, 1]$ follow the update
    \begin{align*}
        x_{t+1} = x_t + \alpha x_t^I.
    \end{align*}
    Then 
    \begin{align*}
        x_t \ge \frac{x_0}{\left(1 - \alpha (I-1) \exp(-\alpha I) x_0^{I-1}t\right)^{\frac{1}{I-1}}}.
    \end{align*}
\end{lemma}
\begin{proof}
    We have that
\begin{align*}
    \alpha &= \frac{x_t - x_{t-1}}{(x_{t-1})^I}\\
    &= \frac{(x_{t})^I}{(x_{t-1})^I} \cdot \frac{x_t - x_{t-1}}{(x_{t})^2}\\
    &\le \frac{(x_{t})^I}{(x_{t-1})^I} \int_{x_{t-1}}^{x_t} \frac{1}{x^I}dx\\
    &= \frac{(x_{t})^I}{(I-1)(x_{t-1})^I} \left(\frac{1}{x_{t-1}^{I-1}} - \frac{1}{x_{t}^{I-1}}\right)\\
    &= (I-1)^{-1}(1 + \alpha x_{t-1}^{(I-1)})^I\left(\frac{1}{x_{t-1}^{I-1}} - \frac{1}{x_{t}^{I-1}}\right)\\
    &\le (I-1)^{-1}\exp(\alpha I)\left(\frac{1}{x_{t-1}^{I-1}} - \frac{1}{x_{t}^{I-1}}\right).
    \end{align*}
    Therefore
    \begin{align*}
        \frac{1}{x_{t}^{I-1}} \le \frac{1}{x_{t-1}^{I-1}} - \alpha (I-1) \exp(-\alpha I),
    \end{align*}
    so summing and solving for $x_t$ yields
    \begin{align*}
        x_t \ge \frac{x_0}{\left(1 - \alpha (I-1) \exp(-\alpha I) x_0^{I-1}t\right)^{\frac{1}{I-1}}}.
    \end{align*}
    
\end{proof}

\bigskip
\section{Scaling Law Derivations}
\label{sec: scaling law}

We have shown that direction $\e_{\pi(p)}$ will be learned at time $(1 \pm o(1)) T_p$ where $T_p$ is defined by 
\[
  T_p := \left( 4I(I-1)\hat\sigma_{2I}^2 a_{\pi(p)}\eta \bar{v}_{p, {\pi(p)}}^2(0) \right)\inv.
\]
Suppose that the signal follows the power law $a_p = p^{-\beta} / Z$ where $\beta > 1/2$ and 
$Z = \sum_{p=1}^P p^\beta$ is the normalizing constant. 
In Section~\ref{sec: idealized dynamics and scaling laws}, we informally derive the scaling law 
$\Loss(t) \propto t^{-(2\beta-1)/\beta} $. In this section, we prove that this is true up to a multiplicative constant (cf.~Corollary~\ref{thm: scaling law}).

To this end, it suffices to (1) argue that teacher neurons $p$ with large signal strength $a_p$ are likely to lie in the set of learned neurons $\{\pi(p) : p \in [P_*]\}$, and (2) bound the fluctuations of $\bar v_{p, \pi(p)}^2(0)$. A lower bound on the fluctuations is given in Lemma \ref{lemma: initialization}(d). The following lemma shows that neurons with large signal strength do indeed get learned.

\begin{lemma}\label{lem:greedy-select-lemma}
    Assume that $a_p \propto p^{-\beta}$ for $\beta > 1/2$. Let $\delta_{\P} = 1/\poly(m)$ be the target failure probability. Then there exists a universal constant $C$ so that, with probability $1 - \delta_{\P}$, all teacher neurons $q$ satisfying $a_q \ge Ca_{P_*}$ lie in the set of learned neurons, i.e $q \in \{\pi(p) : p \in [P_*]\}$.
\end{lemma}
\begin{proof}
    Let $\z_1, \dots, \z_m$ be independent $\mathcal{N}(0, \Id_d)$ variables. We remark that $\{\bar \v_i\}_{i \in [m]}$ is equal in distribution to $\{\z_i/\norm{\z_i}\}_{i \in [m]}$. First, with probability $1 - 2m\exp(-Cd)$, we have that $\norm{\z_i}^2 = (1 \pm 0.5)d$ for all $i \in [m]$. Moreover, $\mathbb{P}(\max_{k \in [m], p \in [P_*]} \abs{z_{k, p}} \ge z) \le 2mP_*e^{-z^2/2}$, and therefore $\max_{k \in [m], p \in [P_*]} z^2_{k, p} \le 2\log(2mP_*/\delta_{\P})$ with probability $1 - \delta_{\P}$. Let us condition on these two events.

    Let $\gamma \ge 1$ be some threshold. We begin by computing $\mathbb{P}(\max_{k \in [m], p > P_*} a_pZ_{k,p}^2 \ge a_{P_*}\gamma)$. By standard Gaussian tail bounds and a union bound, we have that
    \begin{align*}
        \mathbb{P}\left(\max_{k \in [m], p > P_*} a_pZ_{k,p}^2 \ge a_{P_*}\gamma\right) \le \sum_{p > P_*} 2m\exp\left(-\frac{a_{P_*}\gamma}{2a_p}\right)
    \end{align*}
    Substituting $a_p = p^{-\beta}/Z$ for $\beta > \frac12$, we get that 
    \begin{align*}
        \sum_{p > P_*}\exp\left(-\frac{a_{P_*}\gamma}{2a_p}\right) = \sum_{p > P_*}\exp\left(-\frac{\gamma}{2}(\frac{p}{P_*})^\beta\right) \le \int_{P_*}^\infty \exp\left(-\frac{\gamma}{2}(\frac{p}{P_*})^{1/2}\right)dp\\
        = \frac{4\sqrt{P_*}}{\gamma}\exp(-\gamma/2)\sqrt{P_*} + \frac{8P_*}{\gamma^2}\exp(-\gamma/2) \le 12P_*\exp(-\gamma/2).
    \end{align*}
    Therefore
    \begin{align*}
        \mathbb{P}\left(\max_{k \in [m], p > P_*} a_pZ_{k,p}^2 \ge a_{P_*}\gamma\right) \le 24P_*m\exp(-\gamma/2) \le \delta_{\P}
    \end{align*}
    for $\gamma = 2\log(24mP_*/\delta_{\P})$.

    Next, we aim to upper bound the quantity $a_{\pi(P_*)}\bar v^2_{P_*, \pi(P_*)}$. The first case is when $\{\pi(p) : p \in [P_*]\} = [P_*]$. Since $\max_{k \in [m], p \in [P_*]} z^2_{k, p} \le 2\log(2mP_*/\delta_{\P})$, it is clear that $a_{\pi(P_*)}\bar v^2_{P_*, \pi(P_*)} \le 4a_{P_*}\log(2mP_*/\delta_{\P})/d$. Otherwise, there exists some $q \in [P_*]$ such that $\pi(q) > P_*$. We then have that $a_{\pi(P_*)}\bar v^2_{P_*, \pi(P_*)} \le a_{\pi(q)}\bar v^2_{q, \pi(q)} \le 2a_{P_*}\gamma/d = 4a_{P_*}\log(24mP_*/\delta_{\P})/d$.

    Let $\e_q$ be some teacher neuron which was not selected by the greedy maximum selection process, i.e $q \not\in \{\pi(p) : p \in [P_*]\}$. Then we must have $a_q\bar v_{p, q}^2 \le a_{\pi(P_*)}\bar v_{P_*, \pi(P_*)}^2$ for all $p > P_*$. Therefore
    \begin{align*}
        \mathbb{P}(q \not\in \{\pi(p) : p \in [P_*]\}) &\le \mathbb{P}\left(\cup_{p > P_*}a_q\bar v_{p, q}^2 \le a_{\pi(P_*)}\bar v_{P_*, \pi(P_*)}^2\right)\\
        &\le \mathbb{P}\left(\cup_{p > P_*} z_{p, q}^2 \le \frac{6a_{P_*}}{a_q}\log(24mP_*/\delta_{\P})\right).
    \end{align*}
    For $\gamma > 1$, one can bound $\mathbb{P}(Z_i \ge \gamma) \ge \frac{1}{\sqrt{2\pi}}\frac{z}{1 + z^2}e^{-z^2/2} \ge \frac{1}{\sqrt{2\pi}}e^{-3z^2/2}$. Therefore
    \begin{align*}
        \mathbb{P}(q \not\in \{\pi(p) : p \in [P_*]\}) &\le \left(1 - \frac{1}{\sqrt{2\pi}}\exp\left(-\frac{9a_{P_*}}{a_q}\log(24mP_*/\delta_{\P}) \right)\right)^{m - P_*}\\
        &\le \left(1 - \frac{1}{\sqrt{2\pi}}\left(\frac{24mP_*}{\delta_{\P}}\right)^{-\frac{9a_{P_*}}{a_q}}\right)^{m/2}\\
        &\le \exp\left(-\frac{m}{2\sqrt{2\pi}}\left(\frac{24mP_*}{\delta_{\P}}\right)^{-\frac{9a_{P_*}}{a_q}}\right).
    \end{align*}
    If $a_q$ satisfies
    \begin{align*}
        a_q \ge a_{P_*}\cdot \frac{9\log(24mP_*/\delta_{\P})}{\log(\frac{m}{2\sqrt{2\pi}}) - \log\log(P/\delta_{\P})},
    \end{align*}
    then plugging in we obtain $\mathbb{P}(q \not\in \{\pi(p) : p \in [P_*]\}) \le \delta_{\P}/P$. Finally, since $P_* \le m$, for $\delta_{\P} = 1/\poly(m)$ we can upper bound $\frac{9\log(24mP_*/\delta_{\P})}{\log(\frac{m}{2\sqrt{2\pi}}) - \log\log(P/\delta_{\P})} \le C$ for some universal constant $C$. Union bounding over all $q$ yields the desired result.
\end{proof}

Now, we are ready to prove our main theorem on the scaling law. 
\begin{proof}[Proof of Proposition~\ref{thm: scaling law}]
    By Theorem~\ref{thm:online-sgd}, we know that with probability at least $1 - o(1)$, we have 
        \begin{align*}
            1 - \sum_{p \in [P_*]} a_{\pi(p)}^2 \indi\left(t \ge (1 - \Delta/4)T_p \right) - O(\eps_D) \le \mathcal{L}(t) \le 1 - \sum_{p \in [P_*]} a_{\pi(p)}^2 \indi\left(t \ge (1 + \Delta/4)T_p \right) + O(\eps_D).
        \end{align*}
        It suffices to estimate the LHS and RHS. For the RHS, by Lemma \ref{lem:greedy-select-lemma} we have that $\{q : a_q \ge C a_{P_*}\} \subset \{\pi(p) : p \in [P_*]\}$, and by Lemma \ref{lemma: initialization} we have $\min_{p \in P_*} \bar v_{p, \pi(p)}^2 \ge (\log P_*)/d$, and thus
        \begin{align*}
            \sum_{p \in [P_*]} a_{\pi(p)}^2 \indi\left(t \ge \frac{1 + o(1)}{4I(I-1)\hat\sigma_{2I}^2\eta a_{\pi(p)}\bar v^{2I-2}_{p, \pi(p)}(0)} \right) &\ge \sum_{p \in [P_*]} a_{\pi(p)}^2 \indi\left(t \ge \frac{\tilde Cd^{I-1}}{\eta a_{\pi(p)}\log^{2I-2}P_*} \right)\\
            &\ge \sum_{p = 1}^{P_*C^{-1/\beta}} a_{p}^2 \indi\left(t \ge \frac{\tilde Cd^{I-1}}{\eta a_{p}\log^{2I-2}P_*} \right).
        \end{align*}
        Therefore, letting $K = \eta Z^{-1}\tilde C^{-1} \log^{2I-2}P_*$, we have
          \begin{align*}
    \rhs(t)
    &\le \frac{1 + o(1)}{2 Z^2} \sum_{p=1}^{P} p^{-2\beta} 
      \indi\braces{ t \ge \frac{d^{I-1}}{K p^{-\beta}} \lor p \ge P_*C^{-1/\beta} } + O(\eps_D)\\
    &\le \frac{1 + o(1)}{2 Z^2} \sum_{p=1}^{P} p^{-2\beta} 
      \indi\braces{ p \ge \left( K t / d^{I-1} \right)^{1/\beta} \land P_*C^{-1/\beta}} + O(\eps_D)\\
    &\le \frac{1 + o(1)}{2 Z^2} \left[\left( \frac{K t}{d^{I-1}} \right)^{-2} + P_*^{-2\beta}C^2\right]
      + \frac{1 + o(1)}{2 Z^2} \int_{\left( K t / d^{I-1} \right)^{1/\beta} \land P_*C^{-1/\beta}}^\infty q^{-2\beta} \,\rd q + O(\eps_D)\\
    &\le \frac{1 + o(1)}{2 Z^2} \left[\left( \frac{K t}{d^{I-1}} \right)^{-2} + P_*^{-2\beta}C^2\right]
      + \frac{1 + o(1)}{2 Z^2} \frac{1}{2\beta - 1} \left[\left( \frac{K t}{d^{I-1}} \right)^{-(2\beta-1)/\beta} \lor P_*^{-(2\beta - 1)}C^{\frac{2\beta - 1}{\beta}}\right]+ O(\eps_D). 
  \end{align*}
  When $\beta > 1/2$, we have $0 < 2\beta-1\le 2\beta$. Hence, when $t \ge d^{I-1}/K, P_* \ge C^{1/\beta}$ the first term 
  can be merged into the first term. Therefore, 
      \[
        \rhs(t) 
        \le C_\beta\left[\left( \frac{K t}{d^{I-1}} \right)^{-(2\beta-1)/\beta} \lor P_*^{-(2\beta - 1)}\right] + O(\eps_D). 
      \]
    We next consider the LHS. In Lemma \ref{lem:greedy-select-lemma}, we proved that $a_{\pi(P_*)}\bar v_{P_*, \pi(P_*)}^2 \le 4a_{P_*}\log(24mP_*/\delta_{\P})/d$ with probability $1 - \delta_{\P}$. Repeating the argument for all $p \in [P_*]$ and union bounding, with probability $1 - \delta_{\P}$ we have that $a_{\pi(p)}\bar v_{p, \pi(p)}^2 \le 4a_{p}\log(24mP^2_*/\delta_{\P})/d$ for $p \in [P_*]$. We can therefore upper bound the LHS as
        \begin{align*}
            \sum_{p \in [P_*]} a_{\pi(p)}^2 \indi\left(t \ge \frac{1 - o(1)}{4I(I-1)\hat\sigma_{2I}^2\eta a_{\pi(p)}\bar v^{2I-1}_{p, \pi(p)}(0)} \right) &\le \sum_{p \in [P_*]} a_{\pi(p)}^2 \indi\left(t \ge \frac{\tilde c d^{I-1}}{\eta a_{p}\log^{2I-2}m} \right) \\ 
            &\le \sum_{p \in [P_*]} a_{p}^2 \indi\left(t \ge \frac{\tilde c d^{I-1}}{\eta a_{p}\log^{2I-2}m} \right).
        \end{align*}
    Letting $k = \eta Z^{-1}\tilde c^{-1}\log^{2I - 2}m$, we can similarly write
  \begin{align*}
    \lhs(t)
    &\ge \frac{1}{2 Z^2} 
      \sum_{p=1}^{P} p^{-2\beta} \indi\braces{
        t \le \frac{d^{I-1}}{k p^{-\beta}} \lor p \ge P_*
      } - O(\eps_D)\\
    &\ge \frac{1}{2 Z^2} 
      \sum_{p=1}^{P} p^{-2\beta} \indi\braces{ p \ge (k t/d^{I-1})^{1/\beta} \land P_*} - O(\eps_D)\\
    &\ge \frac{1}{2 Z^2} \int_{(k t/d^{I-1})^{1/\beta} \land P_*}^P q^{-2\beta} \,\rd q - O(\eps_D)\\
    &\ge \frac{1}{2 Z^2} \frac{1}{2\beta - 1} 
      \left(
        \left(\frac{k t}{d^{I-1}}\right)^{-(2\beta-1)/\beta} \lor P_*^{-(2\beta - 1)}
        - P^{1-2\beta}
      \right)- O(\eps_D). 
  \end{align*}
  When $t \le 2^{-\beta/(2\beta-1)} P^\beta d^{I-1} /k$, the last term can be merged into the second last term. 
  This gives the lower bound 
  \[
    \lhs(t)
    \ge c_\beta\left[\left(\frac{k t}{d}\right)^{-(2\beta-1)/\beta} \lor P_*^{-(2\beta - 1)}\right] - O(\eps_D).
  \]
  Altogether, the desired claim in part (b) follows from choosing $P_* = \Theta(\frac{m}{\log m})$.

  Finally, we observe that Lemma \ref{lem:greedy-select-lemma} implies that all directions $\e_p$ with $p \le P_*C^{-1/\beta} = \tilde\Theta(\frac{m}{\log m})$ are learned, and Theorem \ref{thm:online-sgd} implies that this learning happens at time $\tilde\Theta(p^\beta d^{I-1}\eta^{-1})$. The conclusion in part (a) directly follows.
\end{proof}

\end{document}